\documentclass[twoside,11pt]{article}

\usepackage[T1]{fontenc}
\usepackage[utf8]{inputenc}

\usepackage[a4paper]{geometry}
\usepackage{RR}

\usepackage[unicode,colorlinks,allcolors=black]{hyperref} %
\usepackage[inline]{enumitem}

\usepackage[titles]{tocloft}
\usepackage{xcolor}

\usepackage{tabularx}

\usepackage{nameref}

\newcommand{\symitem}[3]{\item[\texorpdfstring{\hyperref[#1]{#2}}{#2}] #3\dotfill\pageref{#1}}

\usepackage[colorinlistoftodos]{todonotes} %

\usepackage{mathtools}
\usepackage{amsmath}
\usepackage{amsfonts}
\usepackage{calrsfs}
\DeclareMathAlphabet{\pazocal}{OMS}{zplm}{m}{n}
\newcommand{\FranzEL}{\pazocal{EL}}
\usepackage{amssymb}
\usepackage{amsthm}

\newtheorem{definition}{Definition}
\newtheorem{example}{Example}
\newtheorem{theorem}{Theorem}
\newtheorem{property}{Property}
\newtheorem{proposition}[property]{Proposition}

\newcommand{\subsubsubsection}[1]{\paragraph{#1}}

\usepackage[french,british]{babel}

\usepackage{pgf}
\usepackage{tikz}
\usetikzlibrary{shapes}
\usetikzlibrary{fit}
\usetikzlibrary{calc}
\usetikzlibrary{decorations.pathmorphing}
\usetikzlibrary{decorations.markings}
\usetikzlibrary{patterns}
\usetikzlibrary{patterns.meta}

\usepackage[outputdir={dotfiles/}]{dot2texi}
\tikzstyle{concept} = [
        align=center,
        draw,
        rectangle split,
        rectangle split parts=2,
        rounded corners=.2cm,
        minimum width=1cm,
        minimum height=1cm,
      ]
\tikzstyle{smconcept} = [
        align=center,
        draw,
        rectangle split,
        rectangle split parts=2,
        rounded corners=.2cm,
        minimum width=1cm,
        minimum height=1cm,
      ]
\tikzstyle{nosep} = [ inner sep=0pt, outer sep=0pt ]

\usepackage[
        backend=biber,
	backref=true,               %
	url=true,                      %
        doi=true,
 	maxcitenames=2,         %
 	maxbibnames=100,     %
        style=apa,
 	sorting=nyt,                    %
 	hyperref=true                %
 	]{biblatex} 

\DeclareFieldFormat{bibhyperrefnonest}{%
  \DeclareFieldFormat{bibhyperref}{##1}%
  \bibhyperref{#1}}

\DeclareFieldFormat{doi}{%
   \mkbibacro{DOI}\addcolon\space 
   \ifhyperref 
     {\href{https://doi.org/#1}{\footnotesize\nolinkurl{#1}}} 
     {\footnotesize\nolinkurl{#1}}}

\DeclareFieldFormat{url}{%
   \ifhyperref 
     {\footnotesize\url{#1}}
     {\footnotesize\nolinkurl{#1}}}

\DeclareCiteCommand{\cite}[\mkbibbrackets]
  {\usebibmacro{cite:init}\usebibmacro{prenote}}
  {\usebibmacro{citeindex}\printtext[bibhyperrefnonest]{\usebibmacro{cite}}}
  {}
  {\usebibmacro{postnote}\usebibmacro{cite:post}}

\AtEveryBibitem{\clearfield{note}}
\AtEveryBibitem{\clearfield{language}}

\newcommand\textlb[1]{\textsf{\small #1}}

\setlist{leftmargin=.5cm,topsep=0pt,parsep=-5pt}
\setlist[2]{itemsep=4pt}
\setlist[description]{font={\normalfont\sffamily}}

\newlist{enuminline}{enumerate*}{5}
\setlist[enuminline]{label={(\arabic*)}}

\newlist{enumii}{enumerate*}{5}
\setlist[enumii]{label={(\emph{\roman*})}}

\newlist{enumaa}{enumerate*}{5}
\setlist[enumaa]{label={(\emph{\alph*})}}

\def\XX{X}
\def\xx{x}
\def\zz{z}

\let\emptyset\varnothing

\newcommand\extent[1]{\ensuremath{#1^{\downarrow}}}
\newcommand\intent[1]{\ensuremath{#1^{\uparrow}}}

\newcommand{\forallexists}{\mathrel{\forall\exists}}

\RRNo{9518}
\RRversion{3}
\RRdate{October 2023}
\RRdater{May 2025}
\RRauthor{Jérôme Euzenat\thanks{\begin{tabular}{rl}\\
Affiliation:& Univ. Grenoble Alpes, Inria, CNRS, Grenoble INP, LIG, F-38000 Grenoble, France\\
Contact:& Jerome.Euzenat@inria.fr\end{tabular}}}
\RRtitle{Refondation fonctionnelle progressive de\\
  l'analyse relationnelle de concepts}
\RRetitle{Stepwise functional refoundation of\\
relational concept analysis}
\titlehead{Stepwise functional refoundation of relational concept analysis}
\authorhead{Jérôme Euzenat}

\RRresume{
\begin{otherlanguage}{french} %
L'analyse relationnelle de concepts (RCA) est une extension de l'analyse formelle de concepts qui permet de traiter simultanément plusieurs contextes formels liés.
Elle permet d'induire des théories en logiques de description à partir de données et est utilisée dans diverses applications.
La RCA retourne une unique famille de treillis de concepts bien que, lorsque les données entretiennent des dépendances circulaires, d'autres solutions semblent acceptables.
La sémantique de l'analyse relationnelle de concepts, définie de manière opérationnelle, n'éclaire pas cette question.
Le but de ce rapport est de définir précisément la sémantique de l'analyse relationnelle de concepts et d'identifier des solutions alternatives.
Les solutions acceptables y sont définies comme les familles de treillis de concepts qui appartiennent à l'espace délimité par les contextes initiaux (\emph{bien formées}), ne peuvent supporter de nouveaux attributs (\emph{saturées}) et ne réfèrent qu'à des concepts de la famille (\emph{auto-supportées}).
Nous adoptons une approche fonctionnelle de la RCA en définissant l'espace des solutions bien formées et deux fonctions sur cet espace\,: l'une expansive et l'autre contractive.
Nous montrons que les solutions acceptables sont les points fixes communs aux deux fonctions.
Ce résultat est obtenu progressivement en partant d'une version minimale de RCA qui ne considère qu'un unique contexte et en définissant un espace de contextes et un espace de treillis correspondant.
Ces espaces sont ensuite rassemblés en un espace de paires contexte-treillis qui est étendu en un espace de familles indexées de paires de contexte-treillis représentant les objets manipulés par la RCA.
Ceci permet de montrer que l'algorithme de RCA retourne le plus petit élément de l'ensemble des solutions acceptables.
De plus, il est possible de définir une opération duale qui retourne le plus grand élément.
L'ensemble des solutions acceptables forme un sous-treillis complet de l'intervalle entre ces deux éléments.
Nous étudions en détail la structure de celui-ci et, en particulier, comment les fonctions définies le parcourent.
\end{otherlanguage}
}

\RRabstract{
Relational concept analysis (RCA) is an extension of formal concept analysis dealing with several related formal contexts simultaneously.
It can learn description logic theories from data and has been used within various applications.
However, RCA returns a single family of concept lattices, though, when the data feature circular dependencies, other solutions may be considered acceptable.
The semantics of RCA, provided in an operational way, does not shed light on this issue.
This reports aims at defining precisely the semantics of RCA and identifying alternative solutions.
We characterise these acceptable solutions as those families of concept lattices
which belong to the space determined by the initial contexts (\emph{well-formed}),
which cannot scale new attributes (\emph{saturated}), and
which refer only to concepts of the family (\emph{self-supported}).
We adopt a functional view on the RCA process by defining the space of well-formed solutions and two functions on that space: one expansive and the other contractive.
We show that the acceptable solutions are the common fixed points of both functions.
This is achieved step-by-step by starting from a minimal version of RCA that considers only one single context defined on a space of contexts and a space of lattices.
These spaces are then joined into a single space of context-lattice pairs, which is further extended to a space of indexed families of context-lattice pairs representing the objects manipulated by RCA.
We show that RCA returns the least element of the set of acceptable solutions.
In addition, it is possible to build dually an operation that generates its greatest element.
The set of acceptable solutions is a complete sublattice of the interval between these two elements.
Its structure and how the defined functions traverse it are studied in detail.
}

\RRmotcle{Analyse formelle de concepts -- analyse relationnelle de concepts -- point fixe -- sémantique de point fixe -- dépendance circulaire}
\RRkeyword{Formal concept analysis -- relational concept analysis -- fixed point -- fixed-point semantics -- circular dependency}
\RRprojet{mOeX}
\URRhoneAlpes
\RCGrenoble

\newenvironment{starfootnotes}
  {\par\edef\savedfootnotenumber{\number\value{footnote}}
   \setcounter{footnote}{0}}
 {\par\setcounter{footnote}{\savedfootnotenumber}}

\begin{document}

\begin{starfootnotes}
\makeRR
\end{starfootnotes}

\setcounter{tocdepth}{2}

\renewcommand{\baselinestretch}{0.95}\normalsize
\tableofcontents
\renewcommand{\baselinestretch}{1.0}\normalsize

\section{Introduction}\label{sec:motiv}

Formal concept analysis (FCA) is a well-defined and widely used operation for extracting concept lattices from binary data tables \cite{ganter1999a}.
This means that, from a set of objects described by Boolean attributes (called context), it will generate the set of all descriptions relevant to these objects (called concepts) organised in a lattice following a generalisation order relation between these concepts.
For instance, from a description of people through their diploma, major and current job it is possible to generate concepts about those bachelors in literature who hold a teacher position.
This concept is a subconcept of bachelors in literature, which is a subconcept of that of people having a bachelor degree.
The subconcepts add more constraints (attributes) to the objects they cover and thus cover less objects.
These concepts, in turn, may be used to study the diploma-job matching.

Many other techniques exist in AI and elsewhere for analysing data.
Features that distinguish FCA is that it is symbolic, i.e. concept descriptions are algebraic combinations of attributes, and complete,
i.e.\ FCA provides the lattice of \emph{all} symbolically described concepts covering the input data.
From this set, it is possible to select those which are more relevant to a particular purpose. %
Some numerical techniques would instead provide the concepts which optimise a criterion
or may project the object descriptions in a space that make their separation and grouping easier.
This leads to select the concepts to be considered, for which a description remains to be found.
FCA and these types of techniques are complementary.

Formal concept analysis has been put to work in a variety of applications \cite{ganter2005a,missaoui2022a}, but its Boolean descriptions are limited with respect to real data.
It has thus received many extensions for using concrete domains to overcome this limitation.
For instance, people may be further characterised by their salary, age and hobby.
`Conceptual scaling' can group age and salary into intervals so as to be treated by FCA.
However, such extensions do not cover relations between objects themselves:
the fact that students have been taught by specific teachers or have been registered to specific schools.

Relational concept analysis (RCA) extends FCA by taking into account \emph{relations between objects} \cite{rouanehacene2013a}.
Such relations induce possible `relational attributes' that are added to the initial contexts, through `relational scaling'.
From a family of related scaled contexts, RCA generates a family of dependent concept lattices which may include concepts that would not exist in FCA.
For instance, people may be related to their household, their employer or their properties.
Household, themselves may be described by their income, size and related to their members and the neighbourhood in which they are settled.
This may lead to classify household depending on theses features and creating new attributes for people based on whether they live in a low-income household and whether they have been employed by a school.
These new attributes, in turn, will generate new people concepts such as low-income household teachers and may lead to unveil indirect relations between household and jobs.

Although RCA was initially targeting conceptual description languages such as UML \cite{huchard2007a}, it has been generalised to more varied description logic constructors \cite{rouanehacene2013a}.
Relational concept analysis has been used for instance for analysing the ecological and sanitary quality of water courses \cite{ouzerdine2019a}.
For that purpose, it connects formal contexts corresponding to water courses to data collection points which themselves are connected to measures and to organisms collected in water that can be described by further attributes.
These relations between objects help generating richer concept descriptions comprising relations between concepts.
It is then possible to connect the abundance or scarcity of some species to the presence of some pollutants, e.g. glyphosate.
RCA has also been used for other purposes such as generating link keys used to extract links from RDF data sets \cite{atencia2019z}.

\medskip

In principle, RCA contributes to the completeness of FCA by providing more concepts to classify objects in.
However, some concepts that can be considered acceptable may still not be generated by RCA.
Indeed, in presence of circular relationships between objects, RCA may not identify reciprocally supporting concepts.
For instance, consider that people are related to the schools they have attended, and schools are related to the students that they graduated.
People attended multiple schools and schools had many student.
There may be populations who have attended specific groups of schools and graduated from these.
Although there may be no attribute distinguishing these populations and groups, they can be described by having attended at least one school of the group of schools having graduated only students of the population.
These reciprocally supporting attributes lead to a refined version of the returned concept lattice which contains more concepts.
Such concepts may reveal distinct populations based on various hidden factors not directly available from the data, such as wealth or information whose collection is not possible.

These mutually supported concepts may not be returned by RCA.
However, they determine families of more complete concept lattices covering the data.
Hence, it remains to determine which unique family is returned by RCA.
In this respect, this report may be thought of as a way to restore the completeness of RCA by providing more possible concepts.
This completeness may be useful in applications in which the most relevant solution is not the minimal one.
It may also contribute to unveil more implications or dependencies between concepts.

This problem of alternative RCA results stemmed out of curiosity.
It occurred to us through experimenting with relational concept analysis for extracting link keys.
Although RCA was returning acceptable results, it was easy to identify other acceptable results that it did not return.
When RCA is used for extracting description logic terminologies, it makes sense to return minimal terminologies that may be extended.
But different sets of link keys would return totally different sets of links.
The problem also manifests itself in applications in which developers add artificial identifiers in their data in order to constrain the returned solution to include more concepts \cite{dolques2012a,braud2018a}.

\medskip

To understand why such concepts are not provided, and which concepts are returned by RCA, this report questions its semantics.
The semantics of relational concept analysis has, so far, been provided in a rather operational way \cite{rouanehacene2013b}.
It specifies that RCA returns a family of concept lattices referring to each other that describe the input data and it shows that this result is unique.
However, when there exist cycles in the dependencies between data, several families may satisfy these constraints.
Hence, relational concept analysis needs a more precise and process-independent semantics that defines what it returns.
For that purpose, this report provides a structured description of the space on which relational concept analysis applies.
It then defines acceptable solutions as those families of concept lattices which belong to the space determined by the initial family of formal contexts (\emph{well-formed}), cannot scale new attributes (\emph{saturated}), and refer only to concepts of the family (\emph{self-supported}).

Relational concept analysis is then studied in a functional framework.
It characterises the acceptable solutions as the fixed points of two functions, one expansive, which extends concept lattices as long as there are reasons to generate concepts distinguishing objects, and the other contractive, which reduces concept lattices as long as the attributes they are built on are not supported by remaining concepts.
These functions can be iterated and the acceptable solutions are those families of concept lattices which are fixed points for both (Proposition~\ref{prop:accfp}): there is no reason to either extend or reduce them.
The results provided by RCA are then proved to be the smallest acceptable solution, which is the least fixed point of the expansive function (Proposition~\ref{prop:rcalfp}).
It also offers an alternative semantics based on the greatest element of this set, which is the greatest fixed point of the contractive function (Proposition~\ref{prop:acrgfp}).
The structure of the set of fixed points is further characterised to support algorithmic developments.

\medskip

This report extends the results obtained for the RCA$^0$ restriction of RCA  \cite{euzenat2021a}, which contains a single formal context, hence a single concept lattice, and no attribute, only relations.
In spite of the simplicity of RCA$^0$, the main arguments of this work were already valid for RCA$^0$ which remains a good introduction to the problems faced.
Here, we develop the full argument starting from RCA$^0$ and extending it step-by-step to apply to RCA.
A more direct and synthetic presentation of these results can be found in \cite{euzenat2025a}.
It sometimes uses slightly different concepts.
A full table provides the correspondence between these results in Appendix~\ref{sec:propindex}.

In addition to these more general results, the current report provides
better and more examples,
a revised more general notation,
wider related work, and
insights on the structure of the set of acceptable solutions.
It features many elementary properties and propositions which help keeping their proof manageable so the proofs are given immediately.

We first present the work on which this one builds (relational concept analysis) and relevant related work (Section~\ref{sec:prelim}).
We then provide simple examples illustrating that RCA and RCA$^0$ may accept concept lattices which are not those provided by RCA (Section~\ref{sec:examples}).
RCA$^0$ is then provided with a fixed-point semantics through an expansion function corresponding to the RCA algorithm (Section~\ref{sec:rca0fixpoint}).
We then discuss the notion of self-supported lattices (Section~\ref{sec:selfsup0}) which is characterised as the fixed points of another function that can be seen as dual to that used by RCA.
So far, these characterisations have been provided in parallel on both formal contexts and concept lattices.
Section~\ref{sec:dualspace} reconciles them by unifying both approaches in dependent pairs of contexts and lattices.
This allows us to precisely characterise the space of acceptable solutions, fixed points of these complementary functions, by considering the composition of the corresponding closures.
This is finally generalised to RCA globally by considering families of related context-lattice pairs and showing that the acceptable solutions are those families which are fixed-points for both functions (Section~\ref{sec:rcafixpoint}).
We characterise exactly the results of RCA as the smallest element of this set, we provide an alternative operation returning the greatest element and we study the structure of the set of acceptable solutions (Section~\ref{sec:semantics}).

\section{Preliminaries and related work}\label{sec:prelim}

We mix preliminaries with related work for reasons of space, but also because the report directly builds on this related work.

Hereafter, we use a simple and regular notation:
$K$ denotes contexts, $L$ lattices, $T$ context-lattice pairs and $O$ families of context-lattice pairs.
In addition, $R$ is used for relations, $N$ for names and $D$ for attributes.
Greek letters are devoted to some auxiliary functions ($\eta$, $\kappa$, $\sigma$, $\pi$).
$E$, $F$, $P$ and $Q$ are used for functions names and $^*$ is used for distinguishing the parallel application of such functions.
Table~\ref{tab:notation} provides a list of used symbols.

\begin{table}[!ht]
\begin{center}
\textbf{Sets and structures}
\end{center}
\begin{itemize}
\symitem{def:FCA}{$G$}{set of objects ($g\in G$)}
\symitem{def:FCA}{$M$}{set of attribute ($M\subseteq D$, $m\in M$)}
\symitem{def:FCA}{$I$}{incidence relation ($I\subseteq M\times G$)}
\symitem{def:J}{$J$}{`ternary' element used in conceptual scaling ($J\subseteq G \times M \times W$)}

\symitem{sec:rcasem}{$R$}{set of relations ($R\subseteq G\times G'$, $r\in R$)}
\symitem{def:FCA}{$K$}{formal contexts ($K=\langle G, M, I\rangle$, $K\in \mathcal{K}$)}
\symitem{def:Omega}{$\Omega$}{set of scaling operations ($\varsigma\in\Omega$)}

\symitem{sec:rscaling}{$N$}{concept names (given after extent $c\in N\subseteq 2^G$)}
\symitem{sec:scaling}{$D$}{domain language for expressing attributes (inspired from \hyperref[sec:pattstruct]{Pattern structures})}

\symitem{def:FCA}{$L$}{concept lattice ($L=\langle C,\preceq\rangle$, $L\in \mathcal{L}$)}
\symitem{def:FCA}{$C$}{set of formal concepts ($C\subseteq 2^{G\times M}$, $c\in C$)}

\symitem{def:T}{$T$}{context-lattice pair ($T=\langle K, L\rangle$, $T\in\mathcal{T}$)}
\symitem{def:O}{$O$}{indexed context-lattice pairs ($O=\{T_x\}_{x\in X}$, $O\in\mathcal{O}$) for modelling RCA entirely}
\symitem{sec:scaling}{$\Sigma$}{semantic structure grounding scaling ($\Sigma\in\mathcal{X}$, here $\Sigma\subseteq \{R\}\times \mathcal{L}$)}
\end{itemize}

\begin{center}
\textbf{Functions}
\end{center}
\begin{itemize}
\symitem{def:FCA}{$\mathrm{FCA}$}{:~ $\mathcal{K}\rightarrow\mathcal{L}$ Formal concept analysis (extended to indexed families as \hyperref[def:FCAs]{\ensuremath{\mathrm{FCA}^*}})}
\symitem{sec:fca}{$\eta$}{:~$\mathcal{K}\cup\mathcal{L}\rightarrow 2^{G}$ concept name extractors (\hyperref[sec:rcaname]{$\eta^*$})}
\symitem{sec:kappa}{$\kappa$}{:~$\mathcal{L}\rightarrow\mathcal{K}$ Context extraction operation}
\symitem{def:sigma}{$\sigma$}{:~$\mathcal{K}\times\mathcal{L}\rightarrow\mathcal{K}$ Scaling operation (\hyperref[def:sigmas]{$\sigma^*$})}
\symitem{def:purge}{$\pi$}{:~$\mathcal{L}\rightarrow\mathcal{K}$ Purge function (\hyperref[def:purges]{$\pi^*$})}
\symitem{def:F}{$F$}{:~$\mathcal{L}\rightarrow\mathcal{L}$ Concept lattice expansion}
\symitem{def:E}{$E$}{:~$\mathcal{K}\rightarrow\mathcal{K}$ Context expansion}
\symitem{def:P}{$P$}{:~$\mathcal{K}\rightarrow\mathcal{K}$ Context contraction}
\symitem{def:Q}{$Q$}{:~$\mathcal{L}\rightarrow\mathcal{L}$ Concept lattice contraction}
\symitem{def:Tconst}{$\mathrm{T}$}{:~$\mathcal{K}\rightarrow\mathcal{T}$ Context-lattice pair constructor (\hyperref[def:purges]{\ensuremath{\mathrm{T}^*}})}
\symitem{def:EF}{$EF$}{:~$\mathcal{T}\rightarrow\mathcal{T}$ Context-lattice pair expansion (\hyperref[def:EFs]{$EF^*$}:~$\mathcal{O}\rightarrow\mathcal{O}$)}
\symitem{def:PQ}{$PQ$}{:~$\mathcal{T}\rightarrow\mathcal{T}$ Context-lattice pair contraction (\hyperref[def:PQs]{$PQ^*$}:~$\mathcal{O}\rightarrow\mathcal{O}$)}
\symitem{sec:rca}{$\underline{\mathrm{RCA}}$}{:~$\mathcal{K}^*\rightarrow\mathcal{L}^*$ Relational concept analysis}
\symitem{def:ACR}{$\overline{\mathrm{RCA}}$}{:~$\mathcal{K}^*\rightarrow\mathcal{L}^*$ Relational concept analysis with a greatest fixed-point semantics}
\end{itemize}
\caption{Some symbols used in this document (signatures are simplified omitting $R$ and $\Omega$).}\label{tab:notation}
\end{table}

\subsection{Basics of formal concept analysis}\label{sec:fca}

This report relies only on the most basic results of formal concept analysis expressed as order-preserving functions. 

Formal Concept Analysis (FCA) \cite{ganter1999a} starts with a formal context (hereafter context) $\langle G,M,I\rangle$ where $G$ denotes a set of objects, $M$ a set of attributes, and $I \subseteq G \times M$ a binary relation between $G$ and $M$, called the incidence relation.
The statement $gIm$ is interpreted as `object $g$ has attribute $m$', also noted $m(g)$.
Two operators~$\cdot^{\uparrow}$ and $\cdot^{\downarrow}$ define a Galois connection between the powersets $\langle 2^G, \subseteq\rangle$ and $\langle 2^M, \subseteq\rangle$, with $A \subseteq G$ and $B \subseteq M$: 
\begin{align*}
  A^{\uparrow} = \{m \in M \mid gIm ~ \text{for all} ~ g \in A \}, \\
  B^{\downarrow} =\{g \in G \mid gIm ~ \text{for all} ~ m \in B \}.
\end{align*} 

The operators $\cdot^{\uparrow}$ and $\cdot^{\downarrow}$ are decreasing,
i.e. if $A_1 \subseteq A_2$ then $A_2^{\uparrow} \subseteq A_1^{\uparrow}$ and if $B_1 \subseteq B_2$ then $B_2^{\downarrow} \subseteq B_1^{\downarrow}$.
Intuitively, the less objects there are, the more attributes they share, and dually, the less attributes there are, the more objects have these attributes.
It can be checked that $A \subseteq A^{\uparrow\downarrow}$ and that $B \subseteq B^{\downarrow\uparrow}$,
that $A^{\uparrow} = A^{\uparrow\downarrow\uparrow}$ and that $B^{\downarrow} = B^{\downarrow\uparrow\downarrow}$.

A pair $\langle A,B\rangle\in 2^G\times 2^M$, such that $A^{\uparrow} = B$ and \mbox{$B^{\downarrow} = A$}, is called a formal concept (hereafter concept), where $A=extent(\langle A,B\rangle)$ is the extent and $B=intent(\langle A,B\rangle)$ the intent of $\langle A,B\rangle$.
Moreover, for a formal concept $\langle A,B\rangle$, $A$ and $B$ are closed for the closure operations $\cdot^{\uparrow\downarrow}$ and $\cdot^{\downarrow\uparrow}$, respectively, i.e. $A^{\uparrow\downarrow} = A$ and $B^{\downarrow\uparrow} = B$.

Concepts are partially ordered by 
$\langle A_{1},B_{1}\rangle \leq \langle A_{2},B_{2}\rangle \Leftrightarrow A_{1} \subseteq A_{2}$ or equivalently $B_{2} \subseteq B_{1}$.
With respect to this partial order, the set of all formal concepts is a complete lattice called the concept lattice of $\langle G,M,I\rangle$.
It has for supremum the concept $\top=\langle G, G^{\uparrow}\rangle$ and for infimum the concept $\bot=\langle M^{\downarrow}, M\rangle$.

\begin{example}[Formal concept analysis]\label{ex:fca}
Starting from a context $K_1^0=\langle G_1, M_1^0, I_1^0\rangle$ with $G_1=\{a,b,c\}$, $M_1^0=\{m_1,m_2,m_3\}$ and $I_1^0$ as the incidence relation whose table is given below, the application of $\mathrm{FCA}$ results in the lattice made of the concepts $ABC$, $AB$, $C$ and $\bot$ as:

\begin{center}
\setlength{\tabcolsep}{2pt}
\begin{tikzpicture}[font=\footnotesize]

  \draw (1,3.5) node {$\mathrm{FCA} ( $};

  \draw (3.5,3.5) node {
    \begin{tabular}{r|ccc}
      $K_1^0$ & $m_1$ & $m_2$ & $m_3$\\ \hline
      $a$ & & $\times$ & \\
      $b$ & & $\times$ & \\
      $c$ & $\times$ & & $\times$
    \end{tabular}
    };

    \draw (6.5,3.5) node {$ ) = $};

  \begin{scope}[xshift=8.5cm,yshift=1.75cm]
  \begin{scope}[xscale=.2,yscale=.25]

    \begin{dot2tex}[dot,tikz,codeonly,options=-traw]
      graph {
	graph [nodesep=1.5]
	node [style=smconcept]
	ABC [label="$\empty$
\nodepart{two}
$\empty$"]
	AB [label="$m_2$
\nodepart{two}
$a, b$"]
	C [label="$m_1,m_3$
\nodepart{two}
$c$"]

        C0 [label="$\empty$
\nodepart{two}
$\empty$"]
	ABC -- AB
	ABC -- C
        AB -- C0
        C -- C0
      }
    \end{dot2tex}
    \draw (ABC.east) node[anchor=west] {$ABC$};
    \draw (AB.south west) node[anchor=east] {$AB$};
    \draw (C.south east) node[anchor=west] {$C$};
    \draw (C0.south west) node[anchor=east] {$\bot$};
  \end{scope}

    \draw (0,2.5) node {$L_1^0$:};
\end{scope}
\end{tikzpicture}
\end{center}
By convention, concept lattices are represented by their reflexive-transitive reduction (Haase diagram) in which concepts are displayed in two parts: an upper part representing their intent and a lower part representing their extent.
They only display the proper part of their intent and extent.
Their actual intent is obtained by joining it to the union of the proper intents of their more general concepts.
Conversely, their extent is obtained by joining it to the union of the proper extent of their more specific concepts.
Hence, $ABC=\langle\{a,b,c\}, \varnothing\rangle$ and $\bot=\langle\varnothing,\{m_1,m_2,m_3\}\rangle$.
\end{example}

Formal concept analysis can be considered as a function that associates to a context $\langle G, M, I\rangle$ its concept lattice $\langle C, \leq\rangle=\mathrm{FCA}(\langle G, M, I\rangle)$\label{def:FCA}
(or $\underline{\mathfrak{B}}(G,M,I)$ \cite{ganter1999a}).
This is illustrated by Example~\ref{ex:fca}.
By abuse of language, when a variable $L$ denotes a concept lattice $\langle C, \leq\rangle$, $L$ will also be used to denote $C$.

The concepts that can be created from a context can be identified by their extent.
Hence, $\eta(\langle G, M, I\rangle)=2^{G}$ is the set of all concept names that may be used in any such concept lattice\footnote{A similar remark is made in \cite[\S4.1.2]{wajnberg2020a}.}.
We will identify the concepts by such sets;
the extent of a so-named concept will be the set of objects in its name.
In any specific concept lattice $L=\mathrm{FCA}(K)$, the subset $\eta(L)$\label{def:N} of $\eta(K)$ is the set of names of concepts in this lattice according to this convention as illustrated in Example~\ref{ex:cnames}.

\begin{example}[Concept names]\label{ex:cnames}
Consider the context $K_1^0$ of Example~\ref{ex:fca}.
The set of objects of $K_1^0$ being $G_1=\{a,b,c\}$, the set of concept names that can be created for them in any concept lattice is $\eta(K_1^0)=\{ABC,AB,AC,BC,A,B,C,\bot\}$.
In the specific lattice obtained in Example~\ref{ex:fca}, the set of concept names is $\eta(L_1^0)=\{ABC, AB, C, \bot\}$.

Throughout the report, concepts are named after their extent.
They will be displayed as uppercase character strings.
\end{example}

In order to discuss algorithms performing formal concept analysis, we will restrict ourselves to finite structures, as it is often the case.
In such a case, from finite contexts are generated finite lattices whose concepts have finite extents and intents.

\subsection{Extending formal concept analysis with scaling}\label{sec:scaling}

Formal concept analysis is defined on relatively simple structures hence many extensions of it have been designed.
These may allow formal concept analysis to
\begin{enumaa}
\item deal with more complex input structures, and/or
\item generate more expressive and interpretable knowledge structures.
\end{enumaa}

\subsubsection{Scaling: a generalisation}\label{sec:gscaling}

Scaling is one kind of extension of type ($a$).
It is a way to encode a more complex structure $\Sigma$ into FCA.
For that purpose, a scaling operation $\varsigma$
determines a set $D_{\varsigma,\Sigma}$ of Boolean attributes to be added to a context $K$ from a structure $\Sigma$.
In scaled contexts, these attributes can be interpreted so that the incidence relation $I$ is immediately derived from the attribute $m\in D_{\varsigma,\Sigma}$ following:
\[ \Sigma\models gIm \text{ or } \Sigma\models m(g). \]
In FCA, $D=M$ and $I$ is provided by its matrix:
\[ I\models m(g) \text{ iff } \langle m, g\rangle\in I. \]
Hence, adding attributes $M'$ to a context under such a structure $\Sigma$ consists of adding the attributes and extending the incidence relation according to this interpretation.
It may be performed as:
\[ \mathrm{K}^\Sigma_{+M'}(\langle G, M, I\rangle)=\langle G, M\cup M', I\cup\{ \langle g, m\rangle\in G\times M' ~|~  \Sigma\models m(g) \}\rangle \]
and suppressing them as:
\[ \mathrm{K}^\Sigma_{-M'}(\langle G, M, I\rangle)=\langle G, M\setminus M', I\setminus\{ \langle g, m\rangle\in G\times M'\}\rangle. \]

Applying a scaling operation $\varsigma$ to a context $K$ following a structure $\Sigma$ can thus be decomposed into
\begin{enumii}
\item determining the set $D_{\varsigma,\Sigma}$ of attributes to add, and
\item extending the context with these attributes:
\end{enumii}
\[
  \sigma_{\varsigma}(K, \Sigma)=\mathrm{K}^\Sigma_{+D_{\varsigma,\Sigma}}(K).
\label{def:sigma}\]
This unified view of scaling may be applied to many available scaling operations.
We discuss these below.

\subsubsection{Conceptual scaling}\label{sec:cscaling}

Attributes found in data sets typically do not range on Booleans, but instead in numbers, intervals, strings, etc.
Such data can be represented as a many-valued context $\Sigma=\langle G,M,W,J\rangle$, such that $G$ is a set of objects, $M$ a set of attributes, $W$ a set of values, and $J$\label{def:J} a ternary relation defined on 
$G \times M \times W$.
$\langle g,m,w \rangle \in J$ or simply $m(g) = w$ means that object $g$ takes the value $w$ for the attribute $m$.
In addition, when $\langle g,m,w\rangle \in J$ and $\langle g,m,v\rangle \in J$ then $w = v$ \cite[\S1.3]{ganter1999a}: in FCA, `many-valued' means that the range of an attribute may include more than two values, but for any object, the attribute can only have one of these values.

Conceptual scaling transforms such a many-valued context into a one-valued context.
For instance, for nominal scaling, given a set $W$ of values and a set $N$ of properties taking these values, $D_{=,\Sigma}=\{ n=w | n\in N \text{ and } w\in W \}$ splits the ranges of the multi-valued attributes in $N$ into binary attributes.
Attributes of $D_{=,\Sigma}$ are interpreted as:
\[ \langle G,N,W,J\rangle\models gI(n=w) \text{ iff } \langle g,n,w\rangle\in J. \]

There are other types of scalings and some of them are detailed in Table~\ref{tab:cscaling}.
The same can be built for ordinal scaling, e.g.
\[ \langle G,N,W,J\rangle\models gI(n \leq w) \text{ iff } \langle g,n,v\rangle\in J\wedge v\leq w. \]

\begin{table}
  \setlength\tabcolsep{4pt}
\centering\begin{tabular}{lcccl}
  name & language ($D$) & scale & $\Sigma$ & condition ($m(g)$) \\ \hline
FCA & $m$ & - & $I$ & $\langle m,g\rangle\in I$\\
dichotomic & $n=v$ & monocolumn & $\langle G,N,W,J\rangle$ & $n(g)=v$\\
nominal & $n=w$ & diagonal & $\langle G,N,W,J\rangle$ & $n(g)=w$ $\forall w\in W$\\
ordinal & $n\leq w$ & triangular & $\langle G,N,W,J\rangle$ & $n(g)\leq w$ $\forall w\in W$\\
inter-ordinal & $n\leq w$, $n\geq w$ & - & $\langle G,N,W,J\rangle$ & $n(g)\leq w$ or \\
& & & & $n(g)\geq w$ $\forall w\in W$\\
contranominal & $n\neq w$ & antidiagonal & $\langle G,N,W,J\rangle$ & $n(g)\neq w$  $\forall w\in W$\\
\end{tabular}
\caption{Conceptual scaling operations (inspired from \cite{ganter1999a}).}\label{tab:cscaling}
\end{table}

These scaling operations only use a simple structure, i.e. $\Sigma=\langle G, N, W, J\rangle$ in which everything is stored in $J$ and the attributes are expressed as predicates, e.g. $\cdot=v$ for nominal scaling or $\cdot\leq n$ for ordinal scaling.
$\models$ is the evaluation of the predicate for the value, hence they can have also been called structural scaling.

\begin{example}[Conceptual scaling]\label{ex:cscale}
  Consider a many-valued context $\langle G, N, W, J\rangle$ with $G=\{\textlb{Alice},$ $\textlb{Bob},$ $\textlb{Carol}\}$, $N=\{\textlb{age},\textlb{shoesize}\}$, $W=\mathbb{N}$ and $J$ as below:
\begin{center}\footnotesize
\begin{tabular}{cc}
\begin{minipage}{.45\textwidth}
\begin{center}
    \begin{tabular}{r|cc}
     $J$ & \rotatebox{90}{\textlb{age}} & \rotatebox{90}{\textlb{shoesize}} \\ \hline
      \textlb{Alice} & 12 & 32 \\
      \textlb{Bob} & 14 & 36 \\
      \textlb{Carol}  & 14 & 34
    \end{tabular}
  \end{center}
\end{minipage}
&
\begin{minipage}{.45\textwidth}
  \begin{center}
    \begin{tabular}{r|ccc}
     $I$ & \rotatebox{90}{\textlb{age}$=12$} & \rotatebox{90}{\textlb{age}$=14$} & \rotatebox{90}{\textlb{shoesize}$<35$} \\ \hline
      \textlb{Alice} & $\times$ & &  $\times$ \\
      \textlb{Bob} &  & $\times$ &  \\
      \textlb{Carol} &  & $\times$ &  $\times$
    \end{tabular}
  \end{center}  
\end{minipage}
\end{tabular}
\end{center}
  It is possible to scale \textlb{age} with nominal scaling and \textlb{shoesize} with (partial) ordinal scaling generating the context with $G$, $M=\{\textlb{age}$=12$, \textlb{age}$=14$, \textlb{shoesize}$<35$\}$ and $I$ as above.
\end{example}

\subsubsection{Relational scaling}\label{sec:rscaling}

Relational scaling operations ($\varsigma$) considered in \cite{rouanehacene2013a} create relational attributes from a binary relation $r\subseteq G\times G'$ between two sets of objects $G=\mathrm{dom}(r)$ and $G'=\mathrm{cod}(r)$ and a concept lattice $L$ on $G'$.
$\varsigma(r,c)$ provides the syntactic form of the attribute.
For example, qualified existential scaling ($\varsigma=\exists$) is expressed by $\exists(r,c)=\exists r.c$.
For instance, the employer relation may relate people to companies depending on whether one works for the other.
If the company lattice contains the school concept, then the relational attribute $\exists\textlb{employer}.\textlb{School}$ may be scaled and holds for all people employed by a school.

Thus, the set of qualified existential attributes are:
\[ D_{\exists,\langle r, L\rangle} = \{ \exists r.c ~|~ c\in \eta(L) \}. \]
Such attributes are interpreted, according to a closed-world description logic interpretation, by
\[ \langle r, L\rangle\models gI\exists r.c ~\text{ iff }~ \exists g';\langle g,g'\rangle\in r\wedge
g'\in extent(c).\]

This is illustrated in Example~\ref{ex:rscale}.
\begin{example}[Relational scaling]\label{ex:rscale}
Consider that the relation $q$ is given by the table:
\begin{center}\footnotesize
\setlength{\tabcolsep}{2pt}
  \begin{tabular}{r|ccc}
                       $q$ & $a$ & $b$ & $c$\\ \hline
                       $d$ & $\times$ &   & \\
                       $e$ &   & $\times$ & \\
                       $f$ &   &   & $\times$
  \end{tabular}
\end{center}
between $G_1=\{a, b, c\}$ and $G_2=\{d, e, f\}$
and consider the concept lattice $L_1^0$ corresponding to the context $K_1^0$ of Example~\ref{ex:fca}. 
The concepts in $L_1^0$ have names in $\eta(L_1^0)=\{ABC, AB, C, \bot\}$ (Example~\ref{ex:cnames}) hence scaling by $\sigma_\exists$ will provide the attribute set $\{\exists q.ABC,$ $\exists q.AB,\exists q.C,\exists q.\bot\}$.
The description of the relation $q$ allows to uncover the incidence relation for these (the incidence relation for $\bot$ is always empty for $\exists$ so never displayed in the examples):
\begin{center}\footnotesize
\setlength{\tabcolsep}{2pt}
\begin{tabular}{r|ccc}
                        & \rotatebox{90}{$\exists q.ABC$} & \rotatebox{90}{$\exists q.AB$} & \rotatebox{90}{$\exists q.C$}\\ \hline
                       $d$ & $\times$ & $\times$ & \\
                       $e$ & $\times$ & $\times$ & \\
                       $f$ & $\times$ &   & $\times$
\end{tabular}
\end{center}
If, in addition, the context $K_2^0$ of $G_2$ had two other attributes $n_1$ and $n_2$ whose incidence relation is given as $I_2^0$ displayed below, then the scaling operation would correspond to:
\begin{center}\footnotesize
\setlength{\tabcolsep}{2pt}
\begin{tikzpicture}

  \draw (-2.25,0) node {$\sigma_{\exists} ( $};
  \draw (-.75,0) node {
  \begin{tabular}{r|cc}
    $K_2^0$ & $n_1$ & $n_2$\\ \hline
    $d$ & $\times$ & \\
    $e$ & $\times$ & $\times$\\
    $f$ &   &
  \end{tabular}
    };
  
  \draw (.75,0) node {$ , $};

  \draw (2,0) node {\begin{tabular}{r|ccc}
                       $q$ & $a$ & $b$ & $c$\\ \hline
                       $d$ & $\times$ &   & \\
                       $e$ &   & $\times$ & \\
                       $f$ &   &   & $\times$
  \end{tabular}
    };
  
  \draw (3,0) node {$ , $};

  \begin{scope}[xshift=3.5cm,yshift=-1.65cm]
    \begin{scope}[xscale=.2,yscale=.25]

    \begin{dot2tex}[dot,tikz,codeonly,options=-traw]
      graph {
	graph [nodesep=1.5]
	node [style=smconcept]
	ABC [label="$\empty$
\nodepart{two}
$\empty$"]
	AB [label="$m_2$
\nodepart{two}
$a, b$"]
	C [label="$m_1,m_3$
\nodepart{two}
$c$"]

        C0 [label="$\empty$
\nodepart{two}
$\empty$"]
	ABC -- AB
	ABC -- C
        AB -- C0
        C -- C0
      }
    \end{dot2tex}
    \draw (ABC.east) node[anchor=west] {$ABC$};
    \draw (AB.north west) node[anchor=south] {$AB$};
    \draw (C.south) node[anchor=north] {$C$};
  \end{scope}
  \end{scope}
  
  \draw (7.25,0) node {$ ) = $};

  \draw (9.5,.5) node {
    \begin{tabular}{r|ccccc}
      $K_2^1$ &  \rotatebox{90}{$n_1$} &  \rotatebox{90}{$n_2$} & \rotatebox{90}{$\exists q.ABC$} & \rotatebox{90}{$\exists q.AB$} & \rotatebox{90}{$\exists q.C$}\\ \hline
      $d$ & $\times$ &   & $\times$ & $\times$ & \\
      $e$ & $\times$ & $\times$ & $\times$ & $\times$ & \\
      $f$ &   &   & $\times$ &   & $\times$
    \end{tabular}
};
\end{tikzpicture}
\end{center}
\end{example}
When generating relational attributes, scaling only relies on the \emph{names} of the concepts in $L$.
It is thus possible to define the sets of attributes from the set of names.
The set of relational attributes that can be scaled from $\varsigma$ and $r$ against a set $N$ of concept of names is $D_{\varsigma,r,N}=\{\varsigma(r,c)~|~c\in N\}$.
This can be used with $\eta(L)$, i.e. the names of concepts actually in $L$, or with $\eta(K)$, i.e. the names of all possible concepts to be generated from a context $K$.
Example~\ref{ex:attr} illustrates this.

\begin{example}[Set of relational attributes]\label{ex:attr}
For the context $K_2^0$ of Example~\ref{ex:rscale}, if there is only one relation $q$, whose codomain is $K_1^0$ of Example~\ref{ex:fca}, and the existential scaling operation $\exists$, then the set of possibly scalable relational attributes is:
\[
  D_{\exists,q,\eta(K_1^0)}=\{\exists q.ABC,\exists q.AB,\exists q.AC,\exists q.BC,\exists q.A,\exists q.B,\exists q.C,\exists q.\bot\}.
\]
But using only those concepts from the lattice $L_1^0$ obtained in Example~\ref{ex:fca}, this set is reduced to:
\[
  D_{\exists,q,\eta(L_1^0)}=\{\exists q.ABC,\exists q.AB,\exists q.C,\exists q.\bot\}.
\]
\end{example}

Various relational scaling operations exist, such as existential, strict and wide universal, min and max cardinality, which all follow the classical role restriction semantics of description logics \cite{baader2007a} under the closed-world assumption (see Table~\ref{tab:rscaling}).
The set of relational attributes obtained from a relation $r$ by relational scaling may be large but remains finite as long as $G'=\mathrm{cod}(r)$ is finite.
Cardinality constraints, relying on integers, may entail infinite sets of concepts in theory, but in practice, when $G'$ is finite, the set of meaningful cardinality attributes is bounded by $|G'|$.

\begin{table}
  \setlength\tabcolsep{4pt}
\centering\begin{tabular}{lccl}
            name & language ($D$) & $\Sigma$ & condition ($m(g)$) \\ \hline
existential      & $\exists r$    & $R$ & $r(g)\neq\varnothing$\\
universal (wide) & $\forall r.c$  & $R, L$ & $r(g)\subseteq extent(c)$\\
strict universal & $\forall\exists r.c$ & $R, L$ & $r(g)\neq\varnothing\wedge r(g)\subseteq extent(c)$\\
contains (wide) & $\forall c.r $ & $R, L$ & $extent(c)\subseteq r(g)$\\
strict contains & $\forall\exists c.r $ & $R, L$ & $extent(c)\neq\varnothing$\\
            &  &  &  \quad$\wedge extent(c)\subseteq r(g)$\\
qualified existential & $\exists r.c$ & $R, L$ & $r(g)\cap extent(c)\neq\varnothing$\\
qualified min cardinality  & $\leq_n r.c$   & $R, L$ & $|r(g)\cap extent(c)|\leq n$\\
qualified max cardinality  & $\geq_n r.c$   & $R, L$ & $|r(g)\cap extent(c)|\geq n$\\
\hline
$\forall$-condition & $\forall\langle r,r'\rangle_k$ & $R\times R', L_{C\times C'}$ & $r(g)=_k r'(g')$\\
$\exists$-condition & $\exists\langle r,r'\rangle_k$ & $R\times R', L_{C\times C'}$ & $r(g)\cap_k r'(g')\neq\varnothing$\\
\end{tabular}
\caption{Relational scaling operations (inspired from \cite{rouanehacene2013a,braud2018a}) and additional link key condition scaling operations \cite{atencia2019z}.}\label{tab:rscaling}
\end{table}

In fact, RCA may be considered as a very general way to apply relational scaling across contexts.
Various kinds of relational scaling operations have been provided \cite{braud2018a,atencia2019z,wajnberg2020a}.

\subsubsection{Logical scaling}\label{sec:logicsc}

Logical scaling \cite{prediger1997a} has been introduced for more versatile languages such as description logics and SQL.
It introduces query results within contexts.
In this case, $\Sigma$ is a logical theory or database tables, $D$ the set of formulas of the logic or queries ($Q$) and $\models$ is entailment or query evaluation.
The scaling can be rewritten as:

\[ \Sigma\models gIQ \text{ iff } \Sigma\models Q(g). \]

\noindent Here, the nearly identical notation shows the relevance of this generalisation.
We expressed it with respect to one individual $g$, so it applies to unary queries or formulas with one variable placeholder.
However, it is possible to generalise this to contexts in which individuals in $G$ are elements of the products of sets of individuals.

\subsubsection{Relational scaling as logical scaling}\label{sec:rscaslsc}

The type of scaling used by RCA, relational scaling, can be thought of as an extension of logical scaling based on description logic.

Relational scaling is based on a set of contexts $\{\langle G_\xx, M_\xx, I_\xx\rangle\}_{\xx\in\XX}$, the corresponding lattices $\{L_\xx\}_{\xx\in\XX}=\{\mathrm{FCA}(\langle G_\xx, M_\xx, I_\xx\rangle)\}_{\xx\in\XX}$ and a set $R$ of relations.
This input can be encoded as sets of description logic axioms by:
\begin{align*}
|K_\xx| &= \{ m(g) ~|~ m\in M_\xx\wedge g\in G_\xx\wedge gI_\xx m \},\\
|r| &=\{ r(g,g') ~|~ g\in G_\xx\wedge g'\in G_\zz\wedge \langle g, g'\rangle\in r \},\\
  \|L_\xx\| &= \{ c\equiv\sqcap_{d\in intent(c)} d ~|~ c\in L_\xx \}.
\end{align*}
The elements in $|\cdot|$ are part of an ABox and those in $\|\cdot\|$ are part of a TBox.
They may be combined into a description logic knowledge base $\Sigma=\langle T_\Sigma, A_\Sigma\rangle$ such that:
\begin{align*}
  T_\Sigma &= \bigcup_{\xx\in\XX} \|L_\xx\|,\\
  A_\Sigma &= \bigcup_{\xx\in\XX} |K_\xx| \cup \bigcup_{r\in R} |r|.
\end{align*}

In principle, it should be possible to deduce the extent of each concept and the order between concepts from this:
\begin{align*}
  \Sigma\models c(g) &\text{ iff } g\in extent(c),\\
  \Sigma\models c\sqsubseteq c' &\text{ iff } c\leq c'.
\end{align*}

The attributes provided by relational scaling are description logic concept descriptions, i.e. unary predicates.
They can be interpreted with respect to the knowledge base associated to $\Sigma$:

\[
  \Sigma\models gI\forallexists p.c\text{ iff } \Sigma\models (\forall p.c\sqcap\exists p)(g).
\]

This way of interpreting relational scaling opens the door to introducing arbitrary description logic axioms within $\Sigma$ and thus to use background knowledge.

\subsection{Other extensions}\label{sec:otherext}

There are other extensions providing formal concept analysis with more expressiveness in the expression of intents without scaling (type $b$ extension).
Instead of scaling, they change the structure of the set of attributes, staying within the scope of Galois lattices.
We mention them here briefly.

\subsubsubsection{Logical concept analysis}\label{sec:lca}

Logical concept analysis is an extension of formal concept analysis in which the set of attributes is replaced by logical formulas attached to objects \cite{ferre2000a}. 

$\langle 2^M,\subseteq,\cap,\cup\rangle$ is replaced by $\langle L,\models,\vee,\wedge\rangle$ in which $L$ is a set of logic formulas.
The extension is defined semantically, hence the formula may be thought of as the class of equivalent formulas. 
It could be redefined by using closed sets of formulas instead of single formulas and closed union instead of conjunction.

The context is $\langle G, L, i\rangle$ with $i:G\rightarrow L$ a mapping.
In this case, the two operators $\intent{\cdot}$ and $\extent{\cdot}$ define a Galois connection between $\langle 2^G, \subseteq\rangle$ and $\langle L,\models\rangle$ with $O\subseteq G$ and $\phi\in L$:
\begin{align*}
  \intent{O} & = \bigvee_{o\in O} i(o),\\
  \extent{\phi} & = \{ o\in G ~|~ i(o)\models \phi \}.
\end{align*}
As for scaling, it is possible to rewrite:
\[
  \langle L,\models,\vee,\wedge\rangle\models oI\phi \text{ iff } i(o)\models \phi.
\]

A benefit of plunging the logic in FCA is the definition of contextualised entailment\footnote{Contextualized deduction in \cite{ferre2000a}.} as:
\[ \phi\models^K \psi \text{ iff } \extent{\phi}\subseteq\extent{\psi}. \]
This extension is very `structural': concepts are carrying theories which do not apply to the objects of the concepts.

\subsubsubsection{Generalised formal concept analysis}\label{sec:gfca}

Generalised formal concept analysis \cite{chaudron2000a} attaches to each object $g$ an element  $T=\zeta(g)$ of a lattice $\langle \mathcal{L},\sqsubseteq,\sqcap,\sqcup\rangle$ which replaces
$\langle 2^M,\subseteq,\cap,\cup\rangle$ in formal concept analysis.
Hence a general context is a triple $\langle G, \mathcal{L},\zeta\rangle$,
such that $\zeta:G\rightarrow\mathcal{L}$.
The incidence relation is implicitly defined by:
\[ \langle \mathcal{L},\sqsubseteq,\sqcap,\sqcup\rangle\models gI T \text{ iff } \zeta(g)\sqsupseteq T. \]
\noindent so that the two operators $\intent{\cdot}$ and $\extent{\cdot}$ define a Galois connection between $\langle 2^G, \subseteq\rangle$ and $\langle \mathcal{L},\sqsubseteq\rangle$ with $O\subseteq G$ and $T\in \mathcal{L}$:
\begin{align*}
  O^\uparrow &= \sqcap_{g\in O} \zeta(g),\\
  T^\downarrow &= \{ g\in G ~|~ T\sqsubseteq \zeta(g) \}.
\end{align*}

This theory has been instantiated to the unification of existentially quantified conjunctions of first order atoms, represented as sets of atoms. 
$\sqsubseteq$ is syntactic subsumption, i.e. the fact that there exists a variable substitution on the subsumee such that it is included in the subsumer.
In such a case, $\mathcal{L}$ is the set of such formulas reduced to a non-redundant form and $\sqcap$ is antiunification.
Because of the use of unification, this approach remains syntactic.

\subsubsubsection{Pattern Structures}\label{sec:pattstruct}

It is also possible to avoid scaling and to directly work on complex data, using the formalism of `pattern structures' \cite{ganter2001a,kaytoue2011a}.
Pattern structures generalise FCA in a similar way as logical concept analysis.
In this case, $2^M$ is replaced by elements of a meet-semilattice $\langle D, \sqcap\rangle$ \cite{ganter2001a,kuznetsov2009a}.
The context is now $\langle G, D, \delta\rangle$ with $\delta: G\rightarrow D$ a mapping.

In this case, the two operators $\intent{\cdot}$ and $\extent{\cdot}$ define a Galois connection between $\langle 2^G, \subseteq\rangle$ and $\langle D, \sqsubseteq\rangle$ with $A\subseteq G$ and $d\in D$:
\begin{align*}
  \intent{A} & = \sqcap_{g\in A} \delta(g),\\
  \extent{d} & = \{ g\in G ~|~ d\sqsubseteq\delta(g) \}.
\end{align*}
such that $c\sqsubseteq d \equiv c\sqcap d=c$.

This requires to define:
\begin{enumaa}
\item how to order its elements ($\sqsubseteq$), and
\item how to test that an object satisfies an attribute expression ($gId$ for $d\in D$).
\end{enumaa}
This can be rewritten as for scaling:
\[
  \langle D, \sqsubseteq\rangle\models gId \text{ iff } d\sqsubseteq \delta(g).
\]

Pattern structures \cite{ganter2001a,kuznetsov2009a} provide a more structured attribute language without scaling.

However, these extensions are not directly affected by the problem of context dependencies considered here as the attributes do not refer to concepts.

\subsubsubsection{Relational extensions}

On the contrary, other approaches \cite{kotters2013a,ferre2020a} aim at extracting conceptual structures from $n$-ary relations without resorting to scaling.
Their concepts have intents that can be thought of as conjunctive queries and extents as tuples of objects, i.e. answers to these queries.
Hence, instead of being classes, i.e. monadic predicates, concepts correspond to general polyadic predicates.
For that purpose, they rely on more expressive input, e.g. in Graph-FCA \cite{ferre2020a} the incidence relation is a hypergraph between objects, and produce more expressive representations.
A comparison of RCA and Graph-FCA is provided in \cite{keip2020a}.
Graph-FCA adopts a different approach from RCA but should, in principle, suffer from the same problem as the one considered here as soon as it contains circular dependencies: intents would need to refer to concepts so created, i.e. named subqueries.
This remains to be studied.

\subsubsubsection{Terminological base extraction}

Finally, description logic base mining \cite{baader2008b,guimaraes2023a} and relational concept analysis share the same purpose: inferring a TBox from an ABox (taken as an interpretation).
However, RCA does this by introducing new named concepts based on FCA, though description logic base mining does not introduce new names but uses new concept descriptions inspired from Duquenne-Guigues implication bases \cite{guigues1986a}.
Where, in Example~\ref{ex:rscale}, relational scaling would use attribute $\exists q.AB$, base mining would use the description $\exists q.\exists m_2.\top$.
As soon as cycles occur in context dependencies, this naturally leads to cyclic concept definitions.
This has been interpreted with the greatest fixed-point semantics in $\FranzEL_{gfp}$.
However, led by complexity considerations, work has focused on extracting minimal bases in $\FranzEL$ through unravelling \cite{baader2008b,guimaraes2023a}.
The problem raised in this report is different but applies as well to description logic base mining as soon as it is taken as a knowledge induction task from data:
circular dependencies may lead to different, equally well-behaving, bases that would be worth taking in consideration.

\subsection{A very short introduction to RCA}\label{sec:rca}

Relational Concept Analysis (RCA) \cite{rouanehacene2013a} extends FCA to the processing of relational datasets and allows inter-object relations to be materialised and incorporated into formal concept intents.
RCA can be considered as a way to induce a description logic TBox from a simple ABox \cite{baader2007a}, using specific scaling operations.
It  may also be thought of as a general way to deal with circular references using different scaling operations.

RCA applies
\begin{itemize}
\item a set $\Omega$\label{def:Omega} of relational scaling operations on
\item a family of contexts $K^0=\{\langle G_\xx, M^0_\xx, I^0_\xx\rangle\}_{\xx\in\XX}$ indexed by a finite set $\XX$, such that, if $\xx\neq\zz$, $G_\xx\cap G_\zz=\varnothing$, and
\item a finite set $R$ of binary relations, i.e. relations $r\subseteq G_\xx\times G_\zz$ (with $\xx, \zz\in\XX$).
\end{itemize}
$\langle K^0, R\rangle$ is called a relational context\footnote{We use the term `relational context' instead of `relational context family'.}.
Each context of the family $K^t$ may be abbreviated as $K^t_\xx=\langle G_\xx, M^t_\xx, I^t_\xx\rangle$.
The use of the ordinal superscript $t$ will become clearer in a few lines and can be ignored at that stage.
We note $R_{\xx}=\{r\in R~|~ \mathrm{dom}(r)=G_\xx\}$ and $R_{\xx,\zz}=\{r\in R_{\xx}~|~ \mathrm{cod}(r)=G_\zz\}$.

\subsubsection{Concept names and relational attributes}\label{sec:rcaname}

The notation used for expressing sets of relational attributes can be generalised.
For RCA, $\eta^*(K^0)=\{\eta(K^0_\xx)\}_{\xx\in\XX}$ is the indexed set of all concept names induced from all contexts in $K^0$.
Similarly, for an indexed set of concept lattices $L=\{L_\xx\}_{\xx\in\XX}$, $\eta^*(L)=\{\eta(L_\xx)\}_{\xx\in\XX}$.
Then, given a set $\Omega$ of relational scaling operations, a set $R$ of relations, the set of scalable relational attributes for a concept with respect to an indexed family of sets of concept names $N$ is:
\[ D_{\Omega,R_\xx,N} = \bigcup_{\varsigma\in\Omega}\bigcup_{\zz\in\XX}\bigcup_{r\in R_{\xx,\zz}} D_{\varsigma,r,N_\zz}. \]

This notation is used for identifying the relational attributes that can be scaled for a context $K_\xx$ from a set of concept lattices:
\[ D_{\Omega,R_\xx,L} = D_{\Omega,R_\xx,\eta^*(L)} \]
or the set of all the possible relational attributes that can be scaled for a context $K_\xx$ given a family of formal contexts $K^0$ as 
\[ D_{\Omega,R_\xx,K^0} = D_{\Omega,R_\xx,\eta^*(K^0)}. \]

Since $\forall \zz\in\XX$, $\eta(L_\zz)\subseteq \eta(K^0_\zz)$, then $D_{\Omega,R_\xx,\{L_\zz\}_{\zz\in\XX}}\subseteq D_{\Omega,R_\xx,K^0}$.
This is illustrated by Example~\ref{ex:attr2}.

\begin{example}[Set of relational attributes (cont'd)]\label{ex:attr2}
Following Example~\ref{ex:attr},
if $K_2^0$, whose objects are $\{d, e, f\}$, is also linked to itself ($K_2^0$) by the relation $s$ and the current set of concept names is $\eta(L_2^0)=\{DEF,DE,E\}$, then it would additionally scale the relational attributes:
$D_{\{\exists\},\{s\},\{L_2^0\}}=D_{\exists,s,\eta(L_2^0)}=\{\exists s.DEF, \exists s.DE, \exists s.E\}$.
If, in addition, the strict contains scaling operation ($\forall\exists C.r$, see Table~\ref{tab:rscaling}) is used, then new relational attributes would be:
\begin{align*}
  D_{\{\exists,\forall\exists\},\{q,s\},\{L_1^0,L_2^0\}} = \{ &\exists q.ABC,\exists q.AB,\exists q.C,\exists q.\bot, \exists s.DEF, \exists s.DE, \exists s.E, \\
  & \forall\exists ABC.q,\forall\exists AB.q,\forall\exists C.q,\forall\exists \bot.q, \forall\exists DEF.s, \forall\exists DE.s, \forall\exists E.s\}.
\end{align*}
\end{example}

The semantics of these attributes is provided in the same way as above and noted $\langle R,L\rangle\models gI\varsigma(r,c)$.

\subsubsection{Operations and algorithm}\label{sec:rcaop}

Hereafter, we will consider relational scaling with the structure $\Sigma=\langle R, L\rangle$ made of a set $R$ of binary relations between two sets of objects from $K^0$, and a family $L^t=\{L^t_\xx\}_{\xx\in\XX}$ of concept lattices obtained from $K^t$.

RCA applies relational scaling operations from a set $\Omega$ to each $K^t_\xx\in K^t$ and all relations $r\in R_{\xx,\zz}$ from the set of concepts in the corresponding $L_\zz^{t}=\mathrm{FCA}(K_\zz^t)$.
Such scaling is defined as:
\[ \sigma_\Omega(K_\xx, R, L) = \mathrm{K}^{\langle R, L\rangle}_{+D_{\Omega,R_\xx,L}}(K_\xx). \]

The classical RCA algorithm, that is called here $\underline{\mathrm{RCA}}$, thus relies on $\mathrm{FCA}$ and $\sigma_{\Omega}$.
More precisely, it applies these in parallel on all contexts.
Hence, $\mathrm{FCA}^*$\label{def:FCAs} and $\sigma_{\Omega}^*$\label{def:sigmas} are defined as:
\begin{align*}
  \mathrm{FCA}^*(\{K_\xx\}_{\xx\in\XX}) & = \{\mathrm{FCA}(K_\xx)\}_{\xx\in\XX},\\
  \sigma^*_{\Omega}(\{K_\xx\}_{\xx\in\XX},R, L) & = \left\lbrace \sigma_\Omega(K_\xx,R,L) \right\rbrace_{\xx\in\XX}.
\end{align*}
\noindent such that $L$ is a family of concept lattices.
The whole family of concepts lattices needs to be passed to $\sigma$.

$\underline{\mathrm{RCA}}$ starts from the initial family of contexts $K^0$ and iterates the application of the two operations:
\[
  K^{t+1} = \sigma^*_{\Omega}(K^t,R, \mathrm{FCA}^*(K^{t}))
\]
until reaching a fixed point, i.e. reaching $n$ such that $K^{n+1}=K^{n}$.
Then, $\underline{\mathrm{RCA}}_\Omega(K^0,R)=\mathrm{FCA}^*(K^n)$.

Thus, the $\underline{\mathrm{RCA}}$ algorithm proceeds in the following way:
\begin{enumerate}
\item Initial contexts: $t\leftarrow 0$; $\{\langle G_\xx, M^{t}_\xx, I^{t}_\xx\rangle\}_{\xx\in\XX} \leftarrow \{\langle G_\xx, M_\xx, I_\xx\rangle\}_{\xx\in\XX}$.
\item\label{it:init-rca} $\{L^{t}_\xx\}_{\xx\in\XX} \leftarrow \mathrm{FCA}^*(\{\langle G_\xx, M^{t}_\xx, I^{t}_\xx\rangle\}_{\xx\in\XX})$ (or, for each context, $\langle G_\xx, M^{t}_\xx, I^{t}_\xx\rangle$ the corresponding concept lattice $L^{t}_\xx=\mathrm{FCA}(\langle G_\xx, M^{t}_\xx, I^{t}_\xx\rangle)$ is created using FCA).
\item $\{\langle G_\xx, M^{t+1}_\xx, I^{t+1}_\xx\rangle\}_{\xx\in\XX} \leftarrow \sigma^*_\Omega(\{\langle G_\xx, M^{t}_\xx, I^{t}_\xx\rangle\}_{\xx\in\XX},R,\{L^{t}_\xx\}_{\xx\in\XX})$ (i.e. relational scaling is applied, for each relation $r$ whose codomain lattice has new concepts, generating new contexts $\langle G_\xx, M^{t+1}_\xx, I^{t+1}_\xx\rangle$ including both plain and relational attributes in $M^{t+1}_\xx$).
\item If $\exists \xx\in\XX$ such that $M^{t+1}_\xx\neq M^{t}_\xx$ (scaling has occurred), then $t\leftarrow t+1$; go to Step~\ref{it:init-rca}.
\item Return $\{L^{t}_\xx\}_{\xx\in \XX}$.
\end{enumerate}
This is illustrated by Example~\ref{ex:rca}.

\begin{example}[Relational concept analysis]\label{ex:rca}
Consider two relations $p$ and $q$ defined as:
\begin{center}\footnotesize
\setlength{\tabcolsep}{2pt}
\begin{tikzpicture}

  \draw (-4,0) node {\begin{tabular}{r|ccc}
                       $p$  & $d$ & $e$ & $f$ \\ \hline
                       $a$ & $\times$ &   & \\
                       $b$ &   & $\times$ & \\
                       $c$ &   &   & $\times$
                       \end{tabular}};

  \draw (4,0) node {\begin{tabular}{r|ccc}
                       $q$ & $a$ & $b$ & $c$\\ \hline
                       $d$ & $\times$ &   & \\
                       $e$ &   & $\times$ & \\
                       $f$ &   &   & $\times$
                       \end{tabular}};

\end{tikzpicture}
\end{center}
\noindent and applying on the contexts $K_1^0$ of Example~\ref{ex:fca} and $K_2^0$ of Example~\ref{ex:rscale} (Figure~\ref{fig:ex-3-O1}).

\begin{figure}[!ht]
\setlength{\tabcolsep}{2pt}
\centering
   \begin{tikzpicture}[font=\footnotesize]

  \draw (-6,0) node {\begin{tabular}{r|ccc}
                       $K_1^0$ & $m_1$ & $m_2$ & $m_3$\\ \hline
                       $a$ & & $\times$ & \\
                       $b$ & & $\times$ & \\
                       $c$ & $\times$ & & $\times$
                       \end{tabular}};

  \begin{scope}[xshift=-3cm,yshift=-1cm]
    \begin{scope}[xscale=.25,yscale=.25]
    \begin{dot2tex}[dot,tikz,codeonly,options=-traw]
      graph {
	graph [nodesep=1.5]
	node [style=concept]
	ABC [label="$$
\nodepart{two}
$\empty$"]
	AB [label="$m_2$
\nodepart{two}
$a, b$"]
	C [label="$m_1, m_3$
\nodepart{two}
$c$"]
	C0 [label="$\empty$
\nodepart{two}
$\empty$"]

        ABC -- AB
	ABC -- C
        AB -- C0
        C -- C0
      }
    \end{dot2tex}
    \node[anchor=east] at (ABC.west) {$ABC$};
    \node[anchor=east] at (AB.west) {$AB$};
    \node[anchor=east] at (C.west) {$C$};
    \node[anchor=east] at (C0.west) {$\bot$};
  \end{scope}
  \draw (0,0) node {$L_1^0$:};
  \end{scope}

  \begin{scope}[xshift=2cm,yshift=-1cm]
    \begin{scope}[xscale=.25,yscale=.25]
    \begin{dot2tex}[dot,tikz,codeonly,options=-traw]
      graph {
	graph [nodesep=1.5]
	node [style=concept]
	DEF [label="$\empty$
\nodepart{two}
$f$"]
	DE [label="$n_1$
\nodepart{two}
$d$"]
	E [label="$n_2$
\nodepart{two}
$e$"]

	DEF -- DE
	DE -- E
      }
    \end{dot2tex}
    \node[anchor=east] at (DEF.west) {$DEF$};
    \node[anchor=east] at (DE.west) {$DE$};
    \node[anchor=east] at (E.west) {$E$};
  \end{scope}
  \draw (-.75,0) node {$L_2^0$:};
  \end{scope}

  \draw (5.3,0) node {\begin{tabular}{r|cc}
                       $K_2^0$ & $n_1$ & $n_2$\\ \hline
                       $d$ & $\times$ & \\
                       $e$ & $\times$ & $\times$\\
                       $f$ &   &
                       \end{tabular}};

  \begin{scope}[yshift=-4cm]
  \draw (-6,0) node {\begin{tabular}{r|cccccc}
                       $K_1^1$ & \rotatebox{90}{$m_1$} & \rotatebox{90}{$m_2$} & \rotatebox{90}{$m_3$} & \rotatebox{90}{$\exists p.DEF$} & \rotatebox{90}{$\exists p.DE$} & \rotatebox{90}{$\exists p.E$}\\ \hline
                       $a$ &   & $\times$ &   & $\times$ & $\times$ &   \\
                       $b$ &   & $\times$ &   & $\times$ & $\times$ & $\times$ \\
                       $c$ & $\times$ &   & $\times$ & $\times$ &   &   
                       \end{tabular}};

  \begin{scope}[xshift=-3.5cm,yshift=-2.5cm]
    \begin{scope}[xscale=.25,yscale=.25]
    \begin{dot2tex}[dot,tikz,codeonly,options=-traw]
      graph {
	graph [nodesep=1.5]
	node [style=concept]
	ABC [label="$\exists p.DEF$
\nodepart{two}
$\empty$"]
	AB [label="$m_2$,$\exists p.DE$
\nodepart{two}
$a$"]
	B [label="$\exists p.E$
\nodepart{two}
$b$"]
	C [label="$m_1, m_3$
\nodepart{two}
$c$"]
	C0 [label="$\empty$
\nodepart{two}
$\empty$"]

        ABC -- AB
	ABC -- C
        AB -- B
        B -- C0
        C -- C0
      }
    \end{dot2tex}
    \node[anchor=east] at (ABC.west) {$ABC$};
    \node[anchor=east] at (AB.west) {$AB$};
    \node[anchor=east] at (B.west) {$B$};
    \node[anchor=east] at (C.west) {$C$};
    \node[anchor=east] at (C0.west) {$\bot$};
  \end{scope}
  \draw (0,0) node {$L_1^1$:};
  \end{scope}

  \begin{scope}[xshift=.7cm,yshift=-2.5cm]
    \begin{scope}[xscale=.25,yscale=.25]
    \begin{dot2tex}[dot,tikz,codeonly,options=-traw]
      graph {
	graph [nodesep=1.5]
	node [style=concept]
	DEF [label="$\exists q.ABC$
\nodepart{two}
$\empty$"]
	DE [label="$n_1, \exists q.AB$
\nodepart{two}
$d$"]
	E [label="$n_2$
\nodepart{two}
$e$"]
	F [label="$\exists q.C$
\nodepart{two}
$f$"]
	C0 [label="$\empty$
\nodepart{two}
$\empty$"]

	DEF -- DE
	DEF -- F
        DE -- E
        E -- C0
        F -- C0
      }
    \end{dot2tex}
    \node[anchor=east] at (DEF.west) {$DEF$};
    \node[anchor=east] at (DE.west) {$DE$};
    \node[anchor=east] at (E.west) {$E$};
    \node[anchor=east] at (F.west) {$F$};
    \node[anchor=east] at (C0.west) {$\bot$};
  \end{scope}
  \draw (0,0) node {$L_2^1$:};
  \end{scope}

  \draw (5.3,0) node {\begin{tabular}{r|ccccc}
                       $K_2^1$ &  \rotatebox{90}{$n_1$} &  \rotatebox{90}{$n_2$} & \rotatebox{90}{$\exists q.ABC$} & \rotatebox{90}{$\exists q.AB$} & \rotatebox{90}{$\exists q.C$}\\ \hline
                       $d$ & $\times$ &   & $\times$ & $\times$ & \\
                       $e$ & $\times$ & $\times$ & $\times$ & $\times$ & \\
                       $f$ &   &   & $\times$ &   & $\times$
                       \end{tabular}};
  \end{scope}

  \begin{scope}[yshift=-9.5cm]
  \draw (-6,-1) node {\begin{tabular}{r|ccccccc}
                       $K_1^2$ &  \rotatebox{90}{$m_1$} &  \rotatebox{90}{$m_2$} &  \rotatebox{90}{$m_3$} & \rotatebox{90}{$\exists p.DEF$} & \rotatebox{90}{$\exists p.DE$} & \rotatebox{90}{$\exists p.E$} & \rotatebox{90}{$\exists p.F$}\\ \hline
                       $a$ &   & $\times$ &   & $\times$ & $\times$ &   &   \\
                       $b$ &   & $\times$ &   & $\times$ & $\times$ & $\times$ &   \\
                       $c$ & $\times$ &   & $\times$ & $\times$ &   &   & $\times$
                       \end{tabular}};

  \begin{scope}[xshift=-3.75cm,yshift=-2.5cm]
    \begin{scope}[xscale=.25,yscale=.25]
    \begin{dot2tex}[dot,tikz,codeonly,options=-traw]
      graph {
	graph [nodesep=1.5]
	node [style=concept]
	ABC [label="$\exists p.DEF$
\nodepart{two}
$\empty$"]
	AB [label="$m_2, \exists p.DE$
\nodepart{two}
$a$"]
	B [label="$\exists p.E$
\nodepart{two}
$b$"]
	C [label="$m_1, m_3$,\\\\
$\exists p.F$
\nodepart{two}
$c$"]
	C0 [label="$\empty$
\nodepart{two}
$\empty$"]

        ABC -- AB
	ABC -- C
        AB -- B
        B -- C0
        C -- C0
      }
    \end{dot2tex}
    \node[anchor=east] at (ABC.west) {$ABC$};
    \node[anchor=east] at (AB.west) {$AB$};
    \node[anchor=east] at (B.west) {$B$};
    \node[anchor=east] at (C.west) {$C$};
    \node[anchor=east] at (C0.west) {$\bot$};
  \end{scope}
  \draw (0,0) node {$L_1^2$:};
  \end{scope}

  \begin{scope}[xshift=.45cm,yshift=-2.5cm]
    \begin{scope}[xscale=.25,yscale=.25]
    \begin{dot2tex}[dot,tikz,codeonly,options=-traw]
      graph {
	graph [nodesep=1.5]
	node [style=concept]
	DEF [label="$\exists q.ABC$
\nodepart{two}
$\empty$"]
	DE [label="$n_1, \exists q.AB$
\nodepart{two}
$d$"]
	E [label="$n_2, \exists q.B$
\nodepart{two}
$e$"]
	F [label="$\exists q.C$
\nodepart{two}
$f$"]
	C0 [label="$\empty$
\nodepart{two}
$\empty$"]

	DEF -- DE
	DEF -- F
        DE -- E
        E -- C0
        F -- C0
      }
    \end{dot2tex}
    \node[anchor=east] at (DEF.west) {$DEF$};
    \node[anchor=east] at (DE.west) {$DE$};
    \node[anchor=east] at (E.west) {$E$};
    \node[anchor=east] at (F.west) {$F$};
    \node[anchor=east] at (C0.west) {$\bot$};
  \end{scope}
  \draw (0,0) node {$L_2^2$:};
  \end{scope}

  \draw (5,1.5) node {\begin{tabular}{r|cccccc}
                       $K_2^2$ &  \rotatebox{90}{$n_1$} &  \rotatebox{90}{$n_2$} & \rotatebox{90}{$\exists q.ABC$} & \rotatebox{90}{$\exists q.AB$} & \rotatebox{90}{$\exists q.B$} & \rotatebox{90}{$\exists q.C$}\\ \hline
                       $d$ & $\times$ &   & $\times$ & $\times$ &   & \\
                       $e$ & $\times$ & $\times$ & $\times$ & $\times$ & $\times$ & \\
                       $f$ &   &   & $\times$ &   &   & $\times$
                       \end{tabular}};
  \end{scope}

\end{tikzpicture}
\caption{The three iterations of RCA from the initial contexts $K_1^0$ and $K_2^0$.}\label{fig:ex-3-O1}
\end{figure}
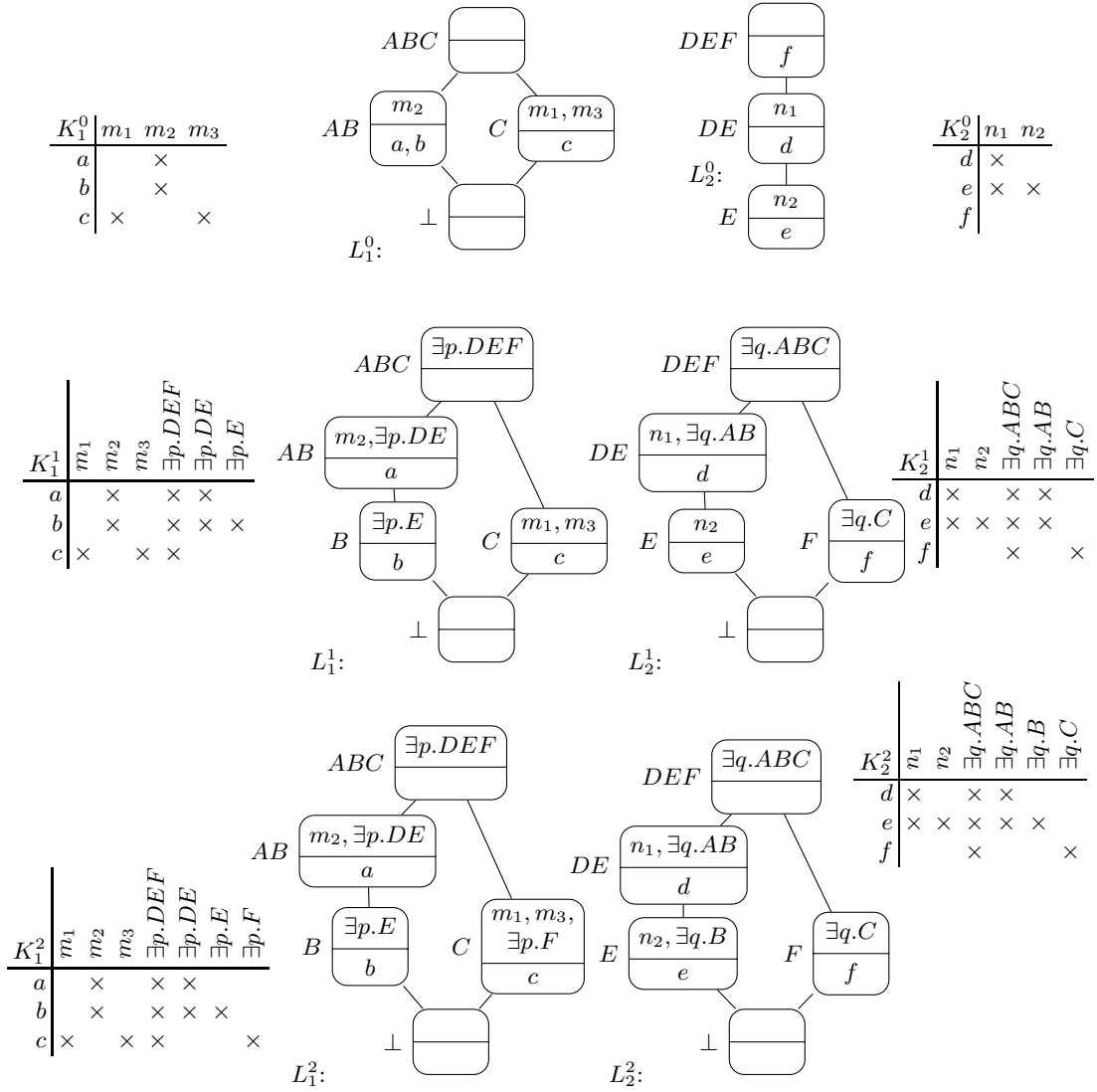

Applying $\mathrm{FCA}^*$ to the two contexts $K_1^0$ and $K_2^0$, provides the very simple lattices $L_1^0$ and $L_2^0$ of Figure~\ref{fig:ex-3-O1} with concepts $ABC$, $AB$, $C$ and $\bot$, and $DEF$, $DE$ and $E$, respectively.
Applying scaling, as seen partially in Example~\ref{ex:rscale}, provides the context $K_1^1$ with new attributes $\exists p.DEF$, $\exists p.DE$, $\exists p.E$ and $K_2^1$ with $\exists q.ABC$, $\exists q.AB$, $\exists q.C$ (Example~\ref{ex:attr}).
Applying $\mathrm{FCA}^*$ to these contexts provides the lattices $L_1^1$ and $L_2^1$ with additional concepts $B$ and $F$.
These can in turn go through scaling and unveil new attributes $\exists p.F$ and $\exists q.B$ added to $K^1_1$ and $K_2^1$ to give $K^2_1$ and $K_2^2$.
$\mathrm{FCA}^*$ introduces the new attributes in the intent of relevant concepts but does not introduce any new concept.
Hence the process stops and $\underline{\mathrm{RCA}}$ returns the family of concept lattices $\{L_1^2, L_2^2\}$.

The result is thus quite different from the $\{L_1^0, L_2^0\}$ that would have been returned by FCA alone.
\end{example}

These operations can be interpreted as generating a description logic T-box from a given A-box.
This can be seen in Example~\ref{ex:rcadl}.

\begin{example}[Relational concept analysis and description logics]\label{ex:rcadl}
As an example, consider the following ABox:
\begin{align*}
A_{12}^0 = \{ &\top_1(a), \top_1(b), \top_1(c), m_1(c), m_2(a), m_2(b), m_3(c), p(a,d), p(b,e), p(c,f),\\
& \top_2(d), \top_2(e), \top_2(f), n_1(d), n_1(e), n_2(e), q(d,a), q(e,b), q(f,c) \}
\intertext{
This can be encoded as the two context $K_1^0$ and $K_2^0$ (Figure~\ref{fig:ex-3-O1}) and the two relations $p$ and $q$ between these of Example~\ref{ex:rca}.}
\intertext{From this, RCA generates the lattices $L_1^2$ and $L_2^2$ (Figure~\ref{fig:ex-3-O1}) which can be interpreted as the description logic T-box:}
T_{12}^2 = \{ & ABC \sqsubseteq \top_1 \sqcap \exists p.DEF, AB \sqsubseteq ABC \sqcap m_2 \sqcap \exists p.DE, B \sqsubseteq AB \sqcap \exists p.E,\\
              & C \sqsubseteq ABC \sqcap m_1 \sqcap m_2 \sqcap \exists p.F, AB\sqcap C \sqsubseteq \bot, DEF \sqsubseteq \top_2 \sqcap \exists q.ABC,\\
              & DE \sqsubseteq DEF \sqcap n_1 \sqcap \exists q.AB, E \sqsubseteq DE \sqcap n_2 \sqcap \exists q.B, F \sqsubseteq DEF \sqcap \exists q.C,\\
              & DE\sqcap F \sqsubseteq \bot
\}
\intertext{with the improved A-box:}
A_{12}^2 = \{ & AB(a), B(b), C(c),  p(a,d), p(b,e), p(c,f),\\
& DE(d), E(e), F(f), q(d,a), q(e,b), q(f,c) \}
\end{align*}
\end{example}

\subsubsection{Properties and semantics}\label{sec:rcasem}

A context $K=\langle G, M, I\rangle$ is a subcontext of another $K'=\langle G', M', I'\rangle$ whenever $G\subseteq G'$, $M\subseteq M'$ and $I=I'\cap(G\times M)$ \cite[p.97]{ganter1999a}.
By abuse of notation, this is noted $K\subseteq K'$.
This is generalised to families of contexts $\{K_\xx\}_{\xx\in \XX}\subseteq\{K'_\xx\}_{\xx\in \XX}$ whenever $\forall\xx\in \XX$, $K_\xx\subseteq K'_\xx$.

$\underline{\mathrm{RCA}}$ always reaches a closed family of contexts for reason of finiteness \cite{rouanehacene2013a} and
the sequence $(K^t)_{t=0}^{n}$ is non-(intent-)contracting, i.e. $\forall t\geq 0, K^{t}\subseteq K^{t+1}$ \cite{rouanehacene2013b}.

The RCA semantics characterises the set of concepts in resulting RCA lattices as all and only those concepts grounded on the initial family of contexts ($K^0$) based on relations ($R$) \cite{rouanehacene2013b}.
This can thus be considered as a well-grounded semantics: an attribute is scaled and applied to an object at iteration $t+1$ only if its condition applies at stage $t$. Hence, everything is ultimately relying on $K^0$.

\cite{rouanehacene2013b} established that $\underline{\mathrm{RCA}}$ indeed finds \emph{the} $K^n$ satisfying these constraints through correctness (the concepts of $\mathrm{FCA}^*(K^n)$ are grounded in $K^0$ through $R$) and completeness (all so-grounded concepts are in $K^n$).

\subsection{Dependencies and cycles}\label{sec:cycles}

As can be seen, relations in RCA define a dependency graph between objects (of different or the same context).
In turn, this graph of objects induces a dependency graph between concepts through the scaled attributes that refer to other concepts.
It also induces a dependency graph between contexts: an edge exists between two contexts if an object of the former is related to an object of the latter.

This report is related to the circular dependencies, i.e. the circuits, that may exist within these graphs.
We say that a set $R$ of relations is \emph{hierarchical} if its object dependency graph is not circular.

Circular dependencies create a problem when one wants to define the family of concept lattices that should be returned by relational concept analysis.
As will be seen in Section~\ref{sec:examples}, there may exists several such families.

In order to explain some specific phenomena in a clearer way, it is possible to study them in various restrictions of RCA.
We introduced some of these that are organised in Figure~\ref{fig:rcarest}:
\begin{itemize}
\item RCA$^n$=RCA: RCA with $n$ contexts
\item RCA$^1$: RCA with one context
\item RCA$^0$: RCA with one \emph{empty} context
\item HRCA: RCA with only hierarchical relations.
\item HRCA$^1$=FCA: FCA
\item FCA$^0$=HRCA$^0$: FCA with empty contexts is clearly not interesting
\end{itemize}

\bigskip

\begin{figure}
\centering
\begin{tikzpicture}[xscale=1.5]
    \draw (6,.5) node (rca0) {RCA$^0$};

    \draw (6,-1) node (hrca0) {HRCA$^0$};
    \draw [->] (hrca0) -- (rca0);

    \draw (8,.5) node (rca1) {RCA$^1$};
    \draw [->] (rca0) -- (rca1);
    \draw (8,-1) node (fca) {HRCA$^1$=FCA};
    \draw [->] (fca) -- (rca1);

    \draw [->] (hrca0) -- (fca);
    \draw (fca.north west) edge[bend left=15,->] (rca0.south east);
    \draw (rca0.south east) edge[bend left=15,->] (fca.north west);

    \draw (10,.5) node (rca2) {RCA$^2$};
    \draw [->] (rca1) -- (rca2);
    \draw (10,-1) node (hrca2) {HRCA$^2$};
    \draw [->] (fca) -- (hrca2);

    \draw (12,.5) node (rca) {RCA};
    \draw (12,-1) node (hrca) {HRCA};
    \draw [->] (hrca) -- (rca);
    \draw [->,dotted] (hrca2) -- (hrca);
    \draw [->,dotted] (rca2) -- (rca);

\end{tikzpicture}
\caption{Relation between different restrictions of RCA (arrows mean: `can be rewritten into'). 
}\label{fig:rcarest}
\end{figure}
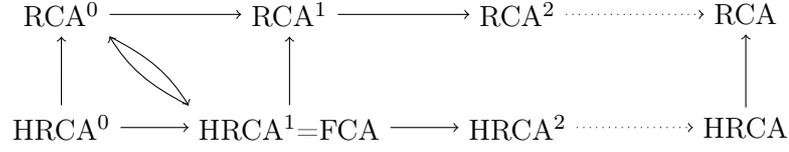

\label{sec:rca0}

We first study the semantics of RCA within RCA$^0$, a special case of RCA.
It is restricted in two ways:
\begin{itemize}
\item It contains only one context ($|X|=1$),
\item which has no attribute ($M_x^0=\emptyset$).
\end{itemize}
Additionally, we consider in the examples below only one single scaling operation: qualified existential scaling ($\Omega=\{\exists\}$).
The results of the report are independent from this choice in examples.

Because RCA$^0$ is a restriction of RCA, we will use the same notation as defined above, thought it operates on simpler structures.
Although RCA$^0$ seems very simple\footnote{An anonymous ICFCA 2021 reviewer complements the remarks of \S\ref{sec:otherext} noting that RCA$^0$ is also very related to Graph-FCA as they both have only one context and using existential scaling.}, FCA can be encoded into RCA$^0$.
Indeed, given a context $\langle G, M, I\rangle$, for each attribute $m\in M$ in the context, a relation $R_m\subseteq G\times G$ can be created such that $\forall g\in G, \langle g, g\rangle\in R_m$ if and only if $gIm$.
Starting with $K^0=\langle G, \varnothing, \varnothing\rangle$, it can be checked that $\sigma^*_{\exists}(K^0,R,\mathrm{FCA}^*(K^0))$ will simply add to $K^0$ one attribute $\exists r_m.\top$ per $m\in M$ which exactly corresponds to $m$.

It is also possible to encode RCA$^0$ into FCA using the following trick:
Given an RCA$^0$ relational context $\langle\{\langle G, \varnothing, \varnothing\rangle\}, R\rangle$, it can be encoded in a single FCA context:
\begin{itemize}
\item $G$ remains the same;
\item $M=\{ p^o | p\in R\wedge o\in G\}$;
\item $o'Ip^o$ iff $\langle o', o\rangle\in p$.
\end{itemize}
As a result, all the information from the relational context has been preserved and FCA will return a result analogous to $\underline{\mathrm{RCA}}^0_{\{\exists\}}$.

Introducing RCA$^0$ is sufficient to hint at the problems and solutions that we want to illustrate, as will now be presented.

\section{Examples}\label{sec:examples}

In order to illustrate the weakness of the RCA semantics, we first carry on the introductory Examples~\ref{ex:fca}, \ref{ex:rscale}--\ref{ex:rca} (\S\ref{sec:elabexrca}).
We then display it on more minimal examples that will be carried over the report:
the minimal example used in \cite{euzenat2021a} for RCA$^0$ (\S\ref{sec:elabexrca0})
and a somewhat equivalent example for RCA in general, i.e. involving more than one context (\S\ref{sec:exrca}).

From such a simple basis, it is possible to consider more complex settings:
\begin{itemize}
\item By using more than two contexts;
\item By using more than two relations between these contexts;
\item By using more than two objects in each context;
\item By using more than zero attributes in the contexts.
\end{itemize}

\subsection{RCA may accept different families of concept lattices}\label{sec:elabexrca}

The simple Example~\ref{ex:rca} (Figure~\ref{fig:ex-3-O1}) does not present a result in which each object is identified by a single class.
Indeed, $a$ has not more attributes than $b$.
This result could also be obtained with far more objects $a'$, $a''$, etc. sharing the attributes of $a$ and $b$, or duplicating other objects.

However, the lattices $L_1^\star$ and $L_2^\star$ displayed in Figure~\ref{fig:ex-3-alt} seem another good way to describe the data given as input to RCA.
There is in fact an objective difference between $AC$ and $BC$: $AC$ denotes all objects connected by $p$ to $DF$ and $BC$ all objects connected by $p$ to $EF$. Reciprocally, $DF$ denotes all objects connected by $q$ to $AC$ and $EF$ those connected by $q$ to $BC$.

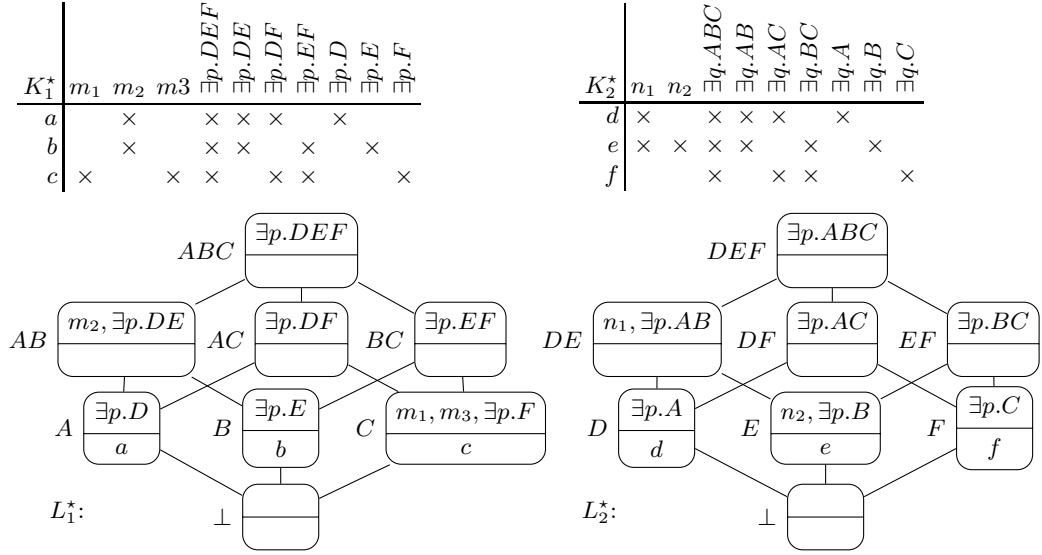
\begin{figure}[!ht]
\setlength{\tabcolsep}{2pt}
\centering
\begin{tikzpicture}[font=\footnotesize]

  \draw (-4,6) node {\begin{tabular}{r|cccccccccc}
                        $K_1^\star$ & $m_1$ & $m_2$ & $m_3$ & \rotatebox{90}{$\exists p.DEF$} & \rotatebox{90}{$\exists p.DE$} & \rotatebox{90}{$\exists p.DF$} & \rotatebox{90}{$\exists p.EF$} & \rotatebox{90}{$\exists p.D$} & \rotatebox{90}{$\exists p.E$} & \rotatebox{90}{$\exists p.F$}\\ \hline
                       $a$ &   & $\times$ &   & $\times$ & $\times$ & $\times$ &   & $\times$ &   &   \\
                       $b$ &   & $\times$ &   & $\times$ & $\times$ &   & $\times$ &   & $\times$ &   \\
                       $c$ & $\times$ &   & $\times$ & $\times$ &   & $\times$ & $\times$ &   &   & $\times$  
                       \end{tabular}};

  \draw (3,6) node {\begin{tabular}{r|ccccccccc}
                        $K_2^\star$ & $n_1$ & $n_2$ & \rotatebox{90}{$\exists q.ABC$} & \rotatebox{90}{$\exists q.AB$} & \rotatebox{90}{$\exists q.AC$} & \rotatebox{90}{$\exists q.BC$} & \rotatebox{90}{$\exists q.A$} & \rotatebox{90}{$\exists q.B$} & \rotatebox{90}{$\exists q.C$}\\ \hline
                       $d$ & $\times$ &    & $\times$ & $\times$ & $\times$ &   & $\times$ &   & \\
                       $e$ & $\times$ & $\times$ & $\times$ & $\times$ &   & $\times$ &   & $\times$ & \\
                       $f$ &   &   & $\times$ &   & $\times$ & $\times$ &   &   & $\times$
                       \end{tabular}};

  \begin{scope}[xshift=-6cm]
    \begin{scope}[xscale=.25,yscale=.25]

    \begin{dot2tex}[dot,tikz,codeonly,options=-traw]
      graph {
	graph [nodesep=1.5]
	node [style=concept]
	ABC [label="$\exists p.DEF$
\nodepart{two}
$\empty$"]
	AB [label="$m_2, \exists p.DE$
\nodepart{two}
$\empty$"]
	AC [label="$\exists p.DF$
\nodepart{two}
$\empty$"]
	BC [label="$\exists p.EF$
\nodepart{two}
$\empty$"]
	A [label="$\exists p.D$
\nodepart{two}
$a$"]
	B [label="$\exists p.E$
\nodepart{two}
$b$"]
	C [label="$m_1, m_3, \exists p.F$
\nodepart{two}
$c$"]
        C0 [label="$\empty$
\nodepart{two}
$\empty$"]
	ABC -- AB
	ABC -- AC
	ABC -- BC
	AB -- A
        AB -- B
        AC -- A
        AC -- C
        BC -- B
        BC -- C
        
        A -- C0
        B -- C0
        C -- C0
      }
    \end{dot2tex}
    \node[anchor=east] at (ABC.west) {$ABC$};
    \node[anchor=east] at (AB.west) {$AB$};
    \node[anchor=east] at (AC.west) {$AC$};
    \node[anchor=east] at (BC.west) {$BC$};
    \node[anchor=east] at (A.west) {$A$};
    \node[anchor=east] at (B.west) {$B$};
    \node[anchor=east] at (C.west) {$C$};
    \node[anchor=east] at (C0.west) {$\bot$};
  \end{scope}
  \draw (0,.5) node {$L_1^\star$:};
  \end{scope}

  \begin{scope}[xshift=1cm]
    \begin{scope}[xscale=.25,yscale=.25]

    \begin{dot2tex}[dot,tikz,codeonly,options=-traw]
      graph {
	graph [nodesep=1.5]
	node [style=concept]
	DEF [label="$\exists p.ABC$
\nodepart{two}
$\empty$"]
	DE [label="$n_1, \exists p.AB$
\nodepart{two}
$\empty$"]
	DF [label="$\exists p.AC$
\nodepart{two}
$\empty$"]
	EF [label="$\exists p.BC$
\nodepart{two}
$\empty$"]
	D [label="$\exists p.A$
\nodepart{two}
$d$"]
	E [label="$n_2, \exists p.B$
\nodepart{two}
$e$"]
	F [label="$\exists p.C$
\nodepart{two}
$f$"]
        C0 [label="$\empty$
\nodepart{two}
$\empty$"]
	DEF -- DE
	DEF -- DF
	DEF -- EF
	DE -- D
        DE -- E
        DF -- D
        DF -- F
        EF -- E
        EF -- F
        
        D -- C0
        E -- C0
        F -- C0
      }
    \end{dot2tex}
    \node[anchor=east] at (DEF.west) {$DEF$};
    \node[anchor=east] at (DE.west) {$DE$};
    \node[anchor=east] at (DF.west) {$DF$};
    \node[anchor=east] at (EF.west) {$EF$};
    \node[anchor=east] at (D.west) {$D$};
    \node[anchor=east] at (E.west) {$E$};
    \node[anchor=east] at (F.west) {$F$};
    \node[anchor=east] at (C0.west) {$\bot$};
  \end{scope}
  \draw (0,.5) node {$L_2^\star$:};

  \end{scope}

\end{tikzpicture}

\caption{Alternative concept lattices for the example of Section~\ref{sec:elabexrca} ($\star$ is simply a way to identify these objects).}\label{fig:ex-3-alt}
\end{figure}

The two corresponding T-Box and A-box would be:
\begin{align*}
T_{12}^\star = \{ & ABC \sqsubseteq \top_1 \sqcap \exists p.DEF, AB \sqsubseteq ABC \sqcap m_2 \sqcap \exists p.DE,
B \sqsubseteq AB \sqcap \exists p.E,\\
              & C \sqsubseteq ABC \sqcap m_1 \sqcap m_2 \sqcap \exists p.F, AB\sqcap C \sqsubseteq \bot, DEF \sqsubseteq \top_2 \sqcap \exists q.ABC,\\
              & DE \sqsubseteq DEF \sqcap n_1 \sqcap \exists q.AB, E \sqsubseteq DE \sqcap n_2 \sqcap \exists q.B, F \sqsubseteq DEF \sqcap \exists q.C,\\
              & DE\sqcap F \sqsubseteq \bot
\}\\
A_{12}^\star = \{ & A(a), B(b), C(c),  p(a,d), p(b,e), p(c,f),\\
& D(d), E(e), F(f), q(d,a), q(e,b), q(f,c) \}
\end{align*}

These lattices share many common points with those returned by RCA: they are also
\begin{enumerate}
\item valid concept lattices,
\item whose contexts extend $K_1^0$ and $K_2^0$ with attributes of $D_{\{\exists\},\{p\},K^0}$ and  $D_{\{\exists\},\{q\},K^0}$ (Example~\ref{ex:attr}),
\item stable for scaling, and
\item such that each attribute refers only to concepts in the lattices.
\end{enumerate}
The only difference with $L_1^2$ and $L_2^2$ is that they are not those returned by RCA.
We will temporarily informally consider lattices sharing these features as \emph{acceptable}.

But, if there exists several acceptable solutions for a given $\Omega$ and $R$, why does RCA only returns one of these, and which one?
To help answering this question, we illustrate the problem with minimal running examples below.

\subsection{Minimal \texorpdfstring{RCA$^0$}{RCA⁰} example}\label{sec:elabexrca0}

As an RCA$^0$ example, consider the following ABox:
\[A_0^0 = \{ \top_0(a), \top_0(b), \top_0(c), \top_0(d), r(a,b), r(b,a), r(c,d), r(d,c), r(a,a), r(b,b) \}.\]
It can be encoded as an empty context ($K_0^0$) from which FCA will generate the concept lattice $L_0^0$ as follows:
\begin{center}
\begin{tikzpicture}[font=\footnotesize]

  \draw (1,3.5) node {$\mathrm{FCA} ( $};

  \draw (4.25,3.5) node {
  \begin{tabular}{r|}
$K_0^0$\\
    \hline
$a$\\
$b$\\
$c$\\
$d$
 \end{tabular}
    };

    \draw (7.25,3.5) node {$ ) = $};

  \begin{scope}[xshift=9.5cm,yshift=3cm]
  \begin{scope}[xscale=.2,yscale=.25]

    \begin{dot2tex}[dot,tikz,codeonly,options=-traw]
      graph {
	graph [nodesep=1.5]
	node [style=smconcept]
	ABCD [label="$\empty$
\nodepart{two}
$a, b, c, d$"]
      }
    \end{dot2tex}
    \draw (ABCD.south east) node[anchor=north west,nosep] {$ABCD$};
  \end{scope}
  \draw (-.75,.5) node {$L_0^0$:};
\end{scope}

\end{tikzpicture}
\end{center}
Scaling with $\sigma_\exists$ and $r$ provides the attribute $\exists r.ABCD$:
\begin{center}
\begin{tikzpicture}[font=\footnotesize]

  \draw (-2.25,0) node {$\sigma_{\exists} ( $};
  \draw (-1,0) node {
  \begin{tabular}{r|}
$K_0^0$\\
    \hline
$a$\\
$b$\\
$c$\\
$d$
 \end{tabular}
    };
  
  \draw (0,0) node {$ , $};

  \draw (2,0) node {
\begin{tabular}{r|cccc}
  $r$  & $a$ & $b$ & $c$ & $d$
     \\\hline
        $a$& $\times$ & $\times$ &\\
        $b$& $\times$ & $\times$ & \\
        $c$&          &          &          & $\times$\\
        $d$&          &          & $\times$ &
 \end{tabular}
    };
  
  \draw (3.5,0) node {$ , $};

  \begin{scope}[xshift=5cm,yshift=-.5cm]
  \begin{scope}[xscale=.2,yscale=.25]

    \begin{dot2tex}[dot,tikz,codeonly,options=-traw]
      graph {
	graph [nodesep=1.5]
	node [style=smconcept]
	ABCD [label="$\empty$
\nodepart{two}
$a, b, c, d$"]
      }
    \end{dot2tex}
    \draw (ABCD.south east) node[anchor=north west,nosep] {$ABCD$};
  \end{scope}
  \draw (-.75,.5) node {$L_0^0$:};

  \end{scope}
  
  \draw (7.25,0) node {$ ) = $};

  \draw (10,0) node {
  \begin{tabular}{r|c}
    $K_0^1$ & \rotatebox{90}{$\exists r.ABCD$} \\ \hline
    $a$ & $\times$ \\
    $b$ & $\times$ \\
    $c$ & $\times$ \\
    $d$ & $\times$ 
  \end{tabular}
};
\end{tikzpicture}
\end{center}
\noindent which run through FCA returns:
\begin{center}
\begin{tikzpicture}[font=\footnotesize]

  \draw (1,3.5) node {$\mathrm{FCA} ( $};

  \draw (4.25,3.5) node {
    \begin{tabular}{r|c}
    $K_0^1$ & \rotatebox{90}{$\exists r.ABCD$} \\ \hline
    $a$ & $\times$ \\
    $b$ & $\times$ \\
    $c$ & $\times$ \\
    $d$ & $\times$ 
    \end{tabular}
    };

    \draw (7.25,3.5) node {$ ) = $};

  \begin{scope}[xshift=9cm,yshift=3cm]
  \begin{scope}[xscale=.2,yscale=.25]

    \begin{dot2tex}[dot,tikz,codeonly,options=-traw]
      graph {
	graph [nodesep=1.5]
	node [style=smconcept]
	ABCD [label="$\exists r.ABCD$
\nodepart{two}
$a, b, c, d$"]
      }
    \end{dot2tex}
    \draw (ABCD.south east) node[anchor=north west,nosep] {$ABCD$};
  \end{scope}
  \draw (-1,.5) node {$L_0^1$:};
\end{scope}

\end{tikzpicture}
\end{center}
Since no new concept has been added, scaling would return $K_0^1$, hence $L_0^1$ is the result returned by RCA$^0$ (and RCA).

However, the concept lattices $L_0'$ and $L_0^\star$ of Figure~\ref{fig:alternatives} are other valid lattices worth considering as acceptable solutions.
As in classical RCA, each concept of these lattices is closed with respect to the specific context scaled by $\exists$ and $r$ from the concepts of the lattice.
Moreover, the lattices are self-supported in the sense that their attributes refer only to their own concepts.

\begin{figure}
\centering
  \begin{tabular}{lr}
    \begin{minipage}{.35\textwidth}
\centering
  \begin{tikzpicture}[font=\footnotesize]
    \begin{scope}[xscale=.25,yscale=.25]
    \begin{dot2tex}[dot,tikz,codeonly,options=-traw]
      graph {
	graph [nodesep=1.5]
	node [style=concept]
	ABCD [label="$\exists r.ABCD$
\nodepart{two}
$\empty$"]
	AB [label="$\exists r.AB$
\nodepart{two}
$a, b$"]
	CD [label="$\exists r.CD$
\nodepart{two}
$c, d$"]

        C0 [label="$\empty$
\nodepart{two}
$\empty$"]
	ABCD -- AB
	ABCD -- CD
        AB -- C0
        CD -- C0
      }
    \end{dot2tex}
    \node[anchor=east] at (ABCD.west) {$ABCD$};
    \node[anchor=east] at (AB.west) {$AB$};
    \node[anchor=east] at (CD.west) {$CD$};
    \node[anchor=east] at (C0.west) {$\bot$};
  \end{scope}
  \draw (0,0) node {$L'_0$:};
  \end{tikzpicture}
\end{minipage}
&
    \begin{minipage}{.65\textwidth}
\centering
  \begin{tikzpicture}[font=\footnotesize]
    \begin{scope}[xscale=.28,yscale=.25]

    \begin{dot2tex}[dot,tikz,codeonly,options=-traw]
      graph {
	graph [nodesep=1.5]
	node [style=concept]
	ABCD [label="$\exists r.ABCD$
\nodepart{two}
$\empty$"]
	ABC [label="$\exists r.ABD$
\nodepart{two}
$\empty$"]
	ABD [label="$\exists r.ABC$
\nodepart{two}
$\empty$"]

	AB [label="$\exists r.AB$
\nodepart{two}
$a, b$"]
	CD [label="$\exists r.CD$
\nodepart{two}
$\empty$"]

	C [label="$\exists r.D$
\nodepart{two}
$c$"]
	D [label="$\exists r.C$
\nodepart{two}
$d$"]

        C0 [label="$\empty$
\nodepart{two}
$\empty$"]
	ABCD -- ABC
	ABCD -- ABD
	ABCD -- CD
	ABC -- AB
        ABC -- C
        ABD -- AB
        ABD -- D
        
        CD -- C
        CD -- D
        
        AB -- C0
        C -- C0
        D -- C0
      }
    \end{dot2tex}
    \node[anchor=east] at (ABCD.west) {$ABCD$};
    \node[anchor=east] at (ABC.west) {$ABC$};
    \node[anchor=east] at (ABD.west) {$ABD$};
    \node[anchor=east] at (AB.west) {$AB$};
    \node[anchor=east] at (CD.west) {$CD$};
    \node[anchor=east] at (C.west) {$C$};
    \node[anchor=east] at (D.west) {$D$};
    \node[anchor=east] at (C0.west) {$\bot$};
  \end{scope}
  \draw (0,.5) node {$L_0^\star$:};
  \end{tikzpicture}
\end{minipage}
\end{tabular}
\caption{Alternative concept lattices ($L_0'$ and $L_0^\star$).}\label{fig:alternatives}
\end{figure}
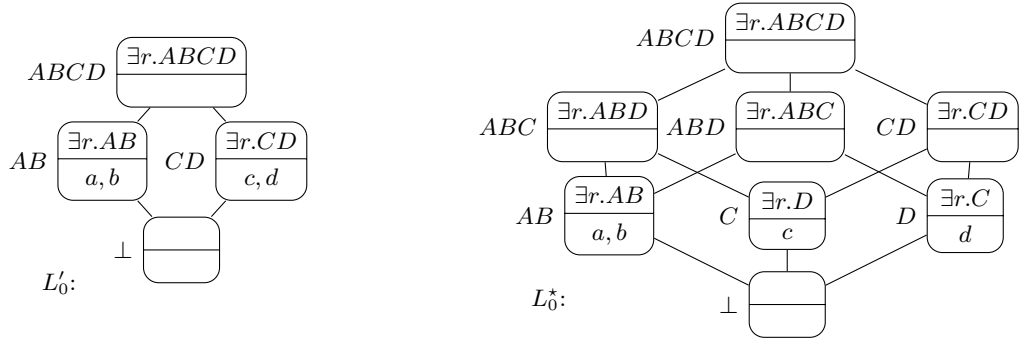

They correspond to different knowledge bases:
\begin{align*}
    T_0^1 & = \{ ABCD\sqsubseteq \exists r.ABCD\}\\
  A_0^1 & = \{ ABCD(a), ABCD(b), ABCD(c), ABCD(d),\\
  & \qquad r(a,b), r(b,a), r(c,d), r(d,c), r(a,a), r(b,b) \}
\intertext{and}
    T_0' & = \{ AB\sqsubseteq \top_0\sqcap\exists r.AB, CD\sqsubseteq \top_0\sqcap\exists r.CD, ABCD\sqsubseteq \exists r.ABCD \}\\
    A_0' & = \{ AB(a), AB(b), CD(c), CD(d), r(a,b), r(b,a), r(c,d), r(d,c), r(a,a), r(b,b) \}
\intertext{and}
    T_0^\star & = \{ AB\sqsubseteq ABC\sqcap ABD\sqcap \exists r.AB,
          C\sqsubseteq ABC\sqcap CD\sqcap \exists r.D,
          D\sqsubseteq ABD\sqcap CD\sqcap \exists r.C,\\
        & \qquad ABC\sqsubseteq ABCD\sqcap\exists r.ABD, ABD\sqsubseteq ABCD\sqcap\exists r.ABC, CD\sqsubseteq ABCD\sqcap\exists r.CD,\\
        & \qquad ABCD\sqsubseteq \exists r.ABCD\}\\
    A_0^\star & = \{ AB(a), AB(b), C(c), D(d), r(a,b), r(b,a), r(c,d), r(d,c), r(a,a), r(b,b) \}
\end{align*}
The problem that there exists several acceptable candidate lattices applies to RCA as a whole because RCA$^0$ is included in RCA.

\subsection{Minimal RCA example}\label{sec:exrca}

As another example, consider the following ABox:

$A_{34}^0 = \{ \top_3(a), \top_3(b), \top_4(c), \top_4(d), p(a,c), p(b,d), q(c,a), q(b,d) \}$

This can be encoded as the two empty contexts $K_3^0$ and $K_4^0$ of Figure~\ref{fig:ex-2-L0} and the two relations $p$ and $q$ of Figure~\ref{fig:ex-2-R}.

\begin{figure}[!ht]
\centering
\begin{tikzpicture}[font=\footnotesize]

  \draw (-4,0) node {\begin{tabular}{r|cc}
                       $p$  & $c$ & $d$  \\ \hline
                       $a$ & $\times$ &    \\
                       $b$ &   & $\times$
                       \end{tabular}};

  \draw (4,0) node {\begin{tabular}{r|cc}
                       $q$ & $a$ & $b$ \\ \hline
                       $c$ &  $\times$ &   \\
                       $d$ &    & $\times$
                       \end{tabular}};

\end{tikzpicture}
\caption{Relations $p$ and $q$ for RCA.}\label{fig:ex-2-R}
\end{figure}

Applying FCA to the two contexts $K_3^0$ and $K_4^0$ provides the very simple lattices $L_3^0$ and $L_4^0$ of Figure~\ref{fig:ex-2-L0}.
From this, RCA generates new context $K_3^1$ and $K_4^1$ through scaling which provides new lattices $L_3^1$ and $L_4^1$ (Figure~\ref{fig:ex-2-O1}).

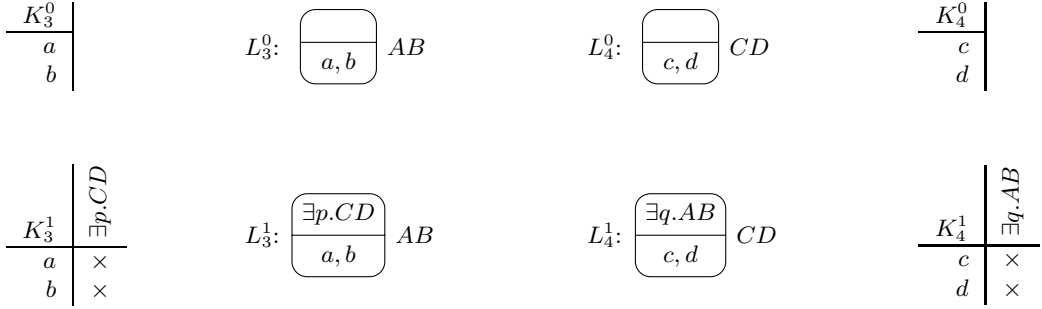
\begin{figure}[!ht]
\centering
   \begin{tikzpicture}[font=\footnotesize]
  \draw (-6,0) node[anchor=west] {\begin{tabular}{r|}
                       $K_3^0$\\ \hline
                       $a$\\
                       $b$\\
                       \end{tabular}};

  \draw (6,0) node[anchor=west] {\begin{tabular}{r|}
                       $K_4^0$\\ \hline
                       $c$\\
                       $d$\\
                       \end{tabular}};

     \begin{scope}[xshift=-1.5cm]
     \begin{scope}[xscale=.25,yscale=.2]
      \draw (0,0) node [style=concept] (C2) {
        \nodepart{two}
        $a, b$};
      \end{scope}
    \node[anchor=west] at (C2.east) {$AB$};
    \draw (-1,0) node {$L_3^0$:};
  \end{scope}

     \begin{scope}[xshift=3cm]
     \begin{scope}[xscale=.25,yscale=.2]
      \draw (0,0) node [style=concept] (C2) {
        \nodepart{two}
        $c, d$};
      \end{scope}
    \node[anchor=west] at (C2.east) {$CD$};
    \draw (-1,0) node {$L_4^0$:};
  \end{scope}

  \begin{scope}[yshift=-2.5cm]
  \draw (-6,0) node[anchor=west] {\begin{tabular}{r|c}
                       $K_3^1$ & \rotatebox{90}{$\exists p.CD$}\\ \hline
                       $a$ & $\times$\\
                       $b$ & $\times$
                       \end{tabular}};

  \draw (6,0) node[anchor=west] {\begin{tabular}{r|c}
                       $K_4^1$ & \rotatebox{90}{$\exists q.AB$}\\ \hline
                       $c$ & $\times$\\
                       $d$ & $\times$
                       \end{tabular}};

   \begin{scope}[xshift=-1.5cm]
     \begin{scope}[xscale=.25,yscale=.2]
       \draw (0,0) node [style=concept] (C2) {
         $\exists p.CD$
        \nodepart{two}
        $a, b$};
      \end{scope}
    \node[anchor=west] at (C2.east) {$AB$};
    \draw (-1,0) node {$L_3^1$:};
  \end{scope}

     \begin{scope}[xshift=3cm]
     \begin{scope}[xscale=.25,yscale=.2]
       \draw (0,0) node [style=concept] (C2) {
         $\exists q.AB$
        \nodepart{two}
        $c, d$};
      \end{scope}
    \node[anchor=west] at (C2.east) {$CD$};
    \draw (-1,0) node {$L_4^1$:};
  \end{scope}
  \end{scope}
  \end{tikzpicture}
\caption{The two iterations of RCA from the initial contexts $K_3^0$ and $K_4^0$.}\label{fig:ex-2-O1}\label{fig:ex-2-L0}
\end{figure}

The lattices $L_3^1$ and $L_4^1$ of Figure~\ref{fig:ex-2-O1} are those returned by RCA as applying scaling from them returns the same contexts $K_3^1$ and $K_4^1$.

However, there could be other acceptable solutions such as those displayed in Figure~\ref{fig:ex-2-alt}.
They are all acceptable solutions for $\{K_3^0,K_4^0\}$ (Figure~\ref{fig:ex-2-O1}) as they satisfy the four conditions of Section~\ref{sec:elabexrca}.

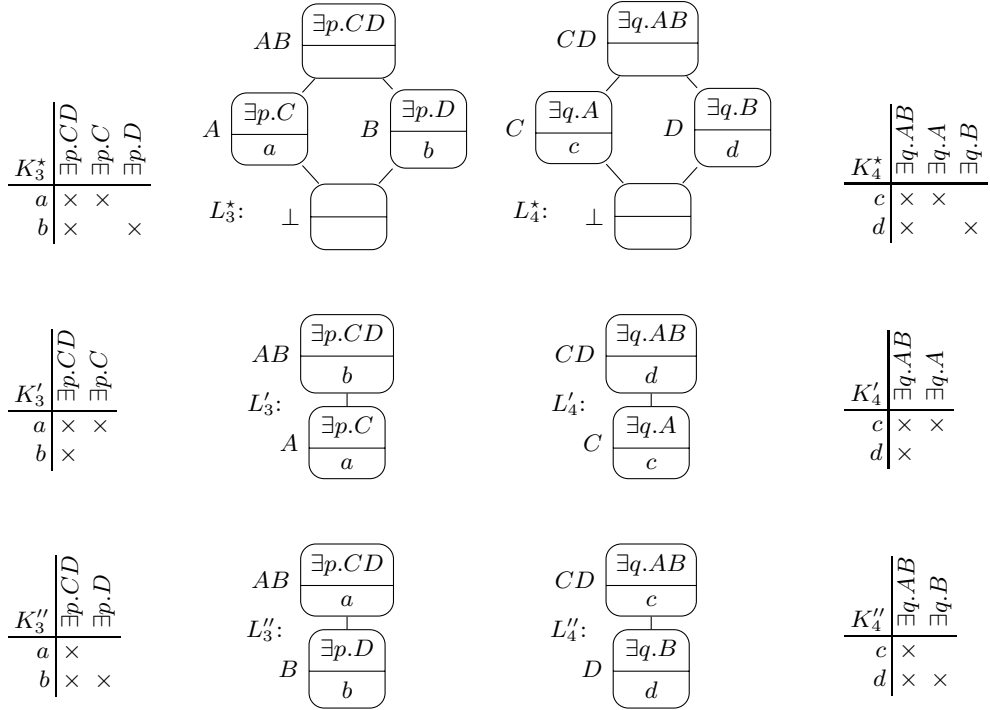
\begin{figure}[!ht]
\centering
\setlength{\tabcolsep}{2pt}
  \begin{tikzpicture}[font=\footnotesize]

  \draw (-5,0) node[anchor=west] {\begin{tabular}{r|ccc}
                       $K_3^\star$ & \rotatebox{90}{$\exists p.CD$}& \rotatebox{90}{$\exists p.C$}& \rotatebox{90}{$\exists p.D$}\\ \hline
                       $a$ & $\times$ & $\times$ &\\
                       $b$ & $\times$ &   & $\times$
                       \end{tabular}};

  \draw (6,0) node[anchor=west] {\begin{tabular}{r|ccc}
                       $K_4^\star$ & \rotatebox{90}{$\exists q.AB$}& \rotatebox{90}{$\exists q.A$}& \rotatebox{90}{$\exists q.B$}\\ \hline
                       $c$ & $\times$ & $\times$ &\\
                       $d$ & $\times$ &   & $\times$
                       \end{tabular}};

    \begin{scope}[xshift=-2cm,yshift=-1cm]
      \begin{scope}[xscale=.25,yscale=.25]
        \begin{dot2tex}[dot,tikz,codeonly,options=-traw]
          graph {
            graph [nodesep=1.5]
            node [style=concept]
            AB [label="$\exists p.CD$
            \nodepart{two}
            $\empty$"]
            A [label="$\exists p.C$
            \nodepart{two}
            $a$"]
            B [label="$\exists p.D$
            \nodepart{two}
            $b$"]
            
            C0 [label="$\empty$
            \nodepart{two}
            $\empty$"]
            AB -- A -- C0
            AB -- B -- C0
          }
        \end{dot2tex}
        \node[anchor=east] at (AB.west) {$AB$};
        \node[anchor=east] at (A.west) {$A$};
        \node[anchor=east] at (B.west) {$B$};
        \node[anchor=east] at (C0.west) {$\bot$};
      \end{scope}
      \draw (0,.5) node {$L_3^\star$:};
    \end{scope}

    \begin{scope}[xshift=2cm,yshift=-1cm]
    \begin{scope}[xscale=.25,yscale=.25]
    \begin{dot2tex}[dot,tikz,codeonly,options=-traw]
      graph {
	graph [nodesep=1.5]
	node [style=concept]
	CD [label="$\exists q.AB$
\nodepart{two}
$\empty$"]
	C [label="$\exists q.A$
\nodepart{two}
$c$"]
	D [label="$\exists q.B$
\nodepart{two}
$d$"]
        C0 [label="$\empty$
\nodepart{two}
$\empty$"]
	CD -- C -- C0
	CD -- D -- C0
      }
    \end{dot2tex}
    \node[anchor=east] at (CD.west) {$CD$};
    \node[anchor=east] at (C.west) {$C$};
    \node[anchor=east] at (D.west) {$D$};
    \node[anchor=east] at (C0.west) {$\bot$};
  \end{scope}
  \draw (0,.5) node {$L_4^\star$:};
  \end{scope}

  \begin{scope}[yshift=-3cm]
  \draw (-5,0) node[anchor=west] {\begin{tabular}{r|cc}
                       $K_3'$ & \rotatebox{90}{$\exists p.CD$}& \rotatebox{90}{$\exists p.C$}\\ \hline
                       $a$ & $\times$ & $\times$\\
                       $b$ & $\times$ &
                       \end{tabular}};

  \draw (6,0) node[anchor=west] {\begin{tabular}{r|cc}
                       $K_4'$ & \rotatebox{90}{$\exists q.AB$}& \rotatebox{90}{$\exists q.A$}\\ \hline
                       $c$ & $\times$ & $\times$\\
                       $d$ & $\times$ &
                       \end{tabular}};

  \begin{scope}[xshift=-1cm,yshift=-1cm]
    \begin{scope}[xscale=.25,yscale=.25]
    \begin{dot2tex}[dot,tikz,codeonly,options=-traw]
      graph {
	graph [nodesep=1.5]
	node [style=concept]
	AB [label="$\exists p.CD$
\nodepart{two}
$b$"]
	A [label="$\exists p.C$
\nodepart{two}
$a$"]

	AB -- A
      }
    \end{dot2tex}
    \node[anchor=east] at (AB.west) {$AB$};
    \node[anchor=east] at (A.west) {$A$};
  \end{scope}
  \draw (-.5,.95) node {$L_3'$:};
  \end{scope}

  \begin{scope}[xshift=3cm,yshift=-1cm]
    \begin{scope}[xscale=.25,yscale=.25]
    \begin{dot2tex}[dot,tikz,codeonly,options=-traw]
      graph {
	graph [nodesep=1.5]
	node [style=concept]
	CD [label="$\exists q.AB$
\nodepart{two}
$d$"]
	C [label="$\exists q.A$
\nodepart{two}
$c$"]
	CD -- C
      }
    \end{dot2tex}
    \node[anchor=east] at (CD.west) {$CD$};
    \node[anchor=east] at (C.west) {$C$};
  \end{scope}
  \draw (-.5,.95) node {$L_4'$:};
  \end{scope}
  \end{scope}

  \begin{scope}[yshift=-6cm]
  \draw (-5,0) node[anchor=west] {\begin{tabular}{r|cc}
                       $K_3''$ & \rotatebox{90}{$\exists p.CD$}& \rotatebox{90}{$\exists p.D$}\\ \hline
                       $a$ & $\times$ & \\
                       $b$ & $\times$ & $\times$
                       \end{tabular}};

  \draw (6,0) node[anchor=west] {\begin{tabular}{r|cc}
                       $K_4''$ & \rotatebox{90}{$\exists q.AB$}& \rotatebox{90}{$\exists q.B$}\\ \hline
                       $c$ & $\times$ & \\
                       $d$ & $\times$ & $\times$
                       \end{tabular}};

  \begin{scope}[xshift=-1cm,yshift=-1cm]
    \begin{scope}[xscale=.25,yscale=.25]
    \begin{dot2tex}[dot,tikz,codeonly,options=-traw]
      graph {
	graph [nodesep=1.5]
	node [style=concept]
	AB [label="$\exists p.CD$
\nodepart{two}
$a$"]
	B [label="$\exists p.D$
\nodepart{two}
$b$"]
	AB -- B
      }
    \end{dot2tex}
    \node[anchor=east] at (AB.west) {$AB$};
    \node[anchor=east] at (B.west) {$B$};
  \end{scope}
  \draw (-.5,.95) node {$L_3''$:};
  \end{scope}

  \begin{scope}[xshift=3cm,yshift=-1cm]
    \begin{scope}[xscale=.25,yscale=.25]
    \begin{dot2tex}[dot,tikz,codeonly,options=-traw]
      graph {
	graph [nodesep=1.5]
	node [style=concept]
	CD [label="$\exists q.AB$
\nodepart{two}
$c$"]
	D [label="$\exists q.B$
\nodepart{two}
$d$"]
	CD -- D
      }
    \end{dot2tex}
    \node[anchor=east] at (CD.west) {$CD$};
    \node[anchor=east] at (D.west) {$D$};
  \end{scope}
  \draw (-.5,.95) node {$L_4''$:};
  \end{scope}
  \end{scope}
\end{tikzpicture}

\caption{Alternative pairs of concept lattices covering the contexts of Figure~\ref{fig:ex-2-O1}.}\label{fig:ex-2-alt}
\end{figure}

On the contrary, Figure~\ref{fig:ex-2-noalt} displays a family of concept lattices $\{L_3^\#, L_4^\#\}$ which is not an acceptable solution.
Although they contain all concepts of $\{L_3^0, L_4^0\}$ and no concept not in $\{L_3^\star, L_4^\star\}$, they would generate more attributes through scaling and applying RCA to their contexts $\{K_3^\#, K_4^\#\}$ would lead to $\{L_3^\star, L_4^\star\}$

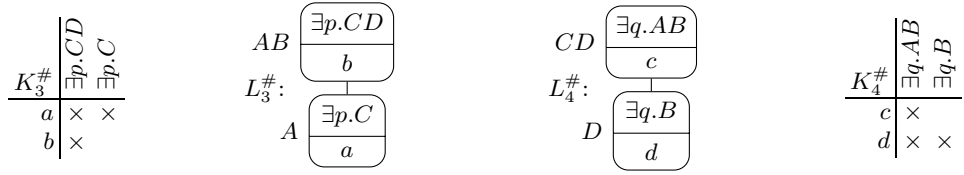
\begin{figure}[ht!]
  \centering
\setlength{\tabcolsep}{2pt}
  \begin{tikzpicture}[font=\footnotesize]
  \draw (-5,0) node[anchor=west] {\begin{tabular}{r|cc}
                       $K_3^\#$ & \rotatebox{90}{$\exists p.CD$}& \rotatebox{90}{$\exists p.C$}\\ \hline
                       $a$ & $\times$ & $\times$\\
                       $b$ & $\times$ & 
                       \end{tabular}};

  \draw (6,0) node[anchor=west] {\begin{tabular}{r|cc}
                       $K_4^\#$ & \rotatebox{90}{$\exists q.AB$}& \rotatebox{90}{$\exists q.B$}\\ \hline
                       $c$ & $\times$ & \\
                       $d$ & $\times$ & $\times$
                       \end{tabular}};

  \begin{scope}[xshift=-1cm,yshift=-1cm]
    \begin{scope}[xscale=.25,yscale=.25]
    \begin{dot2tex}[dot,tikz,codeonly,options=-traw]
      graph {
	graph [nodesep=1.5]
	node [style=concept]
	AB [label="$\exists p.CD$
\nodepart{two}
$b$"]
	A [label="$\exists p.C$
\nodepart{two}
$a$"]

	AB -- A
      }
    \end{dot2tex}
    \node[anchor=east] at (AB.west) {$AB$};
    \node[anchor=east] at (A.west) {$A$};
  \end{scope}
  \draw (-.5,1) node {$L_3^\#$:};
  \end{scope}

  \begin{scope}[xshift=3cm,yshift=-1cm]
    \begin{scope}[xscale=.25,yscale=.25]
    \begin{dot2tex}[dot,tikz,codeonly,options=-traw]
      graph {
	graph [nodesep=1.5]
	node [style=concept]
	CD [label="$\exists q.AB$
\nodepart{two}
$c$"]  
	D [label="$\exists q.B$
\nodepart{two}
$d$"]
	CD -- D
      }
    \end{dot2tex}
    \node[anchor=east] at (CD.west) {$CD$};
    \node[anchor=east] at (D.west) {$D$};
  \end{scope}
  \draw (-.5,1) node {$L_4^\#$:};
  \end{scope}
  \end{tikzpicture}
\caption{A family of concept lattices $\{ L_3^\#, L_4^\#\}$ which is not an acceptable solution.}\label{fig:ex-2-noalt}
\end{figure}

Hence the question raised in the previous section: Why does RCA return only one solution, and which one?
Answering it requires to reconsider the RCA semantics.
More precisely, it requires to define formally which families of concept lattices could be considered as acceptable solutions and which of them is returned by the RCA operation.

\bigskip
The following aims at defining acceptability formally.
For that purpose, we first consider the simpler case of RCA$^0$.
We define expansion functions (Section~\ref{sec:rca0fixpoint}) and contraction functions (Section~\ref{sec:selfsup0}), both from the independent standpoints of contexts and lattices.
Then, because these are in fact dependent, we apply them to context-lattice pairs (Section~\ref{sec:dualspace}).
These results are then extended to RCA (Section~\ref{sec:rcafixpoint}) in which acceptability is related to the set of fixed points of both extended function.
This allows us to precisely define the fixed-point semantics of RCA (Section~\ref{sec:semantics}).
\section{A parallel context-lattice fixed-point semantics for \texorpdfstring{RCA$^0$}{RCA⁰}}\label{sec:rca0fixpoint}

In order to investigate the semantics of relational concept analysis, we adopt a functional standpoint in which RCA is defined as a function in a precisely defined space.
In Sections~\ref{sec:ctxtspace} and \ref{sec:latfixpoint}, we provide two alternative, and equivalent, characterisations of that space, which provide the semantics for RCA$^0$.
In Section~\ref{sec:rca0andlfp}, we relate the well-grounded RCA semantics to these two semantics by showing that RCA$^0$ computes the least fixed point of these functions.

In order to investigate the semantics of relational concept analysis, we need to define the objects on which it applies.
They are determined by three elements given once and for all:
$K^0=\{\langle G_\xx, M^0_\xx, I^0_\xx\rangle\}_{\xx\in\XX}$,
$R$, and
$\Omega$.
Through the application of $RCA$, only $M^t_\xx$ and $I^t_\xx$ change, hence the other may remain implicit.

Although the results of this section aims at RCA$^0$, most of it applies to RCA, hence we use RCA$^0$ only when it matters.

\paragraph{Warning: a nest of fixed points}
RCA is a world of fixed points, hence it is easy to get lost among the various fixed points involved:
\begin{itemize}
\item In description logics, which RCA targets, the semantics of concepts is given by (least) fixed points when circularities occur \cite{nebel1990a};
\item FCA's goal is to compute fixed points: concepts are the result of a closure operation which is also a fixed point \cite{belohlavek2008a};
\item finally, the RCA result is the fixed point of the function that grows a family of concept lattices from the previous one through scaling.
\end{itemize}
The present work is concerned with the fixed points of the latter function taking the others into account.

\subsection{Semantics and properties: the context approach}\label{sec:ctxtspace}

We first study the semantics of RCA from the standpoint of the contexts.
We define precisely the space of contexts in which RCA$^0$ is computed and the functions underlying RCA in that space.

\subsubsection{The lattice $\mathcal{K}$ of \texorpdfstring{RCA$^0$}{RCA⁰} contexts}\label{sec:K}

The contexts considered by RCA are scaled from the initial context using the scaling operations.
In RCA$^0$, they are determined by $K^0=\langle G, \varnothing, \varnothing\rangle$, $R$, and $\Omega$.
But most of this section will ignore it and consider $K^0=\langle G, M, I\rangle$.
Through the application of $RCA$, only $M$ and $I$ change, the latter depending directly from the former (Property~\ref{prop:IdependM0}).

\begin{property}[The incidence relation depends only on the attribute]\label{prop:IdependM0}
Given $\varsigma$ a relational scaling operation, $K=\langle G, M, I\rangle$ a formal context and $r$ a  relation between $G$ and $G_\zz$ a possibly different set of objects.
Given $L$ and $L'$ two concept lattices on $G_\zz$,
for the scaled attributes $m=\varsigma(r,c)\in D_{\varsigma,r,\eta(L)}\cap D_{\varsigma,r,\eta(L')}$,
$\forall g\in G$, $\langle \{r\},L\rangle\models m(g)$ if and only if $\langle \{r\},L'\rangle\models m(g)$.
\end{property}
\begin{proof}
$m=\varsigma(r,c)$ is scaled from a scaling operation $\varsigma$, a relation $r$ and a concept $c$ (possibly a cardinal $n$).
From Table~\ref{tab:rscaling}, $\langle R,L\rangle\models gIm$ only depends on $\varsigma$, $r$ and the extent of $c$.
However, $\varsigma$ and $r$ are independent from $L$ and $L'$.
The concept $c$ is identified by a name which denotes its extent.
Hence, its extent is the same in $L$ and $L'$ to which it belongs.
So whether an object of $g\in G_\xx$ satisfies the attribute $\varsigma(r,c)$ or not depends solely on the attribute $\varsigma(r,c)$ and not on the specific lattice considered.
\end{proof}

Property~\ref{prop:IdependM0} highlights that the incidence of the scaled attributes does only depend on the relation and not on the concept lattices.
This obviously applies to RCA$^0$ with $G_\zz=G$.
This can be extended to the full RCA:

\begin{property}[The incidence relation depends only on the attributes]\label{prop:IdependM}
Given a set $\Omega$ of relational scaling operations, a family $K=\{\langle G_\xx, M_\xx, I_\xx\rangle\}_{\xx\in\XX}$ of formal contexts and a set $R$ of relations on $K$.
Given $L$ and $L'$ two concept lattices on $G_\zz$,
for all scaled attributes $m=\varsigma(r,c)\in D_{\Omega,R_{\xx,\zz},L}\cap D_{\Omega,R_{\xx,\zz},L'}$,
$\forall g\in G_\xx$, $\langle R,L\rangle\models m(g)$ if and only if $\langle R,L'\rangle\models m(g)$.
\end{property}
\begin{proof}
For all $\varsigma\in\Omega$, $r\in R$ and $c\in L$, $\varsigma(r,c)$ is defined independently from the rest of $\Omega$, $R$ and $L$ (or $L'$).
Hence Property~\ref{prop:IdependM0} applies to any such combination.
\end{proof}

This means that the interpretation of an attribute never changes: it is given by its syntactic form $\varsigma(r,c)$ and, in particular, the concept name.
Hence, when adding or suppressing attributes, through $\mathrm{K}^\Sigma_{+M}$ or $\mathrm{K}^\Sigma_{-M}$, $\Sigma$ can be $\langle R, N\rangle$ with $N$ an indexed set of names, and in particular $\eta^*(K^0)$.

In RCA, the set of objects $G_\xx$ does not change.
Property~\ref{prop:IdependM} entails that, if $M_\xx\subseteq M_\xx'$, then $I_\xx\subseteq I_\xx'$.
Thus we are justified in using, for RCA, the definition of subcontexts introduced in Section~\ref{sec:rcasem}.
For comparing two contexts, it suffices to compare their sets of attributes:  if $\forall\xx\in\XX$, $M_\xx\subseteq M_\xx'$, then $K\subseteq K'$.

The attribute language $D_{\Omega,R_\xx,N}$ that can be generated by scaling depends on the finite set of relations $R$, the scaling operations $\Omega$ and the set of possible concepts identified by their standardised names (\S\ref{sec:rscaling} and \ref{sec:rcaname}).
Given $N\subseteq \eta^*(K^0)$ an indexed family of names of concepts that can be in the codomain of relations in $R$, the set of contexts that can be obtained by scaling is
\[
\mathcal{K}^N_{K^0_\xx,R,\Omega}=\{ \mathrm{K}^{\langle R,\eta^*(K^0)\rangle}_{+M}(K^0_\xx)~|~M\subseteq D_{\Omega,R_\xx,N} \}.
\]
\noindent with $\mathrm{K}^{\langle R,\eta^*(K^0)\rangle}_{+M}(.)$ the operation defined in \S\ref{sec:scaling}.
Passing $\eta^*(K^0)$ to $\mathrm{K}$ allows to interpret the generated attributes and to determine $I$.
Below, when we write $\mathcal{K}^N$, the property applies for any $\{N_\xx\}_{\xx\in\XX}$ such that $\forall\xx\in\XX$, $N_\xx\subseteq \eta(K^0_\xx)$, and in particular for $\eta^*(K^0)$.

In RCA$^0$, the set of class names is $\eta(K^0)$.
Hence the attribute language $D_{\Omega,R,K^0}$ is fully determined by the non-changing parts: $G$, the finite set of relations $R$ and the scaling operations $\Omega$.
\[
\mathcal{K}_{K^0,R,\Omega} = \mathcal{K}^{\eta(K^0)}_{K^0,R,\Omega}.
\]

Contexts may be combined by two operations: meet ($\wedge$) and join ($\vee$) on $\mathcal{K}^N_{K^0,R,\Omega}$:

\begin{definition}[Meet and join of contexts]\label{def:andorK}
Given $K, K'\in \mathcal{K}^N_{\langle G, M^0, I^0\rangle,R,\Omega}$
such that $K=\langle G, M^0\cup M, I^0\cup I\rangle$ and  $K'=\langle G, M^0\cup M', I^0\cup I'\rangle$, $K\vee K'$ and $K\wedge K'$ are defined as:
\begin{align*}
  K\vee K' & = \langle G, M^0\cup (M\cup M'), I^0\cup (I\cup I')\rangle, \tag{join}\label{eq:kjoin}\\
  K\wedge K' & = \langle G, M^0\cup (M\cap M'), I^0\cup (I\cap I')\rangle. \tag{meet}\label{eq:kmeet}
\end{align*}
\end{definition}

The set of contexts is closed by meet and join.

\begin{property}[\cite{euzenat2021a}]\label{prop:Kclosedmj} 
$\forall K, K'\in\mathcal{K}^N_{K^0,R,\Omega}$, $K\wedge K'\in\mathcal{K}^N_{K^0,R,\Omega}$ and $K\vee K'\in\mathcal{K}^N_{K^0,R,\Omega}$.
\end{property}
\begin{proof}
Meet and join are defined from the union and intersection of subsets of $D_{\Omega,R,K^0}$ (Definition~\ref{def:andorK}).
But $\mathcal{K}^N_{K^0,R,\Omega}$ is closed by union and intersection of the sets of attributes to add to $M^0$ and the incidence relation is fully determined by the set of attributes (Property~\ref{prop:IdependM}).
Hence, meet and join of contexts in $\mathcal{K}^N_{K^0,R,\Omega}$ belong to $\mathcal{K}^N_{K^0,R,\Omega}$.
\end{proof}

\begin{property}[Commutativity, associativity and absorption of $\vee$ and $\wedge$ on $\mathcal{K}$]\label{prop:comandorK}
For all $K, K', K''$ $\in\mathcal{K}$,
\begin{align}
K\vee K' &= K'\vee K &\text{and} && K\wedge K' &= K'\wedge K, \tag{commutativity}\\
(K\vee K')\vee K'' &= K\vee (K'\vee K'') &\text{and} && (K\wedge K')\wedge K'' &= K\wedge (K'\wedge K''), \tag{associativity}\\
K\wedge(K \vee K') &= K &\text{and} && K \vee (K \wedge K') &= K. \tag{absorption}
\end{align}
\end{property}
\begin{proof} Proofs are given for $\wedge$, those for $\vee$ follow the exact same pattern.
\begin{align*}
K\wedge K' &= \langle G, M, I\rangle\wedge\langle G, M', I'\rangle \\
 &= \langle G, M^0\cup(M\cap M'), I^0\cup(I\cap I')\rangle & \text{Definition~\ref{def:andorK}}\\
 &= \langle G, M^0\cup(M'\cap M), I^0\cup(I'\cap I)\rangle &\text{Commutativity of $\cap$}\\
 &= \langle G, M', I'\rangle\wedge\langle G, M, I\rangle & \text{Definition~\ref{def:andorK}}\\
 &= K'\wedge K & \\
(K\wedge K')\wedge K'' &= (\langle G, M, I\rangle\wedge\langle G, M', I'\rangle)\wedge\langle G, M'', I''\rangle  & \\
 &= (\langle G, M^0\cup(M\cap M'), I^0\cup(I\cap I')\rangle)\\
 &~\wedge\langle G, M'', I''\rangle) & \text{Definition~\ref{def:andorK}} \\
 &= \langle G, M^0\cup((M\cap M')\cap M''), I^0\cup((I\cap I')\cap I'')\rangle & \text{Definition~\ref{def:andorK}} \\
 &= \langle G, M^0\cup(M\cap (M'\cap M'')), I^0\cup(I\cap (I'\cap I''))\rangle &\text{Associativity of $\cap$}\\
 &= \langle G, M^0\cup M, I^0\cup I\rangle\\
 &~\wedge\langle G,  M^0\cup (M'\cap M''), I^0\cup (I'\cap I'')\rangle &\text{Definition~\ref{def:andorK}} \\
 &= \langle G, M^0\cup M, I^0\cup I\rangle\\
 &~\wedge(\langle G,  M^0\cup M', I^0\cup I'\rangle\wedge\langle G,  M^0\cup M'', I^0\cup I''\rangle) &\text{Definition~\ref{def:andorK}} \\
 &= K\wedge (K'\wedge K'') & \text{Definition~\ref{def:andorK}}
\end{align*}

\begin{align*}
K\vee (K\wedge K') &= \langle G, M, I\rangle\vee(\langle G, M, I\rangle)\wedge\langle G, M', I'\rangle) & \\
 &= \langle G, M, I\rangle\vee\langle G, M^0\cup (M\cap M'), I^0\cup (I\cap I')\rangle & \text{Definition~\ref{def:andorK}} \\
 &= \langle G, M^0\cup M\cup(M\cap M'), I^0\cup I\cup (I\cap I')\rangle & \text{Definition~\ref{def:andorK}} \\
 &= \langle G, M^0\cup M, I^0\cup I\rangle & \text{Absorption of $\cup/\cap$}\\
 &= \langle G, M, I\rangle & \\
 &= K  & \qedhere
\end{align*}
\end{proof}

These operations are aligned with context inclusion (Property~\ref{prop:ordereqopK}):

\begin{property}[\cite{euzenat2021a}]\label{prop:ordereqopK}
  $\forall K, K'\in\mathcal{K}^N_{K^0,R,\Omega}, K\subseteq K'\text{  iff  } K=K\wedge K'$.
\end{property}
\begin{proof}
  This property also comes directly from its set theoretic counterpart application to $M$ and $M'$:
  $K\subseteq K' \Leftrightarrow M\subseteq M' \Leftrightarrow M=M\cap M' \Leftrightarrow K=K\wedge K'$.
\end{proof}

This provides the space of contexts with a complete lattice structure (Property~\ref{prop:Kcompl}):

\begin{property}[\cite{euzenat2021a}]\label{prop:Kcompl} 
$\langle \mathcal{K}^N_{K^0,R,\Omega}, \vee, \wedge\rangle$ is a complete lattice.
\end{property}
\begin{proof}
$\mathcal{K}^N_{K^0,R,\Omega}$ is closed by meet and join (Property~\ref{prop:Kclosedmj}).
$\vee$ and $\wedge$ satisfy commutativity, associativity and the absorption law 
(Property~\ref{prop:comandorK}), so this is a lattice.
It is complete because for any subset $S$ of $\mathcal{K}^N_{K^0,R,\Omega}$,
$\bigwedge S=\langle G, M^0\cup\bigcap_{\langle G, M^0\cup M, I^0\cup I\rangle\in S} M,  I^0\cup\bigcap_{\langle G, M^0\cup M, I^0\cup I\rangle\in S} I\rangle$ is well-defined and the greatest lower bound, because $\cap$ induces the greatest lower bounds on attributes and incidence relations.
The same reasoning is possible for the smallest upper bound.
\end{proof}
The proof extends \cite{euzenat2021a} to the non-finite case.

\subsubsection{The context expansion function $F$}\label{sec:F}

We reformulate RCA as based on a main single function,
$F_{K^0,R,\Omega}$, the context expansion function\footnote{Named `complete relational extension' in \cite{rouanehacene2013b}.} attached to a relational context $\langle K^0, R\rangle$ and a set $\Omega$ of scaling operations.

\begin{definition}[Context expansion function \cite{euzenat2021a}]\label{def:F}
  Given a relational context $\langle K^0, R\rangle$ and a set $\Omega$ of relational scaling operations, the function $F_{K^0,R,\Omega}:\mathcal{K}_{K^0,R,\Omega}\rightarrow\mathcal{K}_{K^0,R,\Omega}$ is defined by:
  \[
    F_{K^0,R,\Omega}(K) = \sigma_{\Omega}(K,R,\mathrm{FCA}(K))).
  \]
\end{definition}
The function expression is independent from $K^0$, $K^0$ is used to restrict the domain of the function so that its elements cover $K^0$.
$F_{K^0,R,\Omega}$ is only defined over $\mathcal{K}_{K^0,R,\Omega}$ because scaling is not restricted to an arbitrary $N$.
From now on, we will abbreviate $\mathcal{K}_{K^0,R,\Omega}$ as $\mathcal{K}$ and $F_{K^0,R,\Omega}$ as $F$.
This is legitimate because, for a given relational context, $K^0$, $R$ and $\Omega$ do not change.
$F$ is an extensive and monotone internal operation for $\mathcal{K}$:
\begin{property}[$F$ is internal to $\mathcal{K}$ \cite{euzenat2021a}]\label{prop:FintK}
  $\forall K\in\mathcal{K}$, $F(K)\in\mathcal{K}$.
\end{property}
\begin{proof}
Scaling only adds attributes from $D_{\Omega,R,K^0}$.
\end{proof}

\begin{property}[$F$ is extensive and monotone \cite{euzenat2021a}]\label{prop:Fmonext}
The function $F$, attached to a relational context and a set of scaling operations, satisfies:
\begin{align}
  K &\subseteq F(K),\tag{extensivity}\label{prop:Fext}\\
  K\subseteq K' &\Rightarrow F(K)\subseteq F(K').\tag{monotony}\label{prop:Fmon}
\end{align}
\end{property}
\begin{proof}
\ref{prop:Fext} holds because $F$ can only add to each context in $K$ new relational attributes scaled from $\mathrm{FCA}(K)$.
The set of attributes can thus not be smaller.
\ref{prop:Fmon} holds because $K\subseteq K'$ means that $M\subseteq M'$.
This entails that the set of concepts of $\mathrm{FCA}(K)$ is included in that of $\mathrm{FCA}(K')$, hence the set of attributes $A$ scaled from $K$ is included in the set $A'$ scaled from $K'$.
Since, they are added to $M$ and $M'$, then $M\cup A\subseteq M'\cup A'$, hence $F(K)\subseteq F(K')$.
\end{proof}

Extensivity corresponds to the non-contracting property of the well-grounded semantics \cite{rouanehacene2013b} and monotony is also called order-preservation.

\subsubsection{Fixed points of $F$}\label{sec:fpF}

Given $F$, it is possible to define its sets of fixed points, i.e. the sets of contexts closed for $F$, as:
\begin{definition}[Fixed point \cite{euzenat2021a}]\label{def:Ffixedpoint}
  A context $K\in\mathcal{K}$ is a fixed point for a context expansion function $F$, if $F(K)=K$.
  We call $\mathrm{fp}(F)$ the set of fixed points for $F$.
\end{definition}

Since $\mathcal{K}^N$ is a complete lattice and $F$ is order-preserving (or monotone) on $\mathcal{K}$, then the Knaster-Tarski theorem applies:

\begin{theorem}[Knaster-Tarski theorem \cite{tarski1955a}]\label{prop:knastertarski}
Let $\mathcal{K}$ be a complete lattice and let $F: \mathcal{K}\rightarrow\mathcal{K}$ be an order-preserving function. Then the set of fixed points of $F$ in $\mathcal{K}$ is also a complete lattice.
\end{theorem}
In particular, this warrants that there exists least and greatest fixed points of $F$ in $\mathcal{K}$ (called $\mathrm{lfp}(F)$ and $\mathrm{gfp}(F)$) which can be defined as:
\[
  \mathrm{lfp}(F)=\bigwedge_{K\in \mathrm{fp}(F)} K \text{~~~~and~~~~} \mathrm{gfp}(F)=\bigvee_{K\in \mathrm{fp}(F)} K.
\]

\subsection{Semantics and properties: the lattice approach}\label{sec:latfixpoint}

In formal concept analysis, there is a one-to-one correspondence between contexts and lattices.
Hence the results of the previous section could in principle be derived through reasoning on lattices instead of contexts.
In this section, we approach RCA from the lattice standpoint and we show, unsurprisingly, the close parallel with the context approach.

\subsubsection{The lattice $\mathcal{L}$ of RCA$^0$ concept lattices}\label{sec:L}

From $\mathcal{K}^N_{K^0,R,\Omega}$, one can define $\mathcal{L}^N_{K^0,R,\Omega}$ as the set of images of $\mathcal{K}^N_{K^0,R,\Omega}$ by FCA.
These are concept lattices obtained by applying FCA on $K^0$ extended with a subset of $D_{\Omega,R,N}$:
\[
  \mathcal{L}^N_{K^0,R,\Omega}=\{ \mathrm{FCA}(K^{\langle R, \eta(K^0)\rangle}_{+M}(K^0)) ~|~ M\subseteq D_{\Omega,R,N} \}.
\]

In $\mathrm{RCA}^0$, this time again the set of concept names is limited to those of the single context, $\eta(K^0)$:
\[
  \mathcal{L}_{K^0,R,\Omega}=\mathcal{L}^{\eta(K^0)}_{K^0,R,\Omega}.
\]
For each subset of attributes, the lattice obtained by $\mathrm{FCA}$ is necessarily syntactically different as its concepts refer to different attributes in their intents (at least one of them).
\label{sec:correspKL}
\label{sec:kappa}

There is in fact a bijective correspondence between $\mathcal{L}_{K^0,R,\Omega}$ and $\mathcal{K}_{K^0,R,\Omega}$ \cite[\S1.2]{ganter1999a}.
On the one hand, for any context in $\mathcal{K}_{K^0,R,\Omega}$ corresponds only one lattice by $\mathrm{FCA}$.
On the other hand,
to any concept lattice $\langle C, \leq\rangle\in\mathcal{L}_{K^0,R,\Omega}$ corresponds a formal context: 
\[\kappa(\langle C, \leq\rangle)=\langle \bigcup_{c\in C} extent(c), \bigcup_{c\in C} intent(c), \bigcup_{c\in C}extent(c)\times intent(c)\rangle. \]
$\kappa(L)$ collects the attributes and objects present in $L$ intents to build the unique $M$ and $G$, from which the corresponding $I$ is obtained.

It is such that $\mathrm{FCA}\circ\kappa=id_{\mathcal{K}}$ and $\kappa\circ \mathrm{FCA}=id_{\mathcal{L}}$ ($id_{\mathcal{K}}$ and $id_{\mathcal{L}}$ being the respective identity functions).
This may be stated as:
\begin{property}\label{prop:kappaFCA}
  $K=\kappa(L)\text{ iff } L=\mathrm{FCA}(K)$.
\end{property}
\begin{proof}
$\Rightarrow$) $K=\langle G, M, I\rangle$ $=$ $\langle \bigcup_{c\in C} extent(c), \bigcup_{c\in C} intent(c), \bigcup_{c\in C}extent(c)\times intent(c)\rangle$ $=$ $\kappa(\langle C, \leq\rangle)=\kappa(L)$.
This means that $\langle C, \leq\rangle$ is the concept lattice of a context $\langle G', M', I'\rangle$.
But since $G=\bigcup_{c\in C} extent(c)$ and $M=\bigcup_{c\in C} intent(c)$, then $G=G'$ and $M=M'$.
We need to prove that $I=I'$.
Consider $I\neq I'$, this could be because there exists at least one $g\in G$ and $m\in M$ such that $\langle g, m\rangle\in I\setminus I'$ (or, but not exclusively, $\langle g, m\rangle\in I'\setminus I$).
In this condition, there could not exists $c\in C$ such that $g\in extent(c)$ and $m\in intent(c)$ (resp. it exists).
Then $\langle g, m\rangle\not\in\bigcup_{c\in C}extent(c)\times intent(c)$ (resp. on the contrary it is there) and thus $I\neq \bigcup_{c\in C}extent(c)\times intent(c)$ which contradicts $\langle G, M, I\rangle=\kappa(\langle C,\leq\rangle)$.
Hence, $I=I'$ and $L=\langle C,\leq\rangle$ $=$ $\mathrm{FCA}(\langle G,M,I\rangle)=\mathrm{FCA}(K)$

\medskip
\noindent $\Leftarrow$) $L=\langle C,\leq\rangle$ $=$ $\mathrm{FCA}(\langle G,M,I\rangle)=\mathrm{FCA}(K)$
entails
$\forall c\in C$, $\forall g\in extent(c)$, $\forall m\in intent(c)$, $gIm$, i.e. $extent(c)\times intent(c)\subseteq I$.
In addition, if $gIm$, then there exists $c\in \mathrm{FCA}(\langle G, M, I\rangle)=\langle C, \leq\rangle$ such that $g\in extent(c)$ and $m\in intent(c)$, thus $I\subseteq \bigcup_{c\in C} extent(c)\times intent(c)$.
Moreover, $\top=\langle G,G^\uparrow\rangle\in C$ and $\bot=\langle M^\downarrow, M\rangle\in C$,
hence $\bigcup_{c\in C}extent(c)=G$ and $\bigcup_{c\in C}intent(c)=M$.
Thus $K=\langle G, M, I\rangle$ $=$ $\langle \bigcup_{c\in C}extent(c),$ $\bigcup_{c\in C}intent(c),$ $\bigcup_{c\in C}extent(c)\times intent(c)\rangle=\kappa(L)$.
\end{proof}
We define a specific type of homomorphisms between two concept lattices when concepts are simply mapped into concepts with the same extent and a possibly increased intent\footnote{The results in the remainder of this section are specific to RCA: $\preceq$ is defined with the equality of extents and $\subseteq$ depends only on $M$ because $G$ is always the same.}.

\begin{definition}[Lattice homomorphism \cite{euzenat2021a}]\label{def:Lhomo}
  A concept lattice homomorphism $h: \langle C,\leq\rangle\rightarrow \langle C',\leq'\rangle$ is a function which maps each concept $c\in C$ into a corresponding concept $h(c)\in C'$ such that:
  \begin{itemize}
    \item $\forall c\in C$, $intent(c)\subseteq intent(h(c))$ (or $intent(c)\sqsupseteq intent(h(c))$ if these are considered as description logic concept descriptions),
    \item $\forall c\in C$, $extent(c)= extent(h(c))$, and
    \item $\forall c,d\in C$, $c\leq d\Rightarrow h(c)\leq' h(c)$.
  \end{itemize}
\end{definition}
We note $L\preceq L'$ if there exists a homomorphism from $L$ to $L'$.
In principle, $L\simeq L'$ if $L\preceq L'$ and $L'\preceq L$, but here, $\simeq$ is simply $=$.
This owes to the fact that the homomorphism maps concepts of equal extent, hence, if they hold in both ways, there should be as many concepts in each lattice and these concepts will also have the same intent.

\begin{property}[$\mathrm{FCA}$ is monotone]\label{prop:FCAmon}
  $\forall K=\langle G, M, I\rangle$ and $K'=\langle G, M', I'\rangle$,
  \[ K\subseteq K'\Rightarrow \mathrm{FCA}(K)\preceq \mathrm{FCA}(K'). \]
\end{property}
\begin{proof}
Any concept $c\in\mathrm{FCA}(K)$ characterises a set of objects $extent(c)$ by the set of attributes $intent(c)$ that these objects are the only ones to satisfy.
Since both contexts have the same set of objects $G$, there are not more objects satisfying these in $K'$ and since $M\subseteq M'$ these attributes are still in $K'$, thus $c\in\mathrm{FCA}(K')$.
Hence, it is always possible to define $h$ such as $h(c)=c$.
Then $extent(h(c))=extent(c)$ and $intent(h(c))\supseteq intent(c)$.
Moreover, if $c\leq d$, then $h(c)\leq h(d)$ because $extent(c)\subseteq extent(d)$ entails $extent(h(c))\subseteq extent(h(d))$.
Hence, $\mathrm{FCA}(K)\preceq \mathrm{FCA}(K')$ (Definition~\ref{def:Lhomo}).
\end{proof}
If $K\subseteq K'$, then each concept that can be built from $K$ can be built from $K'$.
The additional attributes in $M'$ can only be used to separate further objects of existing concepts, introducing additional concepts.
All concepts are preserved, possibly with a larger intent, which preserves the homomorphism.

Lattices may be combined by two operations defined from the corresponding operators in $\mathcal{K}^N_{K^0_\xx,R,\Omega}$: meet ($\wedge$) and join ($\vee$) on $\mathcal{L}^N_{K^0,R,\Omega}$:
\begin{definition}[Meet and join of lattices]\label{def:andorL}
Given $L, L'\in \mathcal{L}^N_{\langle G, M^0, I^0\rangle,R,\Omega}$,
\begin{align*}
  L\vee L' & = \mathrm{FCA}( \kappa(L)\vee \kappa(L') ), \tag{join}\label{eq:ljoin}\\
  L\wedge L' & = \mathrm{FCA}( \kappa(L)\wedge \kappa(L') ). \tag{meet}\label{eq:lmeet}
\end{align*}
\end{definition}

The set of lattices is also closed by meet and join:
\begin{property}\label{prop:Lclosedmj}
$\forall L, L'\in\mathcal{L}^N_{K^0,R,\Omega}$, $L\wedge L'\in\mathcal{L}^N_{K^0,R,\Omega}$ and $L\vee L'\in\mathcal{L}^N_{K^0,R,\Omega}$.
\end{property}
\begin{proof}
$\mathcal{L}^N_{K^0,R,\Omega}$ is closed by meet and join since $\mathcal{K}^N_{K^0,R,\Omega}$ is closed by meet and join (Property~\ref{prop:Kclosedmj}) and $\mathcal{L}^N_{K^0,R,\Omega}$ is the image of $\mathcal{K}^N_{K^0,R,\Omega}$ by $\mathrm{FCA}$.
\end{proof}

\begin{property}\label{prop:Lcompl} $\langle \mathcal{L}^N_{K^0,R,\Omega}, \vee, \wedge\rangle$ is a complete lattice.
\end{property}
\begin{proof}
$\mathcal{L}^N_{K^0,R,\Omega}$ is closed by meet and join (Property~\ref{prop:Lclosedmj}).
$\vee$ and $\wedge$ satisfy commutativity, associativity and the absorption laws directly from the union and intersection on contexts (Property~\ref{prop:comandorK}), so this is a lattice.
It is complete because isomorphic, through $\kappa$, to the complete lattices $\mathcal{K}^N_{K^0,R,\Omega}$ (Property~\ref{prop:Kcompl}) and the isomorphism preserves meet and join (consequence of Definition~\ref{def:andorL}).
\end{proof}

\begin{property}\label{prop:ordereqopL}
  $\forall L, L'\in\mathcal{L}_{K^0,R,\Omega}, L\preceq L'\text{  iff  } L=L\wedge L'$.
\end{property}
\begin{proof}
  First, for any $m\in M\cap M'$, the pairs of the incidence relations $I$ and $I'$ for $m$ are the same (Property~\ref{prop:IdependM}).
  \begin{itemize}
    \item[$\Rightarrow$] $L\preceq L'$ means that $\forall c\in L$, $\exists h(c)\in L'$ such that $extent(c)=extent(h(c))$ and $intent(c)\subseteq intent(h(c))$. 
Given that $M=\bigcup_{c\in L} intent(c)\setminus M^0$ and that $M'=\bigcup_{c\in L'} intent(c)\setminus M^0$, then $M\subseteq M'$ (and $I\subseteq I'$ due to Property~\ref{prop:IdependM0}).
Thus, $L=L\wedge L'$ because the contexts on which they are built ($K_{+M}^0$ and $K_{+M\cap M'}^0$) are the same.
    \item[$\Leftarrow$]
      $L=L\wedge L'$ means that $M\subseteq M'$ (and then $I\subseteq I'$ according to Property~\ref{prop:IdependM0}).
      Hence, $\forall c\in L$, the attributes satisfied by $extent(h(c))$ in $L'$ include those satisfied by $extent(c)$ in $L$ and others belonging to $M'\setminus M$.
      Thus, $extent(h(c))$ is the extent of a concept in $L'$ because it contains the only objects satisfying these attributes (it is closed).
      Consequently, $\exists h(c)\in L'$ such that $extent(c)=extent(h(c))$ and $intent(c)\subseteq intent(h(c))$ as $h(c)$ may satisfy additional attributes belonging to $M'\setminus M$, but it satisfies at least all those of $intent(c)$. So, $L\preceq L'$.\qedhere
  \end{itemize}
\end{proof}

\subsubsection{The lattice expansion function $E$}\label{sec:E}

As was done for contexts, it is possible to provide an expansion function for lattices.
We define $E_{K^0,R,\Omega}$, the lattice expansion function attached to a relational context $\langle K^0, R\rangle$ and a set $\Omega$ of scaling operations.

\begin{definition}[Lattice expansion function \cite{euzenat2021a}]\label{def:E}
  Given a relational context $\langle K^0, R\rangle$ and a set $\Omega$ of relational scaling operations the function $E_{K^0,R,\Omega}:\mathcal{L}_{K^0,R,\Omega}\rightarrow\mathcal{L}_{K^0,R,\Omega}$ is defined by:
  \[
    E_{K^0,R,\Omega}(L) = \mathrm{FCA}(\sigma_{\Omega}(\kappa(L),R,L)).
  \]
\end{definition}
Here again, $K^0$ is only used to constrain the domain of the function, not its expression.
From now on, we will abbreviate $\mathcal{L}_{K^0,R,\Omega}$ as $\mathcal{L}$ and $E_{K^0,R,\Omega}$ as $E$.

The definition of $E$ first applies scaling and then $\mathrm{FCA}$, though $F$ does the opposite.
In consequence, $E$ is the function corresponding to $F$ in the sense that
$E\circ \mathrm{FCA}=\mathrm{FCA}\circ F$ (see Figure~\ref{fig:EF}).

\begin{figure}[t]
\centering
  \begin{tikzpicture}[yscale=.8]
    \draw (0,-1) node (K) {$K$};
    \draw (6,-1) node (L) {$L$} node[anchor=west] {$=\mathrm{FCA}(K)$};

    \draw (0,1) node (Kp) {$K'$};
    \draw (6,1) node (Lp) {$L'$} node[anchor=west] {$=\mathrm{FCA}(K')$};

    \draw[->] (K) -- node[left]{$F$} (Kp);
    \draw[->] (L) -- node[right]{$E$} (Lp);

 \begin{scope}[decoration={markings,mark=at position 0.5 with {\arrow{To}}}] 
    \draw[thin,postaction=decorate] (K) .. controls +(1,.5) and +(-1,.5) ..  node[below]{$\mathrm{FCA}$} (L);
    \draw[thin,postaction=decorate] (L) .. controls +(-1,-.5) and +(1,-.5) ..  node[below]{$\kappa$} (K);
    \draw[thin,postaction=decorate] (L) .. controls +(-1,1) and +(1,-1) ..  node[above,sloped] {$\sigma$} (Kp);
    \draw[thin,postaction=decorate] (Kp) .. controls +(1,.5) and +(-1,.5) ..  node[above] {$\mathrm{FCA}$} (Lp);
    \draw[thin,postaction=decorate] (Lp) .. controls +(-1,-.5) and +(1,-.5) ..  node[above]{$\kappa$} (Kp);
  \end{scope}
  
   \draw (-1.75,0) node {$\mathcal{K}$};
   \draw[dotted] (0,-1.5) -- (-1.2,0) -- (0,1.5) -- (1.2,0) -- cycle;

   \draw (7.75,0) node {$\mathcal{L}$};
   \draw[dotted] (6,-1.5) -- (4.8,0) -- (6,1.5) -- (7.2,0) -- cycle;

\end{tikzpicture}
\caption{Relations between $F$ and $E$ through the alternation of $\mathrm{FCA}$ and $\sigma_\Omega$ (from \cite{euzenat2021a}).}\label{fig:EF}
\end{figure}
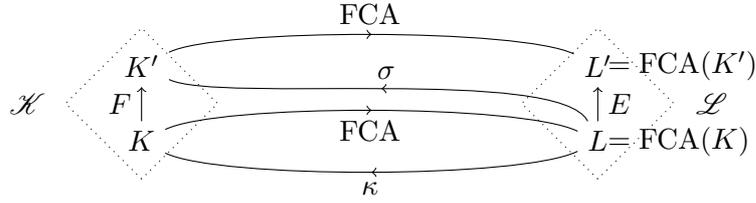

$E$ is an extensive and monotone internal operation for $\mathcal{L}$:
\begin{property}[$E$ is internal to $\mathcal{L}$]\label{prop:EintL}
  $\forall L\in\mathcal{L}$, $E(L)\in\mathcal{L}$.
\end{property}
\begin{proof}
  Given $L\in\mathcal{L}$, $\kappa(L)\in\mathcal{K}$.
  $K=\sigma_\Omega(\kappa(L),R,L)$ adds attributes from $D_{\Omega,R,K^0}$ to $\kappa(L)$, hence $K\in\mathcal{K}$.
  Consequently, $E(L)=\mathrm{FCA}(K)\in\mathcal{L}$.
\end{proof}

\begin{property}[$E$ is monotone and extensive]\label{prop:Emonext}
The function $E$, attached to a relational context $K^0, R$ and a set $\Omega$ of scaling operations, satisfies $\forall L, L'\in\mathcal{L}_{K^0,R,\Omega}$:
\begin{align}
  L\preceq L' &\Rightarrow E(L)\preceq E(L'),\tag{monotony}\label{prop:Emon}\\
  L &\preceq E(L).\tag{extensivity}\label{prop:Eext}
\end{align}
\end{property}
\begin{proof}
\begin{description}[font=\normalfont]
\item[(\ref{prop:Emon})]
  $L\preceq L'$ entails that all concepts of $L$ are found in $L'$ with a larger intent.
  Consequently, $\eta(L)\subseteq \eta(L')$ and $D_{\Omega,R,L}\subseteq D_{\Omega,R,L'}$.
  This entails that $\sigma_\Omega(K,R,L')$ extends $K$ with more attributes than $\sigma_\Omega(K,R,L)$.
  Hence $E(L)\preceq E(L')$ because
  $E(L)$ is the application of $\mathrm{FCA}$ to the same context, to which has been added attributes.
\item[(\ref{prop:Eext})]
  For all $K\in\mathcal{K}$, $L=\mathrm{FCA}(K)$, thus $K\subseteq \sigma_\Omega(K,R,L)$.
  But $E(L)$ $=$ $\mathrm{FCA}(\sigma_{\Omega}(K, R, L))$, so it will have at least all concepts generated by $K$ (identified by extents) because $\sigma$ only adds attributes, hence those allowing to generate a concept remain available and $\mathrm{FCA}$ can only generate more concepts.
    Thus, for each concept $c\in L$ there exists $h(c)\in E(L)$ (with $extent(c)=extent(h(c))$ and possibly with a larger intent, i.e. $intent(c)\subseteq intent(h(c))$), generated by the new scaled attributes.
    Hence, $L\preceq E(L)$.    
    \qedhere
  \end{description}
\end{proof}

Monotony is also called order-preservation.
It corresponds to the non-(intent-)contract\-ing concept property of \cite{rouanehacene2013b}.

\subsubsection{Fixed points of $E$}\label{sec:fpE}

Given $E$, it is possible to define its set of fixed points, i.e. the sets of concept lattices closed for $E$, as:
\begin{definition}[Fixed point]\label{def:Efixedpoint}
  A concept lattice $L\in\mathcal{L}$ is a fixed point for a lattice expansion function $E$, if $E(L)\simeq L$.
  We call $\mathrm{fp}(E)$ the set of fixed points for $E$.
\end{definition}

We can define:
\[
  \mathrm{lfp}(E)=\bigwedge_{L\in \mathrm{fp}(E)} L \text{ and } \mathrm{gfp}(E)=\bigvee_{L\in \mathrm{fp}(E)} L.
\]

Since $\mathcal{L}^N$ is a complete lattice and $E$ is order-preserving (or monotone) on $\mathcal{L}$, then we can apply the Knaster-Tarski theorem.
This warrants that there exists least and greatest fixed points of $E$ in $\mathcal{L}$.

\subsection{Well-grounded and least fixed-point semantics}\label{sec:rca0andlfp}

It is now possible to revise the description of Section~\ref{sec:rcaop} by considering that $F^{\infty}(K^0)=F^n(K^0)$ with $n$ the first integer such that $F^n(K^0)=F^{n+1}(K^0)$.
$\underline{\mathrm{RCA}}$ may thus be redefined as
\[
  \underline{\mathrm{RCA}}_{\Omega}(K^0,R) = \mathrm{FCA}( F^{\infty}(K^0) ),
\]
Alternatively, one may define $E^{\infty}(\mathrm{FCA}(K^0))=F^n(\mathrm{FCA}(K^0))$ with $n$ the first integer $F^n(\mathrm{FCA}(K^0))=F^{n+1}(\mathrm{FCA}(K^0))$.
Then, $\underline{\mathrm{RCA}}$ may be redefined as
\[
  \underline{\mathrm{RCA}}_{\Omega}(K^0,R) = E^{\infty}(\mathrm{FCA}(K^0)).
\]

It thus seems that $\underline{\mathrm{RCA}}$ returns a fixed point of $F$ or $E$.
Hence the question: which fixed point is returned by RCA's well-grounded semantics?

\subsubsection{The RCA well-grounded semantics is the least fixed-point semantics}\label{sec:rca0islfp}

Since $K^0$ belongs to $\mathcal{K}$ and $\mathrm{FCA}(K^0)$ belongs to $\mathcal{L}$, then RCA is indeed based on $E$ and $F$ fixed points.
These are the least fixed points.

\begin{proposition}[The $\underline{\mathrm{RCA}}$ algorithm on a RCA$^0$ context computes the least fixed point \cite{euzenat2021a}]\label{prop:rca0lfp}
  Given $F$ the context expansion function and $E$ the lattice expansion function associated to $K^0$, $R$ and $\Omega$,
\begin{align*}
\underline{\mathrm{RCA}}_{\Omega}(K^0,R) &= \mathrm{FCA}(\mathrm{lfp}(F_{K^0,R,\Omega})),\\
\intertext{and}
\underline{\mathrm{RCA}}_{\Omega}(K^0,R) &= \mathrm{lfp}(E_{K^0,R,\Omega}).
\end{align*}                       
\end{proposition}
\begin{proof}
Concerning the first equation, $\underline{\mathrm{RCA}}_{\Omega}(K^0,R)=\mathrm{FCA}(F^n(K^0))$ for some $n$ at which $F(F^n(K^0))=F^n(K^0)$ \cite{rouanehacene2013a}.
Let $K^{\infty}=F^n(K^0)$, $K^{\infty}\in \mathrm{fp}(F)$ (Definition~\ref{def:Ffixedpoint}).
$\forall K\in \mathrm{fp}(F)$, $K\in\mathcal{K}$, thus $K^0\subseteq K$ because all the contexts in $\mathcal{K}$ contain $M^0$.
By monotony (Property~\ref{prop:Fmonext}), $K^{\infty}=F^n(K^0)\subseteq F^n(K)=K$, because $K$ is a fixed point.
Thus, $K^\infty$ is a fixed point more specific than all fixed points: it is the least fixed point.

Concerning the second equation, $\underline{\mathrm{RCA}}_{\Omega}(K^0,R)=E^n(K^0)$ for some $n$ at which $E(E^n(K^0))=E^n(K^0))$ \cite{rouanehacene2013a}.
  $E(K^0)\in\mathcal{L}$, hence (by Property~\ref{prop:EintL}), $E^n(K^0)\in\mathcal{L}$.
  Moreover, $E(E^n(K^0))=E^n(K^0))$ thus $E^n(K^0))\in \mathrm{fp}(E)$.
  In addition, $\forall L\in \mathrm{fp}(E)$, $E(K^0)\preceq L$ because the context from which $L$ is created contains at least all attributes of $K^0$.
  But if $E^t(K^0)\preceq L$, then $E^{t+1}(K^0)\preceq L$ because by monotony (Property~\ref{prop:Emonext}),  $E^{t+1}(K^0)=E(E^t(K^0))\preceq E(L)$ and $E$ is idempotent on fixed points (by Definition~\ref{def:Efixedpoint}).
  Thus, $\underline{\mathrm{RCA}}_{\Omega}(K^0,R)$ is a fixed point more specific than all fixed points: it is the least fixed point.
\end{proof}

\subsubsection{Greatest fixed point}\label{sec:gfpEF}

A natural question is how to obtain the greatest fixed point.
In fact, under this approach this is (theoretically) surprisingly easy.

\begin{proposition}[\cite{euzenat2021a}]\label{prop:gfpF}
  $\mathrm{gfp}(F_{K^0,R,\Omega})=\mathrm{K}^{\langle R,\eta(K^0)\rangle}_{+D_{\Omega,R,K^0}}(K^0)$.
\end{proposition}
\begin{proof}
  This context is the greatest element of $\mathcal{K}$ as it contains all attributes of $M^0\cup D_{\Omega,R,K^0}$.
  It is also a fixed point because $F$ is extensive (Property~\ref{prop:Fmonext}) and internal (Property~\ref{prop:FintK}).
\end{proof}

The lattice corresponding to the greatest fixed point will be $L=\mathrm{FCA}(\mathrm{gfp}(F_{K^0,R,\Omega}))$.

This result is easy but very uncomfortable.
The obtained lattice may contain many non-supported attributes as shown in Example~\ref{ex:gfprca0}.
Indeed, $\exists r.c$ is well-defined by the incidence relation, but it is of no use to RCA if $c$ does not belong to $L$.

\begin{example}[Greatest fixed point of $F$ in RCA$^0$]\label{ex:gfprca0}
In the example of Section~\ref{sec:elabexrca0}, the attribute $\exists p.A$ belongs to $D_{\Omega,R,K^0}$ though $A$ does not belong to the maximal lattice $L_0^\star$, because it is not a closed concept for FCA.
The fact that both $a$ and $b$ satisfy this attribute makes that it will find its place in the intent of $AB$.
If one considers the lattice in isolation, this is perfectly valid because the scaled context is well-defined: $\exists p.A$ is just an attribute among others satisfied by $a$ and $b$.
However, if the lattice is transformed in a description logic TBox, this is not correct to refer to an undefined class (here $A$).
\end{example}

On the contrary, there may be cases
in which the greatest fixed point is the powerset lattice, i.e. in which all attributes are supported, and the least fixed point of $F$ is directly $\mathrm{FCA}(K^0)$.

This problem is even more embarrassing if one wants to enumerate all acceptable solutions: many of the fixed points of $E$ or $F$ will feature such non-supported attributes.

\section{Self-supported fixed points in \texorpdfstring{RCA$^0$}{RCA⁰}}\label{sec:selfsup0}

In order to define acceptable solutions for RCA, we introduce the notion of self-support.
It specifies that a concept lattice is self-supported if the intents of its concepts only refer to concepts in this lattice.
We describe a function $Q$ which suppresses non-supported attributes and whose closure yields self-supported lattices.
We then identify the acceptable solutions of RCA$^0$ as self-supported fixed points.

\subsection{Self-supported lattices}

Since both $F$ and $E$ are extensive functions, it is possible, starting from anywhere in $\mathcal{K}$ or $\mathcal{L}$, to consider attributes that do not refer to concepts and these attributes will be preserved.
As a consequence, there may exist fixed points with these unwanted attributes and they are also found in the greatest fixed point.
This is not the result that we expect: we need the results to be self-supported.

One may consider identifying such attributes from the greatest fixed point and forbidding them.
However, these meaningless attributes are contextual:
a supported attribute in the greatest fixed point, may not be supported in a smaller lattice.
This is a difficulty for enumerating these fixed points.

\begin{example}[Non self-supported lattice in RCA$^0$]\label{ex:selfsup0}
Figure~\ref{fig:ex-rca0-noalt} (p.\pageref{fig:ex-2-noalt}) shows a context $K_0^\#$ and the associated concept lattice $L_0^\#$ that could be a solution for the example of Section~\ref{sec:elabexrca0} as it belongs to $\mathcal{L}_{\{\exists\},\{r\},K_0^0}$.
However, the lattice is not self-supported because the concept $AB$ uses the attribute $\exists r.B$ in $K_0^\#$ which refers to a concept ($B$) not present in $L_0^\#$.
\end{example}

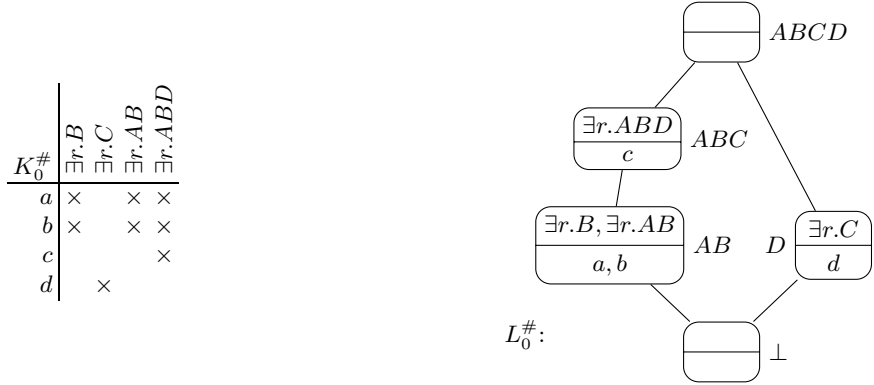
\begin{figure}
\centering
\footnotesize
\begin{tabular}{cc}
\begin{minipage}{.4\textwidth}
\setlength{\tabcolsep}{2pt}
\begin{center}
\begin{tabular}{r|cccc}
$K_0^\#$    & \rotatebox{90}{$\exists r.B$} & \rotatebox{90}{$\exists r.C$} & \rotatebox{90}{$\exists r.AB$} & \rotatebox{90}{$\exists r.ABD$}\\\hline
$a$ & $\times$ &          & $\times$ & $\times$        \\
$b$ & $\times$ &          & $\times$ & $\times$        \\
$c$ &          &          &          & $\times$        \\
$d$ &          & $\times$ &          &         \\
\end{tabular}
\end{center}
\end{minipage} &
\begin{minipage}{.6\textwidth}
\begin{center}
\begin{tikzpicture}[xscale=.3,yscale=.3]
    \begin{dot2tex}[dot,tikz,codeonly,options=-traw]
      graph {
	graph [nodesep=1.5]
	node [style=concept]
	ABCD [label="$\empty$
\nodepart{two}
$\empty$"]

	ABC [label="$\exists r.ABD$
\nodepart{two}
$c$"]

	D [label="$\exists r.C$
\nodepart{two}
$d$"]

	AB [label="$\exists r.B, \exists r.AB$
\nodepart{two}
$a, b$"]

        C0 [label="$\empty$
\nodepart{two}
$\empty$"]
        ABCD -- ABC
        ABCD -- D
        ABC -- AB
        AB -- C0
        D -- C0
      }
    \end{dot2tex}
    \node[anchor=west] at (ABCD.east) {$ABCD$};
    \node[anchor=west] at (ABC.east) {$ABC$};
    \node[anchor=west] at (AB.east) {$AB$};
    \node[anchor=east] at (D.west) {$D$};
    \node[anchor=west] at (C0.east) {$\bot$};
    \draw (0,2.5) node {$L_0^\#$:};
\end{tikzpicture}
\end{center}
\end{minipage}
\end{tabular}
\caption{Non self-supported lattice $L_0^\#$.}\label{fig:ex-rca0-noalt}
\end{figure}

Instead, we consider only self-supported lattices, i.e. lattices whose intents only refer to their own concepts.
\begin{definition}[Self-supported lattices \cite{euzenat2021a}]\label{def:ssl}
A concept lattice $L$ is \emph{self-supported} if $\forall c\in L, intent(c)\subseteq M^0\cup D_{\Omega,R,L}$.
\end{definition}

The definition of self-supported lattices does not provide a direct way to transform a non self-supported lattice into a self-supported one.
Simply suppressing non-supported attributes from intents could result in non concepts (with non-closed extents).
A possible way to solve this problem consists of extracting only the attributes currently in the lattice and to apply $\mathrm{FCA}$ to the resulting context.

For that purpose, we introduce a filtering (or purging) function which suppresses \emph{from the induced context} ($\kappa(L)$) those attributes non supported \emph{by the lattice}:

\begin{definition}[Purging function \cite{euzenat2021a}]\label{def:purge}
The function $\pi_{K^0,R,\Omega}: \mathcal{L}^N\rightarrow\mathcal{K}^N$ returns the context reduced of those attributes not present in a lattice:
\[
  \pi_{K^0,R,\Omega}(L)=\mathrm{K}^{\langle R, N\rangle}_{-(D_{\Omega,R,K^0}\setminus D_{\Omega,R,L})}(\kappa(L)).
\]
\end{definition}
When unambiguous, we refer to $\pi_{K^0,R,\Omega}$ as $\pi$.
The purging function and the following ones are defined over any $N$, as they only restrict the sets of possible contexts and do not expand them.

Like the scaling function, the purging function is only one step: it suppresses currently unsupported attributes, but this may lead to less concepts to be generated by FCA, and thus to other non-supported attributes.
$\pi$ and $\sigma$ are not inverse functions: in particular, $\sigma$ greatly depends on $\Omega$ and $R$ to decide which attributes to scale, through $\pi$ simply suppresses attributes non supported by the lattice(s), independently from $\Omega$, which however determines the attribute language.

\begin{property}\label{prop:pissup0}
$L$ is self-supported if and only if $\mathrm{FCA}(\pi(L))=L$.
\end{property}
\begin{proof}
($\Rightarrow$) $L$ is self-supported means that $\forall c\in L, intent(c)\subseteq M^0\cup D_{\Omega,R,L}$ (Definition~\ref{def:ssl})
which means that, in $\kappa(L)$, all attributes belong to $M^0\cup D_{\Omega,R,L}$.
This leads to $\pi(L)=\kappa(L)$.
However, by Property~\ref{prop:kappaFCA}, $L=\mathrm{FCA}(\kappa(L))$, thus $L=\mathrm{FCA}(\pi(L))$.
($\Leftarrow$)
$L=\mathrm{FCA}(\pi(L))$ if and only if $\pi(L)=\kappa(L)$ (Property~\ref{prop:kappaFCA}).
This means that the attributes of $\kappa(L)$, and thus of $L$, belong to $M^0\cup D_{\Omega,R,L}$.
This is equivalent to being self-supported.
\end{proof}

\subsection{Contraction functions \(Q\) and \(P\)}\label{sec:contr}

Instead of dealing with expansion functions, it is possible to consider contraction functions for contexts or lattices based on $\pi$.

\begin{definition}[Lattice contraction function]\label{def:Q}
The lattice contraction function $Q: \mathcal{L}^N\rightarrow\mathcal{L}^N$ is defined by
\[
  Q(L)=\mathrm{FCA}(\pi_{K^0,R,\Omega}(L)).
\]
\end{definition}

As previously, we will abbreviate $Q_{K^0,R,\Omega}$ as $Q$.
$Q$ is internal to the space of lattices.

\begin{property}[$Q$ is internal to $\mathcal{L}$]\label{prop:QintL}
  $\forall L\in\mathcal{L}^N$, $Q(L)\in\mathcal{L}^N$.
\end{property}
\begin{proof}
  $Q(L)=\mathrm{FCA}(\pi(L))$. $\pi(L)$ retracts attributes from $\kappa(L)$.
  By definition, $\kappa(L)\in\mathcal{K}^N$, hence $\pi(L)\in\mathcal{K}^N$.
  $\pi$ never suppresses attributes from $M^0$ which are all supported (they do not depend on the existence of specific concepts in $L$).
  Consequently, $Q(L)=\mathrm{FCA}(\pi(L))\in\mathcal{L}^N$.
\end{proof}

Contrary to $E$, $Q$ is anti-extensive and monotone:
\begin{property}[$Q$ is anti-extensive and monotone \cite{euzenat2021a}]\label{prop:Qantextmon}
  The function $Q$ satisfies:
\begin{align*}
  Q(L) &\preceq L, \tag{anti-extensivity}\label{eq:Qantiext}\\
  L\preceq L' &\Rightarrow Q(L)\preceq Q(L'). \tag{monotony}\label{eq:Qmono}
\end{align*}
\end{property}
\begin{proof}
  \begin{description}[font=\normalfont]
  \item[(\ref{eq:Qantiext})]
    $\pi(L)\subseteq \kappa(L)$ because $\pi$ simply suppresses attributes from $\kappa(L)$.
    Hence, $\mathrm{FCA}(\pi(L))\preceq \mathrm{FCA}(\kappa(L))$ because the latter contains all concepts of the former (identified by extent) possibly featuring the removed attributes (Property~\ref{prop:FCAmon}).
    Moreover, $\mathrm{FCA}(\kappa(L))=L$ by definition, thus $Q(L)=\mathrm{FCA}(\pi(L))\preceq \mathrm{FCA}(\kappa(L))=L$.
  \item[(\ref{eq:Qmono})] If $L\preceq L'$, then $\kappa(L)\subseteq \kappa(L')$, otherwise $\mathrm{FCA}$ would not generate a smaller lattice.
    In addition, $L\preceq L'$ entails $D_{\Omega,R,L}\subseteq D_{\Omega,R,L'}$ which entails $D_{\Omega,R,K^0}\setminus D_{\Omega,R,L}\supseteq D_{\Omega,R,K^0}\setminus D_{\Omega,R,L'}$, which finally together leads to $M\setminus (D_{\Omega,R,K^0}\setminus D_{\Omega,R,L})\subseteq M'\setminus (D_{\Omega,R,K^0}\setminus D_{\Omega,R,L'})$.
    Then, $\pi(L)\subseteq \pi(L')$ because a smaller context supported by a smaller lattice cannot result in a larger context.
  Hence, $Q(L)=\mathrm{FCA}(\pi(L))\preceq \mathrm{FCA}(\pi(L'))=Q(L')$ by monotony of FCA (Property~\ref{prop:FCAmon}).
  \qedhere
\end{description}
\end{proof}

Similarly, on the context side, the $P$ context contraction function is introduced:

\begin{definition}[Context contraction function]\label{def:P}
The context contraction function $P: \mathcal{K}^N\rightarrow\mathcal{K}^N$ is defined by
\[
  P(K)=\pi_{K^0,R,\Omega}(\mathrm{FCA}(K)).
\]
\end{definition}

Figure~\ref{fig:PQ} displays the relations between these two functions.

\begin{figure}[t]
\centering
  \begin{tikzpicture}[yscale=.8]
    \draw (0,-1) node (K) {$K$};
    \draw (6,-1) node (L) {$L$} node[anchor=west] {$=\mathrm{FCA}(K)$};

    \draw (0,1) node (Kp) {$K'$};
    \draw (6,1) node (Lp) {$L'$} node[anchor=west] {$=\mathrm{FCA}(K')$};

    \draw[<-] (K) -- node[left]{$P$} (Kp);
    \draw[<-] (L) -- node[right]{$Q$} (Lp);

  \begin{scope}[decoration={markings,mark=at position 0.5 with {\arrow{To}}}]
    \draw[thin,postaction=decorate] (K) .. controls +(1,.5) and +(-1,.5) ..  node[below]{$\mathrm{FCA}$} (L);
    \draw[thin,postaction=decorate] (L) .. controls +(-1,-.5) and +(1,-.5) ..  node[below]{$\kappa$} (K);
    \draw[thin,postaction=decorate] (Kp) .. controls +(1,.5) and +(-1,.5) ..  node[above] {$\mathrm{FCA}$} (Lp);
    \draw[thin,postaction=decorate] (Lp) .. controls +(-1,-.5) and +(1,-.5) ..  node[above]{$\kappa$} (Kp);
    \draw[thin,postaction=decorate] (Lp) .. controls +(-1,-1) and +(1,1) ..  node[above,sloped]{$\pi$} (K);
  \end{scope}
  
   \draw (-1.75,0) node {$\mathcal{K}$};
   \draw[dotted] (0,-1.5) -- (-1.2,0) -- (0,1.5) -- (1.2,0) -- cycle;

   \draw (7.75,0) node {$\mathcal{L}$};
   \draw[dotted] (6,-1.5) -- (4.8,0) -- (6,1.5) -- (7.2,0) -- cycle;

\end{tikzpicture}
\caption{Relations between $P$ and $Q$ through the alternation of $\mathrm{FCA}$ and $\pi$  (from \cite{euzenat2021a}).}\label{fig:PQ}
\end{figure}
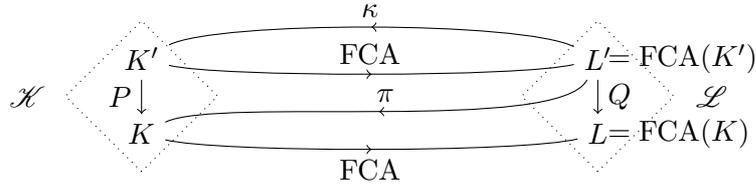

$P$ is internal to $\mathcal{K}^N$:
\begin{property}[$P$ is internal to $\mathcal{K}$]\label{prop:PintK}
  $\forall K\in\mathcal{K}^N$, $P(K)\in\mathcal{K}^N$.
\end{property}
\begin{proof}
$P(K)=\pi(\mathrm{FCA}(K))$ (Definition~\ref{def:P}).
As discussed before (proof of Property~\ref{prop:QintL}), $\pi(\mathrm{FCA}(K))$ retracts some attributes from $\kappa(\mathrm{FCA}(K))$, preserving those of $M^0$.
By definition, $\kappa(\mathrm{FCA}(K))$ $\in$ $\mathcal{K}^N$, hence $P(K)=\pi(\mathrm{FCA}(K))\in\mathcal{K}^N$.
\end{proof}

Contrary to $F$ and according to $Q$, $P$ is anti-extensive and monotone:
\begin{property}[$P$ is anti-extensive and monotone]\label{prop:Pantextmon}
  The function $P$ satisfies:
\begin{align*}
  P(K) &\subseteq K, \tag{anti-extensivity}\label{eq:Pantiext}\\
  K\subseteq K' &\Rightarrow P(K)\subseteq P(K'). \tag{monotony}\label{eq:Pmono}
\end{align*}
\end{property}
\begin{proof}
\begin{description}[font=\normalfont]
\item[(\ref{eq:Pantiext})]
  If $P(K)\not\subseteq K$, this means that there exists a non empty set of attributes $M'$ disjoint from $M$ present in $P(K)$. Such attributes cannot have been brought by $\pi$ since it only suppresses attributes. They should come from either $\mathrm{FCA}$ or $\kappa$.
  However, $\mathrm{FCA}$ does only include in intents attributes from $M$, and $\kappa$ does only extracts attributes from the intents.
  Hence, $P(K)\subseteq K$.
\item[(\ref{eq:Pmono})]
  If $K\subseteq K'$, then $\mathrm{FCA}(K)\preceq \mathrm{FCA}(K')$ by monotony of $\mathrm{FCA}$ (Property~\ref{prop:FCAmon}).
  This means that there exist a lattice homomorphism between $\mathrm{FCA}(K)$ and $\mathrm{FCA}(K')$ for which the intent of all concepts of $\mathrm{FCA}(K)$ is found in that of those of $\mathrm{FCA}(K')$; moreover, all concepts of $\mathrm{FCA}(K)$, as identified by their extent, are found in $\mathrm{FCA}(K')$ (Definition~\ref{def:Lhomo}).
  Hence, necessarily $\kappa(\mathrm{FCA}(K))\subseteq \kappa(\mathrm{FCA}(K'))$ and the supporting concepts in $\mathrm{FCA}(K)$ are still present in $\mathrm{FCA}(K')$, so $P(K)=\pi(\mathrm{FCA}(K))\subseteq \pi(\mathrm{FCA}(K'))=P(K')$.
  \qedhere
\end{description}
\end{proof}

Like $E$ and $F$, $Q$ and $P$ are not closure operations as they are not idempotent.
However, with the same arguments as \cite{rouanehacene2013a}, and in particular the finiteness of contexts (see Section~\ref{sec:rcasem}), it can be argued that the repeated application of $Q$ converges to a self-supported concept lattice.
\begin{property}[Stability of $Q$ \cite{euzenat2021a}]\label{prop:Qstable}
  $\forall L\in\mathcal{L}^N$, $\exists n$; $Q^n(L)=Q^{n+1}(L)$.
\end{property}
\begin{proof}
First, $L$ is a finite concept lattice. 
Moreover, $Q(L)\preceq L$, hence it not possible to build an infinite chain of non-converging application of $Q$ since at each iteration, either $\pi$ suppresses no attribute (and then a fixed point has been reached), or it suppresses at least one attribute and then a strictly smaller context is reached. 
Since the number of scalable attributes is finite and attributes of $M^0$ are not purged, then the process will stop after a finite number of applications of $Q$.
\end{proof}
By convention, we note $Q^\infty$ the closure operation\footnote{Which could be named interior operation as well.} associated with $Q$ and $\mathrm{fp}(Q)$, the set of fixed points of $Q$.

\subsection{Fixed points of \(Q\)}\label{sec:fpQ}

Like with $E$, it is possible to apply the Knaster-Tarski theorem to show that $\langle\mathrm{fp}(Q),\preceq\rangle$ is a complete lattice.

The fixed points of $Q$ are exactly those self-supported lattices in $\mathcal{L}$:
\begin{property}\label{prop:Qss}
  For any $L\in\mathcal{L}^N$, $L$ is self-supported iff $L\in\mathrm{fp}(Q)$.
\end{property}
\begin{proof}
$L\in\mathrm{fp}(Q)$ means that $L=\mathrm{FCA}(\pi(L))=Q(L)$.
But, by Property~\ref{prop:pissup0}, this is equivalent to being self-supported.
\end{proof}

To complete the description of $Q$, it is possible to establish that its least fixed point is $\mathrm{FCA}(K^0)$.

\begin{property}\label{prop:lfpQ}
$\mathrm{lfp}(Q)=\mathrm{FCA}(K^0)$.
\end{property}
\begin{proof}
$\kappa(\mathrm{FCA}(K^0))=K^0$ hence $\pi(\mathrm{FCA}(K^0))=K^0$ because, it is not possible to suppress attributes from $K^0$ which, being a non-scaled context, does not refer to any concept (and in RCA$^0$ this set of attributes is reduced to $\varnothing$).
Thus, $Q(\mathrm{FCA}(K^0))=\mathrm{FCA}(\pi(\mathrm{FCA}(K^0)))$ $=\mathrm{FCA}(K^0)$.
Moreover, $\forall L\in\mathcal{L}^N$, $\mathrm{FCA}(K^0)\preceq L$.
Hence, $\mathrm{FCA}(K^0)$ is a fixed point of $Q$ and all other fixed points are greater.
\end{proof}

\subsection{Relations between \(E\) and \(Q\)}\label{sec:relEQ}

We end up with two operations, $E$ and $Q$, the former extensive and the latter anti-extensive.
If we consider concept lattices from the standpoint of the extents, $Q$ decreases the set of concepts of a lattice and $E$ increases them.

An interesting property of the functions $E$ and $Q$ is that they preserve each other stability:
$E$ has the advantage of preserving self-supportivity (Property~\ref{prop:EintfpQ} replaces Proposition~3 of \cite{euzenat2021a} due to Property~\ref{prop:Qss}):
\begin{property}[$E$ is internal to $\mathrm{fp}(Q)$ {\cite[Prop.3]{euzenat2021a}}]\label{prop:EintfpQ}
$\forall L\in \mathrm{fp}(Q)$, $E(L)\in\mathrm{fp}(Q)$.
\end{property}
\begin{proof}
If $L\in\mathrm{fp}(Q)$, all attributes in intents of $L$ are supported by concepts in $L$ (Property~\ref{prop:Qss}).
$L\preceq E(L)$, so these concepts are still in $E(L)$.
Moreover, $E=\sigma_\Omega\circ\mathrm{FCA}$ and $\sigma_\Omega$ first adds to $\kappa(L)$ attributes which are supported by $L$.
Hence, the attributes in $\kappa(L)$ and those scaled by $\sigma_\Omega$ are still supported by $E(L)$.
\end{proof}

\begin{property}[$Q$ is internal to $\mathrm{fp}(E)$]\label{prop:QintfpE}
$\forall L\in \mathrm{fp}(E)$, $Q(L)\in \mathrm{fp}(E)$.
\end{property}
\begin{proof}
If $L\in\mathrm{fp}(E)$, this means that $E(L)=L$ and, in particular, that $\sigma_\Omega$ does not scale new relational attributes based on the concepts in $L$.
$Q(L)\preceq L$, so that $Q(L)$ does not contain more concepts than $L$.
$Q(L)$ having not more concepts than $L$, $\sigma_\Omega$ cannot scale new attributes ($\sigma_\Omega(Q(L))\subseteq \sigma_\Omega(L)=\varnothing$).
Hence, $Q(L)\in \mathrm{fp}(E)$.
\end{proof}

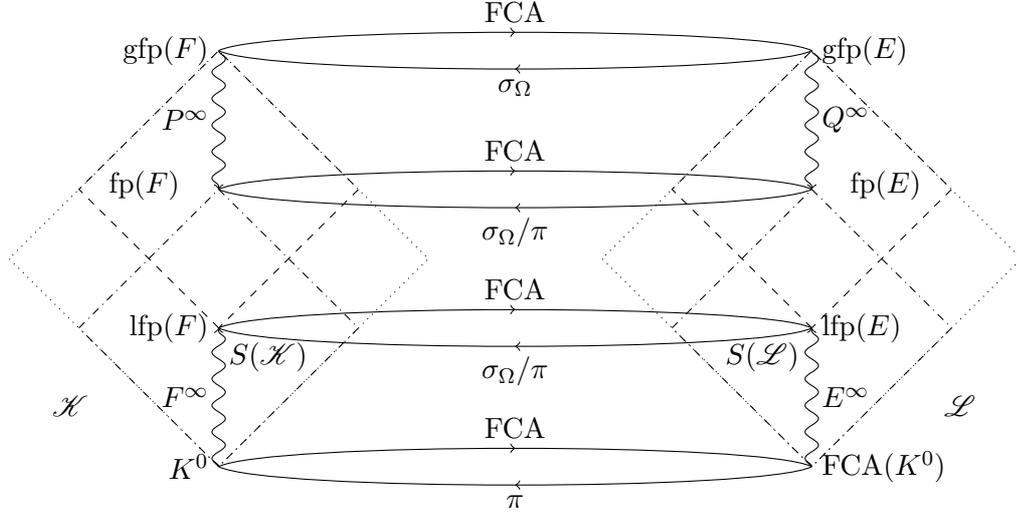
\begin{figure}[t]
\centering
\begin{tikzpicture}[thin,scale=1.3]

  \begin{scope}[xshift=-3cm]
    \begin{scope}[rotate=45]
     \draw[dotted] (-1.5,-1.5) rectangle (1.5,1.5);
     \draw[dashed] (-.5,-.5) rectangle (1.5,1.5);
     \draw[dashdotted] (-1.5,-1.5) rectangle (.5,.5);
     \coordinate (lfpF) at (-.5,-.5);
     \coordinate (k0) at (-1.5,-1.5);
     \coordinate (gfpF) at (1.5,1.5);
     \coordinate (topF) at (.5,.5);
   \end{scope}

   \draw (.5,1.1) node {$\mathrm{fp}(F)$}; 
   \draw (.5,-1.1) node {$\mathrm{fp}(P)$};
   \draw (k0) node[anchor=east] {$\mathrm{lfp}(P)=K^0$}; 
   \draw (lfpF) node[anchor=east] {$\mathrm{lfp}(F)$}; 
   \draw (topF) node[anchor=east] {$\mathrm{gfp}(P)$};
   \draw (gfpF) node[anchor=east] {$\mathrm{gfp}(F)$};

   \draw (-1.5,-1.5) node {$\mathcal{K}$};

   \draw[->,decorate, decoration=snake] (k0) -- node[left] {$F^\infty$} (lfpF);

   \draw[->,decorate, decoration=snake] (gfpF) -- node[left] {$P^\infty$} (topF);
   
 \end{scope}

  \begin{scope}[xshift=3cm]
    \begin{scope}[rotate=45]
     \draw[dotted] (-1.5,-1.5) rectangle (1.5,1.5);
     \draw[dashed] (-.5,-.5) rectangle (1.5,1.5);
     \draw[dashdotted] (-1.5,-1.5) rectangle (.5,.5);
     \coordinate (lfpE) at (-.5,-.5);
     \coordinate (fcak0) at (-1.5,-1.5);
     \coordinate (gfpE) at (1.5,1.5);
     \coordinate (topE) at (.5,.5);
   \end{scope}

   \draw (-.5,1.1) node {$\mathrm{fp}(E)$}; 
   \draw (-.5,-1.1) node {$\mathrm{fp}(Q)$};
   \draw (fcak0) node[anchor=west] {$\mathrm{FCA}(K^0)=\mathrm{lfp}(Q)$};
   \draw (lfpE) node[anchor=west] {$\mathrm{lfp}(E)$}; 
   \draw (topE) node[anchor=west] {$\mathrm{gfp}(Q)$};;
   \draw (gfpE) node[anchor=west] {$\mathrm{gfp}(E)$};

   \draw (1.5,-1.5) node {$\mathcal{L}$};

   \draw[->,decorate, decoration=snake] (fcak0) -- node[right] {$E^\infty$} (lfpE);

   \draw[->,decorate, decoration=snake] (gfpE) -- node[right] {$Q^\infty$} (topE);
   
 \end{scope}

 \begin{scope}[decoration={markings,mark=at position 0.5 with {\arrow{To}}}] 

  \draw[thin,postaction={decorate}] (gfpF) .. controls +(.5,.25) and +(-.5,.25) ..  node[above]{$\mathrm{FCA}$} (gfpE);
  \draw[thin,postaction={decorate}] (gfpE) .. controls +(-.5,-.25) and +(.5,-.25) ..  node[below]{$\sigma_\Omega$} (gfpF);
  \draw[thin,postaction={decorate}] (topF) .. controls +(.5,.25) and +(-.5,.25) ..  node[above]{$\mathrm{FCA}$} (topE);
  \draw[thin,postaction={decorate}] (topE) .. controls +(-.5,-.25) and +(.5,-.25) ..  node[below]{$\sigma_\Omega/\pi$} (topF);
  \draw[thin,postaction={decorate}] (lfpF) .. controls +(.5,.25) and +(-.5,.25) ..  node[above]{$\mathrm{FCA}$} (lfpE);
  \draw[thin,postaction={decorate}] (lfpE) .. controls +(-.5,-.25) and +(.5,-.25) ..  node[below]{$\sigma_\Omega/\pi$} (lfpF);
  \draw[thin,postaction={decorate}] (k0) .. controls +(.5,.25) and +(-.5,.25) ..  node[above]{$\mathrm{FCA}$} (fcak0);
  \draw[thin,postaction={decorate}] (fcak0) .. controls +(-.5,-.25) and +(.5,-.25) ..  node[below]{$\pi$} (k0);

   \end{scope}

\end{tikzpicture}
\caption{The $\mathcal{L}$ (resp. $\mathcal{K}$) lattice and effects of $E$ and $Q$ (resp. $F$ and $P$) for characterising $\mathrm{fp}(E)$ and $\mathrm{fp}(Q)$ (resp. $\mathrm{fp}(F)$ and $\mathrm{fp}(P)$) (from \cite{euzenat2021a}).}\label{fig:lattices}
\end{figure}

In addition, the closure operations associated with the two functions preserve the extrema of each other.

\begin{property}\label{prop:QEbounds}
$Q^\infty(\mathrm{gfp}(E))=\mathrm{gfp}(Q)$
  and
$E^\infty(\mathrm{lfp}(Q))=\mathrm{lfp}(E)$.
\end{property}
\begin{proof}
$\forall L\in\mathcal{L}$, $L\preceq \mathrm{gfp}(E)$ (from Proposition~\ref{prop:gfpF}) and $Q$ and thus $Q^\infty$ is order preserving (Property~\ref{prop:Qantextmon}), hence 
$Q^\infty(L)\preceq Q^\infty(\mathrm{gfp}(E))$.
Moreover, $Q^\infty(\mathrm{gfp}(E))\in\mathrm{fp}(Q)$, thus $Q^\infty(\mathrm{gfp}(E))=\mathrm{gfp}(Q)$.

Similarly, $\forall L\in\mathcal{L}$, $\mathrm{lfp}(Q)\preceq L$ (Property~\ref{prop:lfpQ}) and $E$ and thus $E^\infty$ is order preserving (Property~\ref{prop:Emonext}), hence 
$E^\infty(\mathrm{lfp}(Q))\preceq E^\infty(L)$.
Moreover, $E^\infty(\mathrm{lfp}(Q))\in\mathrm{fp}(E)$, thus $E^\infty(\mathrm{lfp}(Q))=\mathrm{lfp}(E)$.
\end{proof}

The acceptable solutions for RCA$^0$ are the self-supported fixed points of $E$, or said otherwise, the elements of $\mathrm{fp}(E)\cap\mathrm{fp}(Q)$.
Such lattices are both saturated and self-supported well-formed elements of $\mathcal{L}$.
Moreover, by construction of $\mathcal{K}$ and $\mathcal{L}$, they cover $K^0$, i.e. they contain all attributes in $M^0$.

Proposition~\ref{prop:infsupEQ} establishes that the acceptable solutions are circumscribed to the interval sublattice $[\mathrm{lfp}(E)~\mathrm{gfp}(Q)]$ (see also Figure~\ref{fig:lattices}).

\begin{proposition}[\cite{euzenat2021a}]\label{prop:infsupEQ}
$\mathrm{lfp}(E)$ and $\mathrm{gfp}(Q)$ are the infimun and the supremum of $\mathrm{fp}(E)\cap\mathrm{fp}(Q)$ for $\preceq$.
\end{proposition}
\begin{proof}
By definition, $\forall L\in \mathrm{fp}(E)\cap\mathrm{fp}(Q)$, $\mathrm{lfp}(E)\preceq L\preceq\mathrm{gfp}(Q)$.
Moreover,
$Q^{\infty}(\mathrm{lfp}(E))\in\mathrm{fp}(E)$ (Property~\ref{prop:QintfpE})
and
$Q^{\infty}(\mathrm{lfp}(E))\preceq\mathrm{lfp}(E)$ (Property~\ref{prop:Qantextmon}) ,
but $\mathrm{lfp}(E)$ is the least fixed point for $E$,
so $Q^{\infty}(\mathrm{lfp}(E))=\mathrm{lfp}(E)$.
Since, by definition,
$Q^{\infty}(\mathrm{lfp}(E))\in\mathrm{fp}(Q)$, then $\mathrm{lfp}(E)\in\mathrm{fp}(Q)$.
A similar reasoning applies to $E$ and $\mathrm{gfp}(Q)$.
\end{proof}

Thus, the interval sublattice $[\mathrm{lfp}(E)~\mathrm{gfp}(Q)]$ is the smallest interval covering $\mathrm{fp}(E)\cap\mathrm{fp}(Q)$.
However, $\mathrm{fp}(E)\cap\mathrm{fp}(Q)$ does not cover $[\mathrm{lfp}(E)~\mathrm{gfp}(Q)]$ as shown by Example~\ref{ex:neq0}.

\begin{example}[Non fixed points of $\lbrack\mathrm{lfp}(E)\ \mathrm{gfp}(Q)\rbrack$ in RCA$^0$]\label{ex:neq0}
The lattice $L_0^\#$ of Figure~\ref{fig:ex-rca0-noalt} (p.~\pageref{fig:ex-rca0-noalt}) can be checked to belong to the interval $[\mathrm{lfp}(E)\ \mathrm{gfp}(Q)]=[L_0^1\ L_0^\star]$, but it does not belong to $\mathrm{fp}(E)\cap\mathrm{fp}(Q)$: 
it is neither a fixed point for $Q$ (not self-supported) because, as Example~\ref{ex:selfsup0} shows, $B$ does not belong to $L_0^\#$, 
nor for $E$ (not saturated) because $\exists r.ABC$ does not belong to $K_0^\#$.
\end{example}

The definitions and results of the two past sections have been restricted to RCA$^0$ for the sake of clarity.
They will now be generalised.

\section{Unified view of the \texorpdfstring{RCA$^0$}{RCA⁰} space}\label{sec:dualspace}

Although $\mathcal{K}^N$ and $\mathcal{L}^N$ have been presented independently, it is useful to consider the two sets together as, in RCA, lattices in $\mathcal{L}^N$ are an intermediate result of the process which is used for computing the next context.
Instead of dealing with two interrelated spaces independently, we tightly connect them. 
Doing so, we will consider objects which are pairs of contexts and associated concept lattices through $\mathrm{FCA}$.
They are called context-lattice pairs.

\subsection{The lattice \texorpdfstring{$\mathcal{T}$}{𝒯} of context-lattice pairs}\label{sec:T}

From any context in $\mathcal{K}$, it is possible to generate a context-lattice pair using $\mathrm{FCA}$.
The $\mathrm{T}$ constructor does this.

\begin{definition}[$\mathrm{T}$ constructor]\label{def:Tconst}
  Given a context  $K\in\mathcal{K}^N$, $\mathrm{T}:\mathcal{K}^N\rightarrow\mathcal{K}^N\times\mathcal{L}^N$ generates a context-lattice pair, such that:
  \[ \mathrm{T}(K)=\langle K, \mathrm{FCA}(K)\rangle. \]
\end{definition}

We consider the set $\mathcal{T}^N_{K^0,R,\Omega}$ of pairs in $\mathcal{K}^N_{K^0,R,\Omega}\times\mathcal{L}^N_{K^0,R,\Omega}$ such that:\label{def:T}
\[
  \mathcal{T}^N_{K^0,R,\Omega}=\{\langle K, L\rangle\in\mathcal{K}^N_{K^0,R,\Omega}\times\mathcal{L}^N_{K^0,R,\Omega}| L=\mathrm{FCA}(K)\}.
\]
This set is well defined because $\mathcal{K}^N_{K^0,R,\Omega}$ has already been defined and $\mathcal{L}^N_{K^0,R,\Omega}$ are precisely those lattices obtained by $\mathrm{FCA}$ from an element of $\mathcal{K}^N_{K^0,R,\Omega}$.

Alternatively, using Property~\ref{prop:kappaFCA}, it can be defined from $\kappa$:
\[
  \mathcal{T}^N_{K^0,R,\Omega}=\{\langle K, L\rangle\in\mathcal{K}^N_{K^0,R,\Omega}\times\mathcal{L}^N_{K^0,R,\Omega}| K=\kappa(L)\}.
\]

As before, we use $\mathcal{T}_{K^0,R,\Omega}=\mathcal{T}^{\eta(K^0)}_{K^0,R,\Omega}$ and, for any $\langle K, L\rangle\in\mathcal{T}^N_{K^0,R,\Omega}$ we note:
\begin{align*}
  k(\langle K, L\rangle) &= K,\\
  l(\langle K, L\rangle) &= L.
\end{align*}

It is possible to define the meet and join:
\begin{definition}[Meet and join of context-lattice pairs]\label{def:andorT}
Given $T$, $T'\in \mathcal{T}^N_{K^0,R,\Omega}$ $T\vee T'$ and $T\wedge T'$ are defined as:
\begin{align}
  T\vee T' & = \mathrm{T}( k(T)\vee k(T') ), \tag{join}\label{eq:dualjoin}\\
  T\wedge T' & = \mathrm{T}( k(T)\wedge k(T') ). \tag{meet}\label{eq:dualmeet}
\end{align}
\end{definition}
As this definition makes clear, the operations of $\mathcal{T}^N$ only depend on the context part.
But the usual relations with the meet and join on contexts and lattices are preserved:
\begin{property}\label{prop:andorT}
\begin{align*}
  T\vee T' & = \langle k(T)\vee k(T'), l(T)\vee l(T')\rangle,\\
  T\wedge T' & = \langle k(T)\wedge k(T'), l(T)\wedge l(T')\rangle.
\end{align*}
\end{property}
\begin{proof}
This is a simple consequence on the definition of conjunctions and disjunction on context-lattice pairs (Definition~\ref{def:Tconst}) and lattices (Definition~\ref{def:andorL}) as:
\begin{align*}
  L\vee L' & = \mathrm{FCA}( K\vee K' ), \tag{join}\\
  L\wedge L' & = \mathrm{FCA}( K\wedge K' ). \tag{meet}
\end{align*}
\end{proof}

The set of context-lattice pairs is once again closed by meet and join:
\begin{property}\label{prop:Tclosedmj}
$\forall T, T'\in\mathcal{T}^N_{K^0,R,\Omega}$, $T\wedge T'\in\mathcal{T}^N_{K^0,R,\Omega}$ and $T\vee T'\in\mathcal{T}^N_{K^0,R,\Omega}$.
\end{property}
\begin{proof}
$\mathcal{T}^N_{K^0,R,\Omega}$ is closed by meet and join because it is based on $\mathrm{T}$ (Definition~\ref{def:andorT}), $\mathrm{T}$ builds a context-lattice pair in $\mathcal{T}^N_{K^0,R,\Omega}$ from contexts in $\mathcal{K}^N_{K^0,R,\Omega}$ (Definition~\ref{def:Tconst}), and $\mathcal{K}^N_{K^0,R,\Omega}$ itself is closed by meet and join (Property~\ref{prop:Kclosedmj}).
\end{proof}

\begin{property}[Commutativity, associativity and absorption of $\vee$ and $\wedge$ on $\mathcal{T}$]\label{prop:comandorT}
For all $T, T', T''$ $\in\mathcal{T}$,
\begin{align}
T\vee T' &= T'\vee T &\text{and} && T\wedge T' &= T'\wedge T, \tag{commutativity}\\
(T\vee T')\vee T'' &= T\vee (T'\vee T'') &\text{and} && (T\wedge T')\wedge T'' &= T\wedge (T'\wedge T''), \tag{associativity}\\
T\wedge(T \vee T') &= T &\text{and} && T \vee (T \wedge T') &= T. \tag{absorption}
\end{align}
\end{property}
\begin{proof}\belowdisplayskip=-12pt
Proofs are given for $\wedge$, those for $\vee$ follow the exact same pattern.
\begin{align*}
T\wedge T' &= \mathrm{T}(k(T)\wedge k(T')) & \text{ Definition~\ref{def:andorT}} \\
 &= \mathrm{T}(k(T')\wedge k(T)) & \text{ Property~\ref{prop:comandorK}}\\
 &= T'\wedge T & \text{ Definition~\ref{def:andorT}} \\
(T\wedge T')\wedge T'' &= \mathrm{T}(k(T)\wedge k(T'))\wedge T'' & \text{ Definition~\ref{def:andorT}} \\
 &= \mathrm{T}(k(\mathrm{T}(k(T)\wedge k(T')))\wedge k(T'')) & \text{ Definition~\ref{def:andorT}} \\
 &= \mathrm{T}((k(T)\wedge k(T'))\wedge k(T'') & \text{ Definition~\ref{def:Tconst}} \\
 &= \mathrm{T}(k(T)\wedge (k(T')\wedge k(T'')) & \text{ Property~\ref{prop:comandorK}} \\
 &= \mathrm{T}(k(T)\wedge k(\mathrm{T}((k(T')\wedge k(T'')))) & \text{ Definition~\ref{def:Tconst}} \\
 &= T\wedge \mathrm{T}((k(T')\wedge k(T''))) & \text{ Definition~\ref{def:andorT}} \\
 &= T\wedge (T'\wedge T'') & \text{ Definition~\ref{def:andorT}} \\
T\vee (T\wedge T') &= T\vee\mathrm{T}(k(T)\wedge k(T')) & \text{ Definition~\ref{def:andorT}}\\
  &= \mathrm{T}(k(T)\vee k(\mathrm{T}(k(T)\wedge k(T'))) & \text{ Definition~\ref{def:andorT}}\\
  &= \mathrm{T}(k(T)\vee (k(T)\wedge k(T')) & \text{ Definition~\ref{def:Tconst}}\\
  &= \mathrm{T}(k(T)) & \text{ Property~\ref{prop:comandorK}} \\
  &= T & \text{ Definition~\ref{def:Tconst}}. %
\end{align*}
\end{proof}
We also define the order between two context-lattice pairs by combining the orders on contexts and lattices:

\begin{definition}[Order]\label{def:dualorder}
Given $T$, $T'\in \mathcal{T}^N_{K^0,R,\Omega}$,
\[ T\preceq T'\text{ if } k(T)\subseteq k(T')\text{ and } l(T)\preceq l(T'). \]
\end{definition}
Figure~\ref{fig:Tspace} presents the relations between $\mathcal{K}^N$, $\mathcal{L}^N$  and $\mathcal{T}^N$ and their respective orders.
Since $\mathrm{FCA}$ is monotone (Property~\ref{prop:FCAmon}), $T\preceq T'\text{ iff } k(T)\preceq k(T')$.
Like before, we note $T\simeq T'$ if $T\preceq T'$ and $T'\preceq T$, and again $\simeq$ is $=$.

This can be applied to the $\mathrm{T}$ constructor.
\begin{property}\label{prop:Torder}
$\forall K, K'\in\mathcal{K}^N_{K^0,R,\Omega}, \text{if }K\subseteq K'\text{ then } \mathrm{T}(K)\preceq \mathrm{T}(K')$.
\end{property}
\begin{proof}
($\Rightarrow$) This is due to monotony of $\mathrm{FCA}$ (Property~\ref{prop:FCAmon}).
$k(\mathrm{T}(K))=K\subseteq K'=k(\mathrm{T}(K'))$ means that $l(\mathrm{T}(K))=\mathrm{FCA}(K)\preceq \mathrm{FCA}(K')=l(\mathrm{T}(K'))$.
Thus, $\mathrm{T}(K)\preceq \mathrm{T}(K')$.
($\Leftarrow$) $\mathrm{T}(K)\preceq \mathrm{T}(K')$ entails, by Definition~\ref{def:dualorder}, that $K\subseteq K'$.
\end{proof}
This has the consequence that $T=T'$ if and only if $k(T)=k(T')$.

This definition complies with that of meet and join.
\begin{property}\label{prop:ordereqopT}
$\forall T, T'\in\mathcal{T}^N_{K^0,R,\Omega}, T\preceq T'\text{  iff  } T=T\wedge T'$.
\end{property}
\begin{proof}
By Property~\ref{prop:ordereqopK}, we have that $k(T)\preceq k(T')\text{  iff  } k(T)=k(T)\wedge k(T')$ and, by Property~\ref{prop:ordereqopL}, that $l(T)\preceq l(T')\text{  iff  } l(T)=l(T)\wedge l(T')$, consequently $T\preceq T'$ iff $T=T\wedge T'$ by Property~\ref{prop:andorT}.
\end{proof}

This makes $\mathcal{T}^N_{K^0,R,\Omega}$ a complete lattice (Property~\ref{prop:Tcompl}).

\begin{property}\label{prop:Tcompl}
  $\langle \mathcal{T}^N_{K^0,R,\Omega}, \vee, \wedge\rangle$ is a complete lattice.
\end{property}
\begin{proof}
$\mathcal{T}^N_{K^0,R,\Omega}$ is closed by meet and join (Property~\ref{prop:Tclosedmj}).
$\vee$ and $\wedge$ satisfy commutativity, associativity and the absorption laws 
(Property~\ref{prop:comandorT}), so this is a lattice.
It is complete because isomorphic, through $k(\cdot)$ or $l(\cdot)$, to the complete lattices $\mathcal{K}^N_{K^0,R,\Omega}$ (Property~\ref{prop:Kcompl}) or $\mathcal{L}^N_{K^0,R,\Omega}$ (Property~\ref{prop:Lcompl}) and the isomorphisms preserve meet and join (Property~\ref{prop:andorT}).
\end{proof}

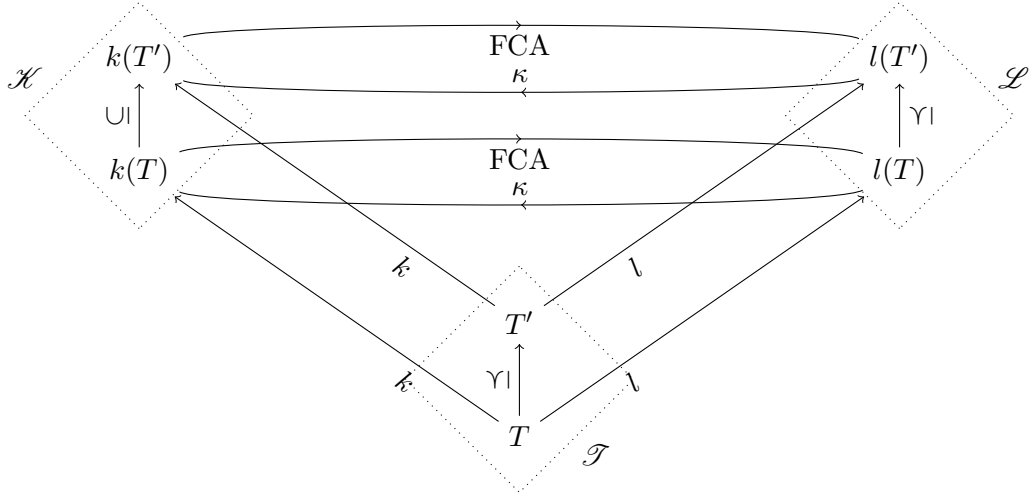
\begin{figure}[t]
\centering
\begin{tikzpicture}
    \draw (-2,.75) node (Kp) {$k(T')$};
    \draw (8,.75) node (Lp) {$l(T')$};

    \draw (-2,-.75) node (K) {$k(T)$};
    \draw (8,-.75) node (L) {$l(T)$};

    \draw[->] (K) -- node[above,sloped]{$\subseteq$} (Kp);
    \draw[->] (L) -- node[below,sloped]{$\preceq$} (Lp);

    \draw (3,-2.75) node (Tp) {$T'$};
    \draw (3,-4.25) node (T) {$T$};
    \draw[->] (T) -- node[above,sloped]{$\preceq$} (Tp);

    \draw[->] (Tp) -- node[below,sloped,near start] {$k$} (Kp);
    \draw[->] (Tp) -- node[below,sloped,near start] {$l$} (Lp);

    \draw[->] (T) -- node[below,sloped,near start] {$k$} (K);
    \draw[->] (T) -- node[below,sloped,near start] {$l$} (L);

 \begin{scope}[decoration={markings,mark=at position 0.5 with {\arrow{To}}}] 
    \draw[thin,postaction=decorate] (K) .. controls +(1,.5) and +(-1,.5) ..  node[below]{$\mathrm{FCA}$} (L);
    \draw[thin,postaction=decorate] (L) .. controls +(-1,-.5) and +(1,-.5) ..  node[above]{$\kappa$} (K);
    \draw[thin,postaction=decorate] (Kp) .. controls +(1,.5) and +(-1,.5) ..  node[below]{$\mathrm{FCA}$} (Lp);
    \draw[thin,postaction=decorate] (Lp) .. controls +(-1,-.5) and +(1,-.5) ..  node[above]{$\kappa$} (Kp);
  \end{scope}
  
   \draw (-3.5,.5) node {$\mathcal{K}$};
   \draw[dotted] (-2,-1.5) -- (-3.5,0) -- (-2,1.5) -- (-.5,0) -- cycle;

   \draw (9.5,.5) node {$\mathcal{L}$};
   \draw[dotted] (8,-1.5) -- (6.5,0) -- (8,1.5) -- (9.5,0) -- cycle;

   \draw (4,-4.5) node {$\mathcal{T}$};
   \draw[dotted] (3,-2) -- (1.5,-3.5) -- (3,-5) -- (4.5,-3.5) -- cycle;

\end{tikzpicture}

\caption{Relations between $\mathcal{T}$, $\mathcal{K}$ and $\mathcal{L}$.}\label{fig:Tspace}
\end{figure}

\subsection{The expansion function \(EF\)}\label{sec:EF}

We reformulate RCA as based on a main single function,
$EF_{K^0,R,\Omega}$, the expansion function attached to a relational context $\langle K^0, R\rangle$ and a set $\Omega$ of scaling operations.

\begin{definition}[Expansion function]\label{def:EF}
Given a relational context $\langle K^0, R\rangle$ and a set $\Omega$ of relational scaling operations, the function $EF_{K^0,R,\Omega}:\mathcal{T}_{K^0,R,\Omega}\rightarrow\mathcal{T}_{K^0,R,\Omega}$ is defined by:
\[
  EF_{K^0,R,\Omega}(T) = \mathrm{T}( \sigma_{\Omega}(k(T),R,l(T)) ).
\]
\end{definition}

This function is an extension of the previous $E$ and $F$:
\begin{property}[$EF$ extends $F$ and $E$]\label{prop:EFisEF}
$EF(T) = \langle F(k(T)), E(l(T))\rangle$.
\end{property}
\begin{proof}
This is the consequence of $F(K)=\sigma_{\Omega}(K,R,\mathrm{FCA}(K))$ (Definition~\ref{def:F}) and $l(T)=\mathrm{FCA}(k(T))$.
On the one hand, $F(k(T))=\sigma_{\Omega}(k(T),R,l(T))$.
On the other hand, $E(L)=\mathrm{FCA}(\sigma_{\Omega}(\kappa(L),R,L))$ (Definition~\ref{def:E}) and $\kappa(l(T))=k(T)$ (Property~\ref{prop:kappaFCA}).
Hence, $EF(T) = \langle \sigma_{\Omega}(k(T),R,l(T)), \mathrm{FCA}(\sigma_{\Omega}(k(T),R,l(T)))\rangle
= \langle F(k(T)), E(l(T))\rangle$.
\end{proof}

As previously, we will abbreviate $\mathcal{T}_{K^0,R,\Omega}$ as $\mathcal{T}$ and $EF_{K^0,R,\Omega}$ as $EF$.

$EF$ is an extensive and monotone internal operation for $\mathcal{T}$:
\begin{property}[$EF$ is internal to $\mathcal{T}$]\label{prop:EFintT}
  $\forall T\in\mathcal{T}$, $EF(T)\in\mathcal{T}$.
\end{property}
\begin{proof}
$EF(T)\in \mathcal{K}_{K^0,R,\Omega}\times\mathcal{L}_{K^0,R,\Omega}$ because $T\in \mathcal{K}_{K^0,R,\Omega}\times\mathcal{L}_{K^0,R,\Omega}$ and $E$ and $F$ are internal to $\mathcal{K}_{K^0,R,\Omega}$ (Property~\ref{prop:FintK}) and $\mathcal{L}_{K^0,R,\Omega}$ (Property~\ref{prop:EintL}), respectively.
Moreover, $EF(T)=T ( F(k(T)) )=\langle F(k(T)), \mathrm{FCA}(F(k(T)))\rangle$ (Definition~\ref{def:EF} and Property~\ref{prop:EFisEF}), hence $l(EF(T))=\mathrm{FCA}(k(EF(T)))$.
\end{proof}

\begin{property}[$EF$ is extensive and monotone]\label{prop:EFmonext}
The function $EF$, attached to a relational context and a set of scaling operations, satisfies:
\begin{align}
  T &\preceq EF(T),\tag{extensivity}\label{prop:EFext}\\
  T\preceq T' &\Rightarrow EF(T)\preceq EF(T').\tag{monotony}\label{prop:EFmon}
\end{align}
\end{property}
\begin{proof}
\ref{prop:EFext} holds because $T\preceq EF(T)$ if and only if $k(T)\preceq k(EF(T))$.
However, $k(EF(T))=F(k(T))$ (Property~\ref{prop:EFisEF}) and $K\preceq F(K)$ (Property~\ref{prop:Fmonext}).
\ref{prop:EFmon} relies on the monotony of $F$ (Property~\ref{prop:Fmonext}) and $E$ (Property~\ref{prop:Emonext}):
$T\preceq T'$ if and only if $k(T)\subseteq k'(T)$ and $l(T)\preceq l'(T)$, but this entail
$F(k(T))\preceq F(k'(T))$ (Property~\ref{prop:Fmonext}) and $E(l(T))\subseteq E(l'(T))$ (Property~\ref{prop:Emonext}), and so $EF(T)\preceq EF(T')$.
\end{proof}

\subsection{The contraction function \(PQ\)}\label{sec:PQ}

It is also possible to consider a single contraction function, $PQ_{K^0,R,\Omega}$, attached to a relational context $\langle K^0, R\rangle$ and a set $\Omega$ of scaling operations.

The context-lattice pairs $\langle K, L\rangle$ may contain many unsupported attributes.
Unsupported attributes are those which refer to classes non existing in the lattice.
Indeed, $\exists r.c$ may be part of the attributes of $K$ well-defined by the incidence relation, but $c$ does not belong to $L$.

$PQ$ may be defined from $\pi$ (\S\ref{sec:selfsup0}) and $T$.

\begin{definition}[Contraction function]\label{def:PQ}
Given a relational context $\langle K^0, R\rangle$ and a set $\Omega$ of relational scaling operations, the function $PQ_{K^0,R,\Omega}:\mathcal{T}^N_{K^0,R,\Omega}\rightarrow\mathcal{T}^N_{K^0,R,\Omega}$ is defined by:
\[
 PQ_{K^0,R,\Omega}(T) = \mathrm{T}( \pi_{K^0,R,\Omega}( l(T) ) ).
\]
\end{definition}

As previously, we will abbreviate $PQ_{K^0,R,\Omega}$ as $PQ$.
This function is also an extension of the previous $Q$ and $P$:
\begin{property}[$PQ$ extends $P$ and $Q$]\label{prop:PQisPQ}
$PQ(T) = \langle P(k(T)), Q(l(T))\rangle$.
\end{property}
\begin{proof}
$\forall T\in\mathcal{T}^N$, $l(T)=\mathrm{FCA}(k(T))$, hence $\pi(l(T))=\pi(\mathrm{FCA}(k(T)))$ and $P(K)=\pi(\mathrm{FCA}(K))$ (Definition~\ref{def:P}), thus $\pi(l(T))=P(k(T))$.
Moreover, $Q(L)=\mathrm{FCA}(\pi(L))$ (Definition~\ref{def:Q}), thus $\mathrm{FCA}(\pi(l(T)))=Q(l(T))$.
So, $\mathrm{T}(\pi(l(T)))=\langle \pi(l(T)), \mathrm{FCA}(\pi(l(T)))\rangle$ $=\langle P(k(T)), Q(l(T))\rangle$.
\end{proof}

$PQ$ is an anti-extensive and monotone internal operation for $\mathcal{T}^N$:
\begin{property}[$PQ$ is internal to $\mathcal{T}^N$]\label{prop:PQintT}
$\forall T\in\mathcal{T}^N$, $PQ(T)\in\mathcal{T}^N$.
\end{property}
\begin{proof}
This follows from $P$ and $Q$ being internal to $\mathcal{K}^N_{K^0,R,\Omega}$ and $\mathcal{L}^N_{K^0,R,\Omega}$, respectively (Property~\ref{prop:PintK} and \ref{prop:QintL}) and Property~\ref{prop:PQisPQ}.
\end{proof}

\begin{property}[$PQ$ is anti-extensive and monotone]\label{prop:PQantextmon}
The function $PQ$, attached to a relational context and a set of scaling operations, satisfies:
\begin{align}
  PQ(T) &\preceq T,\tag{anti-extensivity}\label{prop:PQantext}\\
  T\preceq T' &\Rightarrow PQ(T)\preceq PQ(T').\tag{monotony}\label{prop:PQmon}
\end{align}
\end{property}
\begin{proof}
\ref{prop:PQantext} holds because $PQ(T)\preceq T$ if and only if $P(k(T))=k(PQ(T))\preceq k(T)$ (Property~\ref{prop:PQisPQ}) and $P(K)\preceq K$ (Property~\ref{prop:Pantextmon}).
\ref{prop:PQmon} relies on the monotony of $P$ and $Q$: if $T\preceq T'$, then $k(T)\subseteq k(T')$ and $l(T)\preceq l(T')$, hence $P(k(T))\subseteq P(k(T'))$ (Property~\ref{prop:Pantextmon}) and $Q(l(T))\preceq Q(l(T'))$ (Property~\ref{prop:Qantextmon}), hence $PQ(T)\preceq PQ(T')$.
\end{proof}

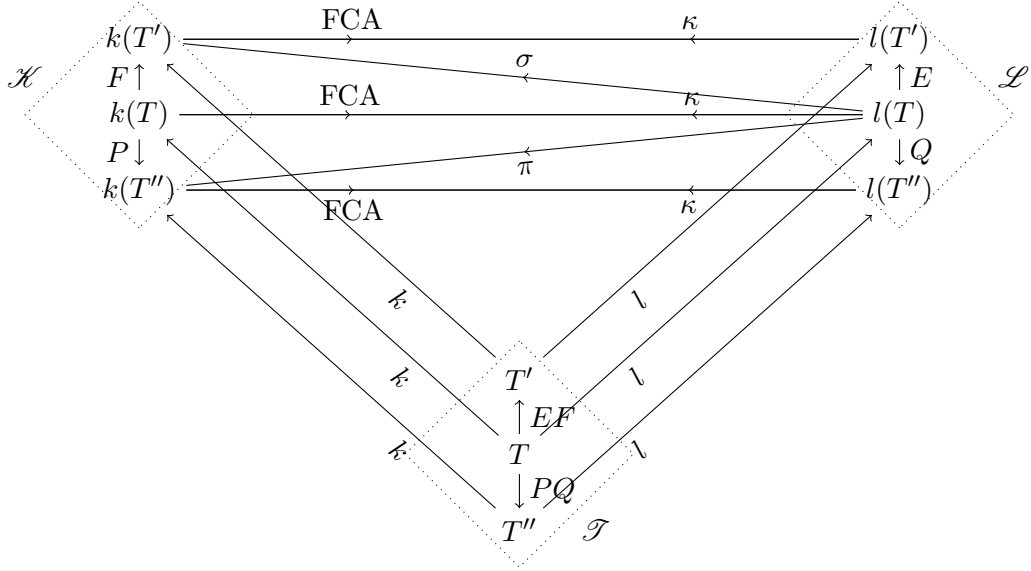
\begin{figure}[t]
\centering
\begin{tikzpicture}
    \draw (-2,1) node (Kp) {$k(T')$};
    \draw (8,1) node (Lp) {$l(T')$};

    \draw (-2,0) node (K) {$k(T)$};
    \draw (8,0) node (L) {$l(T)$};

    \draw (-2,-1) node (Kpp) {$k(T'')$};
    \draw (8,-1) node (Lpp) {$l(T'')$};

    \draw[->] (K) -- node[left]{$F$} (Kp);
    \draw[->] (L) -- node[right]{$E$} (Lp);
    \draw[->] (K) -- node[left]{$P$} (Kpp);
    \draw[->] (L) -- node[right]{$Q$} (Lpp);

    \draw (3,-3.5) node (Tp) {$T'$};
    \draw (3,-4.5) node (T) {$T$};
    \draw (3,-5.5) node (Tpp) {$T''$};
    \draw[->] (T) -- node[right]{$EF$} (Tp);
    \draw[->] (T) -- node[right]{$PQ$} (Tpp);

    \draw[->] (Tp) -- node[below,sloped,near start] {$k$} (Kp);
    \draw[->] (Tp) -- node[below,sloped,near start] {$l$} (Lp);

    \draw[->] (T) -- node[below,sloped,near start] {$k$} (K);
    \draw[->] (T) -- node[below,sloped,near start] {$l$} (L);

    \draw[->] (Tpp) -- node[below,sloped,near start] {$k$} (Kpp);
    \draw[->] (Tpp) -- node[below,sloped,near start] {$l$} (Lpp);
    
 \begin{scope}[decoration={markings,mark=at position 0.5 with {\arrow{To}}}] 
   \draw[thin,postaction=decorate] (L) -- node[above,sloped]{$\sigma$} (Kp);
   \draw[thin,postaction=decorate] (L) -- node[below,sloped]{$\pi$} (Kpp);
 \end{scope}
  
  \begin{scope}[decoration={markings,mark=at position 0.25 with {\arrow{To}}}] 
   \draw[thin,postaction=decorate] (Lp) -- node[near end,above]{$\mathrm{FCA}$} (Kp);
   \draw[thin,postaction=decorate] (L) -- node[near end,above]{$\mathrm{FCA}$} (K);
   \draw[thin,postaction=decorate] (Lpp) -- node[near end,below]{$\mathrm{FCA}$} (Kpp);
  \end{scope}
  
  \begin{scope}[decoration={markings,mark=at position 0.75 with {\arrowreversed{To}}}] 
   \draw[thin,postaction=decorate] (Lp) -- node[near start,above]{$\kappa$} (Kp);
   \draw[thin,postaction=decorate] (L) -- node[near start,above]{$\kappa$} (K);
   \draw[thin,postaction=decorate] (Lpp) -- node[near start,below]{$\kappa$} (Kpp);
  \end{scope}
  
   \draw (-3.5,.5) node {$\mathcal{K}$};
   \draw[dotted] (-2,-1.5) -- (-3.5,0) -- (-2,1.5) -- (-.5,0) -- cycle;

   \draw (9.5,.5) node {$\mathcal{L}$};
   \draw[dotted] (8,-1.5) -- (6.5,0) -- (8,1.5) -- (9.5,0) -- cycle;

   \draw (4,-5.5) node {$\mathcal{T}$};
   \draw[dotted] (3,-3) -- (1.5,-4.5) -- (3,-6) -- (4.5,-4.5) -- cycle;

\end{tikzpicture}

\caption{Relations between $EF$, $PQ$, $E$, $F$, $P$ and $Q$.}\label{fig:EF2PQ}
\end{figure}

Joining together contexts and lattices was a preliminary step to consider families of such pairs to represent the behaviour of relational concept analysis as a whole.
This is done hereafter.

\section{Generalisation to full RCA}\label{sec:rcafixpoint}

So far, we have only considered one context independently from the others.
We now consider RCA in its entirety.

RCA deals with families of contexts.
Its elements are thus simple vectors of the pairs generated by each context.
These vectors will be considered as sets indexed by $\XX$.
All provided definitions can be applied to indexed families of context-lattice pairs,
the order between them will be the product of the piece-wise orders.
All operations remain monotone and extensive (or anti-extensive) as soon as the selected scaling operations are.

The only important change is the notion of self-supported lattices that has to be replaced by self-supported families of context-lattice pairs (\S\ref{sec:selfsup}).
Indeed, if $R=\varnothing$ or only contains endorelations, it is sufficient to work on the families as the free product of pairs.
However, in RCA, this product is constrained by the relations in $R$ which may provide support for otherwise unsupported concepts, invalidating those solutions which do not consider such concepts.
Hence, here the product must be constrained by $R$.

In the following we will thus redefine the objects on which RCA operates (\S\ref{sec:O}), and the expansion (\S\ref{sec:EFs}) and contraction functions (\S\ref{sec:PQs}) based on the notion of self-support  (\S\ref{sec:selfsup}).
We will then consider the fixed points of these functions and their relations (\S\ref{sec:fpEFsPQs}).

\subsection{The lattice \texorpdfstring{$\mathcal{O}$}{𝒪} of families of context-lattice pairs}\label{sec:O}

The input of RCA is given by a family of contexts:
$K^0=\{K^0_\xx\}_{\xx\in\XX}$,
a set $R$ of relations between the objects of these contexts, and
a set $\Omega$ of relational scaling operations.

From this, it is possible to characterise the space $\mathcal{O}_{K^0,R,\Omega}$ associated with RCA by the direct product of the sets of sets of context-lattice pairs associated with each context.

\begin{definition}[$\mathcal{O}_{K^0,R,\Omega}$]\label{def:O}
Given an indexed family of contexts $K^0=\{\langle G_\xx, M_\xx^0, I_\xx^0\rangle\}_{\xx\in\XX}$, 
a set $R$ of relations between the objects of these contexts, and
a set $\Omega$ of relational scaling operations,
the space $\mathcal{O}_{K^0,R,\Omega}$ of indexed families of context-lattice pairs is:
\[ \mathcal{O}_{K^0,R,\Omega}=\prod_{\xx\in\XX} \mathcal{T}^{\eta^*(K^0)}_{K^0_\xx,R,\Omega}. \]
\end{definition}
As before, $\mathcal{O}_{K^0,R,\Omega}$ will simply be referred to as $\mathcal{O}$.

This is well defined because the set of all possible concept extents across all contexts is determined by the set of objects in the context.
This permits us to name unambiguously all the concepts in the family of concept lattices.
In turn, since $\eta^*(K^0)$, $R$ and $\Omega$ do not change and $I$ is determined by $\{M_\xx\}_{\xx\in\XX}$ (Property~\ref{prop:IdependM}), this determines all attributes that can occur in a scaled RCA context.

There is one difference with the RCA$^0$ setting: the scaled attributes depend on $R$ that makes the connection from one context to another, e.g. from $\mathcal{T}_\xx$ to $\mathcal{T}_\zz$.
But since it is possible to name concepts in the lattices generated by the scaled attributes according to their elements in $G_\zz$, then, as soon as $G_\zz$ is finite, the set of scalable attributes in $M_\xx$ is finite and can be established as $D_{\Omega,R_\xx,K^0}$ from the beginning.

The previous notations can be extended:
\begin{align*}
  \mathrm{T}^*( \{ K_\xx\}_{\xx\in\XX} ) &= \{ \mathrm{T}(K_\xx) \}_{\xx\in\XX} = \{ \langle K_\xx, \mathrm{FCA}(K_\xx)\rangle \}_{\xx\in\XX}.\\
\intertext{For any $\{ T_\xx\}_{\xx\in\XX}\in\mathcal{O}_{K^0,R,\Omega}$:
  }
  k(\{ T_\xx\}_{\xx\in\XX}) &= \{ k(T_\xx) \}_{\xx\in\XX},\\
  l(\{ T_\xx\}_{\xx\in\XX}) &= \{ l(T_\xx) \}_{\xx\in\XX},\\
  k_\zz(\{ T_\xx\}_{\xx\in\XX}) &= k( T_\zz ),\\
  l_\zz(\{ T_\xx\}_{\xx\in\XX}) &= l( T_\zz ).
\end{align*}

Finally, for any indexed family of context-lattice pairs $O\in \mathcal{O}$, the family of lattices $l(O)$ is determined directly from $k(O)$: $l(O)=\mathrm{FCA}^*(k(O))$.

We can define $\wedge$ and $\vee$ on $\mathcal{O}_{K^0,R,\Omega}$.
\begin{definition}[Meet and join of families of context-lattice pairs]\label{def:andorO}
Given $O=\{T_{\xx}\}_{\xx\in\XX}$, $O'=\{T'_{\xx}\}_{\xx\in\XX}\in \mathcal{O}_{K^0,R,\Omega}$, $O\vee O'$ and $O\wedge O'$ are defined as:
\begin{align}
  O\vee O' & = \{T_{\xx}\vee T'_{\xx}\}_{\xx\in\XX},\tag{join}\label{eq:rcajoin}\\
  O\wedge O' & = \{T_{\xx}\wedge T'_{\xx}\}_{\xx\in\XX}.\tag{meet}\label{eq:rcameet}
\end{align}
\end{definition}

The set of families of context-lattice pairs is once again closed by meet and join:
\begin{property}\label{prop:Oclosedmj}
$\forall O, O'\in\mathcal{O}^N_{K^0,R,\Omega}$, $O\wedge O'\in\mathcal{O}^N_{K^0,R,\Omega}$ and $O\vee O'\in\mathcal{O}^N_{K^0,R,\Omega}$.
\end{property}
\begin{proof}
$\mathcal{O}_{K^0,R,\Omega}$ is closed by meet and join because meet and join are the piecewise meet and join of context-lattice pairs (Definition~\ref{def:andorO}) and for each $\xx\in\XX$, $\mathcal{T}^{\eta^*(K^0)}_{K_\xx^0,R,\Omega}$ is closed by meet and join (Property~\ref{prop:Tclosedmj}).
\end{proof}

Meet and join also satisfies commutativity, associativity and the absorption law.

\begin{property}[Commutativity, associativity and absorption of $\vee$ and $\wedge$ on $\mathcal{O}$]\label{prop:comandorO} For all $O, O', O''$ $\in\mathcal{O}$,
\begin{align}
O\vee O' &= O'\vee O &\text{and} && O\wedge O' &= O'\wedge O,\tag{commutativity}\\
(O\vee O')\vee O'' &= O\vee (O'\vee O'') &\text{and} && (O\wedge O')\wedge O'' &= O\wedge (O'\wedge O''),\tag{associativity}\\
O\wedge(O \vee O') &= O &\text{and} && O \vee (O \wedge O') &= O.\tag{absorption}
\end{align}
\end{property}
\begin{proof}\belowdisplayskip=-12pt
As before, proofs are given for $\wedge$, those for $\vee$ follow the exact same pattern.
\begin{align*}
O\wedge O' &= \{ T_\xx \wedge T'_\xx \}_{\xx\in\XX} & \text{ Definition~\ref{def:andorO}} \\
 &= \{ T'_\xx \wedge T_\xx \}_{\xx\in\XX} & \text{Property~\ref{prop:comandorT}}\\
 &= O'\wedge O & \text{ Definition~\ref{def:andorO}} \\
(O\wedge O')\wedge O'' &= \{ T_\xx \wedge T'_\xx \}_{\xx\in\XX} \wedge O'' & \text{ Definition~\ref{def:andorO}} \\
 &= \{ (T_\xx \wedge T'_\xx)\wedge T''_\xx \}_{\xx\in\XX} & \text{ Definition~\ref{def:andorO}} \\
 &= \{ T_\xx \wedge (T'_\xx\wedge T''_\xx) \}_{\xx\in\XX} & \text{ Property~\ref{prop:comandorT}}\\
 &= O\wedge \{ T'_\xx \wedge T''_\xx\}_{\xx\in\XX} & \text{ Definition~\ref{def:andorO}} \\
 &= O\wedge (O'\wedge O'') & \text{ Definition~\ref{def:andorO}} \\
O\vee (O\wedge O') &= O\vee \{ T_\xx \wedge T'_\xx \}_{\xx\in\XX} & \text{ Definition~\ref{def:andorO}}\\
  &= \{ T_\xx\vee (T_\xx \wedge T'_\xx) \}_{\xx\in\XX} & \text{ Definition~\ref{def:andorO}}\\
  &= \{ T_\xx \}_{\xx\in\XX} & \text{ Property~\ref{prop:comandorT}} \\
  &= O & \text{ Definition~\ref{def:O}}
\end{align*}
\end{proof}

We also define the order between two objects by combining the previous definitions.

\begin{definition}[Order]\label{def:rcaorder}
Given $O=\{T_\xx\}_{\xx\in\XX}$, $O'=\{T'_\xx\}_{\xx\in\XX}\in \mathcal{O}_{K^0,R,\Omega}$,
\[ O\preceq O'\text{ if } \forall \xx\in\XX, T_\xx\preceq T'_\xx. \]
\end{definition}

Like before, we note $O\simeq O'$ if $O\preceq O'$ and $O'\preceq O$, and again $\simeq$ is $=$.

Property~\ref{prop:Torder} can be generalised: the order between families of context-lattice pairs may be reduced to the order between contexts (and ultimately the order between their sets of attributes).

\begin{property}\label{prop:Oorder}
$\forall O, O'\in\mathcal{O}_{K^0,R,\Omega}, \text{if }k(O)\subseteq k(O')\text{ then } O\preceq O'$.
\end{property}
\begin{proof}
$k(O)\subseteq k(O')$ means that $\forall\xx\in\XX$, $k_\xx(O)\subseteq k_\xx(O')$ which is equivalent to $\langle k_\xx(O), l_\xx(O)\rangle$ $\preceq$ $\langle k_\xx(O'), l_\xx(O')\rangle$ (Property~\ref{prop:Torder}) and hence $O\preceq O'$ (Definition~\ref{def:rcaorder}).
\end{proof}

This order is compatible with meet and join.

\begin{property}\label{prop:ordereqopO}
  $\forall O, O'\in\mathcal{O}_{K^0,R,\Omega}, O\preceq O'\text{  iff  } O=O\wedge O'$.
\end{property}
\begin{proof}
By Property~\ref{prop:ordereqopT}, we have that $T_\xx\preceq T'_\xx\text{  iff  } T_\xx=T_\xx\wedge T'_\xx$ and this $\forall\xx\in\XX$, consequently $O\preceq O'\text{  iff  } O=O\wedge O'$.
\end{proof}

Finally, the set $\mathcal{O}_{K^0,R,\Omega}$ of families of context-lattice pairs is a complete lattice.

\begin{proposition}\label{prop:Ocompl} $\langle \mathcal{O}_{K^0,R,\Omega}, \vee, \wedge\rangle$ is a complete lattice.
\end{proposition}
\begin{proof}
$\mathcal{O}_{K^0,R,\Omega}$ is closed by meet and join (Property~\ref{prop:Oclosedmj}).
$\vee$ and $\wedge$ satisfy commutativity, associativity and the absorption laws (Property~\ref{prop:comandorO}), so this is a lattice.
It is complete because it is the direct product of complete lattices (Property~\ref{prop:Tcompl}).
\end{proof}

\subsection{The expansion function \texorpdfstring{$EF^*$}{EF*}}\label{sec:EFs}

We reformulate RCA as based on a main single function,
$EF^*_{K^0,R,\Omega}$, the expansion function attached to a relational context $\langle K^0, R\rangle$ and a set $\Omega$ of scaling operations.

\begin{definition}[Expansion function]\label{def:EFs}
  Given a relational context $\langle K^0, R\rangle$ and a set $\Omega$, of relational scaling operations the expansion function $EF^*_{K^0,R,\Omega}:\mathcal{O}_{K^0,R,\Omega}\rightarrow\mathcal{O}_{K^0,R,\Omega}$ is defined by:
\[
  EF^*_{K^0,R,\Omega}(O) = \mathrm{T}^*(\sigma^*_\Omega(k(O),R,l(O))).
\]
\end{definition}
As previously, we will abbreviate $EF^*_{K^0,R,\Omega}$ as $EF^*$.
This function is the basis of RCA: it covers scaling and the application of $\mathrm{FCA}^*$ embedded in the function $\mathrm{T}^*$.
Thus, the two steps (3) and (2) of the RCA algorithm (Section~\ref{sec:rcaop}) have been merged into one.

A family of context-lattice pairs is called saturated if it is not possible to scale new relational attributes in any of its contexts.
\begin{definition}[Saturated family of context-lattice pairs]\label{def:saturated}
A family of context-lattice pairs $O\in\mathcal{O}$ is \emph{saturated} if
$\forall \xx\in\XX$, $k_\xx(O)=\sigma_{\Omega}(k_\xx(O),R,l(O))$.
\end{definition}
 
$EF^*$ is thus not anymore in direct connection with the previous $EF$ but it extends it:
\begin{property}\label{prop:EFltEFs}
$\{EF_{K^0_\xx,R,\Omega}(T_\xx)\}_{\xx\in\XX}\preceq EF^*_{K^0,R,\Omega}(\{T_\xx\}_{\xx\in\XX})$.
\end{property}
\begin{proof}
$\forall\xx\in\XX$, $l_\xx(O)\in l(O)$, thus $\sigma_\Omega(k_\xx(O),R,l_\xx(O))\subseteq \sigma_\Omega(k_\xx(O),R,l(O))$ because there are less concepts to scale with: only those in $l_\xx(O)$.
Hence, $\mathrm{T}(\sigma_\Omega(k_\xx(O),R,l_\xx(O)))$ $\preceq$ $\mathrm{T}(\sigma_\Omega(k_\xx(O),R,l(O)))$ (Property~\ref{prop:Torder}).
Therefore this leads to
$\{EF_{K^0,R,\Omega}(T_\xx)\}_{\xx\in\XX}$ $=$ $\{\mathrm{T}(\sigma_\Omega(k_\xx(O),R,l_\xx(O)))\}_{\xx\in\XX}$ $\preceq$ $\{\mathrm{T}(\sigma_\Omega(k_\xx(O),R,l(O)))\}_{\xx\in\XX}$ $=$ $EF^*_{K^0,R,\Omega}(O)$.
\end{proof}

$EF^*$ is an extensive and monotone internal operation for $\mathcal{O}$:
\begin{property}[$EF^*$ is internal to $\mathcal{O}$]\label{prop:EFsintO}
  $\forall O\in\mathcal{O}$, $EF^*(O)\in\mathcal{O}$.
\end{property}
\begin{proof}
$\forall \xx\in\XX$,
$\sigma_\Omega(k_\xx(O),R,l(O))=F(k_\xx(O))\in\mathcal{K}^{\eta^*(K^0)}_{K^0_x,R,\Omega}$
because $\sigma_\Omega$ only scales attributes in $D_{\Omega,R_\xx,K^0}$.
Thus, $EF(T)=\mathrm{T}(\sigma_\Omega(k_\xx(O),R,l(O)))\in\mathcal{T}^{\eta^*(K^0)}_{K^0_x,R,\Omega}$.
Hence, $EF^*(O)=\{\mathrm{T}(\sigma_\Omega(k_\xx(O),R,l(O)))\}_{\xx\in\XX}\in \mathcal{O}$ (Definition~\ref{def:O}).
\end{proof}

\begin{property}[$EF^*$ is extensive and monotone]\label{prop:EFsmonext}
The function $EF^*$ attached to a relational context and a set of scaling operations satisfies:
\begin{align}
  O &\preceq EF^*(O),\tag{extensivity}\label{prop:EFsext}\\
  O\preceq O' &\Rightarrow EF^*(O)\preceq EF^*(O').\tag{monotony}\label{prop:EFsmon}
\end{align}
\end{property}
\begin{proof}
\ref{prop:EFsext} holds because $EF^*$ can only add to $k(O)$ attributes scaled from $l(O)$, hence $\forall \xx\in\XX, k_\xx(O)\subseteq k_\xx(EF^*(O))$.
Thus, by Property~\ref{prop:Oorder}, $O\preceq EF^*(O)$.
\ref{prop:EFsmon} holds because $O\preceq O'$ means that $\forall\xx\in\XX$, $l_\xx(O)\preceq l_\xx(O')$ and $k_\xx(O)\subseteq k_\xx(O')$.
The former entails that $\forall\xx\in\XX$, $\eta(l_\xx(O))\subseteq \eta(l_\xx(O'))$ and consequently, that $D_{\Omega,R_\xx,l(O)}\subseteq D_{\Omega,R_\xx,l(O')}$.
A smaller context ($k_\xx(O)$) is extended by a smaller set of attributes ($D_{\Omega,R_\xx,l(O)}$), thus $k_\xx(EF^*(O))\subseteq k_\xx(EF^*(O'))$.
Hence, by Property~\ref{prop:Oorder}, $EF^*(O)\preceq EF^*(O')$.
\end{proof}

\subsection{Self-supporting families of context-lattice pairs}\label{sec:selfsup}

In a family of context-lattice pairs $O$, there may be a context $k_\xx(O)$ containing attributes which refer to concepts non existing in $l(O)$.
Such an attribute, e.g. $\exists r.c$, belonging to $k_\xx(O)$ may belong to $D_{\Omega,R_\xx,K^0}$ and be well-defined by the incidence relation, but $c$ may not be a concept of $l(O)$.
This is illustrated by Example~\ref{ex:selfsup}.

\begin{example}[Non self-supported families of context-lattice pairs]\label{ex:selfsup}
Figure~\ref{fig:ex-2-noalt} (p.\pageref{fig:ex-2-noalt}) shows a familly of contexts $\{K_3^\#,K_4^\#\}$ and the associated family of concept lattices $\{L_3^\#,L_4^\#\}$ that could be a solution for the example of Section~\ref{sec:exrca} as it belongs to $\mathcal{O}_{\{\exists\},\{p,q\},\{K_3^0,K_4^0\}}$.
However, this family is not self-supported because the context $K_3^\#$ (and thus concept $A$) uses the attribute $\exists p.C$ which refers to a concept ($C$) not present in $L_4^\#$ and similarly for $\exists q.B$ in context $K_4^\#$.
\end{example}

A family of context-lattice pairs in $\mathcal{O}$ containing such attributes will see them preserved by $EF^*$ which only extends the contexts.
This is not the expected result: concepts referred to by attributes are expected to exist in the corresponding lattice.

One may consider identifying such attributes and forbidding them.
However, support is contextual:
one supported attribute in a large family of context-lattice pairs may not be supported in a smaller one.

In order to define acceptable solutions for RCA, we introduce the notion of context supported by a family of concept lattices, i.e. those contexts whose relational attributes only refer to concepts in the lattices.

\begin{definition}[Supported context]\label{def:ksup}
A context $\langle G_\xx, M_\xx, I_\xx\rangle$ is \emph{supported} by a family of indexed lattices $\{L_\zz\}_{\zz\in X}$, with respect to a set $R$ of relations, if $\forall\varsigma(r,c)\in M_\xx$, $r\in R_{\xx,\zz}$ and $c\in L_\zz$.
\end{definition}

By extension, an indexed family of context-lattice pairs is said \emph{self-supported} if each context of the family is supported by the family lattices.
The support of a single context may use several lattices of the family ($l(O)$).

\begin{definition}[Self-supported families of context-lattice pairs]\label{def:ssup}
A family of context-lattice pairs $O\in\mathcal{O}$ is \emph{self-supported}, with respect to a set $R$ of relations, if $\forall\xx\in\XX$, $k_\xx(O)$ is supported by $l(O)$, with respect to $R$.
\end{definition}

As before, the definition of self-supported families of context-lattice pairs does not provide a direct way to transform a non self-supported family into a self-supported one.
For that purpose, we extend the purging function  $\pi$ (Section~\ref{sec:selfsup}) to take into account families of context-lattice pairs.

\begin{definition}[Purging function]\label{def:purges}
The function
  \[\pi^*_{K^0,R,\Omega}: \prod_{\xx\in\XX}\mathcal{L}_{K^0_\xx,R,\Omega}\rightarrow\prod_{\xx\in\XX}\mathcal{K}_{K^0_\xx,R,\Omega}\]
returns the family of contexts reduced of those attributes not present in a family of lattices:
\begin{align*}
\pi_{K^0,R,\Omega}^*(L) & = \{\pi_{K^0,R,\Omega}( L_\xx, L )\}_{\xx\in\XX},\\
\intertext{with}
\pi_{K^0,R,\Omega}( L_\xx, L ) & = \mathrm{K}^{\langle R, L\rangle}_{-(D_{\Omega, R_\xx, K^0)} \setminus D_{\Omega, R_\xx, L})}(\kappa(L_\xx)).
\end{align*}
\end{definition}
When unambiguous, we refer to $\pi^*_{K^0,R,\Omega}$ as $\pi^*$.
As for $EF^*$, $\pi^*$ uses the lattices of the whole family ($l(O)$).

$\pi^*$ can be used to determine if a family of context-lattice pairs is self-supported:
\begin{property}\label{prop:pissup}
$O$ is self-supported if and only if $k(O)=\pi^*(l(O))$.
\end{property}
\begin{proof}
$O$ is self-supported means that $\forall \xx\in\XX$, $k_\xx(O)$ is supported by $l(O)$ (Definition~\ref{def:ssup})
which means that, in $k_\xx(O)$, there is no attribute built from concepts out of $l(O)$, i.e. not belonging to $D_{\Omega,R_\xx,l(O)}$.
By Definition~\ref{def:purge}, this is equivalent to having $\forall\xx\in\XX$, $\pi(l_\xx(O),l(O))=\kappa(l_\xx(O))$.
However, by Property~\ref{prop:kappaFCA}, $k_\xx(O)=\kappa(l_\xx(O))$, thus $k_\xx(O)=\pi(l_\xx(O),l(O))$ and then $k(O)=\pi^*(l(O))$.
\end{proof}
Like $\pi$, $\pi^*$ does not necessary provide a self-supported family of context-lattice pairs in one step.
Hence, it has to be iterated.

We introduce a contraction function, $PQ^*_{K^0,R,\Omega}$, attached to a relational context $\langle K^0, R\rangle$ and a set $\Omega$ of scaling operations, which suppresses non-supported attributes and whose closure yields self-supported families of context-lattice pairs.

\subsection{The contraction function \texorpdfstring{$PQ^*$}{PQ*}}\label{sec:PQs}

Similarly to $EF^*_{K^0,R,\Omega}$, it is possible to define
$PQ^*_{K^0,R,\Omega}$ the contraction function attached to a relational context $\langle K^0, R\rangle$ and a set $\Omega$ of scaling operations.

\begin{figure}[t]
\centering
\begin{tikzpicture}
    \draw (-2,1) node (Kp) {$T'_\xx$};
    \draw (8,1) node (Lp) {$T'_\zz$};

    \draw (-2,0) node (K) {$T_\xx$};
    \draw (8,0) node (L) {$T_\zz$};

    \draw (-2,-1) node (Kpp) {$T''_\xx$};
    \draw (8,-1) node (Lpp) {$T''_\zz$};

    \draw[->] (K) -- node[left]{$EF$} (Kp);
    \draw[->] (L) -- node[right]{$EF$} (Lp);
    \draw[->] (K) -- node[left]{$PQ$} (Kpp);
    \draw[->] (L) -- node[right]{$PQ$} (Lpp);

    \draw (3,-3.5) node (Tp) {$O'$};
    \draw (3,-4.5) node (T) {$O$};
    \draw (3,-5.5) node (Tpp) {$O''$};
    \draw[->] (T) -- node[right]{$EF^*$} (Tp);
    \draw[->] (T) -- node[right]{$PQ^*$} (Tpp);

    \draw[->] (Tp) -- node[below,sloped,near start] {$\pi_\xx$} (Kp);
    \draw[->] (Tp) -- node[below,sloped,near start] {$\pi_\zz$} (Lp);

    \draw[->] (T) -- node[below,sloped,near start] {$\pi_\xx$} (K);
    \draw[->] (T) -- node[below,sloped,near start] {$\pi_\zz$} (L);

    \draw[->] (Tpp) -- node[below,sloped,near start] {$\pi_\xx$} (Kpp);
    \draw[->] (Tpp) -- node[below,sloped,near start] {$\pi_\zz$} (Lpp);
    
    \draw[thin,dashed] (K) -- node[above] {$R_{\xx,\zz}\cup R_{\zz,\xx}$} (L);
    \draw[thin,dashed] (Kp) -- node[above] {$R_{\xx,\zz}\cup R_{\zz,\xx}$} (Lp);
    \draw[thin,dashed] (Kpp) -- node[above] {$R_{\xx,\zz}\cup R_{\zz,\xx}$} (Lpp);
  
   \draw (-3.5,.5) node {$\mathcal{T}_\xx$};
   \draw[dotted] (-2,-1.5) -- (-3.5,0) -- (-2,1.5) -- (-.5,0) -- cycle;

   \draw (9.5,.5) node {$\mathcal{T}_\zz$};
   \draw[dotted] (8,-1.5) -- (6.5,0) -- (8,1.5) -- (9.5,0) -- cycle;

   \draw (4,-5.5) node {$\mathcal{O}$};
   \draw[dotted] (3,-3) -- (1.5,-4.5) -- (3,-6) -- (4.5,-4.5) -- cycle;

\end{tikzpicture}

\caption{Relations between $\mathcal{O}$, $\mathcal{T}_\xx$ and $\mathcal{T}_\zz$.
  In this figure, $\pi_\xx$ represents the projection of the objects to their component indexed by $\xx\in\XX$.
}\label{fig:O2T}
\end{figure}

\begin{definition}[Contraction function]\label{def:PQs}
Given a relational context $\langle K^0, R\rangle$ and a set $\Omega$, of relational scaling operations, the contraction function $PQ^*_{K^0,R,\Omega}:\mathcal{O}_{K^0,R,\Omega}\rightarrow\mathcal{O}_{K^0,R,\Omega}$ is defined by:
\[
  PQ^*_{K^0,R,\Omega}(O) = \mathrm{T}^*(\pi_{K^0,R,\Omega}^*(l(O))).
\]
\end{definition}

This function is thus not anymore in direct connection with the previous $PQ$ but it extends it:
\begin{property}\label{prop:PQltPQs}
$\{PQ_{K^0_\xx,R,\Omega}(T_\xx)\}_{\xx\in\XX}\preceq PQ^*_{K^0,R,\Omega}(\{T_\xx\}_{\xx\in\XX})$.
\end{property}
\begin{proof}
For any $\xx\in\XX$, $l_\xx(O)\in l(O)$, hence $\pi(k_\xx(O),l_\xx(O))\subseteq \pi(k_\xx(O),l(O))$ because there are more attributes to preserve from concepts in $l(O)$.
Thus, $\mathrm{T}(\pi(k_\xx(O),l_\xx(O)))$ $\preceq$ $\mathrm{T}(\pi(k_\xx(O),l(O)))$ (Property~\ref{prop:Torder}).
This has for consequence that $\{PQ_{K^0,R,\Omega}(T_\xx)\}_{\xx\in\XX}$ $=$ $\{\mathrm{T}(\pi(k_\xx(O),l_\xx(O)))\}_{\xx\in\XX}$ $\preceq$ $\{\mathrm{T}(\pi(k_\xx(O),l(O)))\}_{\xx\in\XX}$ $=$ $PQ^*_{K^0,R,\Omega}(O)$.
\end{proof}

As previously, we will abbreviate $PQ^*_{K^0,R,\Omega}$ as $PQ^*$.

$PQ^*$ is an anti-extensive and monotone internal operation for $\mathcal{O}$:
\begin{property}[$PQ^*$ is internal to $\mathcal{O}$]\label{prop:PQsintO}
  $\forall O\in\mathcal{O}$, $PQ^*(O)\in\mathcal{O}$.
\end{property}
\begin{proof}
  $PQ^*(O)$
$=\mathrm{T}^*(\pi^*(l(O)))$ 
$=\mathrm{T}^*(\{\pi(l_\xx(O),l(O))\}_{\xx\in\XX})$ 
$=\{\mathrm{T}(\pi(l_\xx(O),l(O)))\}_{\xx\in\XX}$.
Hence, $PQ^*(O)\in\mathcal{O}$ if $\pi(l_\xx(O),l(O))\in\mathcal{K}^{\eta^*(K^0)}_{K^0_\xx,R,\Omega}$ (Definition~\ref{def:O}).
This is the case because
\begin{enumii}
\item $\mathcal{K}^{\eta^*(K^0)}_{K^0_\xx,R,\Omega}$ contains all contexts extending $K^0_\xx$ with attributes from $D_{\Omega,R_\xx,K^0}$,
\item $k_\xx(O)\in\mathcal{K}^{\eta^*(K^0)}_{K^0_\xx,R,\Omega}$, and
\item $\pi$ only suppresses attributes from $k_\xx(O)$ preserving those of $K^0_\xx$.
\end{enumii}
\end{proof}

\begin{property}[$PQ^*$ is anti-extensive and monotone]\label{prop:PQsantextmon}
The function $PQ^*$ attached to a relational context and a set of scaling operations satisfies:
\begin{align}
  PQ^*(O) &\preceq O,\tag{anti-extensivity}\label{prop:PQsantext}\\
  O\preceq O' &\Rightarrow PQ^*(O)\preceq PQ^*(O').\tag{monotony}\label{prop:PQsmon}
\end{align}
\end{property}
\begin{proof}
\ref{prop:PQsantext} holds because $PQ^*$ can only suppress from $k(O)$ attributes not supported by $l(O)$, hence $\forall \xx\in\XX, k_\xx(PQ^*(O))\subseteq k_\xx(O)$. 
Therefore, by Property~\ref{prop:Oorder}, $PQ^*(O)\preceq O$.
\ref{prop:PQsmon} holds because $O\preceq O'$ means that $\forall\xx\in\XX, k_\xx(O)\subseteq k_\xx(O')$ and $l_\xx(O)\preceq l_\xx(O')$.
This entails that $\eta^*(l(O))\subseteq \eta^*(l(O'))$ and thus, $\forall\xx\in\XX$, $D_{\Omega,R_\xx,l(O)}\subseteq D_{\Omega,R_\xx,l(O')}$.
Because $PQ^*(O)$ suppresses from $k_\xx(O)$ attributes not in $M_\xx^0\cup D_{\Omega,R_\xx,l(O)}$,
this entails that $k_\xx(PQ^*(O))\subseteq k_\xx(PQ^*(O'))$.
Hence, by Property~\ref{prop:Oorder}, $PQ^*(O)\preceq PQ^*(O')$.
\end{proof}

\subsection{The fixed points of \texorpdfstring{$EF^*$}{EF*} and \texorpdfstring{$PQ^*$}{PQ*}}\label{sec:fpEFsPQs}

Given $EF^*$ and $PQ^*$, it is possible to define their sets of fixed points, i.e. the sets of families of context-lattice pairs closed for $EF^*$ and $PQ^*$, as:
\begin{definition}[Fixed points]\label{def:Ofp}
A family of context-lattice pairs $O\in\mathcal{O}$ is a fixed point for a function $\phi$, if $\phi(O)\simeq O$.
  We call $\mathrm{fp}(\phi)$ the set of fixed points for $\phi$.
\end{definition}
This characterises $\mathrm{fp}(EF^*)$ and $\mathrm{fp}(PQ^*)$.

This may be directly expressed
\begin{property}\label{prop:EFsasscaling}
  $O\in\mathrm{fp}(EF^*)$ iff $\sigma^*_{\Omega}(k(O),R,l(O))=k(O)$.
\end{property}
\begin{proof}
$O=\mathrm{T}^*(k(O))$ and $EF^*(O)=\mathrm{T}^*(\sigma^*_{\Omega}(k(O),R,l(O)))$.
($\Leftarrow$) If $\sigma^*_{\Omega}(k(O),R,l(O))=k(O)$, then $EF^*(O)=O$ and thus is a fixed point of $EF^*$.
($\Rightarrow$) If $\sigma^*_{\Omega}(k(O),R,l(O))\neq k(O)$, then $EF^*(O)\neq O$, so it is not a fixed point.
\end{proof}

\begin{property}\label{prop:PQsaspurging}
  $O\in\mathrm{fp}(PQ^*)$ iff $\pi^*(l(O))=k(O)$.
\end{property}
\begin{proof}
$O=\mathrm{T}^*(k(O))$ and $PQ^*(O)=\mathrm{T}^*(\pi^*(l(O)))$.
($\Leftarrow$) If $\pi^*(l(O))=k(O)$, then $PQ^*(O)$ $=$ $O$ and thus is a fixed point of $PQ^*$.
($\Rightarrow$) If $\pi^*(l(O))\neq k(O)$, then $PQ^*(O)\neq O$, so it is not a fixed point.
\end{proof}

Since $\mathcal{O}$ is a complete lattice (Proposition~\ref{prop:Ocompl}) and $EF^*$ and $PQ^*$ are order-preserving (or monotone) on $\mathcal{O}$ (Properties~\ref{prop:EFsmonext} and \ref{prop:PQsantextmon}), then we can apply the Knaster-Tarski theorem (Theorem~\ref{prop:knastertarski}).

Thus, $\langle\mathrm{fp}(EF^*),\preceq\rangle$ and $\langle\mathrm{fp}(PQ^*),\preceq\rangle$ are complete lattices.
This warrants that there exists least and greatest fixed points of $EF^*$ and $PQ^*$ in $\mathcal{O}$.
For such a function $\phi$, operating on the set $\mathcal{O}$, their least and greatest fixed points are:
\[
  \mathrm{lfp}(\phi)=\bigwedge_{O\in \mathrm{fp}(\phi)} O \text{ and } \mathrm{gfp}(\phi)=\bigvee_{O\in \mathrm{fp}(\phi)} O.
\]
The fixed points of these two functions may be further characterised.
The smallest fixed point of $PQ^*$ is the smallest element of $\mathcal{O}$ which cannot be further reduced.
Property~\ref{prop:lfpQ} (apparently not in \cite{euzenat2021a}) can be generalised as:
\begin{property}[Least fixed point of $PQ^*$]\label{prop:lfpPQs}
  \[
    \mathrm{lfp}(PQ^*)=\mathrm{T}^*(K^0).
  \]
\end{property}
\begin{proof}
$\mathrm{T}^*(\pi^*(\mathrm{FCA}^*(K^0)))=\mathrm{T}^*(K^0)$ because
\begin{enumaa}
\item $\forall\xx\in\XX$, $\kappa(\mathrm{FCA}(K_\xx^0))=K_\xx^0$ (Property~\ref{prop:kappaFCA}), and
\item $\pi(\mathrm{FCA}(K^0_\xx),\mathrm{FCA}^*(K^0))=K_\xx^0$ as
it is not possible to suppress attributes from $K^0_\xx$ which, being an initial (unscaled) context, does not comprise any attribute referring to concepts.
\end{enumaa}
Thus, $PQ^*(\mathrm{T}^*(K^0))=\mathrm{T}^*(\pi^*(\mathrm{FCA}^*(K^0)))=\mathrm{T}^*(K^0)$.
Moreover, $\forall O\in\mathcal{O}$, $\mathrm{T}^*(K^0)\preceq O$.
Hence, $\mathrm{T}^*(K^0)$ is a fixed point of $PQ^*$ and all other fixed points are greater.
\end{proof}

The greatest fixed point of $EF^*$ is the family that cannot be further extended (generalising Proposition~\ref{prop:gfpF} \cite{euzenat2021a}):
\begin{property}[Greatest fixed point of $EF^*$]\label{prop:gfpEFs}
\[
  \mathrm{gfp}(EF^*_{K^0,R,\Omega})=\mathrm{T}^*( \{ \mathrm{K}^{\langle R,\eta^*(K^0)\rangle}_{+D_{\Omega,R_\xx,K^0}}(K^0_\xx) \}_{\xx\in\XX} ).
\]
\end{property}
\begin{proof}
This family of context-lattice pairs is the greatest element of $\mathcal{O}$ as $\forall\xx\in\XX$, the context $k_\xx(O)$ contains all attributes of $M_\xx^0\cup D_{\Omega,R_\xx,K^0}$ and due to Property~\ref{prop:Oorder}.
It is also a fixed point because $EF^*$ is extensive (Property~\ref{prop:EFsmonext}) and internal (Property~\ref{prop:EFsintO}).
\end{proof}

The function $EF^*$ and $PQ^*$ converge after a finite number of applications.

\begin{property}[Stability of $EF^*$]\label{prop:EFsstable} $\forall O\in\mathcal{O}$, 
$\exists n; EF^{*n}(O)=EF^{*n+1}(O)$.
\end{property}
\begin{proof}
$EF^*$ can only increase the contexts when there are new concepts in lattices and increase the lattices when contexts grow.
However, the set of attributes that can increase contexts, and the set of concepts that can be in lattices, is finite. 
Hence, at each step either an attribute is added or $n$ has been reached such that the family of context-lattice pairs is the same. 
This is the same argument as that of \cite{rouanehacene2013a}.
\end{proof}

This below is an extension of Proposition~5 of \cite{euzenat2021a}:
\begin{property}[Stability of $PQ^*$]\label{prop:PQsstable} $\forall O\in\mathcal{O}$, 
$\exists n; PQ^{*n}(O)=PQ^{*n+1}(O)$.
\end{property}
\begin{proof}
$PQ^*$ can only decrease the contexts and reduce lattices. 
Since these are finite (and the decrease does not affect the attributes of $K^0$), there exists a $n$ at which the decrease stops.
\end{proof}

The finite application of $EF^*$ and $PQ^*$ as many times as necessary, i.e. to the first $n$ satisfying Properties~\ref{prop:EFsstable} and \ref{prop:PQsstable}, are closure operations denoted by $EF^{*\infty}$ and $PQ^{*\infty}$, respectively.

\begin{property}\label{prop:EFsPQsclos}
  $EF^{*\infty}$ and $PQ^{*\infty}$ are closures.
\end{property}
\begin{proof}
Since $EF^*$ is extensive and monotone (Property~\ref{prop:EFsmonext}), $EF^{*\infty}$ is also extensive and monotone by transitivity of $\preceq$. 
In order to be a closure operation it has to be idempotent.
This is the case, because $\forall O\in\mathcal{O}$, $EF^{*\infty}(O)=EF^{*n}(O)=EF^{*n+1}(O)=EF^*(EF^{*n}(O))$.
Since $EF^{*n}(O)=EF^*(EF^{*n}(O))$, $EF^*$ can be applied $n$ times, leading to $EF^{*\infty}(O)=EF^{*n}(O)=EF^{*n}(EF^{*n}(O))=EF^{*\infty}(EF^{*\infty}(O))$.

The same can be obtained from $PQ^*$, albeit anti-extensive (Property~\ref{prop:PQsantextmon}).
\end{proof}

In addition, they are extrema of the sets of fixed points of their respective functions.
\begin{property}[$EF^*$ and $PQ^*$ return the smallest subsuming and greatest subsumed fixed points]\label{prop:EFsPQslsgsfp}
$\forall O\in\mathcal{O}$,
\begin{align*}
EF^{*\infty}(O) &= \min_\preceq(\mathrm{fp}(EF^*)\cap \{O'|O\preceq O'\}),\\
PQ^{*\infty}(O) &= \max_\preceq(\mathrm{fp}(PQ^*)\cap \{O'|O'\preceq O\}).
\end{align*}
\end{property}
\begin{proof}
$EF^{*\infty}(O)\in\mathrm{fp}(EF^*)$ and $PQ^{*\infty}(O)\in\mathrm{fp}(PQ^*)$ as they satisfy Definition~\ref{def:Ofp}.
Moreover, $EF^{*\infty}(O)\in\{O'|O\preceq O'\}$ and $PQ^{*\infty}(O)\in\{O'|O'\preceq O\}$ as $EF^*$ and $PQ^*$ are respectively extensive and anti-extensive and monotone (Property~\ref{prop:EFsmonext} and \ref{prop:PQsantextmon}).
There cannot be $O'\in\mathrm{fp}(EF^*)\cap \{O'|O\preceq O'\}$ such that $O'\prec EF^{*\infty}(O)$ because otherwise $k(O')\subset k(EF^{*\infty}(O))$ and $k(O)\subseteq k(O')$. In other terms, $O'$ contains all attributes of $O$ but not all attributes of $EF^{*\infty}(O)$. But, $EF^{*\infty}$ only adds scalable attributes and $k(EF^{*\infty}(O))$ contains only attributes scalable from $O$. Hence, $O'$ is not closed for $EF^*$ ($O'\not\in\mathrm{fp}(EF^*)$).

The same holds for $PQ^{*\infty}(O)$, there cannot be $O'\in\mathrm{fp}(PQ^*)\cap \{O'|O'\preceq O\}$ such that $PQ^{*\infty}(O)\prec O'$ because otherwise $k(O')\subseteq k(O)$.
In other terms, all attributes of $O'$ are in $O$ but $O'$ contains all attributes of $PQ^{*\infty}(O)$.
However, $PQ^{*\infty}$ only suppress attributes not supported by those of $O$.
Hence, $O'$ is not closed for $PQ^*$ ($O'\not\in\mathrm{fp}(PQ^*)$), as it would contain non-supported attributes.
\end{proof}

The respective relations of these various objects can be summarised by the following property:
\begin{property}\label{prop:order}
$\forall O\in\mathcal{O}$,
\[
  \mathrm{lfp}(PQ^*)\preceq PQ^{*\infty}(O) \preceq PQ^*(O)\preceq O\preceq EF^*(O)\preceq EF^{*\infty}(O)\preceq\mathrm{gfp}(EF^*).
\]
\end{property}
\begin{proof}
All the inner equations are consequences of the extensivity of $EF^*$ (Property~\ref{prop:EFsmonext}) and anti-extensivity of $PQ^*$ (Property~\ref{prop:PQsantextmon}).
The outer ones owe to the fact that the two closure operations are fixed points (Property~\ref{prop:EFsPQslsgsfp}), thus they are subsumed by, resp. subsuming, their greatest, resp. least, fixed point.
\end{proof}

\subsection{Acceptable solutions}\label{sec:acceptable}

What is called acceptable solutions in Section~\ref{sec:examples} is now rephrased in Definition~\ref{def:acceptable}.

\begin{definition}[Acceptable family of context-lattice pairs]\label{def:acceptable}
Given a family $K^0$ of contexts, a set $\Omega$ of scaling operations and a set $R$ of relations, a family of context-lattice pairs $O$ is acceptable if
\begin{itemize}
\item $O\in\mathcal{O}_{K^0,R,\Omega}$, \hfill (well-formedness)
\item $O$ is saturated, \hfill(saturation)\\
\item $O$ is self-supported. \hfill(self-support)\\
\end{itemize}
\end{definition}

This can be characterised as those families of context-lattice pairs which are fixed points of both $EF^*$ and $PQ^*$.

The fixed points of $EF^*$ are exactly those saturated elements of $\mathcal{O}$:
\begin{property}[Fixed points of $EF^*$ are saturated]\label{prop:satisfpEFs}
$\forall O\in\mathcal{O}$, $O$ is saturated iff $O\in\mathrm{fp}(EF^*)$.
\end{property}
\begin{proof}
$O\in\mathrm{fp}(EF^*)$ means that $k(O)=\sigma^*_\Omega(k(O),R,l(O))$ (Property~\ref{prop:EFsasscaling}) which is equivalent to $\forall\xx\in\XX, k_\xx(O)=\sigma_{\Omega}(k_\xx(O),R,l(O))$ which, by Definition~\ref{def:saturated}, means that $O$ is saturated.
\end{proof}

The fixed points of $PQ^*$ are exactly those self-supported objects in $\mathcal{O}$:
\begin{property}[Fixed points of $PQ^*$ are self-supported]\label{prop:PQsss}\label{prop:ssisfpPQs}
  $\forall O\in\mathcal{O}$, $O$ is self-supported iff $O\in\mathrm{fp}(PQ^*)$.
\end{property}
\begin{proof}
$O$ is self-supported iff $k(O)=\pi^*(l(O))$ (Property~\ref{prop:pissup}) which is equivalent to $O=PQ^*(O)$ (Property~\ref{prop:PQsaspurging}), i.e. $O\in\mathrm{fp}(PQ^*)$.
\end{proof}

Hence, the set of acceptable solutions is $\mathrm{fp}(EF^*)\cap\mathrm{fp}(PQ^*)$.

\begin{proposition}[Acceptable solutions are fixed points of both $EF^*$ and $PQ^*$]\label{prop:accfp}
Given a family $K^0$ of contexts, a set $\Omega$ of scaling operations and a set $R$ of relations,
a family of context-lattice pairs $O$ is acceptable iff $O\in\mathcal{O}_{K^0,R,\Omega}$ and $O\in\mathrm{fp}(EF^*)\cap\mathrm{fp}(PQ^*)$.
\end{proposition}
\begin{proof} $O$ is well-formed as it belongs to $\mathcal{O}_{K^0,R,\Omega}$.
It is saturated if and only if it belongs to $\mathrm{fp}(EF^*)$ (Property~\ref{prop:satisfpEFs}) and it is self-supported if and only if it belongs to $\mathrm{fp}(PQ^*)$ (Property~\ref{prop:ssisfpPQs}).
Hence, $O$ is acceptable (Definition~\ref{def:acceptable}).
\end{proof}

Example~\ref{ex:fxp} illustrates this:

\begin{example}[Acceptable solutions]\label{ex:fxp}
In the example of Section~\ref{sec:exrca}, it can be checked that the given solutions belong to the expected fixed points:
\begin{align*}
EF^*(\{\langle K_3^1, L_3^1\rangle, \langle K_4^1, L_4^1\rangle\})=\{\langle K_3^1, L_3^1\rangle, \langle K_4^1, L_4^1\rangle\}=PQ^*(\{\langle K_3^1, L_3^1\rangle, \langle K_4^1, L_4^1\rangle\}),\\
EF^*(\{\langle K_3^\star, L_3^\star\rangle, \langle K_4^\star, L_4^\star\rangle\})=\{\langle K_3^\star, L_3^\star\rangle, \langle K_4^\star, L_4^\star\rangle\}=PQ^*(\{\langle K_3^\star, L_3^\star\rangle, \langle K_4^\star, L_4^\star\rangle\}),\\
EF^*(\{\langle K_3', L_3'\rangle, \langle K_4', L_4'\rangle\})=\{\langle K_3', L_3'\rangle, \langle K_4', L_4'\rangle\}=PQ^*(\{\langle K_3', L_3'\rangle, \langle K_4', L_4'\rangle\}),\\
\intertext{and}
EF^*(\{\langle K_3'', L_3''\rangle, \langle K_4'', L_4''\rangle\})=\{\langle K_3'', L_3''\rangle, \langle K_4'', L_4''\rangle\}=PQ^*(\{\langle K_3'', L_3''\rangle, \langle K_4'', L_4''\rangle\}).
\end{align*}
and none of the other elements of $\mathcal{O}$ as displayed in Figure~\ref{fig:exrcalattice} (p.\pageref{fig:exrcalattice}).
\end{example}

In lattice theory, saturation and self-support would have been easily called closedness.
The terms saturation and self-support have been chosen in order to differentiate them.

\section{The fixed-point semantics of RCA}\label{sec:semantics}

Now that the acceptable solutions have been characterised structurally and functionally, we can answer our initial question and define the semantics of RCA.
$\underline{\mathrm{RCA}}$ returns the smallest acceptable solution.
It is also the least fixed point of the $EF^*$ function (\S\ref{sec:fullRCAsem}).
Another interesting operation is the one that generates the greatest acceptable solution, which is also the greatest fixed point of $PQ^*$ (\S\ref{sec:gfpsem}).

It is also worth considering obtaining the whole set $\mathrm{fp}(EF^*)\cap\mathrm{fp}(PQ^*)$.
Section~\ref{sec:allfpsem} investigates the structure of $[\mathrm{lfp}(EF^*),~\mathrm{gfp}(PQ^*)]$ and its relation with $\mathrm{fp}(EF^*)\cap\mathrm{fp}(PQ^*)$ towards that goal.
It provides various results that may be exploited to develop efficient algorithms.

\subsection{Classical RCA computes \texorpdfstring{$EF^*$}{EF*}'s least fixed point}\label{sec:fullRCAsem}

$\underline{\mathrm{RCA}}$ as it has been defined in Section~\ref{sec:rcaop} (p.\pageref{sec:rcaop}) may be redefined as
\[
  \underline{\mathrm{RCA}}_{\Omega}(K^0,R) = l( EF_{K^0,R,\Omega}^{*\infty}( \mathrm{T}^*( K^0 ) ))
\]
i.e. $\underline{\mathrm{RCA}}$ iterates $EF^*$ from $\mathrm{T}^*(K^0)$ until reaching a fixed point, and ultimately the corresponding lattices are returned.

It thus seems that $\underline{\mathrm{RCA}}$ returns a fixed point of $EF^*$.
Hence the question: which fixed point is returned by RCA's well-grounded semantics?
This is the least fixed point.

\begin{proposition}[The RCA algorithm computes the least fixed point of $EF^*$]\label{prop:rcalfp}
Given $EF^*$ the expansion function associated to $K^0$, $R$ and $\Omega$,
\[
  \underline{\mathrm{RCA}}_{\Omega}(K^0,R) = l(\mathrm{lfp}(EF^*_{K^0,R,\Omega})).
\]
\end{proposition}
\begin{proof}
$\mathrm{T}^*(K^0)\in\mathcal{O}$, thus $EF^{*\infty}(\mathrm{T}^*(K^0))\in\mathcal{O}$ (Property~\ref{prop:EFsintO}).
Moreover, $EF^{*\infty}( \mathrm{T}^*( K^0 ) )$ $=$ $\min_\preceq(\mathrm{fp}(EF^*)\cap\{O' ~|~ \mathrm{T}^*( K^0 )\preceq O'\})$ (Property~\ref{prop:EFsPQslsgsfp}).
But $\forall O'\in \mathcal{O}$, $\mathrm{T}^*( K^0 )\preceq O'$, hence $EF^{*\infty}( \mathrm{T}^*( K^0 ) )=\min_\preceq(\mathrm{fp}(EF^*) )$.
Thus, $EF^{*\infty}_{K^0,R,\Omega}(\mathrm{T}^*(K^0))$ is a fixed point more specific than all fixed points: it is the least fixed point.
$l(EF^{*\infty}_{K^0,R,\Omega}(\mathrm{T}^*(K^0)))$ is the family of lattices associated with the least fixed point of $EF^*_{K^0,R,\Omega}$.
\end{proof}

\subsection{Greatest fixed-point (of \texorpdfstring{$PQ^*$}{PQ*}) semantics}\label{sec:gfpsem}

It is possible to define $\overline{\mathrm{RCA}}$ as returning the greatest acceptable solution.
The greatest fixed point of $EF^*$ (Property~\ref{prop:gfpEFs}) is not necessarily an acceptable solution because it may not be self-supported.
Said otherwise, it does not belong to $\mathrm{fp}(EF^*)\cap\mathrm{fp}(PQ^*)$ because it is not a fixed point for $PQ^*$.

Alternatively, a dual procedure $\overline{\mathrm{RCA}}$ may be defined as:\label{def:ACR}
\[
  \overline{\mathrm{RCA}}_{\Omega}( K^0, R ) = l( PQ_{K^0,R,\Omega}^{*\infty}( \mathrm{T}^*( \{ \mathrm{K}^{\langle R,\eta^*(K^0)\rangle}_{+D_{\Omega,R_\xx,K^0}}(K^0_\xx) \}_{\xx\in\XX} ) ) ).
\]
and it can be characterised analogously as the greatest fixed point of $PQ_{K^0,R,\Omega}^{*}$.

\begin{proposition}[$\overline{\mathrm{RCA}}$ determines the greatest fixed point of $PQ^*$]\label{prop:acrgfp}
Given $PQ^*$ the contraction function associated to $K^0$, $R$ and $\Omega$,
\[
\overline{\mathrm{RCA}}_{\Omega}(K^0,R) = l(\mathrm{gfp}(PQ^*_{K^0,R,\Omega})).
\]
\end{proposition}
\begin{proof}
  $O^\infty=\mathrm{T}^*( \{ \mathrm{K}^{\langle R,\eta^*(K^0)\rangle}_{+D_{\Omega,R_\xx,K^0}}(K^0_\xx) \}_{\xx\in\XX} )\in\mathcal{O}$, hence $PQ_{K^0,R,\Omega}^{*\infty}( O^\infty )\in\mathcal{O}$ (by Property~\ref{prop:PQsintO}).
  Moreover, $PQ^{*\infty}(O^\infty)$ $=$ $\max_\preceq(\mathrm{fp}(PQ^*)\cap \{O'|O'\preceq O^\infty\})$ (Property~\ref{prop:EFsPQslsgsfp}).
  But $\forall O'\in\mathcal{O}$, $O'\preceq O^\infty$, hence $PQ^{*\infty}(O^\infty) = \max_\preceq(\mathrm{fp}(PQ^*))$.
  Thus, $PQ^{*\infty}_{K^0,R,\Omega}(O^\infty)$ is a fixed point more general than all fixed points: it is the greatest fixed point.

$\overline{\mathrm{RCA}}_{\Omega}(K^0,R)=l( PQ_{K^0,R,\Omega}^{*\infty}( O^\infty ) )$ returns the family of lattices associated with the greatest fixed point of $PQ^*_{K^0,R,\Omega}$.
\end{proof}

In order to find $\mathrm{gfp}(PQ^*)$, the process starts with $\mathrm{T}^*( \{ \mathrm{K}^{\langle R,\eta(K^0_\xx)\rangle}_{+D_{\Omega,R_\xx,K^0}}(K^0_\xx) \}_{\xx\in\XX} )$, the largest family of context-lattice pairs, and iterates the application of $PQ^*$, i.e. the two operations $\pi^*$ and $\mathrm{FCA}^*$, until reaching a fixed point, i.e. reaching $n$ such that $O^{n+1}=O^{n}$.

Thus, the $\overline{\mathrm{RCA}}$ algorithm proceeds in the following way:
\begin{enumerate}
\item Initial contexts: $t\leftarrow 0$; $\{\langle G_\xx, M^{t}_\xx, I^{t}_\xx\rangle\}_{\xx\in\XX} \leftarrow \{\mathrm{K}^{\langle R,\eta^*(K^0)\rangle}_{+D_{\Omega,R_\xx,K^0}}(K^0_\xx)\}_{\xx\in\XX}$
\item\label{it:init-acr} $\{L^{t}_\xx\}_{\xx\in\XX} \leftarrow \mathrm{FCA}^*(\{\langle G_\xx, M^{t}_\xx, I^{t}_\xx\rangle\}_{\xx\in\XX})$ (or, for each context, $\langle G_\xx, M^{t}_\xx, I^{t}_\xx\rangle$ the corresponding concept lattice $L^{t}_\xx=\mathrm{FCA}(\langle G_\xx, M^{t}_\xx, I^{t}_\xx\rangle)$ is created using $\mathrm{FCA}$).
\item $\{\langle G_\xx, M^{t+1}_\xx, I^{t+1}_\xx\rangle\}_{\xx\in\XX} \leftarrow \pi^*(\{L^{t}_\xx\}_{\xx\in\XX})$ (i.e. suppressing from $K^t_\xx$ each attribute in $L^t_\xx$ referring through a relation $r\in R_{\xx,\zz}$ to a concept $c_\zz$ not appearing in $L^t_\zz$).
\item If $\exists \xx\in\XX; M^{t+1}_\xx\neq M^{t}_\xx$ (purging has occurred), then $t\leftarrow t+1$; go to Step~\ref{it:init-acr}.
\item Return: $\{L^{t}_\xx\}_{\xx\in \XX}$.
\end{enumerate}
This algorithm is the dual of the $\underline{\mathrm{RCA}}$ procedure.

Example~\ref{ex:gfp0} shows how this is processed in RCA$^0$ and Example~\ref{ex:gfp} for full RCA.
\begin{example}[Greatest fixed-point semantics for RCA$^0$]\label{ex:gfp0}
Consider the example of Section~\ref{sec:elabexrca0}.
$\mathrm{gfp}(EF^*)=\mathrm{T}^*( \{ \mathrm{K}^{\langle R,\eta(K_0^0)\rangle}_{+D_{\Omega,R_\xx,K_0^0}}(K_0^0) \} )$ (Property~\ref{prop:gfpEFs}) is presented below:
\begin{center}\footnotesize
\setlength{\tabcolsep}{2pt}
\begin{tabular}{r|ccccccccccccccc}
$k(\mathrm{gfp}(EF^*))$    & \rotatebox{90}{$\exists r.A$} & \rotatebox{90}{$\exists r.B$} & \rotatebox{90}{$\exists r.C$} & \rotatebox{90}{$\exists r.D$} & \rotatebox{90}{$\exists r.AB$} & \rotatebox{90}{$\exists r.AC$} & \rotatebox{90}{$\exists r.AD$} & \rotatebox{90}{$\exists r.BC$} & \rotatebox{90}{$\exists r.BD$} & \rotatebox{90}{$\exists r.CD$} & \rotatebox{90}{$\exists r.ABC$} & \rotatebox{90}{$\exists r.ABD$} & \rotatebox{90}{$\exists r.ACD$} & \rotatebox{90}{$\exists r.BCD$} & \rotatebox{90}{$\exists r.ABCD$}\\\hline
$a$ & $\times$ & $\times$ &          &          & $\times$ & $\times$ & $\times$ & $\times$ & $\times$ &          & $\times$ & $\times$ & $\times$ & $\times$ & $\times$        \\
$b$ & $\times$ & $\times$ &          &          & $\times$ & $\times$ & $\times$ & $\times$ & $\times$ &          & $\times$ & $\times$ & $\times$ & $\times$ & $\times$        \\
$c$ &          &          &          & $\times$ &          &          & $\times$ &          & $\times$ & $\times$ &          & $\times$ & $\times$ & $\times$ & $\times$        \\
$d$ &          &          & $\times$ &          &          & $\times$ &          & $\times$ &          & $\times$ & $\times$ &          & $\times$ & $\times$ & $\times$ \\
\end{tabular}
\end{center}
It leads to the following lattice:
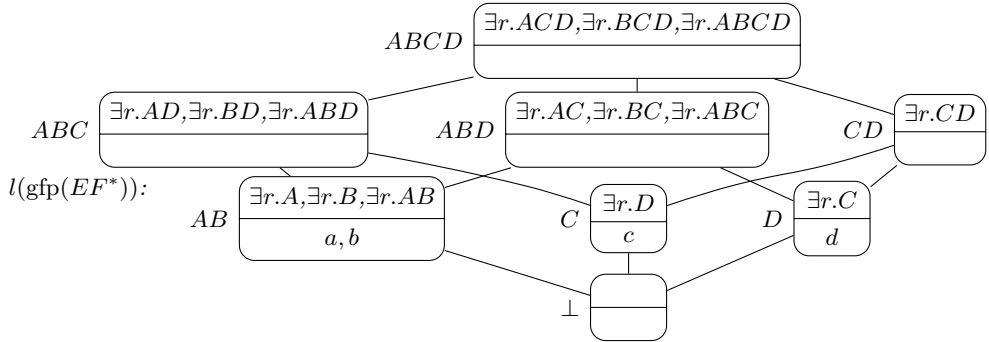
\begin{center}
  \begin{tikzpicture}[font=\footnotesize]
    \begin{scope}[xscale=.32,yscale=.25]

    \begin{dot2tex}[dot,tikz,codeonly,options=-traw]
      graph {
	graph [nodesep=1.5]
	node [style=concept]
	ABCD [label="$\exists r.ACD$,$\exists r.BCD$,$\exists r.ABCD$
\nodepart{two}
$\empty$"]
	ABC [label="$\exists r.AD$,$\exists r.BD$,$\exists r.ABD$
\nodepart{two}
$\empty$"]
	ABD [label="$\exists r.AC$,$\exists r.BC$,$\exists r.ABC$
\nodepart{two}
$\empty$"]

	AB [label="$\exists r.A$,$\exists r.B$,$\exists r.AB$
\nodepart{two}
$a, b$"]
	CD [label="$\exists r.CD$
\nodepart{two}
$\empty$"]

	C [label="$\exists r.D$
\nodepart{two}
$c$"]
	D [label="$\exists r.C$
\nodepart{two}
$d$"]

        C0 [label="$\empty$
\nodepart{two}
$\empty$"]
	ABCD -- ABC
	ABCD -- ABD
	ABCD -- CD
	ABC -- AB
        ABC -- C
        ABD -- AB
        ABD -- D
        
        CD -- C
        CD -- D
        
        AB -- C0
        C -- C0
        D -- C0
      }
    \end{dot2tex}
    \node[anchor=east] at (ABCD.west) {$ABCD$};
    \node[anchor=east] at (ABC.west) {$ABC$};
    \node[anchor=east] at (ABD.west) {$ABD$};
    \node[anchor=east] at (AB.west) {$AB$};
    \node[anchor=east] at (CD.west) {$CD$};
    \node[anchor=east] at (C.west) {$C$};
    \node[anchor=east] at (D.west) {$D$};
    \node[anchor=east] at (C0.west) {$\bot$};
  \end{scope}
  \draw (-.25,1.5) node {$l(\mathrm{gfp}(EF^*))$:};
  \end{tikzpicture}
\end{center}
It can be checked that it is a fixed point for $EF^*$: 
no additional attribute can be scaled.

On the contrary, $PQ^*$ can be applied to $\mathrm{gfp}(EF^*)$ leading to the following result:
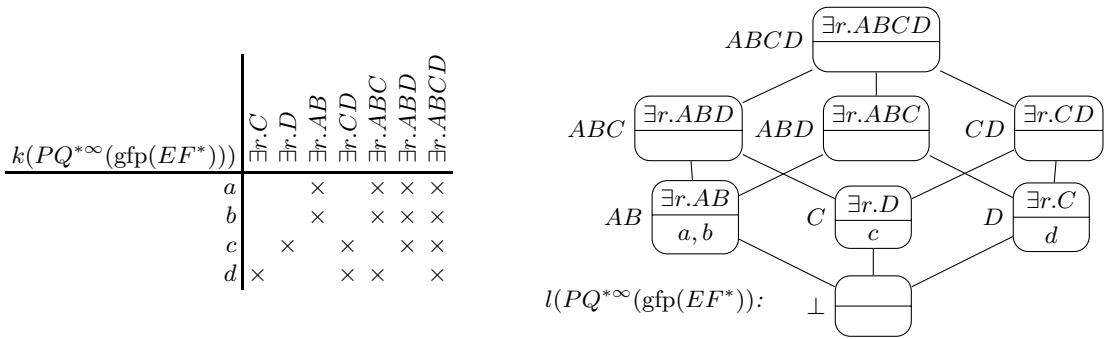
\begin{center}\footnotesize
\begin{tabular}{cc}
\begin{minipage}{.4\textwidth}
\setlength{\tabcolsep}{2pt}
\begin{center}
\begin{tabular}{r|ccccccc}
$k(PQ^{*\infty}(\mathrm{gfp}(EF^*)))$    & \rotatebox{90}{$\exists r.C$} & \rotatebox{90}{$\exists r.D$} & \rotatebox{90}{$\exists r.AB$} & \rotatebox{90}{$\exists r.CD$} & \rotatebox{90}{$\exists r.ABC$} & \rotatebox{90}{$\exists r.ABD$} & \rotatebox{90}{$\exists r.ABCD$}\\\hline
$a$ &          &          & $\times$ &          & $\times$ & $\times$ & $\times$        \\
$b$ &          &          & $\times$ &          & $\times$ & $\times$ & $\times$        \\
$c$ &          & $\times$ &          & $\times$ &          & $\times$ & $\times$        \\
$d$ & $\times$ &          &          & $\times$ & $\times$ &          & $\times$ \\
\end{tabular}
\end{center}
\end{minipage} &
\begin{minipage}{.6\textwidth}
\begin{center}
  \begin{tikzpicture}[font=\footnotesize]
    \begin{scope}[xscale=.28,yscale=.25]

    \begin{dot2tex}[dot,tikz,codeonly,options=-traw]
      graph {
	graph [nodesep=1.5]
	node [style=concept]
	ABCD [label="$\exists r.ABCD$
\nodepart{two}
$\empty$"]
	ABC [label="$\exists r.ABD$
\nodepart{two}
$\empty$"]
	ABD [label="$\exists r.ABC$
\nodepart{two}
$\empty$"]

	AB [label="$\exists r.AB$
\nodepart{two}
$a, b$"]
	CD [label="$\exists r.CD$
\nodepart{two}
$\empty$"]

	C [label="$\exists r.D$
\nodepart{two}
$c$"]
	D [label="$\exists r.C$
\nodepart{two}
$d$"]

        C0 [label="$\empty$
\nodepart{two}
$\empty$"]
	ABCD -- ABC
	ABCD -- ABD
	ABCD -- CD
	ABC -- AB
        ABC -- C
        ABD -- AB
        ABD -- D
        
        CD -- C
        CD -- D
        
        AB -- C0
        C -- C0
        D -- C0
      }
    \end{dot2tex}
    \node[anchor=east] at (ABCD.west) {$ABCD$};
    \node[anchor=east] at (ABC.west) {$ABC$};
    \node[anchor=east] at (ABD.west) {$ABD$};
    \node[anchor=east] at (AB.west) {$AB$};
    \node[anchor=east] at (CD.west) {$CD$};
    \node[anchor=east] at (C.west) {$C$};
    \node[anchor=east] at (D.west) {$D$};
    \node[anchor=east] at (C0.west) {$\bot$};
  \end{scope}
  \draw (.25,.5) node {$l(PQ^{*\infty}(\mathrm{gfp}(EF^*))$:};
  \end{tikzpicture}
\end{center}
\end{minipage}
\end{tabular}
\end{center}
In this case, $PQ^*(\mathrm{gfp}(EF^*))=PQ^{*\infty}(\mathrm{gfp}(EF^*))$, this is not necessarily true as some concepts may be supported by attributes which may be retracted from the lattice due to lack of support.
When full RCA is considered, this may span from context to context.
\end{example}

\begin{example}[Greatest fixed-point semantics for RCA]\label{ex:gfp}
In the cases presented in Section~\ref{sec:elabexrca} and \ref{sec:exrca}, the greatest (or maximum) elements $\{\langle K_1^\star, L_1^\star\rangle, \langle K_2^\star, L_2^\star\rangle\}$ and  $\{\langle K_3^\star, L_3^\star\rangle,$ $\langle K_4^\star, L_4^\star\rangle\}$ of $\mathcal{O}$ are fixed points for both $EF^*$ and $PQ^*$.
Instead, consider the example of Section~\ref{sec:elabexrca} in which the relation $p$ is changed to:

\begin{center}\footnotesize
\setlength{\tabcolsep}{2pt}
\begin{tikzpicture}

  \draw (-4,0) node {\begin{tabular}{r|ccc}
                       $p$ & $d$      & $e$      & $f$ \\ \hline
                       $a$ & $\times$ & $\times$ & \\
                       $b$ & $\times$ & $\times$ & \\
                       $c$ &          &          & $\times$
                       \end{tabular}};

  \draw (4,0) node {\begin{tabular}{r|ccc}
                       $q$ & $a$ & $b$ & $c$\\ \hline
                       $d$ & $\times$ &   & \\
                       $e$ &   & $\times$ & \\
                       $f$ &   &   & $\times$
                       \end{tabular}};

\end{tikzpicture}
\end{center}
\noindent $q$ remaining the same.
The effect of changing $p$ is to make $a$ and $b$ non distinguishable and reduce the sets of supported concepts.

By Property~\ref{prop:gfpEFs}, $\mathrm{gfp}(EF^*)=\mathrm{T}^*( \{ \mathrm{K}^{\langle R,\eta^*(K^0)\rangle}_{+D_{\Omega,R_\xx,K^0}}(K_\xx^0) \}_{\xx\in\{1,2\}}  )=\{\langle K_1^\top, L_1^\top\rangle,$ $\langle K_2^\top, L_2^\top\rangle\}$ which is presented below:

\begin{center}\footnotesize
\setlength{\tabcolsep}{2pt}
\centering
\begin{tikzpicture}[font=\footnotesize]

  \draw (-4,6) node {\begin{tabular}{r|cccccccccc}
                        $K_1^\top$ & $m_1$ & $m_2$ & $m_3$ & \rotatebox{90}{$\exists p.DEF$} & \rotatebox{90}{$\exists p.DE$} & \rotatebox{90}{$\exists p.DF$} & \rotatebox{90}{$\exists p.EF$} & \rotatebox{90}{$\exists p.D$} & \rotatebox{90}{$\exists p.E$} & \rotatebox{90}{$\exists p.F$}\\ \hline
                       $a$ &   & $\times$ &   & $\times$ & $\times$ & $\times$ & $\times$ & $\times$ & $\times$  &   \\
                       $b$ &   & $\times$ &   & $\times$ & $\times$ & $\times$ & $\times$ & $\times$  & $\times$ &   \\
                       $c$ & $\times$ &   & $\times$ & $\times$ &   & $\times$ & $\times$ &   &   & $\times$  
                       \end{tabular}};

  \draw (3,6) node {\begin{tabular}{r|ccccccccc}
                        $K_2^\top$ & $n_1$ & $n_2$ & \rotatebox{90}{$\exists q.ABC$} & \rotatebox{90}{$\exists q.AB$} & \rotatebox{90}{$\exists q.AC$} & \rotatebox{90}{$\exists q.BC$} & \rotatebox{90}{$\exists q.A$} & \rotatebox{90}{$\exists q.B$} & \rotatebox{90}{$\exists q.C$}\\ \hline
                       $d$ & $\times$ &    & $\times$ & $\times$ & $\times$ &   & $\times$ &   & \\
                       $e$ & $\times$ & $\times$ & $\times$ & $\times$ &   & $\times$ &   & $\times$ & \\
                       $f$ &   &   & $\times$ &   & $\times$ & $\times$ &   &   & $\times$
                       \end{tabular}};

  \begin{scope}[xshift=-6cm]
    \begin{scope}[xscale=.25,yscale=.25]

    \begin{dot2tex}[dot,tikz,codeonly,options=-traw]
      graph {
	graph [nodesep=1.5]
	node [style=concept]
	ABC [label="$\exists p.DEF, \exists p.DF, \exists p.EF$
\nodepart{two}
$\empty$"]
	AB [label="$m_2, \exists p.D, \exists p.E, \exists p.DE$
\nodepart{two}
$a$, $b$"]
	C [label="$m_1, m_3, \exists p.F$
\nodepart{two}
$c$"]
        C0 [label="$\empty$
\nodepart{two}
$\empty$"]
	ABC -- AB
	ABC -- C
	AB -- C0
        C -- C0
      }
    \end{dot2tex}
    \node[anchor=east] at (ABC.west) {$ABC$};
    \node[anchor=east] at (AB.west) {$AB$};
    \node[anchor=east] at (C.west) {$C$};
    \node[anchor=east] at (C0.west) {$\bot$};
  \end{scope}
  \draw (0,.5) node {$L_1^\top$:};
  \end{scope}

  \begin{scope}[xshift=1cm]
    \begin{scope}[xscale=.25,yscale=.25]

    \begin{dot2tex}[dot,tikz,codeonly,options=-traw]
      graph {
	graph [nodesep=1.5]
	node [style=concept]
	DEF [label="$\exists p.ABC$
\nodepart{two}
$\empty$"]
	DE [label="$n_1, \exists p.AB$
\nodepart{two}
$\empty$"]
	DF [label="$\exists p.AC$
\nodepart{two}
$\empty$"]
	EF [label="$\exists p.BC$
\nodepart{two}
$\empty$"]
	D [label="$\exists p.A$
\nodepart{two}
$d$"]
	E [label="$n_2, \exists p.B$
\nodepart{two}
$e$"]
	F [label="$\exists p.C$
\nodepart{two}
$f$"]
        C0 [label="$\empty$
\nodepart{two}
$\empty$"]
	DEF -- DE
	DEF -- DF
	DEF -- EF
	DE -- D
        DE -- E
        DF -- D
        DF -- F
        EF -- E
        EF -- F
        
        D -- C0
        E -- C0
        F -- C0
      }
    \end{dot2tex}
    \node[anchor=east] at (DEF.west) {$DEF$};
    \node[anchor=east] at (DE.west) {$DE$};
    \node[anchor=east] at (DF.west) {$DF$};
    \node[anchor=east] at (EF.west) {$EF$};
    \node[anchor=east] at (D.west) {$D$};
    \node[anchor=east] at (E.west) {$E$};
    \node[anchor=east] at (F.west) {$F$};
    \node[anchor=east] at (C0.west) {$\bot$};
  \end{scope}
  \draw (0,.5) node {$L_2^\top$:};

  \end{scope}

\end{tikzpicture}
\end{center}

It can be checked that $\{\langle K_1^\top, L_1^\top\rangle, \langle K_2^\top, L_2^\top\rangle\}$ is a fixed point for $EF^*$: no additional attribute can be scaled.
On the contrary, it is not a fixed point for $PQ^*$ which can be applied to it.

In a first application, it will purge $\langle K_2^\top, L_2^\top\rangle$ to:

\begin{center}\footnotesize
\setlength{\tabcolsep}{2pt}
\begin{tikzpicture}[font=\footnotesize]

  \draw (-3,2) node {\begin{tabular}{r|ccccc}
                        $K_2^\dag$ & $n_1$ & $n_2$ & \rotatebox{90}{$\exists q.ABC$} & \rotatebox{90}{$\exists q.AB$} & \rotatebox{90}{$\exists q.C$}\\ \hline
                       $d$ & $\times$ &    & $\times$ & $\times$ & \\
                       $e$ & $\times$ & $\times$ & $\times$ & $\times$ & \\
                       $f$ &   &   & $\times$ &   & $\times$
                       \end{tabular}};

  \begin{scope}[xshift=3cm]
    \begin{scope}[xscale=.25,yscale=.25]

    \begin{dot2tex}[dot,tikz,codeonly,options=-traw]
      graph {
	graph [nodesep=1.5]
	node [style=concept]
	DEF [label="$\exists p.ABC$
\nodepart{two}
$\empty$"]
	DE [label="$n_1, \exists p.AB$
\nodepart{two}
$d$"]
	E [label="$n_2$
\nodepart{two}
$e$"]
	F [label="$\exists q.C$
\nodepart{two}
$f$"]
        C0 [label="$\empty$
\nodepart{two}
$\empty$"]
	DEF -- DE
	DEF -- F
        DE -- E
        E -- C0
        F -- C0
      }
    \end{dot2tex}
    \node[anchor=east] at (DEF.west) {$DEF$};
    \node[anchor=east] at (DE.west) {$DE$};
    \node[anchor=east] at (E.west) {$E$};
    \node[anchor=east] at (F.west) {$F$};
    \node[anchor=east] at (C0.west) {$\bot$};
  \end{scope}
  \draw (0,.5) node {$L_2^\dag$:};

  \end{scope}

\end{tikzpicture}
\end{center}

A second application will purge $\langle K_1^\top, L_1^\top\rangle$ with respect to $\langle K_2^\dag, L_2^\dag\rangle$:

\begin{center}\footnotesize
\setlength{\tabcolsep}{2pt}
\setlength{\tabcolsep}{2pt}
\centering\begin{tikzpicture}[font=\footnotesize]

  \draw (-3,2) node {\begin{tabular}{r|ccccccc}
                        $K_1^\dag$ & $m_1$ & $m_2$ & $m_3$ & \rotatebox{90}{$\exists p.DEF$} & \rotatebox{90}{$\exists p.DE$} & \rotatebox{90}{$\exists p.E$} & \rotatebox{90}{$\exists p.F$}\\ \hline
                       $a$ &   & $\times$ &   & $\times$ & $\times$ & $\times$  &   \\
                       $b$ &   & $\times$ &   & $\times$ & $\times$ & $\times$ &   \\
                       $c$ & $\times$ &   & $\times$ & $\times$ &   &   & $\times$  
                       \end{tabular}};

  \begin{scope}[xshift=2cm]
    \begin{scope}[xscale=.25,yscale=.25]

    \begin{dot2tex}[dot,tikz,codeonly,options=-traw]
      graph {
	graph [nodesep=1.5]
	node [style=concept]
	ABC [label="$\exists p.DEF$
\nodepart{two}
$\empty$"]
	AB [label="$m_2, \exists p.E, \exists p.DE$
\nodepart{two}
$a$, $b$"]
	C [label="$m_1, m_3, \exists p.F$
\nodepart{two}
$c$"]
        C0 [label="$\empty$
\nodepart{two}
$\empty$"]
	ABC -- AB
	ABC -- C
	AB -- C0
        C -- C0
      }
    \end{dot2tex}
    \node[anchor=east] at (ABC.west) {$ABC$};
    \node[anchor=east] at (AB.west) {$AB$};
    \node[anchor=east] at (C.west) {$C$};
    \node[anchor=east] at (C0.west) {$\bot$};
  \end{scope}
  \draw (0,.5) node {$L_1^\dag$:};
  \end{scope}

\end{tikzpicture}
\end{center}

It can be checked that $\{\langle K_1^\dag, L_1^\dag\rangle, \langle K_2^\dag, L_2^\dag\rangle\}$ is a fixed point for $PQ^*$, but also for $EF^*$.
\end{example}
\bigskip

It may be interesting, for some applications to check if there is only one acceptable solution.
This can easily be characterised by:
\begin{proposition}\label{prop:lfgeqgfp}
  $\mathrm{lfp}(EF_{K^0,R,\Omega}^*)=\mathrm{gfp}(PQ_{K^0,R,\Omega}^*)$
 iff $|\mathrm{fp}(EF_{K^0,R,\Omega}^*)\cap\mathrm{fp}(PQ_{K^0,R,\Omega}^*)|=1$.
\end{proposition}
The proof of this proposition is given in the next section (\ref{sec:allfpsem}) as it relies on further results.

This can be tested using $\underline{\mathrm{RCA}}$ and $\overline{\mathrm{RCA}}$.

FCA can be described as RCA with $R=\varnothing$. 
In this case, $\forall\xx\in\XX$ $D_{\Omega,R_\xx,K^0}=\varnothing$.
Thus, $\mathcal{O}=\{\mathrm{T}^*( K^0 )\}=\{\langle K^0, \mathrm{FCA}(K^0)\rangle\}$ and $\mathrm{fp}(EF^*)=\mathrm{fp}(PQ^*)=\{\mathrm{T}^*( K^0 )\}$.
Hence,
\[
  \underline{\mathrm{RCA}}_{\Omega}(K^0,\varnothing)=\overline{\mathrm{RCA}}_{\Omega}(K^0,\varnothing)=\mathrm{FCA}(K^0).
\]

\subsection{The structure of fixed points}\label{sec:allfpsem}\label{sec:qeeq}

Besides obtaining the least fixed point of $EF^*$ ($\underline{\mathrm{RCA}}_{\Omega}$) or the greatest fixed point of $PQ^*$ ($\overline{\mathrm{RCA}}_{\Omega}$), 
an interesting problem is to obtain all acceptable solutions, i.e. those families of context-lattice pairs belonging to the fixed points of both functions ($\mathrm{fp}(EF^*)\cap\mathrm{fp}(PQ^*)$).

A naive algorithm for this consists in enumerating all elements of the interval and testing if they are fixed points.
This would not be very efficient.
Figure~\ref{fig:exrcalattice} shows that, in our simple Example~\ref{sec:exrca}, among the 16 elements in the interval only 4 belong to $\mathrm{fp}(EF^*)\cap\mathrm{fp}(PQ^*)$.

One way to try to improve on this situation is to understand the structure of the set of fixed points and its relation with the two functions and their closures.
Figure~\ref{fig:fpstruct} illustrates the structure of $\mathcal{O}$ and how $EF^{*\infty}$ and $PQ^{*\infty}$ and their composition traverse this structure.

An interesting property of the functions $EF^*$ and $PQ^*$ is that they preserve each other stability:
\begin{property}[$EF^*$ is internal to $\mathrm{fp}(PQ^*)$]\label{prop:EFsintfpPQs}
$\forall O\in \mathrm{fp}(PQ^*)$, $EF^*(O)\in\mathrm{fp}(PQ^*)$.
\end{property}
\begin{proof}
If $O\in\mathrm{fp}(PQ^*)$, all attributes in intents of $l(O)$ are supported by concepts in $l(O)$ (Property~\ref{prop:ssisfpPQs} and Definition~\ref{def:ssup}).
By Property~\ref{prop:EFsmonext}, $O\preceq EF^*(O)$, so these concepts are still in $l(EF^*(O))$.
Moreover, $EF^*$ only adds to $k(O)$ attributes which are supported by $l(O)$ (they only refer to concepts in $l(O)$).
Hence, the attributes in $k(EF^*(O))$ and those scaled by $\sigma_\Omega$ are still supported by $l(EF^*(O))$.
\end{proof}

\begin{property}[$PQ^*$ is internal to $\mathrm{fp}(EF^*)$]\label{prop:PQsintfpEFs}
$\forall O\in \mathrm{fp}(EF^*)$, $PQ^*(O)\in \mathrm{fp}(EF^*)$.
\end{property}
\begin{proof}
If $O\in\mathrm{fp}(EF^*)$, this means that $EF^*(O)=O$ and, in particular, that $\sigma^*_\Omega$ does not scale new relational attributes based on the concepts in $l(O)$.
By Property~\ref{prop:PQsantextmon}, $PQ^*(O)\preceq O$, so that $l(PQ^*(O))$ does not contain more concepts than $l(O)$, then $\sigma^*_\Omega$ cannot scale new attributes ($\sigma^*_\Omega(k(PQ^*(O)),R,l(PQ^*(O)))\subseteq \sigma^*_\Omega(k(O),R,l(O))=\varnothing$).
Hence, $PQ^*(O)\in\mathrm{fp}(EF^*)$.
\end{proof}

This shows that $\mathrm{fp}(EF^*)\cap\mathrm{fp}(PQ^*)\neq\varnothing$: acceptable solutions always exist.
In addition, the closure operations associated with the two functions preserve their extrema.

\begin{property}\label{prop:EFsPQsbounds}
$PQ^{*\infty}(\mathrm{gfp}(EF^*))=\mathrm{gfp}(PQ^*)$
  and
$EF^{*\infty}(\mathrm{lfp}(PQ^*))=\mathrm{lfp}(EF^*)$.
\end{property}
\begin{proof}
$\forall O\in\mathcal{O}$, $O\preceq \mathrm{gfp}(EF^*)$ (from Property~\ref{prop:order}), and $PQ^{*\infty}$ is order preserving (Property~\ref{prop:EFsPQsclos}), thus
$PQ^{*\infty}(O)\preceq PQ^{*\infty}(\mathrm{gfp}(EF^*))$.
Hence, $\forall O\in\mathrm{fp}(PQ^*)$, $O\preceq PQ^{*\infty}(\mathrm{gfp}(EF^*))$.
Moreover, $PQ^{*\infty}(\mathrm{gfp}(EF^*))\in\mathrm{fp}(PQ^*)$, thus $PQ^{*\infty}(\mathrm{gfp}(EF^*))=\mathrm{gfp}(PQ^*)$.

Similarly, $\forall O\in\mathcal{O}$, $\mathrm{lfp}(PQ^*)\preceq O$ (Property~\ref{prop:order}), and $EF^{*\infty}$ is order preserving (Property~\ref{prop:EFsPQsclos}), thus
$EF^{*\infty}(\mathrm{lfp}(PQ^*))$ $\preceq$ $EF^{*\infty}(O)$.
Hence, $\forall O\in\mathrm{fp}(EF^*)$, $EF^{*\infty}(\mathrm{lfp}(PQ^*))\preceq O$.
Moreover, $EF^{*\infty}(\mathrm{lfp}(PQ^*))\in\mathrm{fp}(EF^*)$, therefore $EF^{*\infty}(\mathrm{lfp}(PQ^*))$ $=$ $\mathrm{lfp}(EF^*)$.
\end{proof}

Proposition~\ref{prop:infsupEFsPQs} complements Property~\ref{prop:order} for elements of $\mathrm{fp}(EF^*)\cap\mathrm{fp}(PQ^*)$.
The elements of $\mathrm{fp}(EF^*)\cap\mathrm{fp}(PQ^*)$ thus belong to the interval 
$[\mathrm{lfp}(EF^*)\ \mathrm{gfp}(PQ^*)]$ (which is more restricted than $[\mathrm{lfp}(PQ^*)\ \mathrm{gfp}(EF^*)]$, see Figure~\ref{fig:fpstruct}).

\begin{proposition}\label{prop:infsupEFsPQs}
$\forall O\in\mathrm{fp}(EF^*)\cap\mathrm{fp}(PQ^*)$,
  \[\mathrm{lfp}(PQ^*)\preceq\mathrm{lfp}(EF^*)\preceq O\preceq\mathrm{gfp}(PQ^*)\preceq\mathrm{gfp}(EF^*). \]
\end{proposition}
\begin{proof}
By Property~\ref{prop:EFsPQsbounds}, $PQ^{*\infty}(\mathrm{gfp}(EF^*))=\mathrm{gfp}(PQ^*)$, but $PQ^*(O)\preceq O$ (Property~\ref{prop:PQsantextmon}) and $PQ^{*\infty}(O)\preceq PQ^*(O)$ (Property~\ref{prop:order}), therefore $\mathrm{gfp}(PQ^*)$ $=$ $PQ^{*\infty}(\mathrm{gfp}(EF^*))$ $\preceq$ $PQ^*(\mathrm{gfp}(EF^*))$ $\preceq$ $\mathrm{gfp}(EF^*)$.
Moreover, by Property~\ref{prop:EFsPQsbounds}, $EF^{*\infty}(\mathrm{lfp}(PQ^*))=\mathrm{lfp}(EF^*)$, but $O\preceq EF^*(O)$ (Property~\ref{prop:EFsmonext}) and $EF^*(O)\preceq EF^{*\infty}(O)$ (Property~\ref{prop:order}), thus $\mathrm{lfp}(PQ^*)\preceq  EF^{*}(\mathrm{lfp}(PQ^*))\preceq EF^{*\infty}(\mathrm{lfp}(PQ^*))=\mathrm{lfp}(EF^*)$.
Finally, $O\in \mathrm{fp}(EF^*)\cap\mathrm{fp}(PQ^*)$ entails that $\mathrm{lfp}(EF^*)\preceq O$ and $O\preceq\mathrm{gfp}(PQ^*)$.
\end{proof}

It is now possible to prove Proposition~\ref{prop:lfgeqgfp}:
\begin{proof}[Proof of Proposition~\ref{prop:lfgeqgfp}]
First, both $\mathrm{lfp}(EF_{K^0,R,\Omega}^*)$ and $\mathrm{gfp}(PQ_{K^0,R,\Omega}^*)$ are among the acceptable solutions.
Indeed, $\mathrm{lfp}(EF^*)=EF^{*\infty}(\mathrm{lfp}(PQ^*))$ (Property~\ref{prop:EFsPQsbounds}), but $\mathrm{lfp}(PQ^*)\in\mathrm{fp}(PQ^*)$ thus $EF^{*\infty}(\mathrm{lfp}(PQ^*))\in\mathrm{fp}(PQ^*)$  (Property~\ref{prop:EFsintfpPQs}), hence $\mathrm{lfp}(EF^*)\in\mathrm{fp}(PQ^*)$.
The same reasoning can be applied to $\mathrm{gfp}(PQ_{K^0,R,\Omega}^*)$ with Property~\ref{prop:PQsintfpEFs}.

\noindent $\Rightarrow$) Since all acceptable solutions are within the interval between both fixed points (Proposition~\ref{prop:infsupEFsPQs}), if these are equal then the interval contains only one object which is the only acceptable solution.

\noindent $\Leftarrow$) If there is only one acceptable solution, then $\mathrm{lfp}(EF_{K^0,R,\Omega}^*)=\mathrm{gfp}(PQ_{K^0,R,\Omega}^*)$.
\end{proof}

Proposition~\ref{prop:infsupEFsPQs} together with the preamble of the proof of Proposition~\ref{prop:lfgeqgfp} determine that $[\mathrm{lfp}(EF^*)\ \mathrm{gfp}(PQ^*)]$ is the smallest interval in which $\mathrm{fp}(EF^*)\cap\mathrm{fp}(PQ^*)$ is included since its bounds are acceptable solutions.
However, acceptable solutions do not cover the whole interval: the converse of Proposition~\ref{prop:infsupEFsPQs} does not hold in general as shown by the counter-example~\ref{ex:neq}.

\begin{example}[Non coverage in RCA]\label{ex:neq}
In the example of Section~\ref{sec:exrca}, 
$\mathrm{lfp}(EF^*)=\{\langle K_1^1, L_1^1\rangle, \langle K_2^1, L_2^1\rangle\}$ and
$\mathrm{gfp}(PQ^*)=\{\langle K_1^\star, L_1^\star\rangle, \langle K_2^\star, L_2^\star\rangle\}$.
The family $\{\langle K_1^\#, L_1^\#\rangle, \langle K_2^\#,$ $L_2^\#\rangle\}$ of Figure~\ref{fig:ex-2-noalt} belongs to $[\mathrm{lfp}(EF^*)\ \mathrm{gfp}(PQ^*)]$, but not to $\mathrm{fp}(EF^*)\cap\mathrm{fp}(PQ^*)$ as mentioned in Example~\ref{ex:selfsup}.
Figure~\ref{fig:exrcalattice} shows that 12 out of 16 elements of the interval are in this situation.
\end{example}

The layout of Figures~\ref{fig:fpstruct} and \ref{fig:conj} do not help understanding the situation, but Figure~\ref{fig:exrcalattice} illustrates the presence of non-acceptable objects within the interval.

This interval may be thought of as an approximation of the situation described by the initial context $K^0$.
For some purposes, this may be sufficient.
However, it may also be interesting to navigate within the set $\mathrm{fp}(EF^*)\cap\mathrm{fp}(PQ^*)$ of fixed points or to compute it.

In order to find the elements of $\mathrm{fp}(EF^*)\cap\mathrm{fp}(PQ^*)$, 
the closure of $EF^*$ and $PQ^*$, $EF^{*\infty}$ and $PQ^{*\infty}$, can be used as functions which maps elements of $\mathcal{O}$ 
into families of context-lattice pairs in $\mathrm{fp}(EF^*)$ and $\mathrm{fp}(PQ^*)$, respectively.
Moreover, Properties~\ref{prop:PQsintfpEFs} and \ref{prop:EFsintfpPQs} entail that $PQ^{*\infty}\circ EF^{*\infty}$ and $EF^{*\infty}\circ PQ^{*\infty}$ map any element of $\mathcal{O}$ into an acceptable family of context-lattice pairs in $\mathrm{fp}(EF^*)\cap\mathrm{fp}(PQ^*)$.
Hence, the set of acceptable solutions are those elements in the image of $\mathcal{O}$ by the composition of these two closure operations, in any order.

\begin{property}\label{prop:imPQsEFsfp}
$\mathrm{Im}(PQ^{*\infty}\circ EF^{*\infty}) = \mathrm{fp}(EF^*)\cap\mathrm{fp}(PQ^*) = \mathrm{Im}(EF^{*\infty}\circ PQ^{*\infty})$.
\end{property}
\begin{proof}
We show it for $PQ^{*\infty}\circ EF^{*\infty}$, the other part is dual:
\begin{description}
\item[$\subseteq$]
  By definition, $\mathrm{Im}(PQ^{*\infty}\circ EF^{*\infty})\subseteq\mathrm{Im}(PQ^{*\infty})=\mathrm{fp}(PQ^*)$.
  Moreover, $\mathrm{Im}(EF^{*\infty})=\mathrm{fp}(EF^*)$, but by Property~\ref{prop:PQsintfpEFs}, if $O\in\mathrm{fp}(EF^*)$, then $PQ^{*\infty}(O)\in\mathrm{fp}(EF^*)$.
  Hence, $\mathrm{Im}(PQ^{*\infty}\circ EF^{*\infty}) \subseteq \mathrm{fp}(EF^*)\cap\mathrm{fp}(PQ^*)$.
\item[$\supseteq$] $\forall O\in\mathrm{fp}(PQ^*)\cap\mathrm{fp}(EF^*)$,
  $O\in\mathrm{fp}(EF^*)$, thus $EF^{*\infty}(O)=O$ and $O\in\mathrm{fp}(PQ^*)$, thus $PQ^{*\infty}(O)=O$.
  Hence, $O=PQ^{*\infty}(EF^{*\infty}(O))=PQ^{*\infty}\circ EF^{*\infty}(O)\in \mathrm{Im}(PQ^{*\infty}\circ EF^{*\infty})$
  and consequently $\mathrm{fp}(EF^*)\cap\mathrm{fp}(PQ^*)\subseteq \mathrm{Im}(PQ^{*\infty}\circ EF^{*\infty})$.\qedhere
\end{description}
\end{proof}

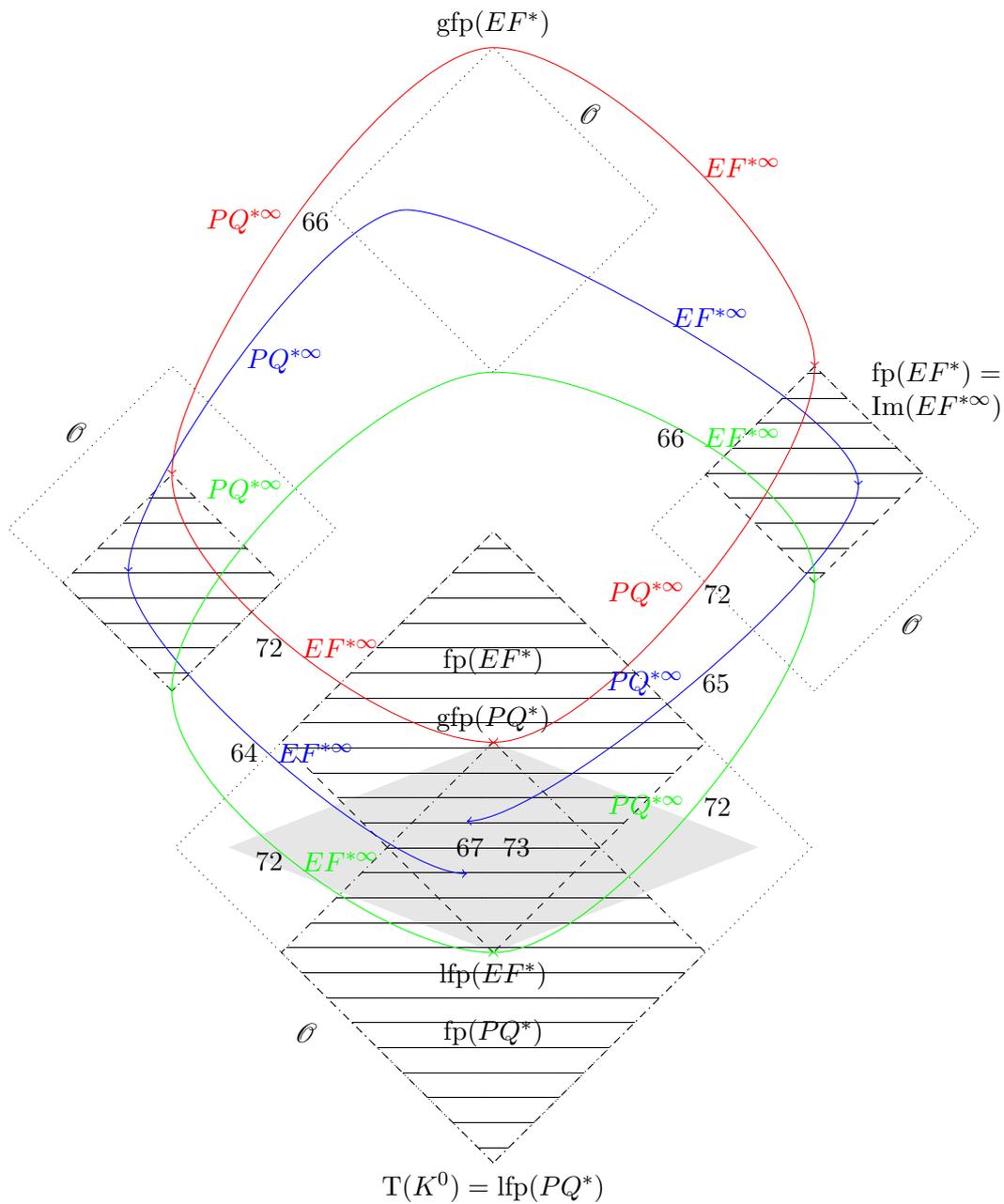
\begin{figure}
\centering
\begin{tikzpicture}[scale=1.5]

  \begin{scope}[scale=.72]
    \begin{scope}[rotate=45]
     \coordinate (lfpEF) at (-.75,-.75);
     \coordinate (fcak0) at (-1.5,-1.5);
     \coordinate (gfpEF) at (1.5,1.5);
     \coordinate (topEF) at (.75,.75);
     \coordinate (point) at (-0.8,0.8);
     \draw[dotted] (fcak0) rectangle (gfpEF);
   \end{scope}

   \draw (gfpEF) {} node[anchor=south] {$\mathrm{gfp}(EF^*)$};

   \draw (1.25,1.25) node {$\mathcal{O}$};

 \end{scope}

  \begin{scope}[xshift=3cm,yshift=-3cm,scale=.72]
    \begin{scope}[rotate=45]
     \coordinate (lfpEFr) at (-.5,-.5);
     \coordinate (fcak0r) at (-1.5,-1.5);
     \coordinate (gfpEFr) at (1.5,1.5);
     \coordinate (topEFr) at (.5,.5);
     \coordinate (EFx) at (0.8,0);
     \draw[dotted] (fcak0r) rectangle (gfpEFr);
     \fill[pattern={Lines[angle=90,distance=10pt]}] (lfpEFr) rectangle (gfpEFr);
     \draw[dashed] (lfpEFr) rectangle (gfpEFr);
   \end{scope}

   \draw (1.2,1.8) node[text width=1cm] {$\mathrm{fp}(EF^*)=$ $\mathrm{Im}(EF^{*\infty})$}; 

   \draw (1.25,-1.25) node {$\mathcal{O}$};

 \end{scope}

  \begin{scope}[xshift=-3cm,yshift=-3cm,scale=.72]
    \begin{scope}[rotate=45]
     \coordinate (lfpEFl) at (-.5,-.5);
     \coordinate (fcak0l) at (-1.5,-1.5);
     \coordinate (gfpEFl) at (1.5,1.5);
     \coordinate (topEFl) at (.5,.5);
     \coordinate (PQx) at (-0.8,0);
     \draw[dotted] (fcak0l) rectangle (gfpEFl);
     \fill[pattern={Lines[distance=10pt]}] (fcak0l) rectangle (topEFl);
     \draw[dashdotted] (fcak0l) rectangle (topEFl);
   \end{scope}

   \draw (-1.25,1.25) node {$\mathcal{O}$};

 \end{scope}

  \begin{scope}[yshift=-6cm,scale=1.4]
    \begin{scope}[rotate=45]
     \coordinate (lfpEFb) at (-.5,-.5);
     \coordinate (fcak0b) at (-1.5,-1.5);
     \coordinate (gfpEFb) at (1.5,1.5);
     \coordinate (topEFb) at (.5,.5);
     \coordinate (EFPQx) at (-.25,0);
     \coordinate (PQEFx) at (0,.25);
     \coordinate (EFPQ) at (0,0);
     \draw[dotted] (fcak0b) rectangle (gfpEFb);
     \fill[gray!20] (topEFb) -- (1.25,-1.25) -- (lfpEFb) -- (-1.25,1.25) -- cycle;
     \draw[dashed] (lfpEFb) rectangle (gfpEFb);
     \draw[dashdotted] (fcak0b) rectangle (topEFb);
     \fill[pattern={Lines[distance=10pt]}] (fcak0b) rectangle (topEFb);
     \fill[pattern={Lines[angle=90,distance=10pt]}] (lfpEFb) rectangle (gfpEFb);
   \end{scope}

   \draw (fcak0b) node[anchor=north] {$\mathrm{T}^*(K^0)=\mathrm{lfp}(PQ^*)$};
   \draw (lfpEFb) {} node[anchor=north] {$\mathrm{lfp}(EF^*)$}; 
   \draw (0,1.25) node {$\mathrm{fp}(EF^*)$}; 
   \draw (0,-1.25) node {$\mathrm{fp}(PQ^*)$};
   \draw (topEFb) node[anchor=south] {$\mathrm{gfp}(PQ^*)$}; 

   \draw (-1.25,-1.25) node {$\mathcal{O}$};

 \end{scope}

 \draw[->,red] (gfpEF) .. controls +(1,0) and +(0,1) .. node[right] {$EF^{*\infty}$} (gfpEFr);
 \draw[->,red] (gfpEF) .. controls +(-1,0) and +(0,1) .. node[left] {$PQ^{*\infty}$} node[right] {\ref{prop:EFsPQsbounds}} (topEFl);
 \draw[->,red] (topEFl) .. controls +(0,-1) and +(-1,0) .. node[right] {$EF^{*\infty}$} node[left] {\ref{prop:PQsEFseqEFsPQsbt}} (topEFb);
 \draw[->,red] (gfpEFr) .. controls +(0,-1) and +(1,0) .. node[left] {$PQ^{*\infty}$} node[right] {\ref{prop:PQsEFseqEFsPQsbt}} (topEFb);
 
   \draw[->,blue] (point) .. controls +(.75,0) and +(0,.75) .. node[right] {$EF^{*\infty}$} (EFx);
   \draw[->,blue] (point) .. controls +(-.75,0) and +(0,.75) .. node[right] {$PQ^{*\infty}$} (PQx);
   \draw[->,blue] (PQx) .. controls +(0,-.75) and +(-.75,0) .. node[right] {$EF^{*\infty}$} node[left] {\ref{prop:EFsintfpPQs}} (EFPQx);
   \draw[->,blue] (EFx) .. controls +(0,-.75) and +(.75,0) .. node[left] {$PQ^{*\infty}$} node[right] {\ref{prop:PQsintfpEFs}} (PQEFx);

 \draw[->,green] (fcak0) .. controls +(1,0) and +(0,1) .. node[right] {$EF^{*\infty}$} node[left] {\ref{prop:EFsPQsbounds}} (lfpEFr);
 \draw[->,green] (fcak0) .. controls +(-1,0) and +(0,1) .. node[left] {$PQ^{*\infty}$} (fcak0l);
 \draw[->,green] (fcak0l) .. controls +(0,-1) and +(-1,0) .. node[right] {$EF^{*\infty}$} node[left] {\ref{prop:PQsEFseqEFsPQsbt}} (lfpEFb);
 \draw[->,green] (lfpEFr) .. controls +(0,-1) and +(1,0) .. node[left] {$PQ^{*\infty}$} node[right] {\ref{prop:PQsEFseqEFsPQsbt}} (lfpEFb);

 \draw (EFPQ) node[right] {\ref{prop:EFsPQsinfPQsEFs}};
 \draw (EFPQ) node[left] {\ref{prop:infsupEFsPQs}};
   
\end{tikzpicture}
\caption{Illustration of Properties~\ref{prop:EFsintfpPQs}, \ref{prop:PQsintfpEFs}, \ref{prop:EFsPQsbounds}, \ref{prop:infsupEFsPQs}, \ref{prop:imPQsEFsfp}, \ref{prop:PQsEFseqEFsPQsbt} and \ref{prop:EFsPQsinfPQsEFs}. The figure displays four times $\mathcal{O}$ and the images of $\mathrm{gfp}(EF^*)$ (red), a random family of context-lattice pairs (blue) and $\mathrm{lfp}(PQ^*)=\mathrm{T}^*(K^0)$ (green) through $PQ^{*\infty}$ (left) and $EF^{*\infty}$ (right). $\mathrm{fp}(EF^*)$ is drawn in vertical lines; $\mathrm{fp}(PQ^*)$ in horizontal lines and the grey area depicts the interval $[\mathrm{lfp}(EF^*)\ \mathrm{gfp}(PQ^*)]$.}\label{fig:fpstruct}
\end{figure}

In addition, these functions are monotone and idempotent.

\begin{property}\label{prop:PQsEFsmonid}
  $PQ^{*\infty}\circ EF^{*\infty}$ (resp. $EF^{*\infty}\circ PQ^{*\infty}$) is order-preserving and idempotent:
  \begin{align*}
    \forall O, O'\in\mathcal{O}, O\preceq O' &\Rightarrow (PQ^{*\infty}\circ EF^{*\infty})(O)\preceq (PQ^{*\infty}\circ EF^{*\infty})(O'),\tag{monotony}\label{eq:PQsEFsmon}\\
    (PQ^{*\infty}\circ EF^{*\infty})\circ (PQ^{*\infty}\circ EF^{*\infty})(O) &= PQ^{*\infty}\circ EF^{*\infty}(O).\tag{idempotence}\label{eq:PQsEFsid}
  \end{align*}
\end{property}
\begin{proof}
  We prove it for $PQ^{*\infty}\circ EF^{*\infty}$, the $EF^{*\infty}\circ PQ^{*\infty}$ case is strictly dual.
\begin{description}
\item[\ref{eq:PQsEFsmon}] is obtained as the combination of order-preservation of the two functions: $O\preceq O'$, hence $EF^{*\infty}(O)\preceq EF^{*\infty}(O')$, and thus $PQ^{*\infty}\circ EF^{*\infty}(O)\preceq PQ^{*\infty}\circ EF^{*\infty}(O')$ (applying Property~\ref{prop:EFsPQsclos} twice).
\item[\ref{eq:PQsEFsid}] is obtained from Property~\ref{prop:imPQsEFsfp}: $\forall O\in\mathcal{O}$, $PQ^{*\infty}\circ EF^{*\infty}(O)\in\mathrm{fp}(EF^*)\cap\mathrm{fp}(PQ^*)$, hence $PQ^{*\infty}\circ EF^{*\infty}(O)=O$ and $PQ^{*\infty}\circ EF^{*\infty}\circ PQ^{*\infty}\circ EF^{*\infty}(O)=O$, thus $PQ^{*\infty}\circ EF^{*\infty}\circ PQ^{*\infty}\circ EF^{*\infty}(O) = PQ^{*\infty}\circ EF^{*\infty}(O)$.\qedhere
\end{description}
\end{proof}

The monotony of these functions entails that $\mathrm{fp}(EF^*)\cap\mathrm{fp}(PQ^*)$ is a complete lattice:
\begin{proposition}\label{prop:imPQsEFscompl}
$\langle\mathrm{fp}(EF^*)\cap\mathrm{fp}(PQ^*), \preceq\rangle$ is a complete sublattice of $\langle\mathcal{O}, \preceq\rangle$.
\end{proposition}
\begin{proof}
$\mathrm{fp}(EF^*)\cap\mathrm{fp}(PQ^*)=\mathrm{Im}(PQ^{*\infty}\circ EF^{*\infty})$ (Property~\ref{prop:imPQsEFsfp}) and $\mathrm{Im}(PQ^{*\infty}\circ EF^{*\infty})=\mathrm{fp}(PQ^{*\infty}\circ EF^{*\infty})$ due to idempotence (Property~\ref{prop:PQsEFsmonid}), hence the Knaster-Tarski theorem can be applied based on Property~\ref{prop:PQsEFsmonid} (monotony), concluding that it is a complete lattice.
It is included in $\mathcal{O}$, thus this is a sublattice of $\langle\mathcal{O}, \preceq\rangle$.
\end{proof}
This is illustrated by Example~\ref{ex:interval}.

\begin{example}[Interval lattice]\label{ex:interval}
Figure~\ref{fig:exrcalattice} shows all elements of $[\mathrm{lfp}(EF^*)\ \mathrm{gfp}(PQ^*)]$ for the example of Section~\ref{sec:exrca}.
It can be observed that $\langle\mathrm{fp}(EF^*)\cap\mathrm{fp}(PQ^*), \preceq\rangle$ is a proper sublattice of $\mathcal{O}$.
Actually only 4 out of 16 possible objects in the interval are acceptable.

In the figure, direct edges corresponding to $EF^*$ or $PQ^*$, from lattice pairs of level 2 and 4, are drawn in solid or dashed, respectively.
All the objects, of level 3, that are not comparable with the two intermediate fixed points map to the extrema of the interval and thus $EF^*$ are $PQ^*$ are not displayed.
\end{example}

However, the functions $PQ^{*\infty}\circ EF^{*\infty}$ and $PQ^{*\infty}\circ EF^{*\infty}$ are not necessarily extensive, nor anti-extensive (see Figure~\ref{fig:conj} and Example~\ref{ex:notEFPQeqPQEF}, p.~\pageref{ex:notEFPQeqPQEF}).
Hence, they would not be closure operations.

\begin{figure}
\tikzstyle{object} = [rounded corners,dotted,draw]
\tikzstyle{admobject} = [object,solid]
  \centering
  \begin{tikzpicture}[xscale=.6,font=\footnotesize]

    \begin{scope}[very thin]
        \draw (0,2) node (AB) {$AB$};
        \draw (-1,1) node (A) {$A$};
        \draw (1,1) node (B) {$B$};
        \draw (0,0) node (bot) {$\bot$};

        \draw (AB) -- (A);
        \draw (AB) -- (B);
        \draw (A) -- (bot);
        \draw (B) -- (bot);

        \draw (1.5,0) node {\footnotesize $L_3^\star, L_4^\star$};

      \begin{scope}[xshift=3cm]
        \draw (0,2) node (CD) {$CD$};
        \draw (-1,1) node (C) {$C$};
        \draw (1,1) node (D) {$D$};
        \draw (0,0) node (bot) {$\bot$};

        \draw (CD) -- (C);
        \draw (CD) -- (D);
        \draw (C) -- (bot);
        \draw (D) -- (bot);
      \end{scope}
      \node [admobject,fit=(AB) (CD) (bot) (A) (D)] (gfp) {};
    \end{scope}

    \begin{scope}[yshift=-4cm]
    \begin{scope}[very thin,xshift=-9cm]
        \draw (0,2) node (AB) {$AB$};
        \draw (-1,1) node (A) {$A$};
        \draw (1,1) node (B) {$B$};
        \draw (0,0) node (bot) {$\bot$};

        \draw (AB) -- (A);
        \draw (AB) -- (B);
        \draw (A) -- (bot);
        \draw (B) -- (bot);

      \begin{scope}[xshift=2cm]
        \draw (0,2) node (CD) {$CD$};
        \draw (0,1) node (C) {$C$};

        \draw (CD) -- (C);
      \end{scope}
      \node [object,fit=(AB) (CD) (bot) (A)] (l11) {};
    \end{scope}

    \begin{scope}[very thin,xshift=-3cm]
        \draw (0,2) node (AB) {$AB$};
        \draw (-1,1) node (A) {$A$};
        \draw (1,1) node (B) {$B$};
        \draw (0,0) node (bot) {$\bot$};

        \draw (AB) -- (A);
        \draw (AB) -- (B);
        \draw (A) -- (bot);
        \draw (B) -- (bot);

      \begin{scope}[xshift=2cm]
        \draw (0,2) node (CD) {$CD$};
        \draw (0,1) node (D) {$D$};

        \draw (CD) -- (D);
      \end{scope}
      \node [object,fit=(AB) (CD) (bot) (A)] (l12) {};
    \end{scope}

    \begin{scope}[very thin,xshift=3cm]
        \draw (0,2) node (AB) {$AB$};
        \draw (0,1) node (A) {$A$};

        \draw (AB) -- (A);

      \begin{scope}[xshift=2cm]
        \draw (0,2) node (CD) {$CD$};
        \draw (-1,1) node (C) {$C$};
        \draw (1,1) node (D) {$D$};
        \draw (0,0) node (bot) {$\bot$};

        \draw (CD) -- (C);
        \draw (CD) -- (D);
        \draw (C) -- (bot);
        \draw (D) -- (bot);
      \end{scope}
      \node [object,fit=(AB) (CD) (bot) (D)] (l13) {};
    \end{scope}

    \begin{scope}[very thin,xshift=9cm]
        \draw (0,2) node (AB) {$AB$};
        \draw (0,1) node (B) {$B$};

        \draw (AB) -- (B);

      \begin{scope}[xshift=2cm]
        \draw (0,2) node (CD) {$CD$};
        \draw (-1,1) node (C) {$C$};
        \draw (1,1) node (D) {$D$};
        \draw (0,0) node (bot) {$\bot$};

        \draw (CD) -- (C);
        \draw (CD) -- (D);
        \draw (C) -- (bot);
        \draw (D) -- (bot);
      \end{scope}
      \node [object,fit=(AB) (CD) (bot) (D)] (l14) {};
    \end{scope}
    \end{scope}

    \begin{scope}[yshift=-8cm]
    \begin{scope}[very thin,xshift=-5.5cm]
        \draw (0,2) node (AB) {$AB$};
        \draw (-1,1) node (A) {$A$};
        \draw (1,1) node (B) {$B$};
        \draw (0,0) node (bot) {$\bot$};

        \draw (AB) -- (A);
        \draw (AB) -- (B);
        \draw (A) -- (bot);
        \draw (B) -- (bot);

      \begin{scope}[xshift=1.5cm]
        \draw (0,2) node (CD) {$CD$};
      \end{scope}
      \node [object,fit=(AB) (CD) (A) (bot)] (l21) {};
    \end{scope}

    \begin{scope}[very thin,xshift=5.5cm]
        \draw (0,2) node (AB) {$AB$};

      \begin{scope}[xshift=1.5cm]
        \draw (0,2) node (CD) {$CD$};
        \draw (-1,1) node (C) {$C$};
        \draw (1,1) node (D) {$D$};
        \draw (0,0) node (bot) {$\bot$};

        \draw (CD) -- (C);
        \draw (CD) -- (D);
        \draw (C) -- (bot);
        \draw (D) -- (bot);
      \end{scope}
      \node [object,fit=(AB) (CD) (D) (bot)] (l22) {};
    \end{scope}

    \end{scope}

    \begin{scope}[yshift=-8.5cm]
    \begin{scope}[very thin,xshift=-10cm,yshift=.25cm]
        \draw (0,2) node (AB) {$AB$};
        \draw (0,1) node (A) {$A$};

        \draw (AB) -- (A);

      \begin{scope}[xshift=1.5cm]
        \draw (0,2) node (CD) {$CD$};
        \draw (0,1) node (C) {$C$};

        \draw (CD) -- (C);
      \end{scope}
      \draw (.75,.5) node (O0) {\footnotesize $L_3', L_4'$};
      \node [admobject,fit=(AB) (CD) (O0) (A) (C)] (fp1) {};
    \end{scope}

    \begin{scope}[very thin,xshift=-2cm,yshift=.25cm]
        \draw (0,2) node (AB) {$AB$};
        \draw (0,1) node (A) {$A$};

        \draw (AB) -- (A);

      \begin{scope}[xshift=1.5cm]
        \draw (0,2) node (CD) {$CD$};
        \draw (0,1) node (D) {$D$};

        \draw (CD) -- (D);
      \end{scope}
      \draw (.75,.5) node (O0) {\footnotesize $L_3^\#, L_4^\#$};
      \node [object,fit=(AB) (CD) (O0)] (l23) {};
    \end{scope}

    \begin{scope}[very thin,xshift=2cm]
        \draw (0,2) node (AB) {$AB$};
        \draw (0,1) node (B) {$B$};

        \draw (AB) -- (B);

      \begin{scope}[xshift=1.5cm]
        \draw (0,2) node (CD) {$CD$};
        \draw (0,1) node (C) {$C$};

        \draw (CD) -- (C);
      \end{scope}
      \node [object,fit=(AB) (CD) (C) (B)] (l24) {};
    \end{scope}

    \begin{scope}[very thin,xshift=10cm,yshift=.25cm]
        \draw (0,2) node (AB) {$AB$};
        \draw (0,1) node (B) {$B$};

        \draw (AB) -- (B);

      \begin{scope}[xshift=1.5cm]
        \draw (0,2) node (CD) {$CD$};
        \draw (0,1) node (D) {$D$};

        \draw (CD) -- (D);
      \end{scope}
      \draw (.75,.5) node (O0) {\footnotesize $L_3'', L_4''$};
      \node [admobject,fit=(AB) (CD) (O0) (B) (D)] (fp2) {};
    \end{scope}
    \end{scope}

    \begin{scope}[yshift=-12cm]
    \begin{scope}[very thin,xshift=-9cm]
        \draw (0,2) node (AB) {$AB$};
        \draw (0,1) node (A) {$A$};

        \draw (AB) -- (A);

      \begin{scope}[xshift=1.5cm]
        \draw (0,2) node (CD) {$CD$};
      \end{scope}
      \node [object,fit=(AB) (CD) (A)] (l41) {};
    \end{scope}

    \begin{scope}[very thin,xshift=-3cm]
        \draw (0,2) node (AB) {$AB$};
        \draw (0,1) node (B) {$B$};

        \draw (AB) -- (B);

      \begin{scope}[xshift=1.5cm]
        \draw (0,2) node (CD) {$CD$};
      \end{scope}
      \node [object,fit=(AB) (CD) (B)] (l42) {};
    \end{scope}

    \begin{scope}[very thin,xshift=3cm]
        \draw (0,2) node (AB) {$AB$};

      \begin{scope}[xshift=1.5cm]
        \draw (0,2) node (CD) {$CD$};
        \draw (0,1) node (C) {$C$};

        \draw (CD) -- (C);
      \end{scope}
      \node [object,fit=(AB) (CD) (C)] (l43) {};
    \end{scope}

    \begin{scope}[very thin,xshift=9cm]
        \draw (0,2) node (AB) {$AB$};

      \begin{scope}[xshift=1.5cm]
        \draw (0,2) node (CD) {$CD$};
        \draw (0,1) node (D) {$D$};
        \draw (CD) -- (D);
      \end{scope}
    \end{scope}
      \node [object,fit=(AB) (CD) (D)] (l44) {};
    \end{scope}

    \begin{scope}[yshift=-13cm]
    \begin{scope}[very thin]
      \draw (0,0) node (AB) {$AB$};

      \begin{scope}[xshift=1.5cm]
        \draw (0,0) node (CD) {$CD$};
      \end{scope}
      \draw (.75,-.5) node (O0) {\footnotesize $L_3^1, L_4^1$};

      \node [admobject,fit=(AB) (CD) (O0)] (lfp) {};
    \end{scope}
    \end{scope}

    \begin{scope}[dotted]
    \draw[solid,<-] (gfp.south) -- (l11.north);
    \draw[solid,<-] (gfp.south) -- (l12.north);
    \draw[solid,<-] (gfp.south) -- (l13.north);
    \draw[solid,<-] (gfp.south) -- (l14.north);

    \draw[dashed,<-] (fp1.north) -- (l11.south);
    \draw[dashed,<-] (fp1.north) -- (l13.south);
    \draw (l21.north) -- (l11.south);
    \draw (l21.north) -- (l12.south);
    \draw (l22.north) -- (l11.south);
    \draw (l23.north) -- (l12.south);
    \draw (l23.north) -- (l13.south);
    \draw (l22.north) -- (l14.south);
    \draw (l24.north) -- (l14.south);
    \draw (l24.north) -- (l13.south);
    \draw[dashed,<-] (fp2.north) -- (l14.south);
    \draw[dashed,<-] (fp2.north) -- (l12.south);

    \draw[solid,<-] (fp1.south) -- (l41.north);
    \draw[solid,<-] (fp1.south) -- (l43.north);
    \draw (l21.south) -- (l41.north);
    \draw (l21.south) -- (l44.north);
    \draw (l22.south) -- (l41.north);
    \draw (l23.south) -- (l44.north);
    \draw (l23.south) -- (l43.north);
    \draw (l22.south) -- (l42.north);
    \draw (l24.south) -- (l42.north);
    \draw (l24.south) -- (l43.north);
    \draw[solid,<-] (fp2.south) -- (l42.north);
    \draw[solid,<-] (fp2.south) -- (l44.north);

    \draw[dashed,<-] (lfp.north) -- (l41.south);
    \draw[dashed,<-] (lfp.north) -- (l42.south);
    \draw[dashed,<-] (lfp.north) -- (l43.south);
    \draw[dashed,<-] (lfp.north) -- (l44.south);
  \end{scope}
  
\end{tikzpicture}
\caption{All the families of lattices belonging to $[\mathrm{lfp}(EF^*)\ \mathrm{gfp}(PQ^*)]$ in the example of Section~\ref{sec:exrca}. 
Those in $\mathrm{fp}(EF^*)\cap\mathrm{fp}(PQ^*)$ are within solid boxes.
As usual, only direct edges are displayed.
Solid arrows show direct application of $EF^*$ and dashed arrows show direct application of $PQ^*$.
Dotted arrows are order relations not corresponding to $EF^*$ or $PQ^*$ applications.}\label{fig:exrcalattice}
\end{figure}

For any family of context-lattice pairs within the fixed points, i.e. $\mathrm{fp}(EF^*)$ \textit{or} $\mathrm{fp}(PQ^*)$ (the vertically or horizontally stripped area of Figure~\ref{fig:conj}), the two functions are equal.

\begin{property}\label{prop:PQsEFseqEFsPQsonUnion}
$\forall O\in\mathrm{fp}(EF^*)\cup\mathrm{fp}(PQ^*), PQ^{*\infty}\circ EF^{*\infty}(O) = EF^{*\infty}\circ PQ^{*\infty}(O)$.
\end{property}
\begin{proof}
For any lattice $O$ belonging to $\mathrm{fp}(EF^*)\cap\mathrm{fp}(PQ^*)$,
$PQ^*(O)=EF^*(O)=O$, hence $PQ^{*\infty}\circ EF^{*\infty}(O) = EF^{*\infty}\circ PQ^{*\infty}(O)=O$.
Similarly, for any lattice $O$ belonging to $\mathrm{fp}(EF^*)$, then $EF^{*\infty}(O)=O$, so $PQ^{*\infty}\circ EF^{*\infty}(O)=PQ^{*\infty}(O)$.
However, by Property~\ref{prop:EFsintfpPQs}, since $O\in\mathrm{fp}(EF^*)$, $PQ^{*\infty}(O)\in\mathrm{fp}(EF^*)$.
This means that $EF^{*\infty}\circ PQ^{*\infty}(O)=PQ^{*\infty}(O)$ as well.
The same can be proved for $O\in\mathrm{fp}(PQ^*)$ with Property~\ref{prop:PQsintfpEFs}.
\end{proof}

What is actually shown by the proof of Property~\ref{prop:PQsEFseqEFsPQsonUnion} is that:
\begin{align*}
\text{if } O\in\mathrm{fp}(EF^*) &\text{ then } PQ^{*\infty}\circ EF^{*\infty}(O) = EF^{*\infty}\circ PQ^{*\infty}(O)=PQ^{*\infty}(O),\\
\text{if } O\in\mathrm{fp}(PQ^*) &\text{ then } PQ^{*\infty}\circ EF^{*\infty}(O) = EF^{*\infty}\circ PQ^{*\infty}(O)=EF^{*\infty}(O).
\end{align*}

In particular, this applies to the bounds of $\mathrm{fp}(EF^*)\cup\mathrm{fp}(PQ^*)$:
\begin{property}\label{prop:PQsEFseqEFsPQsbt}
\begin{align*}
  EF^{*\infty}\circ PQ^{*\infty}(\mathrm{gfp}(EF^*)) &= PQ^{*\infty}\circ EF^{*\infty}(\mathrm{gfp}(EF^*)) = \mathrm{gfp}(PQ^*),\\
\intertext{and}
  PQ^{*\infty}\circ EF^{*\infty}(\mathrm{lfp}(PQ^*)) &= EF^{*\infty}\circ PQ^{*\infty}(\mathrm{lfp}(PQ^*)) = \mathrm{lfp}(EF^*).
\end{align*}
\end{property}
\begin{proof}
The first part of these equations are consequences of Property~\ref{prop:PQsEFseqEFsPQsonUnion}, since $\mathrm{gfp}(EF^*)$ and $\mathrm{lfp}(PQ^*)$ belong to $\mathrm{fp}(EF^*)$ and $\mathrm{fp}(PQ^*)$, respectively. 
The second part is due to $PQ^{*\infty}\circ EF^{*\infty}(\mathrm{gfp}(EF^*)) = PQ^{*\infty}(\mathrm{gfp}(EF^*))$ and $EF^{*\infty}\circ PQ^{*\infty}(\mathrm{lfp}(PQ^*))$ $=$ $EF^{*\infty}(\mathrm{lfp}(PQ^*))$ for the same reason that $\mathrm{gfp}(EF^*)\in \mathrm{fp}(EF^*)$ and $\mathrm{lfp}(PQ^*)\in \mathrm{fp}(PQ^*)$, respectively.
Property~\ref{prop:EFsPQsbounds} shows that the second terms correspond to $\mathrm{gfp}(PQ^*)$ and $\mathrm{lfp}(EF^*)$, respectively.
\end{proof}

Examples~\ref{ex:notEFPQeqPQEFrca0} and \ref{ex:notEFPQeqPQEF} show that $EF^{*\infty}\circ PQ^{*\infty}(O^\#) \prec O^\#\prec PQ^{*\infty}\circ EF^{*\infty}(O^\#)$ hence that the equality does not hold in general.

\begin{example}[Counterexample to equality in RCA$^0$]\label{ex:notEFPQeqPQEFrca0}
Consider $O_0^\#=\langle K_0^\#,L_0^\#\rangle$ of Figure~\ref{fig:ex-rca0-noalt} (p.~\pageref{fig:ex-rca0-noalt}) in the context of Example~\ref{sec:elabexrca0} (p.~\pageref{sec:elabexrca0}).
$O_0^\#\in\mathcal{O}$ because all attributes belong to $M_0^0\cup D_{\Omega,R,K^0}$ and $L_0^\#=\mathrm{FCA}^*(K_0^\#)$.

In fact, $O_0^\#\in[\mathrm{lfp}(EF^*)\ \mathrm{gfp}(PQ^*)]=[\langle K_0^1, L_0^1\rangle\ \langle K_0^\star, L_0^\star\rangle]$,
but $O_0^\#\not\in\mathrm{fp}(EF^*)\cap\mathrm{fp}(PQ^*)$ as explained in Example~\ref{ex:neq0}.

Indeed, applying $EF^{*\infty}$ returns $EF^{*\infty}(O_0^\#)$:
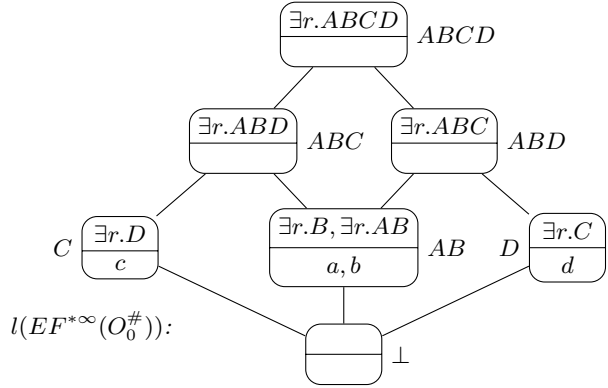
\begin{center}\footnotesize
\begin{tabular}{cc}
\begin{minipage}{.35\textwidth}
\setlength{\tabcolsep}{2pt}
\begin{center}
\begin{tabular}{r|ccccccc}
$k(EF^{*\infty}(O_0^\#))$    & \rotatebox{90}{$\exists r.B$} & \rotatebox{90}{$\exists r.C$} & \rotatebox{90}{$\exists r.D$} & \rotatebox{90}{$\exists r.AB$} & \rotatebox{90}{$\exists r.ABC$} & \rotatebox{90}{$\exists r.ABD$}& \rotatebox{90}{$\exists r.ABCD$}\\\hline
$a$ & $\times$ &          &          & $\times$ & $\times$ & $\times$ & $\times$        \\
$b$ & $\times$ &          &          & $\times$ & $\times$ & $\times$ & $\times$        \\
$c$ &          &          &          &          &          & $\times$ & $\times$        \\
$d$ &          & $\times$ & $\times$ &          & $\times$ &          & $\times$ \\
\end{tabular}
\end{center}
\end{minipage} &
\begin{minipage}{.6\textwidth}
\begin{center}
\begin{tikzpicture}[xscale=.3,yscale=.3]
    \begin{dot2tex}[dot,tikz,codeonly,options=-traw]
      graph {
	graph [nodesep=1.5]
	node [style=concept]
	ABCD [label="$\exists r.ABCD$
\nodepart{two}
$\empty$"]

	ABC [label="$\exists r.ABD$
\nodepart{two}
$\empty$"]

	ABD [label="$\exists r.ABC$
\nodepart{two}
$\empty$"]

	D [label="$\exists r.C$
\nodepart{two}
$d$"]

	AB [label="$\exists r.B, \exists r.AB$
\nodepart{two}
$a, b$"]

	C [label="$\exists r.D$
\nodepart{two}
$c$"]

        C0 [label="$\empty$
\nodepart{two}
$\empty$"]
        ABCD -- ABC
        ABCD -- ABD
        ABC -- C
        ABC -- AB
        ABD -- AB
        ABD -- D
        C -- C0
        AB -- C0
        D -- C0
      }
    \end{dot2tex}
    \node[anchor=west] at (ABCD.east) {$ABCD$};
    \node[anchor=west] at (ABC.east) {$ABC$};
    \node[anchor=west] at (ABD.east) {$ABD$};
    \node[anchor=west] at (AB.east) {$AB$};
    \node[anchor=east] at (C.west) {$C$};
    \node[anchor=east] at (D.west) {$D$};
    \node[anchor=west] at (C0.east) {$\bot$};
    \draw (1,3) node {$l(EF^{*\infty}(O_0^\#))$:};
\end{tikzpicture}
\end{center}
\end{minipage}
\end{tabular}
\end{center}

Applying $PQ^{*\infty}$ returns $PQ^{*\infty}(O_0^\#)$:
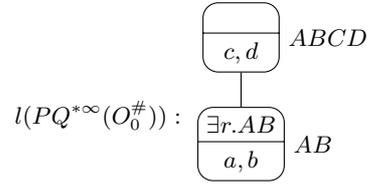
\begin{center}\footnotesize
\begin{tabular}{cc}
\begin{minipage}{.4\textwidth}
\setlength{\tabcolsep}{2pt}
\begin{center}
\begin{tabular}{r|c}
$k(PQ^{*\infty}(O_0^\#))$    & \rotatebox{90}{$\exists r.AB$} \\\hline
$a$ & $\times$ \\
$b$ & $\times$ \\
$c$ &          \\
$d$ &          \\
\end{tabular}
\end{center}
\end{minipage} &
\begin{minipage}{.6\textwidth}
\begin{center}
\begin{tikzpicture}[xscale=.3,yscale=.3]
    \begin{dot2tex}[dot,tikz,codeonly,options=-traw]
      graph {
	graph [nodesep=1.5]
	node [style=concept]
	ABCD [label="$\empty$
\nodepart{two}
$c, d$"]

	AB [label="$\exists r.AB$
\nodepart{two}
$a, b$"]
        ABCD -- AB
      }
    \end{dot2tex}
    \node[anchor=west] at (ABCD.east) {$ABCD$};
    \node[anchor=west] at (AB.east) {$AB$};
    \draw (-4,3) node {$l(PQ^{*\infty}(O_0^\#)):$};
\end{tikzpicture}
\end{center}
\end{minipage}
\end{tabular}
\end{center}

None of $EF^{*\infty}(O_0^\#)$ nor $PQ^{*\infty}(O_0^\#)$ is a fixed point for $PQ^*$ and $EF^*$, respectively.
Indeed, $PQ^{*\infty}\circ EF^{*\infty}(O_0^\#)$ is:
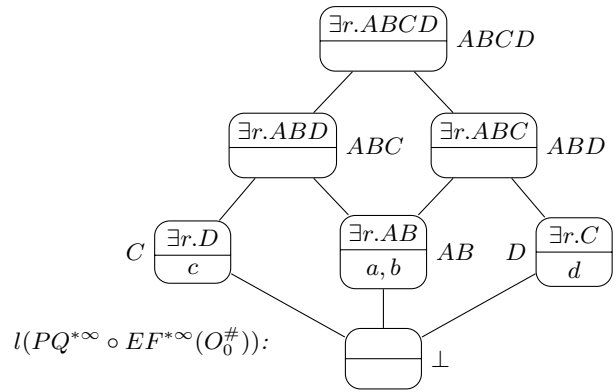
\begin{center}\footnotesize
\begin{tabular}{cc}
\begin{minipage}{.35\textwidth}
\setlength{\tabcolsep}{2pt}
\begin{center}
\begin{tabular}{r|cccccc}
$k(PQ^{*\infty}\circ EF^{*\infty}(O_0^\#))$    & \rotatebox{90}{$\exists r.C$} & \rotatebox{90}{$\exists r.D$} & \rotatebox{90}{$\exists r.AB$} & \rotatebox{90}{$\exists r.ABC$} & \rotatebox{90}{$\exists r.ABD$}& \rotatebox{90}{$\exists r.ABCD$}\\\hline
$a$ &          &          & $\times$ & $\times$ & $\times$ & $\times$        \\
$b$ &          &          & $\times$ & $\times$ & $\times$ & $\times$        \\
$c$ &          &          &          &          & $\times$ & $\times$        \\
$d$ & $\times$ & $\times$ &          & $\times$ &          & $\times$ \\
\end{tabular}
\end{center}
\end{minipage} &
\begin{minipage}{.65\textwidth}
\begin{center}
\begin{tikzpicture}[xscale=.3,yscale=.3]
    \begin{dot2tex}[dot,tikz,codeonly,options=-traw]
      graph {
	graph [nodesep=1.5]
	node [style=concept]
	ABCD [label="$\exists r.ABCD$
\nodepart{two}
$\empty$"]

	ABC [label="$\exists r.ABD$
\nodepart{two}
$\empty$"]

	ABD [label="$\exists r.ABC$
\nodepart{two}
$\empty$"]

	D [label="$\exists r.C$
\nodepart{two}
$d$"]

	AB [label="$\exists r.AB$
\nodepart{two}
$a, b$"]

	C [label="$\exists r.D$
\nodepart{two}
$c$"]

        C0 [label="$\empty$
\nodepart{two}
$\empty$"]
        ABCD -- ABC
        ABCD -- ABD
        ABC -- C
        ABC -- AB
        ABD -- AB
        ABD -- D
        C -- C0
        AB -- C0
        D -- C0
      }
    \end{dot2tex}
    \node[anchor=west] at (ABCD.east) {$ABCD$};
    \node[anchor=west] at (ABC.east) {$ABC$};
    \node[anchor=west] at (ABD.east) {$ABD$};
    \node[anchor=west] at (AB.east) {$AB$};
    \node[anchor=east] at (C.west) {$C$};
    \node[anchor=east] at (D.west) {$D$};
    \node[anchor=west] at (C0.east) {$\bot$};
    \draw (0,2.5) node {$l(PQ^{*\infty}\circ EF^{*\infty}(O_0^\#))$:};
\end{tikzpicture}
\end{center}
\end{minipage}
\end{tabular}
\end{center}

\noindent and $EF^{*\infty}\circ PQ^{*\infty}(O_0^\#)$:
\begin{center}\footnotesize
\begin{tabular}{cc}
\begin{minipage}{.4\textwidth}
\setlength{\tabcolsep}{2pt}
\begin{center}
\begin{tabular}{r|cc}
$k(EF^{*\infty}\circ PQ^{*\infty}(O_0^\#))$    & \rotatebox{90}{$\exists r.AB$} & \rotatebox{90}{$\exists r.ABCD$} \\\hline
$a$ & $\times$ & $\times$ \\
$b$ & $\times$ & $\times$ \\
$c$ &          & $\times$ \\
$d$ &          & $\times$ \\
\end{tabular}
\end{center}
\end{minipage} &
\begin{minipage}{.6\textwidth}
\begin{center}
\begin{tikzpicture}[xscale=.3,yscale=.3]
    \begin{dot2tex}[dot,tikz,codeonly,options=-traw]
      graph {
	graph [nodesep=1.5]
	node [style=concept]
	ABCD [label="$\exists r.ABCD$
\nodepart{two}
$c, d$"]

	AB [label="$\exists r.AB$
\nodepart{two}
$a, b$"]
        ABCD -- AB
      }
    \end{dot2tex}
    \node[anchor=west] at (ABCD.east) {$ABCD$};
    \node[anchor=west] at (AB.east) {$AB$};
    \draw (-5,3) node {$l(EF^{*\infty}\circ PQ^{*\infty}(O_0^\#))$:};
\end{tikzpicture}
\end{center}
\end{minipage}
\end{tabular}
\end{center}

Now, $PQ^{*\infty}\circ EF^{*\infty}(O_0^\#)$ and $EF^{*\infty}\circ PQ^{*\infty}(O_0^\#)$ belong to $\mathrm{fp}(EF^*)\cap\mathrm{fp}(PQ^*)$.
Yet they are not isomorphic.
In fact, $EF^{*\infty}\circ PQ^{*\infty}(O_0^\#)\prec PQ^{*\infty}\circ EF^{*\infty}(O_0^\#)$
This is the result of $\sigma$ which may add needed support ($ABD$ from $ABC$) and $\pi$ which may suppress unsupported concepts ($ABC$ missing $ABD$).
\end{example}

\begin{example}[Counterexample to equality in RCA]\label{ex:notEFPQeqPQEF}
Consider Example~\ref{sec:exrca} (p.~\pageref{sec:exrca}),
$\mathrm{lfp}(EF^*)$ $=$ $O_{12}^1=\{\langle K_1^1, L_1^1\rangle, \langle K_2^1, L_2^1\rangle\}$ and $\mathrm{gfp}(PQ^*)=O_{12}^\star=\{\langle K_1^\star, L_1^\star\rangle, \langle K_2^\star, L_2^\star\rangle\}$.
$O_{12}^\#$ $=$ $\{\langle K_1^\#,$ $L_1^\#\rangle, \langle K_2^\#, L_2^\#\}$ belongs to $[\mathrm{lfp}(EF^*)\ \mathrm{gfp}(PQ^*)]$ but not to $\mathrm{fp}(EF^*)\cap\mathrm{fp}(PQ^*)$.
It happens that $EF^*(O_{12}^\#)=O_{12}^\star$ and $PQ^*(O_{12}^\#)=O_{12}^1$, hence $PQ^{*\infty}\circ EF^{*\infty}(O_{12}^\#)$ $=$ $PQ^*\circ EF^*(O_{12}^\#)=O_{12}^\star$ and $EF^{*\infty}\circ PQ^{*\infty}(O_{12}^\#)=EF^*\circ PQ^*(O_{12}^\#)=O_{12}^1$.
These two objects are not isomorphic.
What can be said, in this case, is that $EF^{*\infty}\circ PQ^{*\infty}(O_{12}^\#)\prec PQ^{*\infty}\circ EF^{*\infty}(O_{12}^\#)$.
This is the result of $\sigma$ which may add needed support (for $C$ and $B$ from $A$ and $C$) and $\pi$ which may suppress unsupported concepts ($A$ missing $C$ and $D$ missing $B$).

It is not necessary that the results of the closure be the bounds of the interval as is shown for any object of the second and fourth lines of the lattice of Figure~\ref{fig:exrcalattice}.
Example~\ref{ex:notEFPQeqPQEFrca0} illustrates this even better.
\end{example}
  
It may be that, as illustrated by Example~\ref{ex:notEFPQeqPQEF}, when $PQ^{*\infty}$ is first applied, it suppresses non-supported attributes which cannot be recovered by scaling.
Conversely, $EF^{*\infty}$ applied first may scale attributes ($\exists p.D$ in Example~\ref{ex:notEFPQeqPQEF}) which generate new concepts ($B$) supporting previously non-supported concept ($D$). 
These will not be suppressed any more.

Property~\ref{prop:EFsPQsinfPQsEFs} shows that, in addition, there is still a homomorphism between the two resulting objects.

\begin{property}\label{prop:EFsPQsinfPQsEFs}
  $\forall O\in\mathcal{O}, EF^{*\infty}\circ PQ^{*\infty}(O)\preceq PQ^{*\infty}\circ EF^{*\infty}(O)$.
\end{property}
\begin{proof}
$PQ^{*\infty}(O)\preceq O$ by Property~\ref{prop:order}.
But $EF^{*\infty}$ is monotone (Property~\ref{prop:EFsPQsclos}), hence $EF^{*\infty}\circ PQ^{*\infty}(O)\preceq EF^{*\infty}(O)$.
$PQ^{*\infty}$ is also monotone (Property~\ref{prop:EFsPQsclos}), thus $PQ^{*\infty}\circ EF^{*\infty}\circ PQ^{*\infty}(O)\preceq PQ^{*\infty}\circ EF^{*\infty}(O)$.
However, $PQ^{*\infty}\circ EF^{*\infty}(O)\in\mathrm{fp}(EF^*)\cup\mathrm{fp}(PQ^*)$ so $PQ^{*\infty}\circ EF^{*\infty}(O)=EF^{*\infty}\circ PQ^{*\infty}(O)$ (Property~\ref{prop:PQsEFseqEFsPQsonUnion}).
Thus,  $PQ^{*\infty}\circ EF^{*\infty}\circ PQ^{*\infty}(O)=EF^{*\infty}\circ PQ^{*\infty}\circ PQ^{*\infty}(O)=EF^{*\infty}\circ PQ^{*\infty}(O)$.
This means that $EF^{*\infty}\circ PQ^{*\infty}(O)\preceq PQ^{*\infty}\circ EF^{*\infty}(O)$.
\end{proof}
\begin{proof}[Alternative proof]
The same reasoning can be held from $O\preceq EF^{*\infty}(O)$ (Property~\ref{prop:order}) and $EF^{*\infty}$ and $PQ^{*\infty}$ being monotone (Property~\ref{prop:EFsPQsclos}).
Hence, $PQ^{*\infty}(O)\preceq PQ^{*\infty}\circ EF^{*\infty}(O)$ and $EF^{*\infty}\circ PQ^{*\infty}(O)\preceq EF^{*\infty}\circ PQ^{*\infty}\circ EF^{*\infty}(O)$.
But, $PQ^{*\infty}\circ EF^{*\infty}(O)\in\mathrm{fp}(EF^*)\cup\mathrm{fp}(PQ^*)$, so $PQ^{*\infty}\circ EF^{*\infty}(O)=EF^{*\infty}\circ PQ^{*\infty}(O)$ (Property~\ref{prop:PQsEFseqEFsPQsonUnion}).
Thus $EF^{*\infty}\circ PQ^{*\infty}\circ EF^{*\infty}(O)=PQ^{*\infty}\circ EF^{*\infty}\circ EF^{*\infty}(O)=PQ^{*\infty}\circ EF^{*\infty}(O)$.
Hence, $EF^{*\infty}\circ PQ^{*\infty}(O)\preceq PQ^{*\infty}\circ EF^{*\infty}(O)$.
\end{proof}
The alternative proof is given here to show that starting from the $EF^*$ or $PQ^*$ give the same result.

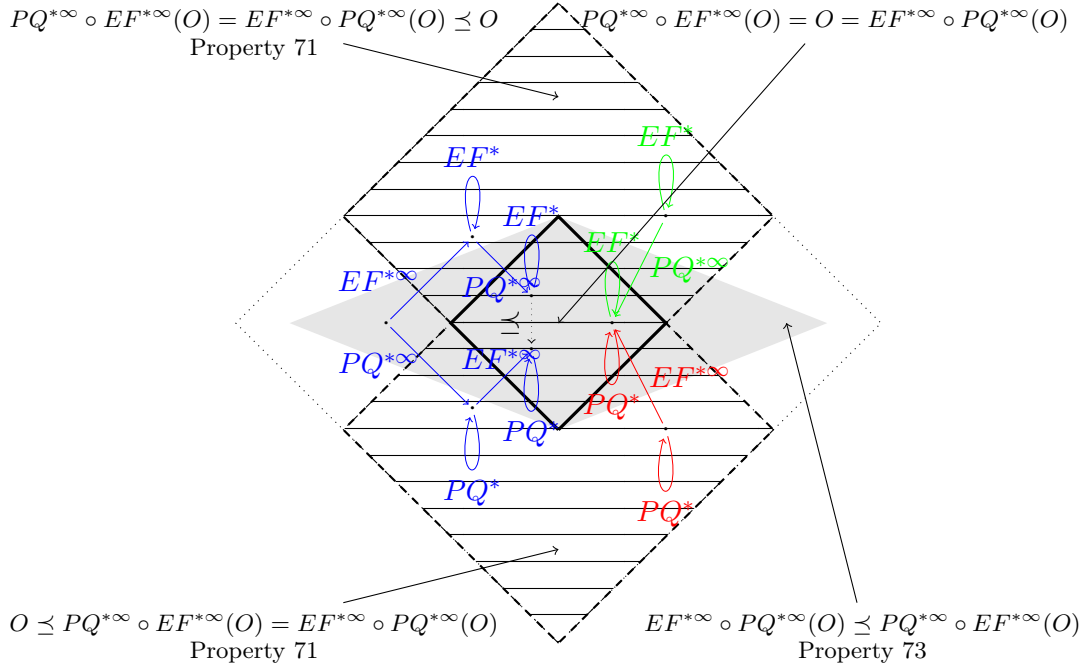
\begin{figure}
\centering\begin{tikzpicture}[yscale=2,xscale=2]

  \tikzstyle{point}=[inner sep=0pt,outer sep=0pt]
  
    \begin{scope}[rotate=45]
     \coordinate (lfpEb) at (-.5,-.5);
     \coordinate (fcak0b) at (-1.5,-1.5);
     \coordinate (gfpEb) at (1.5,1.5);
     \coordinate (topEb) at (.5,.5);
     \coordinate (QEx) at (.1,.1);
     \coordinate (EQx) at (-.1,-.1);
     \draw[dotted] (fcak0b) rectangle (gfpEb);
     \fill[gray!20] (topEb) -- (1.25,-1.25) -- (lfpEb) -- (-1.25,1.25) -- cycle;
     \draw[dashed] (lfpEb) rectangle (gfpEb);
     \draw[dashdotted] (fcak0b) rectangle (topEb);
     \fill[pattern={Lines[distance=10pt]}] (fcak0b) rectangle (topEb);
     \fill[pattern={Lines[angle=90,distance=10pt]}] (lfpEb) rectangle (gfpEb);

     \draw[thick,dashed] (gfpEb) -- (-.5,1.5) -- (-.5,.5) -- (-1.5,.5) -- (fcak0b) -- (.5,-1.5) -- (.5,-.5) -- (1.5,-.5) -- (gfpEb);
     \draw[very thick] (lfpEb) -- (-.5,.5) -- (topEb) -- (.5,-.5) -- (lfpEb);
     \draw (.25,-.25) node[style=point] (fp) {$\cdot$};

      \draw (0,.25) node[style=point] (fp2) {$\cdot$};
      \draw (-.25,0) node[style=point] (fp3) {$\cdot$};

      \draw (1,0) node[style=point] (fpEr) {$\cdot$};
      \draw (0,-1) node[style=point] (fpQr) {$\cdot$};

     \draw (-.8,0) node[style=point] (fpQ) {$\cdot$};
     \draw (0,.8) node[style=point] (fpE) {$\cdot$};
     \draw (-.8,.8) node[style=point] (nofp) {$\cdot$};

   \end{scope}

   \begin{scope}[very thin]
     \path[->,green] (fp) edge[loop above] node[anchor=south] {$EF^*$} (fp);
     \path[->,red] (fp) edge[loop below] node[anchor=north] {$PQ^*$} (fp);

     \path[->,green] (fpEr) edge[loop above] node[anchor=south] {$EF^*$} (fpEr);
     \draw[->,green] (fpEr) -- node[right] {$PQ^{*\infty}$} (fp);

     \path[->,red] (fpQr) edge[loop below] node[anchor=north] {$PQ^*$} (fpQr);
     \draw[->,red] (fpQr) -- node[right] {$EF^{*\infty}$} (fp);
        
     \draw[->,blue] (nofp) -- node[left] {$PQ^{*\infty}$} (fpQ);
     \draw[->,blue] (nofp) -- node[left] {$EF^{*\infty}$} (fpE);
     \path[->,blue] (fpE) edge[loop above] node[anchor=south] {$EF^*$} (fpE);
     \path[->,blue] (fpQ) edge[loop below] node[anchor=north] {$PQ^*$} (fpQ);
     \draw[->,blue] (fpE) -- node[below] {$PQ^{*\infty}$} (fp2);
     \draw[->,blue] (fpQ) -- node[above] {$EF^{*\infty}$} (fp3);
     \path[->,blue] (fp2) edge[loop above] node[anchor=south] {$EF^*$} (fp2);
     \path[->,blue] (fp3) edge[loop below] node[anchor=north] {$PQ^*$} (fp3);

     \draw[->,dotted] (fp2) -- node[left] {$\preceq$} (fp3);
   \end{scope}

   \begin{scope}[font=\footnotesize]
     \draw (-2,2) node (fpEF) {$PQ^{*\infty}\circ EF^{*\infty}(O)=EF^{*\infty}\circ PQ^{*\infty}(O)\preceq O$};
     \draw ($(fpEF.south)+(0,-1pt)$) node {Property~\ref{prop:PQsEFseqEFsPQsonUnion}};
     \draw (-2,-2) node (fpPQ) {$O\preceq PQ^{*\infty}\circ EF^{*\infty}(O)=EF^{*\infty}\circ PQ^{*\infty}(O)$};
     \draw ($(fpPQ.south)+(0,-1pt)$) node {Property~\ref{prop:PQsEFseqEFsPQsonUnion}};
     \draw (1.75,2) node (fpEFPQ) {$PQ^{*\infty}\circ EF^{*\infty}(O)=O=EF^{*\infty}\circ PQ^{*\infty}(O)$};
     \draw (2,-2) node (rest) {$EF^{*\infty}\circ PQ^{*\infty}(O)\preceq PQ^{*\infty}\circ EF^{*\infty}(O)$};
     \draw ($(rest.south)+(0,-1pt)$) node {Property~\ref{prop:EFsPQsinfPQsEFs}};
     
     \draw[->] (fpEFPQ) -- (0,0);
     \draw[->] (fpEF) -- (0,1.5);
     \draw[->] (fpPQ) -- (0,-1.5);
     \draw[->] (rest) -- (1.5,0);
   \end{scope}
 \end{tikzpicture}
\caption{Illustration of the position of $PQ^{*\infty}\circ EF^{*\infty}(O)$ and $EF^{*\infty}\circ PQ^{*\infty}(O)$ depending on $O$'s origin (in dotted, $\mathcal{O}$, in dashed $\mathrm{fp}(EF^*)\cup\mathrm{fp}(PQ^*)$, in plain $\mathrm{fp}(EF^*)\cap\mathrm{fp}(PQ^*)$). 
Colours correspond to that of Figure~\ref{fig:fpstruct}: green starting from $\mathrm{fp}(EF^*)$, red starting from $\mathrm{fp}(PQ^*)$, blue starting outside of them.
}\label{fig:conj}
\end{figure}

It is thus unclear what to do with $EF^{*\infty}\circ PQ^{*\infty}$ and $PQ^{*\infty}\circ EF^{*\infty}$ in general.
For instance, if one needs an operation to map elements of $\mathcal{O}$ to $\mathrm{fp}(EF^*)\cap\mathrm{fp}(PQ^*)$, which one is preferable?
There may be an interest in studying the interval $[EF^{*\infty}\circ PQ^{*\infty}(O)\ PQ^{*\infty}\circ EF^{*\infty}(O)]$. Does it contain only fixed points or no fixed points? Are these the image of other lattices?
This question can be answered if $O$ can be compared to these bounds (Proposition~\ref{prop:intervnotfp}): the intermediate families are \emph{not} fixed points.

\begin{proposition}\label{prop:intervnotfp}
$\forall O\in \mathcal{O}\setminus (\mathrm{fp}(EF^*)\cap\mathrm{fp}(PQ^*))$:
\begin{itemize}
\item if $O\preceq PQ^{*\infty}\circ EF^{*\infty}(O)$, then $\forall O'\in [O\ PQ^{*\infty}\circ EF^{*\infty}(O)[$, $O'\not\in\mathrm{fp}(EF^*)\cap\mathrm{fp}(PQ^*)$,
\item if $EF^{*\infty}\circ PQ^{*\infty}(O)\preceq O$, then $\forall O'\in ]EF^{*\infty}\circ PQ^{*\infty}(O)\ O]$, $O'\not\in\mathrm{fp}(EF^*)\cap\mathrm{fp}(PQ^*)$.
\end{itemize}
\end{proposition}
\begin{proof}
Considering the first item of the proposition, $O\preceq PQ^{*\infty}\circ EF^{*\infty}(O)$ can only occur if $EF^{*\infty}(O)\in\mathrm{fp}(PQ^*)$, i.e. $PQ^{*\infty}\circ EF^{*\infty}(O)=EF^{*\infty}(O)$. 
Indeed, if this were not the case, then $PQ^{*\infty}$ would suppress attributes from $EF^{*\infty}(O)$. 
However, since  $O\preceq PQ^{*\infty}\circ EF^{*\infty}(O)$, these could not be attributes from $O$, but only attributes  added by $EF^{*\infty}$. 
But since $EF^{*\infty}$ only adds attributes if they are supported and it starts with attributes from $O$, this is not possible.
Thus, if $O'\in [O\ PQ^{*\infty}\circ EF^{*\infty}(O)[$ then $O'\in [O\ EF^{*\infty}(O)[$.
However, $O'$ cannot be a fixed point for $EF^*$ because it contains all attributes of $O$ which would scale to generate all those of $EF^{*\infty}(O)$.
Hence $O'\not\in\mathrm{fp}(EF^*)\cap\mathrm{fp}(PQ^*)$.

The second item has a similar proof: $EF^{*\infty}\circ PQ^{*\infty}(O)\preceq O$ can only occur if $PQ^{*\infty}(O)\in\mathrm{fp}(EF^*)$, i.e. $EF^{*\infty}\circ PQ^{*\infty}(O)=PQ^{*\infty}(O)$.
Indeed, if this were not the case, then $EF^{*\infty}$ would generate attributes from $PQ^{*\infty}(O)$. 
However, since  $EF^{*\infty}\circ PQ^{*\infty}(O)\preceq O$, these could only be attributes of $O$ which were suppressed by $PQ^{*\infty}$ due to lack of support.
But this is not possible because if they lacked support in $O$, there is not more support for them in $PQ^{*\infty}(O)$, which only reduces $O$.
Thus, if $O'\in ]EF^{*\infty}\circ PQ^{*\infty}(O)\ O]$, then $O'\in ]PQ^{*\infty}(O)\ O]$.
However, $O'$ cannot be a fixed point for $PQ^*$ because it contains less attributes than $O$: if these attributes lacked supports in $O$, they would still lack it in $O'$.
Hence, $O'\not\in\mathrm{fp}(EF^*)\cap\mathrm{fp}(PQ^*)$.
\end{proof}

This is illustrated by Example~\ref{ex:intervnotfp}:
\begin{example}[]\label{ex:intervnotfp}
In Example~\ref{ex:notEFPQeqPQEF} (p.~\pageref{ex:notEFPQeqPQEF}),
$O_{34}^\#=\{\langle K^\#_3,L^\#_3\rangle, \langle K^\#_4,L^\#_4\rangle\}$ is a fixed point for neither $EF^*$ nor $PQ^*$.
$PQ^{*\infty}\circ EF^{*\infty}(O_{34}^\#)$ $=$ $EF^{*\infty}(O_{34}^\#)$ $=$ $O_{34}^\star$ and
$EF^{*\infty}\circ PQ^{*\infty}(O_{34}^\#)$ $=$ $PQ^{*\infty}(O_{34}^\#)$ $=$ $O_{34}^1$.
None of the objects in the interval between $O_{34}^\#$ and either $O_{34}^\star$ or $O_{34}^1$ belongs to $\mathrm{fp}(EF^*)\cap\mathrm{fp}(PQ^*)$ as can be seen in Figure~\ref{fig:exrcalattice}.
\end{example}

This result cannot be generalised to the interval $]EF^{*\infty}\circ PQ^{*\infty}(O)\ PQ^{*\infty}\circ EF^{*\infty}(O)[$ as shown by Example~\ref{ex:fpinint} and Example~\ref{ex:fpinint0}.

\begin{example}[The subinterval may contain fixed points in RCA]\label{ex:fpinint}
Following Example~\ref{ex:intervnotfp},
$O_{34}'=\{\langle K'_3, L'_3\rangle, \langle K'_4, L'_4\rangle\}$ belongs to
$]EF^{*\infty}\circ PQ^{*\infty}(O_{34}^\#)\ PQ^{*\infty}\circ EF^{*\infty}(O_{34}^\#)[~=~]O_{34}^1,$ $\ O_{34}^\star[$ as can be observed in Figure~\ref{fig:exrcalattice}.
However, $O'_{34}\in\mathrm{fp}(EF^*)\cap\mathrm{fp}(PQ^*)$.
\end{example}

\begin{example}[The subinterval may contain fixed points in RCA$^0$]\label{ex:fpinint0}
The family of context-lattice pairs $O_0^\#$ of Example~\ref{ex:notEFPQeqPQEFrca0} (p.~\pageref{ex:notEFPQeqPQEFrca0}), is such that:
\begin{enumaa}
\item $EF^{*\infty}(O_0^\#)\neq PQ^{*\infty}\circ EF^{*\infty}(O_0^\#)$,
\item $PQ^{*\infty}(O_0^\#)\neq EF^{*\infty}\circ PQ^{*\infty}(O_0^\#)$, and
\item $EF^{*\infty}\circ PQ^{*\infty}(O_0^\#)\prec PQ^{*\infty}\circ EF^{*\infty}(O_0^\#)$.
\end{enumaa}

Note that none of $EF^{*\infty}\circ PQ^{*\infty}(O_0^\#)$ nor  $PQ^{*\infty}\circ EF^{*\infty}(O_0^\#)$ are the bounds of $\mathrm{fp}(EF^*)\cap \mathrm{fp}(PQ^*)$ ($O_0^1$ and $O_0^\star$) contrary to Example~\ref{ex:fpinint}.

Moreover, consider $O_0''$ defined as follows:
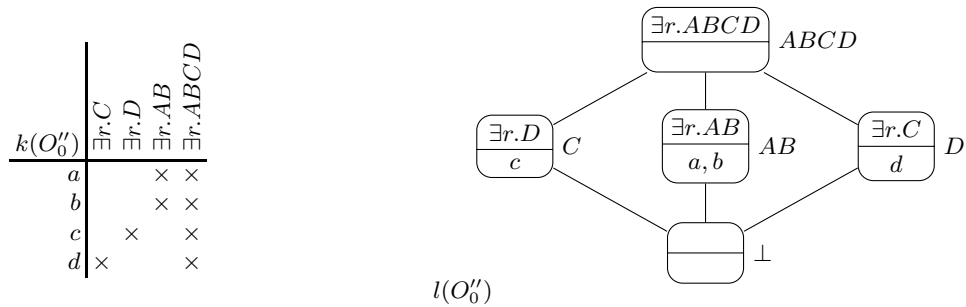
\begin{center}\footnotesize
\begin{tabular}{cc}
\begin{minipage}{.4\textwidth}
\setlength{\tabcolsep}{2pt}
\begin{center}
\begin{tabular}{r|cccc}
$k(O_0'')$    & \rotatebox{90}{$\exists r.C$} & \rotatebox{90}{$\exists r.D$} & \rotatebox{90}{$\exists r.AB$} & \rotatebox{90}{$\exists r.ABCD$} \\\hline
$a$ &          &          & $\times$ & $\times$ \\
$b$ &          &          & $\times$ & $\times$ \\
$c$ &          & $\times$ &          & $\times$ \\
$d$ & $\times$ &          &          & $\times$ \\
\end{tabular}
\end{center}
\end{minipage} &
\begin{minipage}{.6\textwidth}
\begin{center}
\begin{tikzpicture}[xscale=.3,yscale=.3]
    \begin{dot2tex}[dot,tikz,codeonly,options=-traw]
      graph {
	graph [nodesep=1.5]
	node [style=concept]
	ABCD [label="$\exists r.ABCD$
\nodepart{two}
$\empty$"]

	C [label="$\exists r.D$
\nodepart{two}
$c$"]

	AB [label="$\exists r.AB$
\nodepart{two}
$a, b$"]

	D [label="$\exists r.C$
\nodepart{two}
$d$"]

	C0 [label="$\empty$
\nodepart{two}
$\empty$"]
        ABCD -- C
        ABCD -- AB
        ABCD -- D
        C -- C0
        AB -- C0
        D -- C0
      }
    \end{dot2tex}
    \node[anchor=west] at (ABCD.east) {$ABCD$};
    \node[anchor=west] at (AB.east) {$AB$};
    \node[anchor=west] at (C.east) {$C$};
    \node[anchor=west] at (D.east) {$D$};
    \node[anchor=west] at (C0.east) {$\bot$};
    \draw (0,0) node {$l(O_0'')$:};
\end{tikzpicture}
\end{center}
\end{minipage}
\end{tabular}
\end{center}
$O_0''\in ]EF^{*\infty}\circ PQ^{*\infty}(O_0^\#), PQ^{*\infty}\circ EF^{*\infty}(O_0^\#)[$
and $O_0''\in\mathrm{fp}(EF^*)\cap \mathrm{fp}(PQ^*)$.

\bigskip
This counter-example is not sensible to conditions (a) and (b) above:
(a) can be relaxed, by simply adding $\exists r.ABCD$ to $K_0^\#$,
(b) can be relaxed by suppressing $\exists r.B$ from $K_0^\#$, and
both can be relaxed together.
In each of these cases, $EF^{*\infty}\circ PQ^{*\infty}(O_0^\#)$ and $PQ^{*\infty}\circ EF^{*\infty}(O_0^\#)$ will not be changed, preventing to discard the presence of acceptable solutions ($O_0''$) in the interval.
\end{example}

Proposition~\ref{prop:intervnotfp} can however be useful algorithmically. 
Indeed, if one considers the non-acceptable objects of Figure~\ref{fig:exrcalattice} on the same line as $O_{34}^\#$, then this result identifies as non acceptable two objects on the second and fourth line without testing them.

\section{Conclusion}

We addressed the questions of which family of concept lattices was returned by relational concept analysis and, more generally, which such families could be considered acceptable.

This report provides an
answer to these questions by characterising the acceptable families of context-lattice pairs that describe a particular initial family of contexts as those families which are well-formed, saturated and self-supported.
It identifies the results returned by relational concept analysis as the smallest element of this set.
It also defines an alternative operation providing its greatest elements.
The structure of the set of acceptable solutions has been further characterised.

To that extent the report defines the set of well-formed objects $\mathcal{O}$, a function $EF^*$, generalising RCA, expanding a family, and a function $PQ^*$ contracting a family.
The fixed points of these functions characterise the saturated families and the self-supported families respectively.
Hence, the acceptable solutions are those elements of the intersection of the fixed points of both functions ($\mathrm{fp}(EF^*)\cap\mathrm{fp}(PQ^*)$).

These results rely fundamentally on the finiteness of the structure and monotony of the operations.
Dealing with infinite structures would jeopardise the construction of the sets of scalable relational attributes (Section~\ref{sec:rcaname}), however as soon as the termination of the application of the operations is preserved, this should not be a problem.
Non-monotonic operations could be induced by non-monotonic scaling operations. 
Such operations would prevent relational concept analysis to work properly and would require fully different mechanisms.

In FCA, conceptual scaling is considered as a human-driven analysis tool: a knowledgeable person could provide attributes to be scaled for describing better the data to be analysed.
In RCA, scaling is used as an extraction tool, with the drawback to potentially generate many attributes.
By only extracting the least fixed point, RCA avoids generating too many of them.
This is useful when generating a description logic TBox because all concepts are well-defined and necessary, but other contexts may benefit from exploiting other solutions.

Beyond the minimal common acceptable lattices returned by $\underline{\mathrm{RCA}}$ and the most detailed ones that $\overline{\mathrm{RCA}}$ returns, algorithms may be developed for returning all acceptable solutions \cite{atencia2021b}.
The characterisation of the structure of the space of acceptable solutions aims at contributing to this goal.
However, our work does not provide an `efficient' way to obtain all elements of this set.
The characterisation of the structure of the space of acceptable solutions aims at contributing to this goal.

This work also opens perspectives for helping users to identify the acceptable solution that they prefer.
Beyond generating all solutions, another option is to offer users the opportunity to guide the navigation among them.
The structure of admissible solutions and the associated functions may be fruitfully exploited in order to help users finding an acceptable solution featuring the concepts and attributes that they want and not unnecessary ones.

An anonymous reviewer remarked that variations of RCA, such as those based on AOC-posets \cite{dolques2013a}, may receive the same treatment.
This is a perspective worth pursuing, that may lead to generalise the results presented here.

Finally, the position of relational concept analysis with respect to formal concept analysis and Galois connections would be worth investigating.
On the one hand, this work shows that, contrary to other extensions that use scaling to encode a problem within FCA, RCA cannot be encoded in FCA.
Indeed, RCA admits various fixed points contrary to FCA.
RCA is not just the application of a product or sequence of FCA, but relations between contexts introduce constraints between them leading to the possibility of alternative fixed points.
Hence, an encoding would not be direct, so that it provides RCA solutions directly.
On the other hand, other generalisations of FCA get closer to general Galois connections by extending the structure of attributes.
The open question is whether RCA is another instance of a Galois connection extending FCA or if these two need a common generalisation.

\subsection*{Acknowledgements}

This work has been partially funded by the ANR Elker project (ANR-17-CE23-0007-01).
The author thanks 
Petko Valtchev for comments and suggestions on an earlier version of this work,
Marianne Huchard for taking the pain to describe the trick used for forcing discrimination, 
Amedeo Napoli for explaining me RCA (and so many other things),
Jérôme David for reviewing the text,
and Philippe Besnard for pointing to the Knaster-Tarski theorem.
Anonymous JAIR reviewers helped clarifying the report.

\printbibliography[title={References},heading=bibnumbered] %

\newpage
\appendix

\section{Correspondences between propositions}\label{sec:propindex}

\enlargethispage{7\baselineskip}
{\footnotesize

\begin{center}
\begin{tabular}{ccccc}
\cite{euzenat2021a} & & RR-9518 (v1) & & \cite{euzenat2025a}\\
 &  & Property~\ref{prop:IdependM} & $=$ & Property~1 \\
 &  & Property~\ref{prop:Kclosedmj} & $=$ & Property~2 \\
Property~2 & $=$ & Property~\ref{prop:ordereqopK} (3) & & \\
Property~1 & $=$ & Property~\ref{prop:Kcompl} (2) & & \\
Property~3 & $=$ & Property~\ref{prop:FintK} (4) & & \\
Property~4 & $=$ & Property~\ref{prop:Fmonext} (5) & & \\
Theorem~1 & $=$ & Theorem~\ref{prop:knastertarski} (6) (Knaster-Tarski) & $=$ & Theorem~1\\
 &  & Property~\ref{prop:kappaFCA} (7) & $=$ & Property~3 \\
 &  & Property~\ref{prop:FCAmon} (8) & $=$ & Property~4 \\
 &  & Property~\ref{prop:Lclosedmj} & $=$ & Property~5 \\
Proposition~3 & $=$ & Property~\ref{prop:EintL} (11) & & \\
Proposition~1 & $=$ & Proposition~\ref{prop:rca0lfp} (13) & & \\
Proposition~2 & $=$ & Proposition~\ref{prop:gfpF} (14) & & \\
Proposition~4 & $=$ & Property~\ref{prop:Qantextmon} (16) & & \\
Proposition~5 & $>$ & Property~\ref{prop:QintL} (19) & & \\
Proposition~5 & $>$ & Property~\ref{prop:Qstable} (19) & & \\
Proposition~6 & $=$ & Proposition~\ref{prop:infsupEQ} (25) & & \\
 &  & Property~\ref{prop:andorT} (26) & $=$ & Property~6 \\
 &  & Property~\ref{prop:Tclosedmj} & $=$ & Property~7 \\
 &  & Property~\ref{prop:Torder} (27) & $=$ & Property~8 \\
 &  & Property~\ref{prop:Oclosedmj} & $=$ & Property~9 \\
 &  & Property~\ref{prop:Oorder} (36) & $=$ & Property~10 \\
 & & Proposition~\ref{prop:Ocompl} (38) & $=$ & Proposition~11\\
 &  & Property~\ref{prop:EFsintO} (40) & $=$ & Property~12 \\
 &  & Property~\ref{prop:EFsmonext} (41) & $=$ & Property~13 \\
 &  & Property~\ref{prop:pissup} & $=$ & Property~14 \\
 &  & Property~\ref{prop:PQsintO} (43) & $=$ & Property~15 \\
 &  & Property~\ref{prop:PQsantextmon} (44) & $=$ & Property~16 \\
 &  & Property~\ref{prop:EFsasscaling} (45) & $=$ & Property~17 \\
 &  & Property~\ref{prop:PQsaspurging} (46) & $=$ & Property~18 \\
 &  & Property~\ref{prop:lfpPQs} (47) & $=$ & Property~19 \\
 &  & Property~\ref{prop:gfpEFs} (48) & $=$ & Property~20 \\
 &  & Property~\ref{prop:EFsstable} (49) & $=$ & Property~21 \\
 &  & Property~\ref{prop:PQsstable} (49) & $=$ & Property~22 \\
 &  & Property~\ref{prop:EFsPQsclos} (50) & $=$ & Property~23 \\
 &  & Property~\ref{prop:EFsPQslsgsfp} (51) & $=$ & Property~24 \\
 &  & Property~\ref{prop:order} (52) & $=$ & Property~25 \\
 &  & Property~\ref{prop:satisfpEFs} & $=$ & Property~26 \\
 &  & Property~\ref{prop:ssisfpPQs} & $=$ & Property~27 \\
 & & Proposition~\ref{prop:accfp} (53) & $=$ & Proposition~28\\
 & & Proposition~\ref{prop:rcalfp} (54) & $=$ & Proposition~29\\
 & & Proposition~\ref{prop:acrgfp} (55) & $=$ & Proposition~30\\
 & & Proposition~\ref{prop:lfgeqgfp} (56) & $=$ & Proposition~31\\
 &  & Property~\ref{prop:EFsintfpPQs} (58) & $=$ & Property~32 \\
 &  & Property~\ref{prop:PQsintfpEFs} (57) & $=$ & Property~33 \\
 &  & Property~\ref{prop:EFsPQsbounds} (59) & $=$ & Property~34 \\
 & & Proposition~\ref{prop:infsupEFsPQs} (60) & $=$ & Proposition~35\\
 &  & Property~\ref{prop:imPQsEFsfp} (61) & $=$ & Property~36 \\
 &  & Property~\ref{prop:PQsEFsmonid} (62) & $=$ & Property~37 \\
 & & Proposition~\ref{prop:imPQsEFscompl} (63) & $=$ & Proposition~38\\
 &  & Property~\ref{prop:PQsEFseqEFsPQsonUnion} (64) & $=$ & Property~39 \\
 &  & Property~\ref{prop:PQsEFseqEFsPQsbt} (65) & $=$ & Property~40 \\
 &  & Property~\ref{prop:EFsPQsinfPQsEFs} (66) & $=$ & Property~41 \\
 & & Proposition~\ref{prop:intervnotfp} (67) & $=$ & Proposition~42\\
\end{tabular}
\end{center}
}

\end{document}

\endinput %